\documentclass[11pt]{article}

\usepackage{adjustbox}
\usepackage[utf8]{inputenc} 
\usepackage[T1]{fontenc}    
\usepackage{hyperref}       
\usepackage{url}            
\usepackage{booktabs}       
\usepackage{amsfonts}       
\usepackage{nicefrac}       
\usepackage{microtype}      

\usepackage[normalem]{ulem}

\usepackage{wrapfig}
\usepackage{mathtools}
\usepackage{parskip}
\usepackage{mathrsfs}
\usepackage{algorithm}
\usepackage{algorithmic}
\usepackage{relsize}
\usepackage{amsmath,amscd}
\usepackage[margin=1in]{geometry}
\usepackage{mathtools}
\usepackage{amssymb}
\usepackage{amsthm}
\usepackage{tikz-cd}
\usepackage{color}
\usepackage{framed}
\usepackage{bm}
\usepackage{epsfig}
\usepackage{natbib}
\usepackage{authblk}
\usepackage{graphicx}
\usepackage{subfig}
\usepackage{caption}

\usepackage{cleveref}

\usetikzlibrary{positioning,fit,calc}
\tikzset{block/.style={draw,thick,text width=2cm,minimum height=1cm,align=center},
         line/.style={-latex}
}

\newtheorem{example}{Example} 
\newtheorem{theorem}{Theorem}
\newtheorem{lemma}[theorem]{Lemma} 
\newtheorem{proposition}[theorem]{Proposition} 
\newtheorem{remark}[theorem]{Remark}
\newtheorem{corollary}[theorem]{Corollary}
\newtheorem{definition}[theorem]{Definition}

\newtheorem{claim}[theorem]{Claim} 

\newcommand{\diam}{\mathrm{diam}}

\newcommand{\mean}{\mathrm{mean}}
\newcommand{\tr}{\mathrm{tr}}

\newcommand{\vol}{\mathrm{vol}}
\newcommand{\supp}{\mathrm{supp}}
\newcommand{\dW}[1]{d_{\mathrm{W},{#1}}}

\newcommand{\R}{\mathbb{R}}
\newcommand{\eps}{\varepsilon}

\newcommand{\norm}[1]{\left\lVert#1\right\rVert}
\newcommand{\pf}{\mathcal{P}_f}

\newcommand{\bxe}{B_\eps^{d_X}(x)}

\newcommand{\we}{\mathbf{W}_{1}^{\mathsmaller{(\eps)}}}

\newcommand{\dcov}{d_{\mathrm{cov}}}
\newcommand{\dgt}{d^{\mathsmaller{(\eps,\lambda)}}_{{\alpha,d_X}}}

\newcommand{\dw}{d_{\mathrm{W},2}}

\newcommand{\gt}{\mathbf{G}^{\mathsmaller{(\eps,\lambda)}}_2}

\newcommand{\wt}[1]{\mathbf{W}_{#1}}

\newcommand{\deps}{d_{\alpha,d_X}^{\mathsmaller{(\eps)}}}

\newcommand{\depsbeta}{d_{\beta,d_X}^{\mathsmaller{(\eps)}}}

\newcommand{\lms}{L^{\mathrm{ms}}}

\newcommand{\dP}{d_\mathrm{P}}

\newcommand{\me}{m^{\mathsmaller{(\eps)}}}

\newcommand{\lc}{\left(}
\newcommand{\rc}{\right)}
\newcommand{\ls}{\left|}
\newcommand{\rs}{\right|}
\newcommand{\lb}{\left\{}
\newcommand{\rb}{\right\}}

\newcommand{\dgw}[1]{d_{\mathrm{GW},#1}}

\newcommand{\Rp}{\R_{\geq 0}}

\newcommand{\zane}[1]{}

\newcommand{\facundo}[1]{}

\newcommand{\zhengchao}[1]{}

\begin{document}

\renewcommand{\Authfont}{\normalsize}
\renewcommand{\Affilfont}{\small}
\setlength{\affilsep}{0.25em}

\title{The Wasserstein Transform}

\author[1]{Kun Jin}
\author[2]{Facundo M\'emoli}
\author[3]{Zane Smith}
\author[4]{Zhengchao Wan}

\affil[1]{Samsung Research America}
\affil[2]{Department of Mathematics, Rutgers University}
\affil[3]{Department of Computer Science and Engineering, University of Minnesota}
\affil[4]{Department of Mathematics, University of Missouri}
\affil[]{\texttt{kun.jin.810@gmail.com}, \texttt{facundo.memoli@rutgers.edu}, \texttt{zane3g@gmail.com}, \texttt{zwan@missouri.edu}}

\date{}

\maketitle

\begin{abstract}
We introduce the Wasserstein Transform (WT), a general unsupervised framework for updating distance structures on given data sets with the purpose of enhancing features and denoising. Our framework represents each data point by a probability measure reflecting the neighborhood structure of the point, and then updates the distance by computing the Wasserstein distance between these probability measures. The Wasserstein Transform is a general method which extends the mean shift family of algorithms. We study several instances of WT, and in particular, in one of the instances which we call the Gaussian Transform (GT), we utilize Gaussian measures to model neighborhood structures of individual data points. GT is computationally cheaper than other instances of WT since there exists closed form solution for the $\ell^2$-Wasserstein distance between Gaussian measures. We study the relationship between different instances of WT and prove that each of the instances is stable under perturbations. We devise iterative algorithms for performing the above-mentioned WT and propose several strategies to accelerate GT, such as an observation from linear algebra for reducing the number of matrix square root computations. We examine the performance of the Wasserstein Transform method in many tasks, such as denoising, clustering, image segmentation and word embeddings.
\end{abstract}


\newpage
\tableofcontents

\newpage
\section{Introduction}

Outliers and noise are usually inevitable when acquiring new data and these will eventually deteriorate the performance of downstream machine learning tasks on a given data set. For example, in hierarchical clustering, the classical single linkage clustering is known to be susceptible to the so-called chaining effect which is caused by a particular type of outliers \citep{jardine1967structure}; see Figure \ref{fig:chaining} for an illustration.

\begin{figure}[htb]
    \centering
     
		    \includegraphics[width=.65\textwidth]{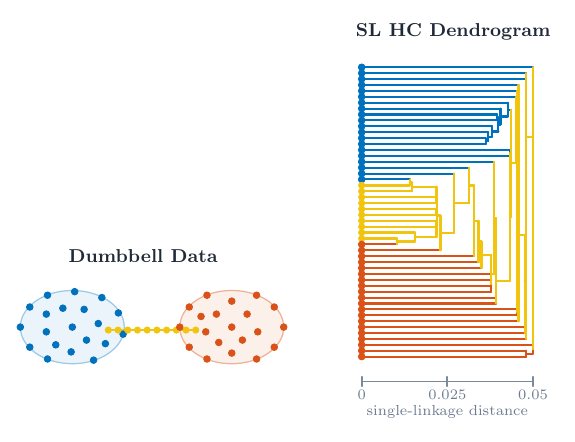} 
        
    \caption{\textbf{Illustration of the chaining effect.} The left figure shows a dumbbell shape data set, and the right figure shows its single-linkage hierarchical clustering dendrogram.} 
    \label{fig:chaining} 
\end{figure}

Usually outliers have different neighborhood structure than that of other data points. For example, in the case of shown in Figure \ref{fig:chaining}, points along the chain have one dimensional neighborhoods whereas points in the two blobs have two dimensional neighborhoods. This observation motivates us to pursue the following approach to denoising data sets: 
we update the distance function of the data set by absorbing a penalty of structural difference between neighborhoods of data points. More precisely, we first represent each data point by a probability measure based on neighborhood/context information, and then compute the Wasserstein distance, a concept from optimal transport, between these probability measures to generate a new distance function for the data set.

Optimal transport (OT) is concerned with finding cost efficient ways of deforming a given source probability distribution into a target distribution \citep{villani2003topics,villani2008optimal,santambrogio2015optimal}.
In recent years, ideas from OT have found applications in machine learning and in data analysis in general. Applications range from image equalization \citep{delon2004midway}, shape interpolation \citep{solomon2015convolutional}, image/shape \citep{solomon2016entropic,rubner1998metric} and document classification \citep{kusner2015word,rolet2016fast}, semi-supervised learning \citep{solomon2014wasserstein}, to population analysis of Gaussian processes \citep{mallasto2017learning}, domain adaptation \citep{courty2017optimal}, deep neural network~\citep{arjovsky2017wasserstein,gulrajani2017improved} and natural language processing 
(NLP)~\citep{alvarez-melis-jaakkola-2018-gromov}, etc.

\begin{figure}[htb]
\includegraphics[width=\linewidth]{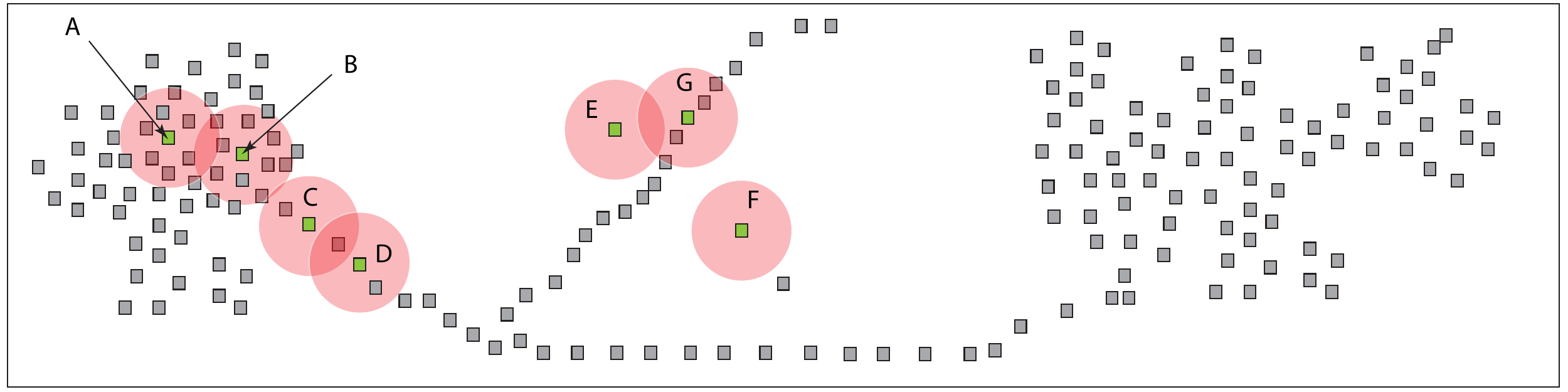}
\caption{\textbf{Illustration of the Wasserstein Transform.} Different points in the given point cloud in $\mathbb{R}^2$ have different neighborhood structures. These neighborhood structures are captured by representing them via a probability measure, e.g., the empirical probability measure supported on the given neighborhood. In a nutshell, the Wasserstein Transform produces a new distance function on the data set by computing the Wasserstein distance between these measures. } \label{fig:interp-wt}
\end{figure}

From the approach mentioned above we extract an unsupervised framework called the \emph{Wasserstein Transform} (WT) for enhancing features and/or to denoise data sets. The intuitive idea of WT is illustrated in \Cref{fig:interp-wt}. We examine several important instances of WT which are applicable in many experimental scenarios. In particular, we propose the \emph{Gaussian Transform} (GT), a computationally efficient instance of WT, which can be fine tuned to adapt anisotropy of data sets, an often desired feature in methods for image denoising and image segmentation \citep{perona1990scale,wang2004image}. In GT we model each point in a given point cloud as a Gaussian measure (whose covariance is estimated locally) and then generate a new distance between pairs of points $x$ and $x'$ on the data set via the computation of the $\ell^2$-Wasserstein distance between the Gaussian at $x$ and the one at $x'$. This computation of the $\ell^2$-Wasserstein distance between Gaussians is efficient due to the existence of closed-form formula \citep{givens1984class} in terms of products and square roots of the covariance matrices involved.

We also interpret our proposed feature enhancing method, the Wasserstein Transform, as both a generalization and a strengthening of mean shift (MS) \citep{cheng1995mean,fukunaga1975estimation} which can operate on general metric spaces. Although mean shift has been generalized to data living on Riemannian manifolds \citep{subbarao2009nonlinear,shamir2006mesh}, our interpretation departs from the ones in those papers in that we do not attempt to estimate a local mean or median of the data but, instead, we use the local density of points to iteratively directly adjust the distance function on the metric space.

We study the relationship between different instances of WT and prove that each of the instances is stable under perturbations. We also draw connections between WT and the renowned Ricci flow \cite{hamilton1982three} to explain the geometric intuition behind WT. We devise iterative algorithms for performing the aforementioned WT and propose a set of strategies to accelerate GT, such as an observation from linear algebra for reducing the number of matrix square root computations involved. We perform extensive experiments and our results demonstrate that the WT family of methods boost the performance significantly in tasks related to denoising, clustering, image segmentation and the generation of word embeddings, etc. The stability results proved later in the paper are technically involved and underpin the robustness claims for these transforms.

This paper is a significant expansion of \cite{memoli2019wasserstein} where a preliminary version of WT was discussed.

\paragraph{Related work.} Word embedding is an important task in NLP which aims at generating representations of words in order to encode semantic similarity between words. Traditionally, words are represented by vectors and such representations are learned via word contexts \citep{pennington2014glove,mikolov2013distributed}. A recent trend in the NLP community consists of representing words as probability measures. For example,  \cite{Vilnis2014WordRV} represented words by Gaussian distributions to allow modeling uncertainty and asymmetry in word embedding. This approach was later extended in \citep{athiwaratkun2017multimodal} by representing words via Gaussian mixture models.

As a generalization of the above application in NLP,  \cite{muzellec2018generalizing} have studied possible advantages of embedding general data points into the space of Gaussian measures equipped with the $\ell^2$-Wasserstein distance. Later, \cite{frogner2019learning} investigated a similar embedding into the space consisting of general probability measures with finite support. 

One main goal of all the above works \citep{Vilnis2014WordRV,muzellec2018generalizing,frogner2019learning} is to learn an optimal (via certain loss functions) embedding of data points (e.g., words) into a certain space of probability measures equipped with either the KL-divergence or the Wasserstein distance. This is fundamentally different from our approach which aims at updating the distance of the given data set based on local information without referring to a learning process. 

A recent work by
\cite{singh2020context} is in spirit closely related to our idea: they proposed to represent words in a given corpus via probability measures reflecting context information and then to utilize the Wasserstein distance between these probability measures to perform on NLP tasks such as sentence representation or word similarity\footnote{We developed our approach independently of theirs and in fact, the original version of \citep{singh2020context} was published at around the same time as our published preliminary version \citep{memoli2019wasserstein}. We remark that their method, though specially tailored for NLP tasks, fits into our general framework of WT.}.

\paragraph{Organization of the paper.} In Section \ref{sec:background}, we recall some basic concepts used in this paper. In Section \ref{sec:wt-framework}, we introduce the Wasserstein Transform framework. Several different instances of WT methods are discussed and in particular, we formulate the mean shift method as an instance of WT. Section \ref{sec:theory} is devoted to study various theoretical properties of different instances of WT including stability results. In Section \ref{sec:algorithm}, we describe algorithms for WT methods and introduce several computational techniques for GT. In Section \ref{sec:experiments}, we apply WT methods to several data analysis tasks. Most of the proofs of results in this paper are provided in Section \ref{sec:wt-proof}.
\section{Preliminaries}\label{sec:background}
\subsection{Basic concepts about metric spaces and probability measures}\label{sec: probability}

\paragraph{Metric spaces.}
A metric space consists of a pair $(X,d_X)$ where $X$ is a set and $d_X:X\times X\rightarrow\Rp$ is a non-negative function satisfying the following three conditions, for any $ x,x',x''\in X$: 

\begin{enumerate}
\item $d_X(x,x')=0$ if and only if $x=x'$;

\item $d_X(x,x')=d_X(x',x)$;

\item $\,d_X(x,x'')\leq d_X(x,x')+d_X(x',x'').$
\end{enumerate}
We call $d_X$ a distance function on $X$.

\begin{remark}[Ultrametric spaces]
A special type of metric spaces we will use later in this paper is \emph{ultrametric spaces}. Given a metric space $(X,d_X)$, we say $d_X$ is an \emph{ultrametric} if $d_X$ satisfies the so-called strong triangle inequality:
for any $x,x',x''\in X$
\[d_X(x,x')\leq\max\lc d_X(x,x''),d_X(x'',x')\rc.\]
We then call $(X,d_X)$ an ultrametric space. We will henceforth use $u_X$ instead of $d_X$ to denote an ultrametric on $X$.
\end{remark}

In the sequel of this paper, we need the following concepts related to metric spaces:

\begin{enumerate}
    \item Inside any metric space $X$, a closed ball with radius $\eps>0$ and with center $x\in X$ is defined as the set $B_\eps(x)\coloneqq\{x'\in X:\,d_X(x,x')\leq \eps\}$. Sometimes we write $B_\eps(x)$ as $B^{d_X}_\eps(x)$ or $B^{X}_\eps(x)$ to emphasize the underlying metric or metric space. A closed ball with radius $\eps$ is also called an \emph{$\eps$-ball} in the sequel.

    \item Based on the notion of closed balls, given any subset $A\subseteq X$, we define its $\eps$-neighborhood $A^\eps$ by 
$$A^\eps\coloneqq\cup_{x\in A}B_\eps(x).$$

\end{enumerate}

\paragraph{Probability measures.}
In this paper, we will use various collections of Borel probability measures on a given metric space $X$:
\begin{enumerate}
    \item $\mathcal{P}(X)$: the collection of all Borel probability measures.
    \item $\mathcal{P}_f(X)$: the collection of all Borel probability measures with \emph{full support}.
    \item $\mathcal{P}_p(X)$: the collection of all Borel probability measures with \emph{finite} $p$-moment where $p\in[1,\infty]$, i.e., $\alpha\in\mathcal{P}_p(X)$ iff $\forall x_0\in X,\quad\norm{d_X(\cdot,x_0)}_{L^p(\alpha)}<\infty.$ 
\end{enumerate}
Note that $\mathcal{P}(X)=\mathcal{P}_p(X)$ whenever $X$ is bounded.

Given two metric spaces $X$ and $Y$, any measurable map $\varphi:X\rightarrow Y$ induces the \emph{pushforward} map $\varphi_\#:\mathcal{P}(X)\rightarrow\mathcal{P}(Y)$: for any $\alpha\in\mathcal{P}(X)$, its pushforward $\varphi_\#\alpha$ is the unique measure such that for any measurable set $A\subseteq Y$
\[\varphi_\#\alpha(A)=\alpha(\varphi^{-1}(A)).\]

Any `nice' probability measure on $\R^m$ comes with the notions of mean and covariance matrix. 

\begin{definition}\label{def:mean and covariance}
Given any $\alpha\in\mathcal{P}_1(\R^m)$, the \emph{mean} $\mu_\alpha$ of $\alpha$ is defined as follows:
$$\mu_\alpha\coloneqq \int_{\R^m}x\,\alpha(dx).$$
If moreover $\alpha\in\mathcal{P}_2(\R^m)$, we define the \emph{covariance matrix} $\Sigma_\alpha$ of $\alpha$ as follows:
$$\Sigma_\alpha\coloneqq\int_{\mathbb{R}^m}\lc{x}-{\mu_\alpha}\rc\otimes \lc {x}-{\mu_\alpha}\rc\,\alpha(dx),$$
where $x\otimes y$ is the bi-linear form on $\R^m$ such that $x\otimes y(u,v)=\langle x,u\rangle\cdot\langle y,v\rangle$ for any $u,v\in\R^m$. 
\end{definition}

For later use, we let $\mathcal{P}_{1,2}(\R^m)\coloneqq\mathcal{P}_{1}(\R^m)\cap \mathcal{P}_{2}(\R^m)$.

Finally, we recall the definition for \emph{Gaussian measures} on Euclidean spaces. Given $\mu\in \R^m$ and a positive semi-definite matrix $\Sigma\in\R^{m\times m}$, the corresponding Gaussian measure $\gamma\coloneqq\mathcal{N}(\mu,\Sigma)$ is determined by the following Gaussian distribution function:
\[f(x)=\frac{1}{\sqrt{(2\pi)^m\det(\Sigma)}}\exp\lc{-\frac{1}{2}(x-\mu)^\mathrm{T}\Sigma^{-1}(x-\mu)}\rc,\,\forall x\in \R^m.\]

Notice that the Gaussian measure $\gamma\in\mathcal{P}_{1,2}(\R^m)$ and moreover, its mean $\mu_\gamma$ coincides with $x$ and its covariance $\Sigma_\gamma$ coincides with $\Sigma$.

\paragraph{Transition kernels.} One of the central notions in defining our Wasserstein Transform is the notion of transition kernels from the theory of Markov chains.
Given a metric space $(X,d_X)$, a \emph{transition kernel} on $X$ is any measurable map $m:X\rightarrow\mathcal{P}(X)$ (w.r.t. the weak topology on $\mathcal{P}(X)$). We let $\mathfrak{T}(X)$ denote the collection of all transition kernels on $X$. More generally, given any two metric spaces $X$ and $Y$, we let $\mathfrak{T}(X,Y)$ denote the collection of all measurable maps $m:X\rightarrow\mathcal{P}(Y)$, which we call transition kernels on $Y$ induced by $X$.

\subsection{Optimal transport concepts}\label{sec:ot concepts}
Let $(X,d_X)$ be a metric space and let $p\in[1,\infty]$. There exists a natural distance $\dW{p}^X$, the \emph{$\ell^p$-Wasserstein distance}  \citep{villani2008optimal}, for comparing probability measures in $\mathcal{P}_p(X)$:
\begin{definition}[$\ell^p$-Wasserstein distance]\label{def:p-w-dist}
For any metric space $X$ and 
any $\alpha,\beta\in\mathcal{P}_p\left(X\right)$, the $\ell^p$-Wasserstein distance between $\alpha$ and $\beta$ is defined as follows:
$$d_{\mathrm{W},p}^{(X,d_X)}(\alpha,\beta) \coloneqq  \left(\inf_{\pi\in \mathcal{C}(\alpha,\beta)} \int_{X\times X} (d_X(x,x'))^p\,\pi(dx\times dx')\right)^\frac{1}{p},$$
and for $p=\infty$:
$$d_{\mathrm{W},\infty}^{(X,d_X)}(\alpha,\beta) \coloneqq  \inf_{\pi\in\mathcal{C}(\alpha,\beta)}\sup_{(x,x')\in\supp(\pi)}d_X(x,x'),$$
where $\mathcal{C}(\alpha,\beta)$ is the set of all \emph{couplings} $\pi$ (also named \emph{transport plans}) between $\alpha$ and $\beta$, i.e., $\pi$ is a probability measure on $X\times X$ with marginals $\alpha$ and $\beta$, respectively. When the underlying metric space $(X,d_X)$ is clear from the context, we will simply write $d_{\mathrm{W},p}$ instead of $d_{\mathrm{W},p}^{(X,d_X)}$.
\end{definition}

\begin{remark}
If $X$ is complete and separable, then there exists $\pi\in\mathcal{C}(\alpha,\beta)$ such that $d_{\mathrm{W},p}^{(X,d_X)}(\alpha,\beta) = \norm{d_X}_{L^p(\pi)}$ (cf. \cite[Theorem 4.1]{villani2008optimal}). We call any such $\pi$ \emph{an optimal coupling} and let $\mathcal{C}_p^\mathrm{pot}(\alpha,\beta)$ denote the set of all optimal couplings.
\end{remark}

We then define the Wasserstein space of $X$ as $W_p(X)\coloneqq(\mathcal{P}_p(X),\dW{p})$.

Solving the optimization problem for computing the Wasserstein distance is usually time-consuming \citep{cuturi2013sinkhorn}. However, there are some special cases where the Wasserstein distance admits explicit formulas.

\begin{example}[Wasserstein distance between Dirac measures]\label{remark:deltas}
A simple but important example \citep{villani2008optimal} is that for points $x,x'\in X$, if one considers the Dirac measures $\delta_x$ and $\delta_{x'}$ supported at those points (which will be probability measures), then for any $p\in[1,\infty]$ we have that
$d_{\mathrm{W},p}(\delta_x,\delta_{x'})=d_X(x,x')$, i.e., the ``ground'' distance between $x$ and $x'$.
\end{example}

\begin{example}[Wasserstein distance on the real line]\label{ex:dw on R}
It is shown in page 73 of \cite{villani2003topics} that given two probability measures $\alpha,\beta\in\mathcal{P}_1(\R)$, we have
\begin{equation}\label{eq:closed-form-real line}
   \dW{1}(\alpha,\beta)=\int_{-\infty}^\infty|F_\alpha(t)-F_\beta(t)|dt, 
\end{equation}
where $F_\alpha$ and $F_\beta$ are the cumulative distribution functions of $\alpha$ and $\beta$, respectively. A similar expression is true for $p>1$.
\end{example}

\begin{example}[Wasserstein distance between Gaussians]\label{rmk:Gaussian distance}
In the case of Gaussian measures on Euclidean spaces, the $\ell^2$-Wasserstein distance enjoys a \emph{closed form formula} which allows for efficient computation. Given two Gaussian measures $\gamma_1=\mathcal{N}(x_1,\Sigma_1)$ and $\gamma_2=\mathcal{N}(x_2,\Sigma_2)$ on $\R^m$, we have that
\begin{equation}\label{eq:w2 btw gaussian}
    d_{\mathrm{W},2}(\gamma_1,\gamma_2)=\lc{\norm{x_1-x_2}^2+ (\dcov(\Sigma_1,\Sigma_2))^2}\rc^{\frac{1}{2}},
\end{equation}
where 
\begin{equation}\label{eq:dcov}
    \dcov(\Sigma_1,\Sigma_2)\coloneqq\lc\tr\left(\Sigma_1+\Sigma_2-2\left(\Sigma_1^\frac{1}{2}\Sigma_2\Sigma_1^\frac{1}{2}\right)^\frac{1}{2}\right)\rc^\frac{1}{2}.
\end{equation}
See \cite{givens1984class} for a proof. Note that ${\dcov}$ is also known as the Bures distance \citep{bures1969extension} between positive semi-definite matrices. 
\end{example}

Note that given any probability measure $\alpha\in\mathcal{P}_{1,2}(\R^m)$ on $\R^m$, we can induce a Gaussian measure $\gamma_\alpha=\mathcal{N}(\mu_\alpha,\Sigma_\alpha)$ on $\R^m$, where $\mu_\alpha$ and $\Sigma_\alpha$ denote the mean and the covariance of $\alpha$, respectively.

The following result provides lower bounds for the Wasserstein distance between arbitrary probability measures on $\R^m$.

\begin{lemma}[Lower and upper bounds on Euclidean spaces]\label{lm:bound}
Let $\alpha_1,\alpha_2\in\mathcal{P}_{1,2}(\R^m)$. 

Then, we have the following inequalities:
\begin{enumerate}
    \item $\norm{\mu_{\alpha_1}-\mu_{\alpha_2}}\leq d_{\mathrm{W},2}(\gamma_{\alpha_1},\gamma_{\alpha_2})\leq d_{\mathrm{W},2}(\alpha_1,\alpha_2).$
    \item $\norm{\mu_{\alpha_1}-\mu_{\alpha_2}}\leq d_{\mathrm{W},1}(\alpha_1,\alpha_2).$
\end{enumerate}
\end{lemma}

\begin{proof}
1. The leftmost inequality is an immediate consequence of the formula for the $\ell^2$-Wasserstein distance between Gaussian measures:
$$d_{\mathrm{W},2}(\gamma_{\alpha_1},\gamma_{\alpha_2})=\lc{\norm{\mu_{\alpha_1}-\mu_{\alpha_2}}^2+(\dcov(\Sigma_{\alpha_1},\Sigma_{\alpha_2}))^2}\rc^\frac{1}{2}\geq \norm{\mu_{\alpha_1}-\mu_{\alpha_2}}.$$
The rightmost inequality was proved in \citep{gelbrich1990formula}.

2. This inequality was proved in \cite{rubner1998metric}.
\end{proof}

\paragraph{Kantorovich duality.} One of the most useful tools for calculating $\dW{1}$ is Kantorovich duality (see for example \cite[Remark 6.5]{villani2008optimal}), which states that for any compact metric space $X$ and $\alpha,\beta\in\mathcal{P}(X)$,
\[\dW{1}(\alpha,\beta)=\sup\left\{\left|\int f(x) \alpha(dx)-\int f(x) \beta(dx)\right|:f:X\rightarrow\R\text{ is 1-Lipschitz}\right\}.\]

Notice that for any 1-Lipschitz function $f:X\rightarrow\R$ and  $c\in\R$, the translation $f+c$ of $f$ will not affect the value $\left|\int f(x) \alpha(dx)-\int f(x) \beta(dx)\right|$. Moreover, $|f(x)-f(y)|\leq d_X(x,y)\leq\diam(X)$. Hence, when $X$ is bounded we can impose on $f$ the condition that $\sup_{x\in X}|f(x)|\leq\diam(X)$, which implies that
\begin{equation}\label{eq:kant duality}
    \dW{1}(\alpha,\beta)=\sup\left\{\left|\int f(x) \alpha(dx)-\int f(x) \beta(dx)\right|:f\text{ is 1-Lipschitz and }\sup_{x\in X}|f(x)|\leq\diam(X).\right\}
\end{equation}

\paragraph{Relationship with the Prokhorov distance.}\label{sec:prokhorov}
Given a metric space $(X,d_X)$, the \emph{Prokhorov distance} $\dP^{X}(\alpha,\beta)$ (or simply $\dP (\alpha,\beta)$, when the underlying metric space is clear) between $\alpha,\beta\in\mathcal{P}(X)$ is defined as
$$\dP^{X}(\alpha,\beta)\coloneqq\inf\{r>0\,:\, \alpha(A)\leq \beta(A^r)+r,\,\forall A\subseteq X\}.$$
Though seemingly asymmetric, $\dP $ is actually symmetric; see \cite{dudley2018real} for more details.

The following result exhibits a relationship between $\dP$ and $\dW{1}$:
\begin{lemma}[{\cite[Theorem 3]{gibbs2002choosing}}]\label{lemma:dp}
Given a compact metric space $(X,d_X)$, for any $\alpha,\beta\in \mathcal{P}(X)$, we have the following relationship between the Wasserstein  and the Prokhorov distances:
\[\lc \dP (\alpha,\beta)\rc ^2\leq \dW{1}(\alpha,\beta)\leq\lc 1+\diam(X)\rc \dP (\alpha,\beta).\]
\end{lemma}

\begin{remark}\label{rmk:pr}
If $\alpha$ and $\beta$ are not fully supported, then by restricting the metric $d_X$ to $S\coloneqq\supp(\alpha)\cup\supp(\beta)\subseteq X$, the rightmost inequality above can be improved to
$\dW{1}(\alpha,\beta)\leq\lc 1+\diam(S)\rc \dP (\alpha,\beta).$
\end{remark}

\section{Framework and instances of the Wasserstein Transform}\label{sec:wt-framework}

In this section, we first introduce the general framework of Wasserstein Transform (WT) and we then introduce several important instances of WT.

\subsection{The Wasserstein Transform}
The essential idea behind the Wasserstein Transform is to first capture local information of the data and then induce a new distance function between pairs of points based on the dissimilarity between their respective neighborhoods. Localization operators are gadgets that capture these neighborhoods.

\begin{definition}[Localization operator] Let $(X,d_X)$ be a metric space. Recall from \Cref{sec: probability} that $\mathcal{P}_f(X)$ denotes the collection of all $\alpha\in\mathcal{P}(X)$ with full support and that $\mathfrak{T}(X)$ denotes the collection of all transition kernels on $X$. Then, a \emph{localization operator} $L$ 
{on $X$} is any map 
$$L:\pf(X)\rightarrow\mathfrak{T}(X), $$
i.e., $L$ sends every $\alpha\in\mathcal{P}_f(X)$ to a measurable map $m_{\alpha,d_X}^L:X\rightarrow \mathcal{P}(X)$.
\end{definition}

Two simple examples of localization operators are given as follows. 
\begin{enumerate}
    \item[(a)] Given $\alpha$ in $\mathcal{P}_f(X)$, let $m_{\alpha,d_X}^L(x)\equiv \alpha,\forall x\in X$, which assigns to all points in $X$ the (same) reference probability measure $\alpha$. This is a trivial example in that it does not localize the  measure $\alpha$ at all.  
    \item[(b)] For any $\alpha$ in $\mathcal{P}_f(X)$, let $m_{\alpha,d_X}^L(x)=\delta_x,\forall x\in X$. This is a legitimate localization operator but it does not retain any information from $\alpha$.
\end{enumerate}
More interesting localization operators will be considered in \Cref{sec:instance of LO}.

After specifying a localization operator $L$ and given $\alpha\in\mathcal{P}_f(X)$, one associates each point $x$ in $X$ with a probability measure $m_{\alpha,d_X}^L(x)$, and then obtains a new metric space by considering the Wasserstein distance between each pair of these measures. See \Cref{fig:interp-wt-once} for an illustration.

\begin{figure}[htb]
\includegraphics[width=\linewidth]{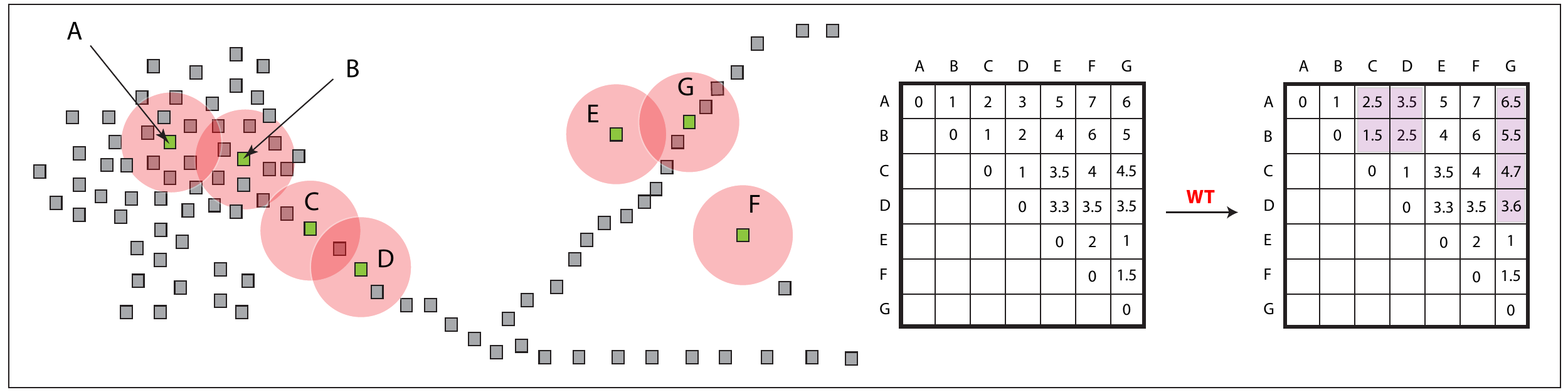}
\caption{\textbf{Illustration of the Wasserstein Transform (continued from \Cref{fig:interp-wt}).} Under the Wasserstein transform, the distance between points with similar neighborhood structures, e.g., A and B, or C and D, will remain more or less the same. In contrast, the distance between points, e.g., B and C, or E and G, with different neighborhood structures will be enlarged. The cartoon on the right sketches the corresponding change in the distance matrix.} \label{fig:interp-wt-once}
\end{figure}

\begin{definition}[The Wasserstein Transform]\label{def:WT}
Let $(X,d_X)$ be a given metric space and let $\alpha \in \mathcal{P}_f(X)$. Given a localization operator $L$ and $p\in[1,\infty]$, the \emph{Wasserstein Transform} $\mathbf{W}_{L,p}$ applied to $\alpha$ defines the distance function $d_{\alpha,d_X}^{L,p}$ on $X$, which we will refer to as the \emph{WT distance}, by
$$d_{\alpha,d_X}^{L,p}(x,x')\coloneqq \dW{p}\lc m_{\alpha,d_X}^{L}(x),m_{\alpha,d_X}^{L}(x')\rc ,\forall x,x'\in X.$$
By $\mathbf{W}_{L,p}(\alpha)$ we will henceforth denote the resulting metric space\footnote{Technically speaking, $d_{\alpha,d_X}^{L,p}$ is a \emph{pseudodistance}, i.e.,  $d_{\alpha,d_X}^{L,p}(x,x')=0$ may not imply $x=x'$. A pseudometric space $(X,d_X)$ can be canonically transformed into a metric space by identifying points at distance 0 (cf. \cite{burago2001course}). Hence, this pseudodistance will not pose any obstacle for applying WT iteratively as suggested in Remark \ref{rmk:iterating WT}.} $\lc X,d_{\alpha,d_X}^{L,p}\rc.$
\end{definition}

In this paper, we restrict ourselves to the cases $p=1$ and $p=2$.

\begin{remark}[Measurability of the identity map]\label{rmk:measurable identity}
The assumption that $m_{\alpha,d_X}^L:X\rightarrow\mathcal{P}(X)$ is measurable implies the measurability of the identity map $(X,d_X)\rightarrow\lc X,d_{\alpha,d_X}^{L,p}\rc $.
\end{remark}

\begin{remark}[Iterating the Wasserstein Transform]\label{rmk:iterating WT}
The Wasserstein Transform can be iterated any desired number of times with the purpose of successively enhancing features and/or reducing noise. After applying the Wasserstein Transform once to $\alpha \in \mathcal{P}_f(X)$, the metric space $(X,d_X)$ is transformed into $\lc X,d_{\alpha,d_X}^{L,p}\rc $. We still denote by $\alpha$ the pushforward of $\alpha$ under the identity map $(X,d_X)\rightarrow\lc X,d_{\alpha,d_X}^{L,p}\rc $, which is measurable due to Remark \ref{rmk:measurable identity}. Then, we can apply the Wasserstein Transform again to $\alpha$ on $\lc X,d_{\alpha,d_X}^{L,p}\rc $ and so on (see \Cref{fig:iterate wt} for an illustration).
\end{remark}

\begin{figure}
\begin{center}
\begin{tikzpicture}
  \node (a) {$(X,d_X)$};
  \node[block,right=3.5em of a,label=below:{\scriptsize Localization operator}] (b) {$L$};
  \node[block,right=of b] (c) {$\dW{p}$};
  \node[right=3em of c] (d) {$\wt{p}^L(\alpha) = \lc X,d_{\alpha,d_X}^{L,p}\rc $};

  \node[draw,dashed, red, line width=1mm, inner xsep=5mm,inner ysep=6mm,fit=(b)(c),label={Wasserstein Transform of $\alpha$}]{};
  \draw[line] (a)-- (b);
  \draw[line] (b) -- (c);
  \draw[line] (c)-- (d);
  \draw[line] (d.south) -- ++(0,-1.5cm) -| node [near start,above]{iterate} ($(a.south) + (0.4cm, 0em)$);
\end{tikzpicture}

\end{center}
\caption{\textbf{Illustration of iterative WT.} The box contains the whole process of applying once WT. Note that the whole process can be iterated by letting $\lc X,d_{\alpha,d_X}^{L,p}\rc$ be the input space.}
\label{fig:iterate wt}
\end{figure}

\subsubsection{Extrinsic Wasserstein Transform}\label{sec:extrinsic WT}
The Wasserstein Transform defined in Definition \ref{def:WT} operates on a general coordinate-free metric space. However, in many applications, data points possess coordinates, i.e., either they are point clouds in an ambient Euclidean space $\R^m$ or they possess Euclidean features. It is then reasonable to absorb local information of the given data set by incorporating this coordinate information. Moreover, by taking into account the underlying Euclidean structure, one may take advantage of the various properties of the Wasserstein distance that we have seen in Section \ref{sec:ot concepts}. Given these considerations, we describe the framework of \emph{extrinsic Wasserstein Transform}.

The extrinsic WT works with any metric space $(X,d_X)$ equipped with a \emph{coordinatization map}, i.e., any continuous map $\varphi:X\rightarrow \R^m$. We first define the notion of \emph{extrinsic localization operators} associated with such data.
The notation in the following definition is a bit involved, so we will explain it via examples immediately after giving the definition:

\begin{definition}[Extrinsic localization operator]\label{def:extrinsic localization operator}
Let $(X,d_X)$ be a metric space with a coordinatization map $\varphi:X\rightarrow \R^m$. Let $\psi:\R^m\times\mathcal{P}_{1,2}(\R^m)\rightarrow\mathcal{P}(\R^m)$ be any measurable map. For any localization operator $L:\pf(X)\rightarrow\mathfrak{T}(X)$, we then define
$$L_{\psi,\varphi}:\pf(X)\rightarrow\mathfrak{T}(X,\R^m)$$
by sending every $\alpha\in\mathcal{P}_f(X)$ to the measurable map 
$$m_{\alpha,d_X}^{L_{\psi,\varphi}}:X\rightarrow \mathcal{P}(\R^m) \text{ defined by }x\mapsto \psi\left(\varphi(x), \varphi_\#m_{\alpha,d_X}^L(x)\right),\,\forall x\in X.$$
We call $L_{\psi,\varphi}$ the \emph{extrinsic localization operator} on $X$ induced by $L$, $\psi$ and $\varphi$. 
\end{definition}

\begin{figure}
\centering
\includegraphics[width=0.82\linewidth]{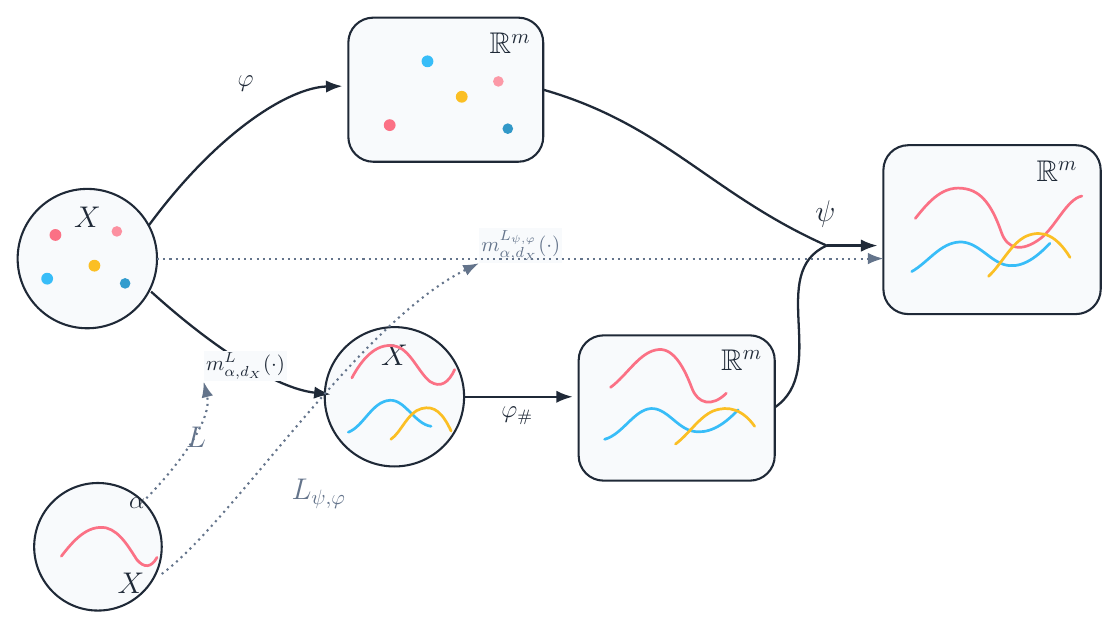}
\caption{\textbf{Extrinsic localization operator attached to a point $x$.} The figure mirrors Definition \ref{def:extrinsic localization operator}: the measure $\alpha$ on $X$ is localized by $L$ to produce $m^L_{\alpha,d_X}(x)$, the coordinate map $\varphi$ and the pushforward map $\varphi_\#$ send the point and the localized measure into $\R^m$, and the map $\psi$ combines these two inputs to produce $m_{\alpha,d_X}^{L_{\psi,\varphi}}(x)$.}
\label{fig:extrinsic WT}
\end{figure}
See Figure \ref{fig:extrinsic WT} for a visualization.

The coordinatization map $\varphi:X\rightarrow \R^m$ can be the inclusion map when $X$ is a subset of $\R^m$.
For the map $\psi:\R^m\times\mathcal{P}_{1,2}(\R^m)\rightarrow\mathcal{P}(\R^m)$, we only consider the following two cases; the usefulness of them will be shown very soon in \Cref{sec:ms} and \Cref{sec:gt}:
\begin{enumerate}
    \item $\psi$ sends every $(y,\alpha)$ to the Dirac measure supported on the mean $\mu_\alpha\coloneqq\mean(\alpha)$.
    \item $\psi$ sends every $(y,\alpha)$ to the Gaussian measure $\mathcal{N}\lc y,\Sigma_\alpha\rc$, where $\Sigma_\alpha$ denotes the covariance matrix of $\alpha$ (cf. Definition \ref{def:mean and covariance}).
\end{enumerate}

\begin{definition}[The extrinsic Wasserstein Transform]\label{def:extrinsic WT}
Let $(X,d_X)$ be a metric space. Let $\varphi:X\rightarrow \R^m$ be a continuous map and let $\psi:\R^m\times\mathcal{P}_{1,2}(\R^m)\rightarrow\mathcal{P}(\R^m)$ be a measurable map. Consider $\alpha \in \mathcal{P}_f(X)$. Given any localization operator $L$ on $X$ and $p\in[1,\infty]$, the \emph{extrinsic Wasserstein Transform} $\mathbf{W}_{L_{\psi,\varphi},p}$ applied to $\alpha$ defines the distance function $d_{\alpha,d_X}^{L_{\psi,\varphi},p}$ on $X$ by
$$d_{\alpha,d_X}^{L_{\psi,\varphi},p}(x,x')\coloneqq \dW{p}^{\R^m}\lc m_{\alpha,d_X}^{L_{\psi,\varphi}}(x),m_{\alpha,d_X}^{L_{\psi,\varphi}}(x')\rc ,\forall x,x'\in X.$$
By $\mathbf{W}_{L_{\psi,\varphi},p}(\alpha)$ we denote the resulting metric space $\lc X,d_{\alpha,d_X}^{L_{\psi,\varphi},p}\rc.$
\end{definition}

\subsection{Instances of localization operators and corresponding Wasserstein Transforms}\label{sec:instance of LO}
In this section, we introduce three types of localization operators and thus three different types of Wasserstein Transforms. 

To help readers navigate our notation, we summarize our acronyms in Table \ref{tab:WT-acronym} and our symbols in Table \ref{tab:WT-symbol}.

\begin{table}[H]
    
    \caption{\textbf{Useful acronyms.}}
    \centering
\adjustbox{max width=\textwidth}{\begin{tabular}{ |c|c|c|  }
 \hline
 \multicolumn{3}{|c|}{\textbf{Acronyms for different instances of WT}} \\
 \hline
 localization operator  &  general & $p$ is specified \\
 \hline
 \hline
 general  & WT &   WT$p$\\
 \hline
 kernel localization&  KL-WT& KL-WT$p$\\
 \hline
 local truncation&   LT-WT & LT-WT$p$\\
 \hline
 mean shift (w.r.t. local truncation)&  MS& N/A\\
 \hline
 Gaussian localization/Gaussian Transform&   GT & N/A\\
 \hline
\end{tabular}}
\label{tab:WT-acronym}
\end{table}

\begin{table}[H]

    \caption{\textbf{Useful symbols.} When $p=1$, we usually suppress $p$ from symbols above, e.g., $\deps\coloneqq d^{\mathsmaller{(\eps)},1}_{{\alpha,d_X}}$. When $(X,d_X)\subseteq\R^m$ is an Euclidean space, we usually suppress $d_X$ from symbols below, e.g., $\me_{\alpha}(x)\coloneqq\me_{\alpha,d_X}(x)$.}
    \centering
\adjustbox{max width=\textwidth}{\begin{tabular}{ |c||c|c|c|c|c|c|  }
 \hline
 \multicolumn{7}{|c|}{\textbf{Symbols for WT}} \\
 \hline
  &general & extrinsic & KL-WT & LT-WT & MS & GT\\
 \hline
 localization operators   & $L$    & $L_{\psi,\varphi}$ & $L_K$ & N/A  & $\lms$& N/A \\
 \hline
 localized measures at $x$ &$m_{\alpha,d_X}^{L}(x)$  & $m_{\alpha,d_X}^{L_{\psi,\varphi}}(x)$ &   $m_{\alpha,d_X}^{K}(x)$ &$\me_{\alpha,d_X}(x)$& $m_{\alpha,d_X}^{\lms}(x)$& $\gamma_{\alpha,d_X}^{\mathsmaller{(\eps,\lambda)}}( {x})$\\
 \hline
 generated distances & $d_{\alpha,d_X,p}^{L}$    &$d_{\alpha,d_X,p}^{L_{\psi,\varphi}}$&  $d_{\alpha,d_X,p}^{K}$ &$d^{\mathsmaller{(\eps)}}_{{\alpha,d_X,p}}$ &$d^{\lms}_{\alpha,d_X}$& $d^{\mathsmaller{(\eps,\lambda)}}_{{\alpha,d_X}}$\\
 \hline
 Wasserstein Transforms    & $\mathbf{W}_{p}^L$   &$\mathbf{W}_{p}^{L_{\psi,\varphi}}$&   $\mathbf{W}_{p}^K$ & $\mathbf{W}_{p}^\mathsmaller{(\eps)}$ &$\wt{1}^{\lms}$ & $\gt$\\
 \hline
\end{tabular}}
\label{tab:WT-symbol}
\end{table}

\subsubsection{Kernel Localization}
There is a large class of localization operators defined via \emph{kernel functions}. Let $(X,d_X)$ be a \emph{compact} metric space. A \emph{kernel function} $K:X\times X\rightarrow\Rp$ is any \emph{bounded measurable} function. Given $\alpha\in\mathcal{P}_f(X)$, for each $x\in X$ we assume that $\int_XK(x,y)\alpha(dy)>0$ and we define the probability measure 
$$m_{\alpha,d_X}^K(x)(dy)\coloneqq \frac{K(x,y)\alpha(dy)}{\int_XK(x,y)\alpha(dy)}.$$  Due to the proposition below, the assignment $x\mapsto m_{\alpha,d_X}^K(x)$ is actually measurable. Therefore, this assignment defines a localization operator on $X$, and we call this localization operator the \textbf{kernel localization} and denote it by $L_K$.

\begin{proposition}[Measurability]
For any $p\in[1,\infty)$, if the kernel function $K$ is such that $\int_XK(x,y)\alpha(dy)>0$ for all $x\in X$, then the map $X\rightarrow W_p(X)$ defined by $x\mapsto m_{\alpha,d_X}^K(x)$ is measurable.
\end{proposition}
\begin{proof}
By Proposition 7.26 in \cite{bertsekas1996stochastic} and the fact that $\dW{p}$ metrizes the weak topology \citep[Theorem 6.9]{villani2008optimal}, we only need to show that for any fixed measurable set $A\subseteq X$, the map
\[x\mapsto m_{\alpha,d_X}^K(x)(A)= \frac{\int_AK(x,y)\alpha(dy)}{\int_XK(x,y)\alpha(dy)}\in\R\]
is measurable. This then follows directly from the Fubini-Tonelli Theorem (see for example \cite[Theorem 2.37]{folland1999real}).
\end{proof}

We denote by $d_{\alpha,d_X}^{K,p}$ the new metric generated by the corresponding Wasserstein Transform with respect to the $\ell^p$-Wasserstein distance and when $p=1$, we simply write $d_{\alpha,d_X}^{K}\coloneqq d_{\alpha,d_X}^{K,1}$. We denote by $\mathbf{W}_{K,p}$ the version of the Wasserstein Transform thus induced.  We use the abbreviation KL-WT for the \emph{Wasserstein Transform with kernel localization}.

One special type of kernel functions are \emph{rotationally symmetric kernels}, i.e., those for which there exists a function $\mathcal{K}:\Rp\rightarrow\Rp$ such that 
$$K(x,x')\coloneqq\mathcal{K}\lc{d_X(x,x')}\rc,\quad\forall x,x'\in X.$$

Sometimes, rotationally symmetric kernels are specified together with a bandwidth $\eps>0$, and in this case kernels are usually defined via squared distance functions as follows
\begin{equation}\label{eq:kernel with bandwidth}
    K_\eps(x,x')\coloneqq\mathcal{K}\lc\frac{d_X^2(x,x')}{\eps^2}\rc,\quad\forall x,x'\in X.
\end{equation}

\paragraph{Local truncation.} We now concentrate on a particular type of kernel localization which we call \textbf{local truncation}. Let $\mathcal{K}(t)=\chi_{[0,1]}(t)$ be the indicator function of the interval $[0,1]$. Then, for any $\eps>0$, this function induces a rotationally symmetric kernel $K_\eps$ with bandwidth $\eps$ on $X$ (cf. \Cref{eq:kernel with bandwidth}). We call the kernel localization operator corresponding to $K_\eps$ the \emph{local truncation}, and denote it by $L^{\mathsmaller{(\eps)}}$. More explicitly, we have for every $x\in X$ a probability measure $\me_{\alpha,d_X}(x)$ as follows: for any measurable set $A\subseteq X$,
\[\me_{\alpha,d_X}(x)(A)\coloneqq\frac{\int_AK_\eps(x,y)\alpha(dy)}{\int_XK_\eps(x,y)\alpha(dy)}= \frac{\alpha\lc B_\eps^{d_X}(x)\cap A\rc}{\alpha\lc B_\eps^{d_X}(x)\rc}.\]
The assumption that $\alpha$ is fully supported and the fact that $\eps>0$ together guarantee that the denominator $\alpha\lc B_\eps^{d_X}(x)\rc$ is positive.

When $X$ is finite, $X=\{x_1,\ldots,x_n\}$, and $\alpha$ is its empirical measure, the formula above becomes  
\[\me_{\alpha,d_X}(x)(A) = \frac{\#\{i:\,x_i\in A\,\,\mbox{{\small and}}\,\,d_X(x_i,x)\leq \eps\}}{\#\{i:\,d_X(x_i,x)\leq \eps\}}.\]

We denote by $\mathbf{W}_{p}^\mathsmaller{(\eps)}$ the corresponding version of the Wasserstein Transform with respect to the $\ell^p$-Wasserstein distance. When $p=1$, we denote by $\deps$ the new metric generated by $\mathbf{W}_{1}^\mathsmaller{(\eps)}$. Then, $\mathbf{W}_{1}^\mathsmaller{(\eps)}(\alpha)=\lc X,\deps\rc$.

Using the fact that the $\ell^1$-Wasserstein distance on $\R$ admits a closed form expression (cf. \Cref{eq:closed-form-real line}) we are able to prove the following Taylor expansion which permit interpreting the behavior of WT.
\begin{remark}\label{rem:taylor-1d}
Let $X =\R$ and assume that $\alpha\in\mathcal{P}_f(X)$ has a smooth density $f$ with bounded $L^1$-norm. Then, we have the following asymptotic formula for $\deps(x,x')$ when $\eps\rightarrow 0$ (whose proof can be found in \Cref{sec:wt-proof}):
\[\mbox{for $x'>x$ and $f(x),f(x')\neq 0$,}\,\,\deps(x,x')=x'-x+\frac{1}{3}\lc\frac{f'(x')}{f(x')}-\frac{f'(x)}{f(x)}\rc\eps^2+O(\eps^3).\]

The interpretation is that after one iteration of the Wasserstein Transform $\mathbf{W}_{\eps,1}$ of $\alpha$, pairs of points $x$ and $x'$ on very dense areas (reflected by large values of $f(x)$ and $f(x')$) will be at roughly the same distance they were before applying the Wasserstein Transform. However, if one of the points, say $x'$ is in a sparse area (i.e. $f(x')$ is small), then the Wasserstein Transform will push it away from $x$. It is also interesting what happens when $x$ and $x'$ are both critical points of $f$: in that case the distance does not change (up to order $\eps^2$). See Figure \ref{fig:MS}.
\end{remark}

We use the abbreviation LT-WT$p$ for the Wasserstein Transform induced by local truncation and by the $\ell^p$-Wasserstein distance. We also sometimes suppress $p$ and simply use LT-WT when referring to the generic concept.

\begin{figure}
\includegraphics[width=\linewidth]{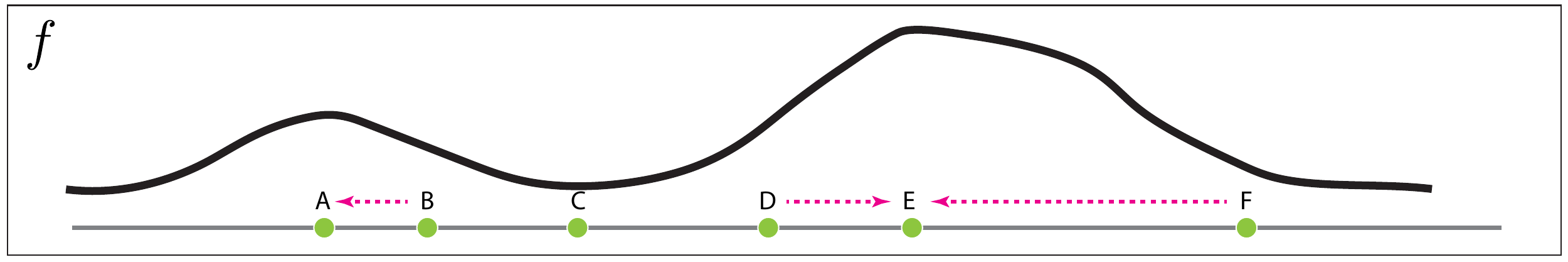}
\caption{\textbf{Illustration of Remark \ref{rem:taylor-1d}.} After applying one iteration of $\we$, both the distance between $A,C$ and the distance between $C,E$ should remain almost the same since these are all critical points of $f$. Since $f'$ has negative sign at $B$ and $B$ lies to the right of $A$, $B$ will be pushed towards $A$ according to the formula in Remark \ref{rem:taylor-1d}, while $D$ will be pushed away from $A$ since $f'(D)>0$ and it lies to the right of $A$. Similarly both $D$ and $F$ are pushed towards $E$.}
\label{fig:MS}
\end{figure}

\subsubsection{Mean shift}\label{sec:ms}
Mean shift (MS) is a mode seeking method for Euclidean data which operates by iteratively updating each data point until convergence according to a rule that moves points towards the weighted barycenter of their neighbors \citep{cheng1995mean,fukunaga1975estimation}. More specifically, given a point cloud $X = \{x_1,\ldots,x_n\}$   in $\R^m$, a function $\mathcal{K}:\Rp\rightarrow \Rp$,  and a bandwidth $\eps>0$, then in the $k$th iteration the $i$th point is shifted as follows:  
\begin{equation}\label{eq:meanshift formula}
    x_i^{k+1}=\frac{\sum_{j=1}^n \mathcal{K}\lc \frac{\norm{x_j^k-x_i^k}^2}{\eps^2}\rc  \, x_j^k}{\sum_{j=1}^n \mathcal{K}\lc \frac{\norm{x_j^k-x_i^k}^2}{\eps^2}\rc },\,\,\mbox{where $x_i^0=x_i$}.
\end{equation}

One of our observations is that the mean shift method can be placed within the framework of the extrinsic Wasserstein Transform. To see this point, consider the localization operator $L_{K_\eps}$ induced by the kernel function $K_\eps(x_i,x_j)\coloneqq\mathcal{K}\lc\frac{\|x_i-x_j\|^2}{\eps^2}\rc$ for $x_i,x_j\in X$. Let $\varphi:X\rightarrow\R^m$ denote the inclusion map (which serves as the coordinatization map as mentioned in Section \ref{sec:extrinsic WT}). Define $\psi:\R^m\times \mathcal{P}_{1,2}(\R^m)\rightarrow\mathcal{P}(\R^m)$ by sending each $(y,\beta)$ to $\delta_{\mean(\beta)}$. Then, we define the extrinsic localization operator $\lms_{K_\eps}\coloneqq (L_{K_\eps})_{\psi,\phi}$ following the notation from Section \ref{sec:extrinsic WT}. More explicitly, for every $\alpha\in\mathcal{P}_f(X)$, and for $x\in X$,
$$m^{\lms_{K_\eps}}_{\alpha,d_X}(x)\coloneqq \delta_{\mean\lc m_{\alpha,d_X}^{L_{K_\eps}}(x)\rc}\in \mathcal{P}(\R^m).$$

In words, at a fixed point $x$, $\lms_{K_\eps}$ applied to a measure $\alpha$ at $x$ first localizes $\alpha$ via $L_{K_\eps}$ to obtain $m^{L_{K_\eps}}_{\alpha,d_X}(x)$, and then further localizes this measure by only retaining information about its mean inside the ambient Euclidean space. Let $\wt{1}^{L_{K_\eps}^\mathrm{ms}}$ denote the corresponding extrinsic Wasserstein Transform.

Then, we have the following proposition.
\begin{proposition}\label{prop:MS is WT}
When considering the uniform probability measure $\alpha$ on $X=\{x_1,\ldots,x_n\}\subseteq\R^m$, the distance function generated by MS agrees with the one generated by WT w.r.t. the extrinsic localization operator $L_{K_\eps}^\mathrm{ms}$ on $X$. More specifically, for any $i,j=1,\ldots,n$ we have that
\[\norm{x_i^1-x_j^1} = \dW 1\lc m_{\alpha,d_X}^{L_{K_\eps}^\mathrm{ms}}(x_i),m_{\alpha,d_X}^{L_{K_\eps}^\mathrm{ms}}(x_j)\rc.\]
\end{proposition}

\begin{proof}
For each $x_i\in X$, we obtain the following formula which agrees with the result of applying one iteration of mean shift to the points in $X$ (cf. \Cref{eq:meanshift formula}):
$$\mean\lc m_{\alpha,d_X}^{K_\eps}(x_i)\rc=\frac{\sum_{j=1}^n K_\eps(x_i,x_j) \, x_j}{\sum_{j=1}^n K_\eps(x_i,x_j)}=\frac{\sum_{j=1}^n \mathcal{K}\lc \frac{\|x_j-x_i\|^2}{\eps^2}\rc  \, x_j}{\sum_{j=1}^n \mathcal{K}\lc \frac{\|x_j-x_i\|^2}{\eps^2}\rc }=x_i^1.$$

Since by \Cref{remark:deltas}, the Wasserstein distance between Dirac measures equals the ground distance between their support points, we then have for all $x_i,x_j\in X$ that
\begin{align*}
    \norm{x_i^1-x_j^1} =&\norm{\mean\lc m_\alpha^{K_\eps}(x_i)\rc-\mean\lc m_\alpha^{K_\eps}(x_j)\rc}=\dW 1\lc m_{\alpha,d_X}^{L_{K_\eps}^\mathrm{ms}}(x_i),m_{\alpha,d_X}^{L_{K_\eps}^\mathrm{ms}}(x_j)\rc.
\end{align*}

\end{proof}

Based on the proposition above, that the metric space $\wt{1}^{L_{K_\eps}^\mathrm{ms}}(\alpha)$ contains the same information as the collection of mean shift points $\{x_i^1\}_{i=1}^n$ above follows from the fact that any finite set in $\R^m$ can be reconstructed up to rigid transformations from its interpoint distance matrix \citep{blumenthal1953theory}.

\subsubsection{The Gaussian Transform}
\label{sec:gt}
We now introduce another instance of the extrinsic WT which we call the \emph{Gaussian Transform} (GT). This variant arose from our attempt at finding computationally less expensive alternatives to LT-WT when restricted to Euclidean data sets: the closed form formula for $\dw$ between Gaussian measures (cf. Example \ref{rmk:Gaussian distance}) is the key to devising the Gaussian Transform. 

Consider a metric space $(X,d_X)$ and a continuous coordinatization map $\varphi:X\rightarrow\R^m$. Let $\alpha\in\mathcal{P}_f(X)$. Recall the local truncation $L^{\mathsmaller{(\eps)}}$ which assigns to every $x\in X$ the probability measure $\me_{{\alpha,d_X}}(x)\coloneqq\frac{\alpha|_{B_{\eps}^{d_X}(x)}}{\alpha\left(B_{\eps}^{d_X}(x)\right)}.$ Consider its pushforward $\varphi_\#\me_{{\alpha,d_X}}(x)\in\mathcal{P}(\R^m)$ under the coordinatization map $\varphi$. If we further assume that $X$ is \emph{proper}\footnote{We say a metric space is \emph{proper} if all of its closed balls with finite radii are compact. Common examples include finite spaces, compact spaces and the Euclidean spaces.}, then we have that $\varphi_\#\me_{{\alpha,d_X}}(x)\in\mathcal{P}_{1,2}(\R^m)$. We denote by $\mu^{\mathsmaller{(\eps)}}_{{\alpha,d_X}}(x)$ the mean of $\varphi_\#\me_{{\alpha,d_X}}(x)$ and by $\Sigma_{{\alpha,d_X}}^{\mathsmaller{(\eps)}}(x)$ the covariance matrix of $\varphi_\#\me_{{\alpha,d_X}}(x)$ (cf. Definition \ref{def:mean and covariance}). 
\begin{figure}[htb]
    \centering
    \includegraphics[width=0.8\linewidth]{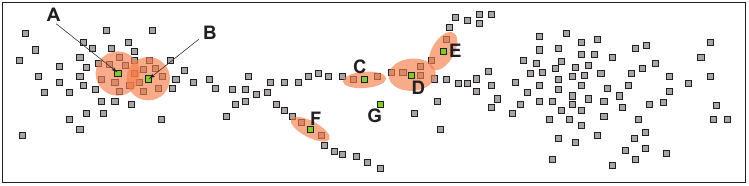}
    \caption{\textbf{Illustration of GT.} Different points in the given point cloud in $\mathbb
 R^2$ have different neighborhood structures, which are captured by their local covariance matrices. GT generates a Gaussian measure \emph{on each data point} based on local covariance matrices. In the figure, we represent these Gaussians by ellipses (iso-contours of Gaussian distribution functions at a given height). Note that these localized Gaussian measures reflect the neighborhood structures: the Gaussian is round when the neighborhood is relatively isotropic (\textbf{A} and \textbf{B}); the Gaussian is {flat} when the neighborhood is oblong (\textbf{C, D, E} and \textbf{F}); and the Gaussian is degenerate for an isolated point (cf. \textbf{G}). }
    \label{fig:ill-cov}
\end{figure}

\begin{remark}[Simplified notation]
When $(X,d_X)$ is a subspace of $\R^m$ and $\varphi$ is the inclusion map, we usually suppress $d_X$ and simply write $\me_{\alpha}(x)$, $\mu^{\mathsmaller{(\eps)}}_{{\alpha}}(x)$, etc.
\end{remark}

We introduce a hyperparameter $\lambda\geq 0$ and define $\psi_\lambda:\R^m\times\mathcal{P}_{1,2}(\R^m)\rightarrow\mathcal{P}(\R^m)$ by sending every $(y,\alpha)$ to the Gaussian measure $\mathcal{N}\lc y,\lambda\Sigma_\alpha\rc$. Then, $L^{\mathsmaller{(\eps)}},\varphi$ and $\psi_\lambda$ together give rise to an extrinsic localization operator which we call the \emph{Gaussian localization}. More explicitly, the Gaussian localization operator assigns to each point $x\in X$ a Gaussian measure 
$$\gamma_{\alpha,d_X}^{\mathsmaller{(\eps,\lambda)}}( {x})\coloneqq\mathcal{N}\left(\varphi(x),\lambda\cdot\Sigma_{\alpha,d_X}^{\mathsmaller{(\eps)}}(x)\right).$$
We denote by $\gt$ the corresponding Wasserstein Transform with the $\ell^2$ Wasserstein distance $\dW{2}$ and also call $\gt$ the \emph{Gaussian Transform} (GT). See also \Cref{fig:ill-cov} for an illustration of the GT in a scenario similar to the one depicted in \Cref{fig:chaining}.

By invoking the explicit formula for computing the $\ell^2$-Wasserstein distance $\dW{2}$ between Gaussian measures (cf. \Cref{rmk:Gaussian distance}), we now write down an explicit formula for the distance associated with $\gt$, which we call the \emph{GT distance}.

\begin{definition}[GT distance]\label{def:gtd}
Given parameters $\lambda\geq 0$ and $\varepsilon>0$, we define the GT distance $d^{\mathsmaller{(\eps,\lambda)}}_{{\alpha,d_X}}(x,x')$ between $x,x'\in X$ as follows:

\begin{align}
    &d^{\mathsmaller{(\eps,\lambda)}}_{{\alpha,d_X}}(x,x')\coloneqq \dW{2}\left(\gamma_{\alpha,d_X}^{\mathsmaller{(\eps,\lambda)}}( {x}),\gamma_{\alpha,d_X}^{\mathsmaller{(\eps,\lambda)}}( {x'}) \right)
\label{eq:gtd}\\
    =&\left(\norm{\varphi(x)-\varphi(x')}^2+\lambda\cdot \lc\dcov\,\left(\Sigma_{{\alpha,d_X}}^{\mathsmaller{(\eps)}}(x),\Sigma_{{\alpha,d_X}}^{\mathsmaller{(\eps)}}(x')\right)\rc^2\right)^\frac{1}{2}. \label{eq:gtd2}
\end{align}
\end{definition}

\begin{remark}
When $X$ is a subspace of $\R^m$ and $\varphi:X\rightarrow\R^m$ is the inclusion map, we will sometimes omit mentioning the map $\varphi$ and write the GT distance as follows
\[d^{\mathsmaller{(\eps,\lambda)}}_{{\alpha,d_X}}(x,x')=\left(\norm{x-x'}^2+\lambda\cdot \lc\dcov\,\left(\Sigma_{{\alpha,d_X}}^{\mathsmaller{(\eps)}}(x),\Sigma_{{\alpha,d_X}}^{\mathsmaller{(\eps)}}(x')\right)\rc^2\right)^\frac{1}{2}.\]
\end{remark}

\begin{remark}[Effect of $\lambda$]\label{rmk:GT-lambda}
Intuitively speaking,
$\lambda$ in \Cref{def:gtd}, as a hyperparameter, determines the influence of local structures when comparing data points. In \Cref{sec:experiments}, we empirically examine the effect of $\lambda$ and show how tuning $\lambda$ could lead to good performance of GT on various tasks.
From another perspective, the hyperparameter $\lambda$ gives rise to a one parameter family of methods where the cases $\lambda=0,1$ are special:
Assume that $X$ is a subspace of $\R^m$ and $\varphi:X\rightarrow\R^m$ is the inclusion map. When $\lambda=0$, the GT distance $d^{\mathsmaller{(\eps,0)}}_{{\alpha,d_X}}$ is the same as the underlying Euclidean distance. When $\lambda=1$, 
in \Cref{prop:two line same as gt} we identify a special case where the GT distance is equal to the distance generated by LT-WT2. 
 
\end{remark}

\begin{remark}[Heat kernel localization]
In the case of Riemannian manifolds $X$, although Gaussian measures may not be defined, the heat kernel $K(t,\cdot,\cdot):X\times X\rightarrow\Rp$ serves as a natural generalization of the Gaussian distribution. The property that $\int_X K(t,x,y)\mathrm{vol}_M(dy)=1$ induces, for each point $x\in X$, a probability measure with density function $K(t,x,\cdot)$. In the case of finite metric spaces, one can instead compute the heat kernel for the graph Laplacian given some graph structure on the spaces (e.g., kNN graph or $\eps$-neighbor graph). Heat kernels induced by graph Laplacians have been widely used for dimension reduction and data representation \citep{belkin2003laplacian,coifman2006diffusion}. Similarly to the continuous case, the heat kernel induced by the graph Laplacian equips each point with a probability measure. In this way, one would obtain an analogue of GT in the case of general finite metric spaces. See \cite{jin2021optimal} for an exploration along these lines.
\end{remark}

\section{Interpretation and theoretical properties}\label{sec:theory}
In this section, we discuss some theoretical properties of the Wasserstein Transform induced by \emph{kernel} and \emph{Gaussian localization}. The outline of this section is as follows. In \Cref{sec:ricci flow} we draw connections between LT-WT and  the Ricci flow to help readers gain a geometric intuition behind LT-WT. In \Cref{sec:wt ultrametric} we provide an exact calculation of LT-WT on ultrametric spaces. In \Cref{sec: gt vs lt-wt} we identify certain similarity between GT and LT-WT. In \Cref{sec:gt aniso} we study the neighborhood structure generated by the GT distance. Finally in \Cref{sec:wt stability theorems}, we establish stability results for KL-WT, LT-WT, MS and GT. 

In order to facilitate the flow of this section, most technical proofs are relegated to \Cref{sec:wt-proof}. 

\subsection{LT-WT as Ricci flow}\label{sec:ricci flow}
Applying WT iteratively to a metric space allows consecutively evolution of distance functions. This is conceptually analogous to the study of manifold flows. In particular, in this section, we unveil a connection between WT and the renowned Ricci flow on manifolds \citep{hamilton1982three}.

First we point out a relationship between the Ricci curvature and the LT-WT distance, as observed by \cite{ollivier2009ricci}. Let $(X,g_X)$ be a complete $n$-dimensional Riemannian manifold. We denote by $d_X$ its associated geodesic distance function. Then, for small $\eps>0$ and when the probability measure $\alpha$ is chosen to be the uniform measure $\alpha\coloneqq\frac{1}{\vol(X)}\vol$, it was proved in \cite{ollivier2009ricci} that for any two close enough points $x,x'\in X$ the following Taylor expansion holds
\begin{equation}\label{eq:ollivier}
    \deps(x,x')=d_X(x,x')\lc1-\frac{\eps^2}{2(n+2)}\mathrm{Ric}(v,v)\rc
\end{equation}
up to higher order terms. Here $\mathrm{Ric}$ denotes the Ricci tensor and $v$ is the unit vector at $x$ pointing towards $y$ (i.e., $v$ is the tangent vector at $x$ along the unique unit speed geodesic issuing from $x$ to $y$\footnote{Here we used the assumption that $x$ and $x'$ are close enough.}). 

Through \Cref{eq:ollivier} we then reveal a connection between LT-WT and the renowned Ricci flow \citep{hamilton1982three}: Let $X$ be a manifold with a smooth family $\{g_X(t)\}_{t\in(a,b)}$ of Riemannian metrics. We say $\{g_X(t)\}_{t\in(a,b)}$ is a \emph{Ricci flow} if
\begin{equation}\label{eq:ricci}
    \frac{\partial g_X(t)}{\partial t}=-2\mathrm{Ric}_{g_X(t)},
\end{equation}
where $\mathrm{Ric}_{g_X(t)}$ denotes the Ricci tensor with respect to the Riemannian metric $g_X(t)$.

Since in the context of WT we only care about the distance function, we use the following lemma (a direct consequence of Remark 6 in \citep{mccann2010ricci}) to describe the instantaneous rate of change of the geodesic distance under the Ricci flow.
\begin{lemma}
Let $X$ be a compact manifold with a smooth family $\{g_X(t)\}_{t\in(-\eps,\eps)}$ of Riemannian metrics where $\eps>0$ is small. Given two points $x,x'\in X$ close enough (with respect to $g_X(0)$), we have that 
\begin{equation}\label{eq:derivative ricci flow}
    \frac{d}{dt}\Big|_{t=0}d_X(x,x',t)=\frac{1}{2d_X(x,x',0)}\int_0^1\frac{\partial g_X}{\partial t}\lc\frac{\partial\sigma}{\partial s},\frac{\partial\sigma}{\partial s}\rc\Big|_{t=0} ds,
\end{equation}
where $d_X(x,x',t)$ denotes the geodesic distance between $x$ and $x'$ with respect to $g_X(t)$, and $\sigma:[0,1]\times(-\eps,\eps)\rightarrow X$ is a smooth function such that $\sigma(s,t)=\sigma_t(s)$ is the unique geodesic connecting $x,x'$ with respect to metric $g_X(t)$.
\end{lemma}

Suppose that $g_X(t)$ satisfies the Ricci flow \Cref{eq:ricci}. Let $\gamma:[0,d_X(x,x',0)]\rightarrow X$ denote the unique unit speed geodesic (w.r.t. $g_X(0)$) connecting $x$ and $x'$. For any $\tau\in[0,d_X(x,x',0)]$, let $v_\tau\coloneqq\gamma'(\tau)$ be the velocity of $\gamma$ at $\tau$ and let $v\coloneqq v_0$ be the initial velocity.
Note that $v$ is now the unit vector at $x$ pointing towards $x'$. Then, we have from \Cref{eq:derivative ricci flow} and by a change-of-variables formula that
\begin{align*}
    \frac{d}{dt}\Big|_{t=0}d_X(x,x',t)&=\frac{1}{2d_X(x,x',0)}\int_0^1\frac{\partial g_X}{\partial t}\lc\frac{\partial\sigma}{\partial s},\frac{\partial\sigma}{\partial s}\rc\Big|_{t=0} ds\\
    &=\frac{1}{d_X(x,x',0)}\int_0^1-\mathrm{Ric}_{g_X(t)}\lc\frac{\partial\sigma}{\partial s},\frac{\partial\sigma}{\partial s}\rc\Big|_{t=0} ds\\
    &=\int_0^{d_X(x,x',0)}-\mathrm{Ric}_{g_X(0)}\lc v_\tau,v_\tau\rc d\tau\\
    & = d_X(x,x',0)\lc-\mathrm{Ric}_{g_X(0)}(v,v)+O(d_X(x,x',0)) \rc.
\end{align*}
Therefore, for $\delta\coloneqq d_X(x,x',0)$ we have that
\begin{align}\label{eq:Ricdistance}
    d_X(x,x',t)&=d_X(x,x',0)+t\cdot d_X(x,x',0)\lc-\mathrm{Ric}_{g_X(0)}(v,v)+O(d_X(x,x',0)) \rc+O(t^2)\\
    & = d_X(x,x',0)\lc 1-t\cdot\mathrm{Ric}_{g_X(0)}(v,v)\rc+O(t\delta^2+t^2).
\end{align}
Comparing Equation (\ref{eq:Ricdistance}) with Equation (\ref{eq:ollivier}), we conclude that the distance function $\deps$ generated by LT-WT coincides with the distance function $d_X(\cdot,\cdot,\frac{\eps^2}{2(n+2)})$ generated by Ricci flow up to high order terms. In this way, one can regard LT-WT as a discrete version of the Ricci flow for metric measure spaces. \cite{ollivier2009ricci} has already suggested some ideas for \emph{discrete} Ricci flows and \cite{ni2019community} implemented these ideas in the case of graphs with the goal of tackling certain network analysis problems.

\subsection{LT-WT on ultrametric spaces}\label{sec:wt ultrametric}

In one special case, we can theoretically characterize the behavior of LT-WT.
In this section, we show that applying LT-WT to probability measures on an \emph{ultrametric} space is equivalent to applying the so-called \emph{closed quotient operation} to the ultrametric space. This closed quotient operation is one of the fundamental operations on ultrametric spaces considered in \citep{memoli2022p} and it is closely related to the dendrogram representation of ultrametric spaces.

We now introduce the closed quotient operation as follows. Given an ultrametric space $(X,u_X)$ and any $\eps\geq 0$, we induce on $(X,u_X)$ an equivalence relation $\sim_\eps$ such that $x\sim_\eps x'$ iff $u_X(x,x')\leq \eps$. That $\sim_\eps$ is an equivalence relation follows from the fact that $u_X$ satisfies the strong triangle inequality. We denote by $[x]_\eps$ the equivalence class of $x\in X$ given the equivalence relation $\sim_\eps$. Let $X_\eps$ denote the set of all equivalence classes and define an ultrametric $u_{X_\eps}$ on $X_\eps$ as follows:

\begin{equation}\label{eq:closed quotient}
    u_{X_\eps}\left([x]_\eps,[x']_\eps\right) \coloneqq  \left\{
\begin{array}{cl}
u_X(x,x') & \mbox{if $[x]_\eps\neq[x']_\eps$}\\
0 & \mbox{if $[x]_\eps=[x']_\eps$.}
\end{array}
\right.
\end{equation}

Then, the $\eps$-closed quotient of $(X,u_X)$ is defined to be the ultrametric space $\lc X_{\eps},u_{X_{\eps}}\rc$. 
Now, we unveil the following relationship between LT-WT and the closed quotient operation (see Figure \ref{fig:WT-quotient} for an illustration):
\begin{proposition}\label{prop:WT as quotient}
Let $(X,u_X)$ be an ultrametric space. Let $\alpha\in\mathcal{P}_f(X)$ and $\eps\geq 0$, then the metric space $\we(\alpha)$ generated by LT-WT is isometric to the $\eps$-closed quotient $\lc X_{\eps},u_{X_{\eps}}\rc$.
\end{proposition}

\begin{figure}[htb]
    \centering
    \includegraphics[width=0.5\linewidth]{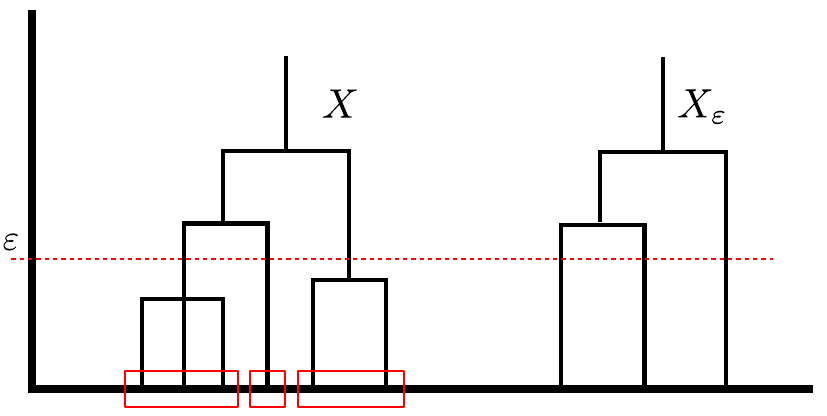}
    \caption{\textbf{Illustration of Proposition \ref{prop:WT as quotient}.} Consider the ultrametric space $X$ represented by the dendrogram on the left. There are only three distinct $\eps$-balls shown by the red boxes. Points within the same $\eps$-ball have the same $\eps$-neighborhood and thus they will be identified after applying LT-WT once. This partially explains why $\we(\alpha)$ is isometric to $\lc X_{\eps},u_{X_{\eps}}\rc$ holds.}
    \label{fig:WT-quotient} 
\end{figure}

\subsection{GT as an approximation of LT-WT2}\label{sec: gt vs lt-wt}

\begin{figure}[htb]
    \centering
    \includegraphics[width=0.7\linewidth]{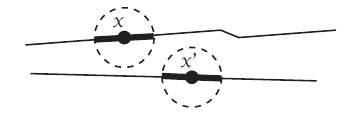}
    \caption{\textbf{Illustration of Proposition \ref{prop:two line same as gt}.}  }
    \label{fig:non-intersect} 
\end{figure}

In applications such as
crack detection in Civil Engineering~\citep{yamaguchi2010fast,koch2015review}, one may often encounter data points concentrated on line segments or curves. For a data set $X\subseteq\R^2$ such as the one as shown in Figure~(\ref{fig:non-intersect}), each ball in $X$ is the intersection of an Euclidean ball and $X$. Generically, such a ball consists of points lying on a tree-like shape. However, when the center of a given ball is away from intersections and the radius is chosen to be small, the ball is usually of the form of a line segment.

For such special type of data sets $X$, after applying GT and LT-WT2 to $X$, it turns out that the resulting distances between points agree with each other, if the $\eps$-neighborhoods of these points are line segments.

\begin{proposition}\label{prop:two line same as gt}
Let $X\subseteq\mathbb{R}^2$ be the union of several line segments endowed with the ambient Euclidean distance. Consider points $x,x'\in X$ away from both intersection points and end points of the line segments. Choose a small $\eps>0$ such that $B_\eps(x)$ and $B_\eps(x')$ are line segments centered at $x$ and $x'$, respectively (see $x$ and $x'$ in \Cref{fig:non-intersect}). Then, we have that (recall symbols from \Cref{tab:WT-symbol}) 
$$ d^{\mathsmaller{(\eps,\lambda)}}_{{\alpha,d_X}}(x,x')=d^{\mathsmaller{(\eps)}}_{{\alpha,d_X,2}}(x,x'),\quad\text{ where }\lambda=1. $$
\end{proposition}

This proposition follows directly from the observation below.

\begin{lemma}\label{prop:line}
Let $l_1$ and $l_2$ be two non-intersecting line segments in $\R^2$ with length $s_1$ and $s_2$, respectively. Let $\alpha_i$ be the {normalized length measure} on $l_i$ for $i=1,2$. Let $\mu_{\alpha_i}$ be the mean of $\alpha_i$ and let $\Sigma_{\alpha_i}$ be the covariance matrix of $\alpha_i$ for $i=1,2$. For $i=1,2$ let $\gamma_{\alpha_i}$ be the Gaussian measure $\mathcal{N}(\mu_{\alpha_i},\Sigma_{\alpha_i})$. Then, we have that
$$\dW{2}(\gamma_{\alpha_1},\gamma_{\alpha_2})=\dW{2}(\alpha_1,\alpha_2).$$
\end{lemma}

\subsection{Anisotropy properties of GT}\label{sec:gt aniso}
\begin{figure}[htb]
\vspace{-0.1in}
    \centering
    \subfloat[Original Data]{
		\includegraphics[width=0.4\linewidth]{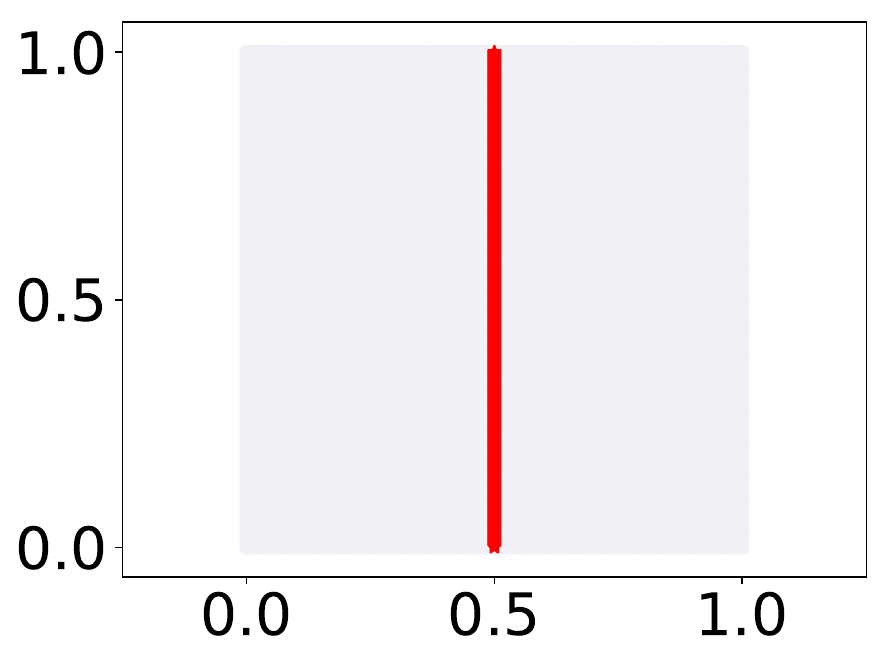}
	}
	\subfloat[Neighborhood $B_\eps^{\lambda}(x_0)$]{
		\includegraphics[width=0.4\linewidth]{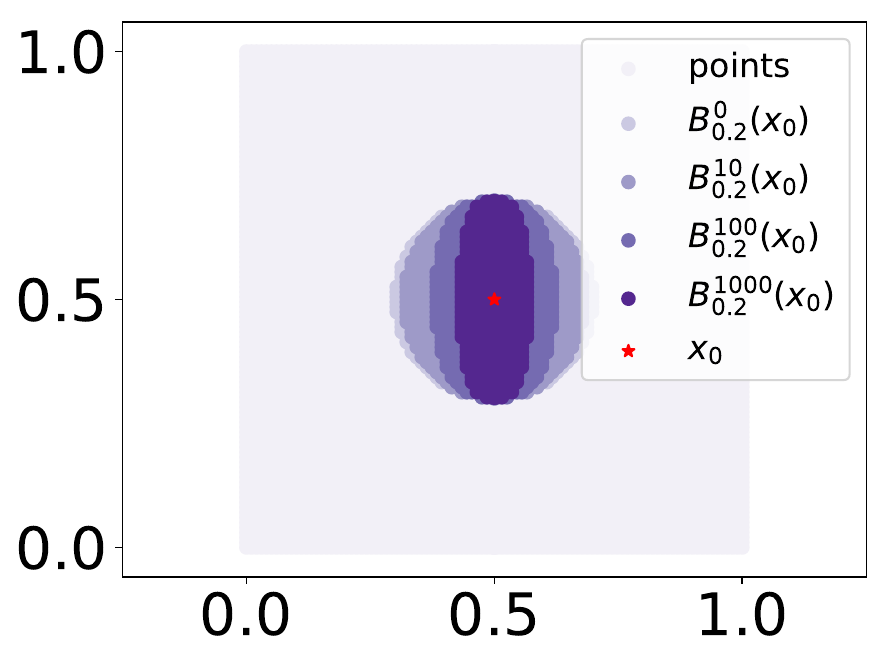}
	}
	
    \caption{\textbf{Illustration of anisotropic neighborhood.} The figure on the left shows the original data with $100 \times 100$ grid points inside $[0,1]\times[0,1]$ and 1001 evenly spaced points on the line segment from $(0.5,0)$ to $(0.5,1)$. The figure on the right shows the $\eps$-neighborhood of the point $x_0$=(0.5, 0.5) with respect to the GT distance $\dgt$ under various choices of $\lambda$, where: $\eps=0.2$ and $\alpha$ is the normalized empirical measure. We abbreviate $B_\eps^{\lambda,\alpha}(x_0)$ to $B_\eps^{\lambda}(x_0)$.}
    \label{fig:nbd-ellip}
    \vspace{-0.1in}
\end{figure}

Now, we assume that $X=\R^m$ and that $\alpha\in\mathcal{P}_f(X)$ has a smooth non-vanishing density function $f$ with respect to the $m$-dimensional Lebesgue measure $\mathcal{L}_m$.
We denote by $B_\eps^{\lambda,\alpha}(x_0)\coloneqq B_\eps^{d^{\mathsmaller{(\eps,\lambda)}}_{{\alpha,d_X}}}(x_0)$ the $\eps$-ball centered at $x_0\in X$ with respect to the GT distance $d^{\mathsmaller{(\eps,\lambda)}}_{{\alpha,d_X}}$, which we will refer to as the \emph{GT-$\eps$-neighborhood} of $x_0$. 
We study the asymptotic shape of $B^{\lambda,\alpha}_\eps(x_0)$ when $\eps$ tends to 0 with $\lambda$ at a precise rate and prove that \emph{$B^{\lambda,\alpha}_\eps(x_0)$ is asymptotically an ellipsoid}. 
See Figure~(\ref{fig:nbd-ellip}) for an illustration of the theorem.  This asymptotic study demonstrates that GT-neighborhoods are anisotropic and indicates that GT is sensitive to boundaries/edges in image data and thus suggests potential applications to edge detection and enhancement tasks in image processing.
See our image segmentation experiment in Section \ref{sec:img-seg}. Anisotropy sensitive ideas, such as anisotropic diffusion \citep{perona1990scale} or anisotropic mean shift \citep{wang2004image}, are prevalent in the literature and have been applied to image denoising and image segmentation. See also \citep{martinez2013multiscale,martinez2020shape}  for applications to shape analysis.

To precisely state our result, we first introduce the following notation 
$$\mathfrak{S}_\eps\lc B_\eps^{\lambda,\alpha}(x_0)\rc\coloneqq\lb x_0+\frac{1}{\eps}(x-x_0):\,x\in B_\eps^{\lambda,\alpha}(x_0)\rb$$ as a rescaling of the ball $B_\eps^{\lambda,\alpha}(x_0)$ around its center $x_0$.

To determine a relationship between $\lambda$ and $\eps$, notice that $(\dcov)^2$ between covariance matrices (in \Cref{eq:gtd2} for determining $d^{\mathsmaller{(\eps,\lambda)}}_{{\alpha,d_X}}$) of points $\eps$-close to each other is of order $O(\eps^8)$. This is negligible compared to the Euclidean distance term (whose order is $O(\eps^2)$). Hence we choose $\lambda=\eps^{-6}$ in order for the term $\lambda(\dcov)^2$ to be comparable with the Euclidean distance term. 

Then, we state our theorem regarding the asymptotic shape of $B_\eps^{\lambda,\alpha}(x_0)$ as follows.

\begin{theorem}\label{thm:ellipsoid}
Let $\lambda=\eps^{-6}$,  then the closure of $\liminf_{\eps\rightarrow 0}\mathfrak{S}_\eps\lc B_\eps^{\lambda,\alpha}(x_0)\rc$ is an \emph{ellipsoid} in $\R^m$.
\end{theorem}


\subsection{Stability theorems}\label{sec:wt stability theorems}

The goal of this section is to establish the stability of the Wasserstein Transform induced by both the kernel localization and the Gaussian localization. Moreover, a sampling convergence theorem for WT induced by Lipschitz kernels will be established as an application of our stability results. Throughout this section we will assume that $(X,d_X)$ is a compact metric space.


\subsubsection{WT under kernel localization}
In this section, we consider stability results for kernel localizations and $p=1$. We will consider the case of Lipschitz kernels and truncation kernels.

\paragraph{Positive Lipschitz kernels.} A rotationally invariant kernel $K:X\times X\rightarrow \Rp$ on the compact metric space $X$ is called a Lipschitz kernel, if and only if there exists a Lipschitz function $\mathcal{K}:\R\rightarrow [0,\infty)$ such that $K(x,x')=\mathcal{K}(d_X(x,x'))$. Recall that the associated localization operator is defined for any $\alpha\in \mathcal{P}(X),x\in X$ and $ A\subseteq X, $ as $ m^K_{\alpha,d_X} (x)(A)=\frac{\int_AK(x,y)\alpha(dy)}{\int_XK(x,y)\alpha(dy)}.$ In order to guarantee that the denominator is nonzero, we always assume $\mathcal{K}$ to be positive.

Then, we have the following  stability result for the kernel localization of two different probability measures on the \emph{same} compact metric space $X$. The stability is expressed in terms of the $\ell^1$-Wasserstein distance itself.

\begin{theorem}[Stability of kernel localization]\label{thm:lip kernel stab}
There exist positive constant $M$ and $N$ depending on $\mathcal{K}$ and $X$, such that for any $x\in X$,
$$\dW{1}\lc  m^K_{\alpha,d_X} (x), m^K_{\beta,d_X} (x)\rc \leq\frac{2L\,\diam(X)+M}{N}\dW{1}(\alpha,\beta).$$
\end{theorem}

Theorem \ref{thm:lip kernel stab} indicates that if $\alpha$ and $\beta$ are similar in terms of the Wasserstein distance, then \emph{for every point $x\in X$} the localized measures $m^K_{\alpha,d_X}(x)$ and $ m^K_{\beta,d_X}(x)$ will also be similar. The following stability result for the KL-WT metric is a direct consequence of Theorem \ref{thm:lip kernel stab}.

\begin{corollary}[Stability of $d_{\alpha,d_X}^K$]\label{coro:stb-metric-lipschitz}
With the same notation as in \Cref{thm:lip kernel stab}, we have that
$$\sup_{x,x'\in X}\ls d_{\alpha,d_X}^K(x,x')-d_{\beta,d_X}^K(x,x')\rs\leq  \frac{4L\,\diam(X)+2M}{N}\dW{1}(\alpha,\beta).$$
\end{corollary}

\subparagraph{Sampling convergence of KL-WT}
Based on the stability result (Theorem \ref{thm:lip kernel stab}) we are able to provide a statistical analysis for WT induced by positive Lipschitz kernels.

We are going to invoke the Gromov-Wasserstein distance (\cite{memoli2011gromov}) in order to establish the consistency/sampling convergence of KL-WT.

\begin{definition}[Gromov-Wasserstein distance]
Let $\mathcal{X}\coloneqq(X,d_X,\alpha_X)$ and $\mathcal{Y}\coloneqq(Y,d_Y,\alpha_Y)$ be two compact metric measure spaces. Then, for any $p\in[1,\infty)$, the $\ell^p$-Gromov-Wasserstein distance $\dgw p$ between $\mathcal{X}$ and $\mathcal{Y}$ is defined as follows:
\[\dgw p(\mathcal{X},\mathcal{Y})\coloneqq\inf_{\pi\in\mathcal{C}(\alpha_X,\alpha_Y)}\lc\iint_{X\times Y\times X\times Y}|d_X(x,x')-d_Y(y,y')|^p\pi(dx\times dy)\pi(dx'\times dy')\rc^\frac{1}{p}.\]
\end{definition}

Then, we have the following sampling convergence result for KL-WT:

\begin{theorem}[Sampling convergence]\label{thm:sampling convergence}
Given a compact metric space $(X,d_X)$, let $\alpha\in\mathcal{P}(X)$. Let $X_1,X_2,\cdots$ be i.i.d. random variables with distribution $\alpha$. Let $\alpha_n$ denote the empirical measure $\frac{1}{n}\sum_{i=1}^{n}\delta_{X_i}$, then for any positive Lipschitz kernel $K$ we have
$$ \lim_{n\rightarrow\infty}\dgw{1}\big( \lc X,d_{\alpha_n,d_X}^K,\alpha_n\rc,\lc X,d_{\alpha,d_X}^K,\alpha\rc\big) =0,\quad a.s.$$
\end{theorem}

Notice that the probability measure $\alpha$ on $\lc X,d_{\alpha,d_X}^K\rc$ is understood as the pushforward of $\alpha$ under the identity map $(X,d_X)\rightarrow \lc X,d_{\alpha,d_X}^K\rc$. The same interpretation holds for $\alpha_n$. The proof of the theorem relies on the following continuity property of $d_{\alpha,d_X}^K$ which is itself interesting:

\begin{theorem}[Continuity of $d_{\alpha,d_X}^K$ w.r.t. $d_X$]\label{thm:lip kernel is lip}
Given a compact metric space $(X,d_X)$ and a positive $L$-Lipschitz kernel $K$, consider any $\alpha\in\mathcal{P}(X)$ (not necessarily fully supported). Then, there exist constants $M,N>0$ depending on $K$ and $(X,d_X)$, such that for every $x,x'\in X$
$$d_{\alpha,d_X}^K(x,x') \leq\frac{L(M+N)\diam(X)}{N^2}d_X(x,x').$$
\end{theorem}

Moreover, under certain conditions on the underlying metric space $(X,d_X)$, we can bound the expected rate of convergence $\mathbb{E}\lc \dgw{1}\lc \lc X,d_{\alpha_n,d_X}^K,\alpha_n\rc,\lc X,d_{\alpha,d_X}^K,\alpha\rc\rc \rc$ in terms of $n$. To realize this, we need the following notion from \cite{dudley1969speed}. Given $\alpha\in\mathcal{P}(X)$ and $\eps>0$, let $N_X(\alpha,\eps)$ denote the minimal number of sets with diameter bounded above by $2\eps$ which cover $X\backslash A$ for a set $A\subseteq X$ such that $\alpha(A)\leq\eps$. Let
\[k(\alpha)\coloneqq\limsup_{\eps\rightarrow 0}\log_{\frac{1}{\eps}}(N_X(\alpha,\eps)).\]

\begin{corollary}[Convergence rate]\label{coro:convergence rate}
For any $\eps>0$, there exists a constant $C>0$ depending on $K$ and $(X,d_X)$ such that
\[\mathbb{E}\lc \dgw{1}\lc \lc X,d_{\alpha_n,d_X}^K,\alpha_n\rc,\lc X,d_{\alpha,d_X}^K,\alpha\rc\rc \rc\leq C\, n^{-\frac{1}{k(\alpha)+2+\eps}},\]
for $n$ large enough.
\end{corollary}

\begin{remark}
The convergence rate $n^{-\frac{1}{k(\alpha)+2+\eps}}$ in the above corollary follows directly from the convergence rate result of empirical probability measures under the \emph{Prokhorov distance} in Theorem 4.1 of \cite{dudley1969speed}. \cite{weed2019sharp} established certain sharp convergence rate result of empirical probability measures under the \emph{Wasserstein distance}. One can of course apply their results to obtain another type of convergence rate in the above corollary, but that refinement is beyond the scope of this paper.
\end{remark}

\paragraph{Local truncation.} Recall that the local truncation kernel is defined through the function $\mathcal{K}(t)=\xi_{[0,1]}(t)$ which is neither positive nor Lipschitz. Thus, in order to obtain stability results for local truncation, we will impose restrictions to the measures that we will consider. Below, we introduce the \emph{Bishop-Gromov condition} on measures (not restricted to probability measures). 

\begin{definition}\label{def:bg}
Given $\Lambda>0$, we say that a Borel measure $\alpha$ on $X$ satisfies the  $\Lambda$-Bishop-Gromov condition if for all $x\in \supp(\alpha),r_1\geq r_2>0$ one has
$\frac{\alpha(B_{r_1}(x))}{\alpha(B_{r_2}(x))}\leq \left(\frac{r_1}{r_2}\right)^\Lambda.$ We denote by $\mathcal{P}^{\mathcal{BG}(\Lambda)}(X)$ the set of all such {probability} measures on $X$.
\end{definition}

This definition originates from \cite{burago2019spectral}, where the intuition rooted from curvature bounds in Riemannian geometry. See Example \ref{ex:bgr} for more details. The class of probability measures satisfying the Bishop-Gromov condition (for some parameter) is rich. Here are some examples. 
\begin{example}\label{ex:bgr}
The normalized volume measure on any compact Riemannian manifold $X$ belongs to $\mathcal{P}^{\mathcal{BG}(\Lambda)}(X)$, where $\Lambda$ depends on $\dim(X)$ and a lower bound of the Ricci curvature of $X$.
\end{example}
\begin{example}
For a bounded convex open subspace $X$ of $\R^m$, the normalized Lebesgue measure restricted on $X$ belongs to $\mathcal{P}^{\mathcal{BG}(m)}(X)$.
\end{example}

Assume that $\diam(X)<D$ for some $D>0$. For $\Lambda>0$,  for $t\in\Rp$ let 
\begin{equation}\label{eq:psi l d}
    \psi_{\Lambda,D}(t)\coloneqq \min\left(1,\left(\frac{t}{D}\right)^\Lambda\right).
\end{equation}
We further define for every $t\geq 0$  
\begin{equation}\label{eq:phi l d e}
    \Phi_{\Lambda,D,\eps}(t)\coloneqq \frac{t}{\psi_{\Lambda,D}(\eps)} + \left(\lc 1+\frac{t}{\eps}\rc ^\Lambda-1\right).
\end{equation}

\begin{remark}\label{rem:Phi}
$\Phi_{\Lambda,D,\eps}(t)$ is an increasing function of $t$ and $\lim_{t\rightarrow 0}\Phi_{\Lambda,D,\eps}(t)=0$.
\end{remark}
Similarly to the case of KL-WT, we first establish the following stability result for local truncations of two different probability measures on the \emph{same} compact metric space $X$ via the $\ell^1$-Wasserstein distance.
\begin{theorem}[Stability of local truncations]\label{thm:stb}
Let $\Lambda>0$. If $\alpha,\beta\in\mathcal{P}^{\mathcal{BG}(\Lambda)}(X)\cap\mathcal{P}_f(X)$, then we have that for any $\eps>0$,
\[\sup_{x\in  X}\dW{1}\lc\me_{\alpha,d_X}(x),\me_{\beta,d_X}(x)\rc\leq (1+2\eps)\Phi_{\Lambda,D,\eps}\lc \sqrt{\dW{1}(\alpha,\beta)}\rc.\]
\end{theorem}


We then obtain the following stability result for the WT distance $\deps$ (w.r.t. the local truncation).

\begin{corollary}[Stability of $\deps$]\label{coro:stb-metric}
With the same notation as in \Cref{thm:stb}, we have that
$$\sup_{x,x'\in X}\ls\deps(x,x')-\depsbeta(x,x')\rs\leq 2(1+2\eps)\Phi_{\Lambda,D,\eps}\lc \sqrt{\dW{1}(\alpha,\beta)}\rc.$$
\end{corollary}

\subsubsection{Mean shift}
Let $X\subseteq\R^m$ be a compact subspace of the $m$-dimensional Euclidean space. Assume that $\diam(X)\leq D$ for some $D>0$.
Let $\eps>0$. In this section, we obtain stability results for MS with respect to both positive Lipschitz kernels and local truncations via stability results for kernel localizations. To the best of our knowledge, this is the first time such results have been established in the literature.

Consider the kernel $K_\eps(x,y)=\mathcal{K}\lc\frac{\|x-y\|^2}{\eps^2}\rc$ with $\mathcal{K}$ being a positive $L$-Lipschitz function. We have the following stability result for mean shift with respect to the kernel $K_\eps$.

\begin{theorem}[Stability of mean shift for Lipschitz kernels]\label{thm:ms-smooth}
Let $C=C(\eps,L,D)\coloneqq\frac{2L\,D}{\eps^2}$. There exist positive constants $M$ and $N$ depending only on $\mathcal{K}$, $\eps$ and $X$ such that for all $x\in X$ and all $\alpha,\beta\in\mathcal{P}(X)$ we have:
\[\norm{\mean\lc m_{\alpha,d_X}^{K_\eps}(x)\rc-\mean\lc m_{\beta,d_X}^{K_\eps}(x)\rc}\leq\frac{2\,CD+M}{N}\dW{1}(\alpha,\beta). \]
\end{theorem}

For local truncation induced MS, we have the following stability result:

\begin{theorem}[Stability of mean shift for local truncations]\label{thm:stb-ms}
Let $\Lambda>0$ and let $\alpha,\beta\in \mathcal{P}^{\mathcal{BG}(\Lambda)}(X)\cap\mathcal{P}_f(X)$. Then, under the same notation as in Theorem \ref{thm:stb}, we have that
\[\sup_{x\in X} \left\|\mathrm{mean}\lc \me_{\alpha,d_X}(x)\rc- \mathrm{mean}\lc \me_{\beta,d_X}(x)\rc \right\|\leq (1+2\eps)\,\Phi_{\Lambda,D,\eps}\lc \sqrt{\dW{1}(\alpha,\beta)}\rc.\]
\end{theorem}



\subsubsection{Gaussian Transform}

In this section, we assume that \emph{$X$ is a subspace of $\R^m$ and $\varphi:X\rightarrow\R^m$ is the inclusion map.} Moreover, we assume that $\diam(X)<D$ for some $D>0$.

We then have the following stability result for GT whose proof is relegated to \Cref{sec:wt-proof}.

\begin{theorem}[Stability of GT]\label{thm:GT-stable}
Let $c>0$ and let $\alpha,\beta\in\mathcal{P}_f^{c}(X)\cap \mathcal{P}^{\mathcal{BG}(\Lambda)}(X)$. Then, there exists a positive constant $A=A(\eps,m,D)$ such that 
$$\norm{d^{\mathsmaller{(\eps,\lambda)}}_{{\alpha,d_X}}-d^{\mathsmaller{(\eps,\lambda)}}_{{\beta,d_X}}}_\infty\leq 2\sqrt{m\,\lambda\,\Psi_{\Lambda,D,\eps}^{c,A,m}\lc \dW{\infty}(\alpha,\beta)\rc}.$$
\end{theorem}

Here $\Psi_{\Lambda,D,\eps}^{c,A,m}:[0,\infty)\rightarrow[0,\infty)$, which we define in Page \pageref{proof of gt stable}, is an increasing function such that $\Psi_{\Lambda,D,\eps}^{c,A,m}(0)=0$. $\mathcal{P}_f^{c}(X)$ denotes the set of all $\alpha\in\pf(X)\subseteq\mathcal{P}(\R^m)$ such that $\alpha(S)\leq c\cdot \mathcal{L}_m(S)$ for all measurable $S$ where $\mathcal{L}_m$ denotes the Lebesgue measure on $\R^m$. Such a restriction of probability measures was used for proving a stability theorem for one type of local covariance matrices in \cite{martinez2020shape}.

Notice that in the above stability theorem, we use $\dW\infty$ to measure the dissimilarity between probability measures instead of $\dW{1}$ as used in previous stability results for KL-WT and MS. $\dW\infty$ in fact appears naturally when comparing the covariance matrices inherent to GT (cf. Lemma \ref{lm:cmbound}). Since $\dW\infty\geq\dW{1}$ \citep{givens1984class}, this stability result is slightly weaker than those of KL-WT and MS. However, to remedy this we subsequently develop a stronger stability result for GT w.r.t. smooth kernels.

\paragraph{Smooth kernels.} Similarly as in the case of MS, GT can also be defined with respect to kernels different from the truncation kernel. In particular, we define GT for a certain type of smooth kernels and establish its stability.

The following definition of smooth kernels is adapted from Definition 2 and the assumptions of Theorem 1 in \cite{martinez2020shape} .

\begin{definition}\label{def:smooth k}
Let $\mathcal{K}:[0,\infty)\rightarrow(0,\infty)$ be a bounded and differentiable function such that:
\begin{enumerate}
    \item $M_m\coloneqq\int_0^\infty r^{\frac{m}{2}-1}\mathcal{K}(r)\,dr<\infty$.
    \item There exists $C>0$ such that $r\mathcal{K}(r)\leq C,\,\forall r\in[0,\infty)$.
    \item There exists $L>0$ such that $|\mathcal{K}'(r)|\leq L$ and $r^\frac{3}{2}|\mathcal{K}'(r)|\leq L$ for $r\in[0,\infty)$.
\end{enumerate}
Then, given any $\eps\in(0,\infty)$, we define the smooth kernel $K_\eps:\R^m\times \R^m\times(0,\infty)\rightarrow \R$ associated with $\mathcal{K}$ by 
\[K_\eps(x,y)\coloneqq\mathcal{K}\lc\frac{\norm{y-x}^2}{\eps^2}\rc. \]
\end{definition}

Now, we define the GT distance with respect to smooth kernels and state our main result as follows. For $\alpha\in\mathcal{P}_f(X)$ and any $x\in X$, consider the probability measure $\frac{K_\eps(x,y)\,\alpha(dy)}{\int_{\R^m}K_\eps(x,y)\,\alpha(dy)}$. It belongs to $\mathcal{P}_{1,2}(X)$ since $\mathcal{K}$ is bounded and satisfies condition 2 in Definition \ref{def:smooth k}. We denote by $\mu_{\alpha,d_X}^{K_\eps}(x)$ the mean and by ${\Sigma}_{\alpha,d_X}^{K_\eps}(x)$ the covariance matrix of the probability measure $\frac{K_\eps(x,y)\,\alpha(dy)}{\int_{\R^m}K_\eps(x,y)\,\alpha(dy)}$.

Then, with respect to a smooth kernel $K_\eps$ we define the GT distance $d_{\alpha,d_X}^{\mathsmaller{(K_\eps,\lambda)}}(x,x')$ between $x,x'\in X$ as follows:
\begin{equation}\label{eq:gt-smooth}
    d_{\alpha,d_X}^{\mathsmaller{(K_\eps,\lambda)}}(x,x')\coloneqq\left(\norm{x-x'}^2+\lambda\cdot \lc\dcov\,\lc\Sigma_{\alpha,d_X}^{K_\eps}(x),\Sigma_{\alpha,d_X}^{K_\eps}(x')\rc\rc^2\right)^\frac{1}{2}.
\end{equation}

\begin{theorem}[Stability of GT for smooth kernels]\label{thm:gt-stable-smooth}
Let $X\subseteq \R^m$. Then, there exists a positive constant $C>0$ depending on $\mathcal{K},\eps,D$ and $X$ such that for $\alpha,\beta\in \pf(X)$, we have
\[\norm{d_{\alpha,d_X}^{\mathsmaller{(K_\eps,\lambda)}}-d_{\beta,d_X}^{\mathsmaller{(K_\eps,\lambda)}}}_\infty\leq 2\sqrt{ m\,\lambda\,C\,\dW{1}(\alpha,\beta)}. \]
\end{theorem}

\section{Algorithms and optimization techniques}\label{sec:algorithm}
As already mentioned in Remark \ref{rmk:iterating WT}, WT can be iteratively applied to a given data set. In this paper, we will mainly focus on the iterative implementation of three instances of WT, namely, LT-WT, MS (w.r.t. the truncation localization) and GT. In this section, we introduce iterative algorithms for those three methods and provide various computational techniques for accelerate GT. Finally, we will provide complexity analysis for all the three methods.

\subsection{Algorithms for iterative MS, LT-WT and GT} 

Algorithms for iterative MS, LT-WT and GT are given in Algorithms \ref{alg:ms}, \ref{alg:wt} and \ref{alg:gt}, respectively. The algorithms for iterative MS and LT-WT (Algorithms \ref{alg:ms}, \ref{alg:wt}) are self-explanatory whereas the iterative GT algorithm (Algorithm \ref{alg:gt}) is more involved than the other two. Below we provide a detailed explanation of the iterative GT algorithm.

\paragraph{Explanation of the iterative GT algorithm.}
Given a finite metric space $(X,d_X)$ and a coordinatization map $\varphi^0=\varphi:X\rightarrow\R^m$, we first compute the GT distance $d^{\mathsmaller{(\eps,\lambda)}}_{{\alpha,d_X}}$. We then update the coordinatization map $\varphi^0$ to $\varphi^1:X\rightarrow\R^m$ in a way similar to what MS algorithm does: we send each $x\in X$ to $\mean\lc m_{\alpha,d}^{\mathsmaller{(\eps)}}(x)\rc$ where $d\coloneqq d^{\mathsmaller{(\eps,\lambda)}}_{{\alpha,d_X}}$. Up to now, we have a new metric space $\left( X,d^{\mathsmaller{(\eps,\lambda)}}_{{\alpha,d_X}}\right)$ and a new coordinatization map $\varphi^1:X\rightarrow\R^m$, based on which we can repeat the above process.

In \Cref{alg:gt}, the roles of the successive $\varphi^k$s are implicitly encoded by the successive $x_i^k$s: the embedding $\varphi^k:X\rightarrow\R^m$ described above sends each $x_i\in X$ to $x_i^{k}\in\R^m$. In this way, the set $X^k$ is simply the image of $X$ under $\varphi^k$ and notation such as $\Sigma^{\mathsmaller{(\eps)}}_{\alpha,D^k}\lc x_i^{k+1}\rc $ and $d^{\mathsmaller{(\eps,\lambda)}}_{{\alpha,D^k}}\lc x_i^{k+1},x_j^{k+1}\rc $ is understood as the usual notation described in \Cref{sec:gt} when applying GT to $(X,D^k)$ with the map $\varphi^{k+1}$. We also refer to the coordinatization map updating process as the \emph{point updating process}. We refer to the first round of GT distance computation as the initialization step. With this initialization step, when $\lambda=0$, \Cref{alg:gt} reduces to \Cref{alg:ms}.

\begin{algorithm}[htb]
\caption{Iterative MS}
\begin{algorithmic}[1]
\STATE \textbf{Input:} Points $X = \left\{ x_1, x_2, ..., x_n \right\} \in \mathbb{R}^{n\times m}$, probability measure $\alpha = \{ \alpha_1, \alpha_2, ..., \alpha_n\}$
\STATE \textbf{Initialization:} $k=0$
\WHILE{ $k< \; \mathrm{max\_iter}$}
\FOR{$i \in [n]$}
\STATE $x_i^{k+1}=\mean\left(m_{\alpha}^{\mathsmaller{(\eps)}}\lc x_i^k\rc \right)$
\ENDFOR
\STATE $k = k + 1$
\ENDWHILE
\STATE \textbf{Output:} $X^k = \left\{ x_1^k, x_2^k, ..., x_n^k \right\} \in \mathbb{R}^{n\times m}$
\end{algorithmic}
\label{alg:ms}
\end{algorithm}

\begin{algorithm}[htb]
\caption{Iterative LT-WT}
\begin{algorithmic}[1]
\STATE \textbf{Input:} Probability measure $\alpha = \{ \alpha_1, \alpha_2, ..., \alpha_n\}$, distance matrix $D$ (on a set of points labeled by $X = \left\{ x_1, x_2, ..., x_n \right\}$)
\STATE \textbf{Initialization:} $k=0$; $D^k=D$
\WHILE{ $k< \; \mathrm{max\_iter}$}
\FOR{$i\in[n]$}
\STATE Compute the localized probability measure $m_{\alpha,D^{k}}^{\mathsmaller{(\eps)}}(x_i)$
\ENDFOR

\FOR{$i,j\in[n]$}
\STATE Let $D^{k+1}(x_i,x_j)=d^{\mathsmaller{(\eps)}}_{{\alpha,D^k}}(x_i,x_j)$.
\ENDFOR
\STATE $k = k + 1$
\ENDWHILE
\STATE \textbf{Output:} $D^k$
\end{algorithmic}
\label{alg:wt}
\end{algorithm}

\begin{algorithm}[htb]
\caption{Iterative GT}
\begin{algorithmic}[1]
\STATE \textbf{Input:} Points $X = \left\{ x_1, x_2, ..., x_n \right\} \in \mathbb{R}^{n\times m}$, probability measure $\alpha = \{ \alpha_1, \alpha_2, ..., \alpha_n\}$, distance matrix $D$
\STATE \textbf{Initialization:} $k=0$; for $i,j\in[n]$, $x_i^k=x_i$ and $D^k\lc x_i^k,x_j^k\rc =d^{\mathsmaller{(\eps,\lambda)}}_{{\alpha,D}}\lc x_i^k,x_j^k\rc $

\WHILE{ $k< \; \mathrm{max\_iter}$}

\FOR{$i\in[n]$}
\STATE $x_i^{k+1}=\mean\left(m_{\alpha,D^{k}}^{\mathsmaller{(\eps)}}\lc x_i^k\rc \right)$

\STATE Compute the covariance matrix $\Sigma^{\mathsmaller{(\eps)}}_{\alpha,D^k}\lc x_i^{k+1}\rc $
\ENDFOR

\FOR{$i,j\in[n]$}
\STATE Let $D^{k+1}\lc x_i^{k+1},x_j^{k+1}\rc =d^{\mathsmaller{(\eps,\lambda)}}_{{\alpha,D^k}}\lc x_i^{k+1},x_j^{k+1}\rc $.
\ENDFOR

\STATE $k = k + 1$
\ENDWHILE
\STATE \textbf{Output:}  $D^k$, $X^k = \left\{ x_1^k, x_2^k, ..., x_n^k \right\} \in \mathbb{R}^{n\times m}$
\end{algorithmic}
\label{alg:gt}
\end{algorithm}

\begin{figure}[htb]
	\includegraphics[width=\linewidth]{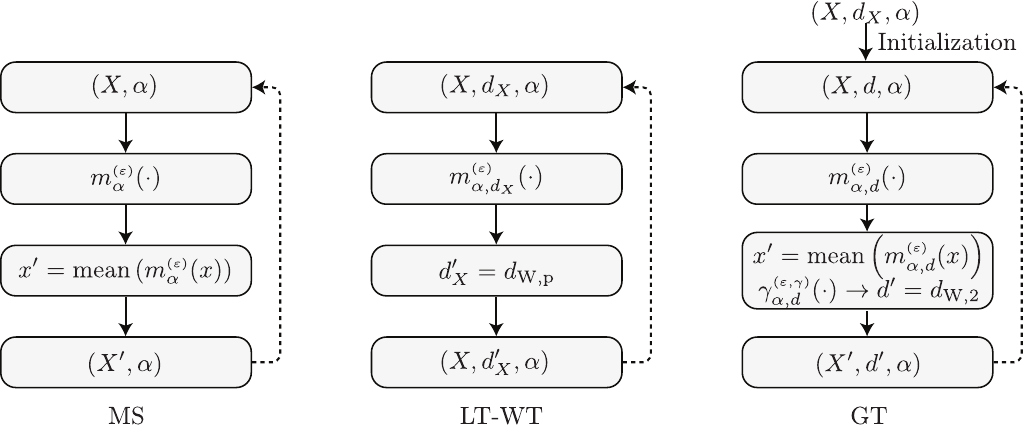}
    \caption{\textbf{Algorithmic structures of MS, LT-WT and GT.}  In MS and GT, the set $X$ is assumed to be a subset of $\R^m$, i.e., there exists an implicit coordinatization map $\varphi$. $X'$ denotes the set of updated points $x'$. $d'$ denotes the updated distance.}
    \label{fig:algcom}
\end{figure}

\paragraph{Structural comparison of the three algorithms.} 
Whereas the MS algorithm is a \emph{point updating} process and the LT-WT algorithm is a \emph{distance updating process}, the GT algorithm is actually a \emph{hybrid} between the MS and the LT-WT algorithms. \emph{It is composed of both a point updating and a distance updating process}. See Figure~(\ref{fig:algcom}) for an illustration.  Thus, the GT algorithm inherently provides us with two features, a point cloud and a distance matrix, which can be leveraged in different applications and thus provides an advantage over LT-WT. For example, the point updating process permits adapting GT to the task of image segmentation (cf. Section \ref{sec:img-seg}), whereas LT-WT is not applicable to this task. 
\begin{remark}
As pointed out in Remark \ref{rmk:GT-lambda}, when $\lambda=0$, $d^{\mathsmaller{(\eps,0)}}_{{\alpha,d_X}}$ boils down to the Euclidean distance. Consequently, when $\lambda=0$, the iterative GT algorithm coincides with the MS algorithm.
\end{remark}

\subsection{The computation of \texorpdfstring{$\dcov$}{dcov}: a new formula}
The computation of the Bures distance $\dcov$ (cf. \Cref{eq:dcov}) requires computing the term $\lc\Sigma_1^\frac{1}{2}\Sigma_2\Sigma_1^\frac{1}{2}\rc^\frac{1}{2}$, for which one has to carry out square root computations twice: once for $\Sigma_1^\frac{1}{2}$ and another one for $\left(\Sigma_1^\frac{1}{2}\Sigma_2\Sigma_1^\frac{1}{2}\right)^\frac{1}{2}$.  Since all we care about is the trace of $\lc\Sigma_1^\frac{1}{2}\Sigma_2\Sigma_1^\frac{1}{2}\rc^\frac{1}{2}$, by the proposition below it turns out that, for each pair of covariance matrices $\Sigma_1$ and $\Sigma_2$ we only need to compute the trace of the matrix $\left(\Sigma_1\Sigma_2\right)^\frac{1}{2}$.

\begin{proposition} 
\label{thm:redux}
    Given two $m\times m$ positive semi-definite matrices $A$ and $B$, we have that
    $$\mathrm{tr}\left(\left(A^\frac{1}{2}BA^\frac{1}{2}\right)^\frac{1}{2}\right)=\sum_{\lambda\in \mathrm{spec}(AB)}\lambda^\frac{1}{2},$$
    where $\mathrm{spec}(AB)$ denotes the multiset of eigenvalues of $AB$ counted with their multiplicities. 
\end{proposition}

\begin{proof}
The main idea is to prove that $AB$ and $A^\frac{1}{2}BA^\frac{1}{2}$ share the same spectrum\footnote{In the course of proving the theorem, we found a discussion website \citep{eigen} where the user Ahmad Bazzi proved this proposition. In the sequel, we present our original proof which is different from the one by Ahmad Bazzi.}. 

The case when one of the matrices is invertible is trivial (see for example Equation (4) in \cite{bhatia2019matrix}). Without loss of generality assume, in this case, that $A$ is invertible. Then, we have $AB= A^{\frac{1}{2}}\cdot A^\frac{1}{2}BA^\frac{1}{2}\cdot A^{-\frac{1}{2}},$ which implies that $AB$ and $A^\frac{1}{2}BA^\frac{1}{2}$ are similar to each other and thus share the same spectrum. Thus, the sum of square roots of eigenvalues of $AB$ counted with multiplicity is the same as the sum of square roots of eigenvalues of $A^\frac{1}{2}BA^\frac{1}{2}$ counted with multiplicity, which is exactly $\mathrm{tr}\lc\left(A^\frac{1}{2}BA^\frac{1}{2}\right)^\frac{1}{2}\rc$.

Now, assume that both $A$ and $B$ are singular. Let $B_t\coloneqq tI_m+B$ where $t\geq 0$ and $I_m$ is the $m\times m$ identity matrix. $B_t$ is then positive definite and thus invertible when $t>0$. Then, by the previous analysis, $A^\frac{1}{2}B_tA^\frac{1}{2}$ and $AB_t$ share the same spectrum for all $t>0$. Since $A^\frac{1}{2}B_tA^\frac{1}{2}=tA+A^\frac{1}{2}BA^\frac{1}{2}$ and $AB_t=tA+AB$, by the continuity of eigenvalues, by letting $t$ approach 0 we conclude that $A^\frac{1}{2}BA^\frac{1}{2}$ and $AB$ share the same spectrum.
Therefore, $\mathrm{tr}\lc\left(A^\frac{1}{2}BA^\frac{1}{2}\right)^\frac{1}{2}\rc=\mathrm{tr}\lc(AB)^\frac{1}{2}\rc. $ 
\end{proof}

\subsection{The neighborhood mechanism and other acceleration methods for GT}\label{sec: gt other accelerate}

In this section we devise acceleration strategies for GT which include a certain neighborhood mechanism and other acceleration methods via a refined analysis of \Cref{alg:gt}. 

\subsubsection{The neighborhood mechanism: acceleration of the point updating process} 
\label{sec:nbh-mech}
In the $k$th iteration of \Cref{alg:gt}, we need to compute the GT distance between each pair of points in $X^{k+1}$ (line 8 to line 10). This is necessary since in the end we want to generate a distance matrix as output of GT. However, in many application scenarios such as image segmentation, only the coordinates of data points are required. When applying GT to such tasks, it is then unnecessary to keep track of the complete GT distance matrices through subsequent iterations. Based on this observation, we develop the \emph{neighborhood mechanism} and devise a modification (\Cref{alg:nnm}) of the iterative GT algorithm (\Cref{alg:gt}). This modified GT algorithm only generates as output a point cloud/coordinatization of data points which coincides with the one generated by the iterative GT algorithm and does not keep track of the complete GT distance matrices through iterations. This compromise enables a refinement of the modified GT algorithm which is faster than the original iterative GT algorithm.

The simple yet crucial observation which leads to our neighborhood mechanism is that the $\eps$-ball for the GT distance is ``smaller'' than the Euclidean $\eps$-ball of the same radius. Based on this observation, in each iteration of the GT algorithm, we restrict the computation of GT distances to pairs of points whose Euclidean distances are no larger than $\eps$. We call this restriction method the \emph{neighborhood mechanism}. We then state our observation formally as follows:

\begin{proposition} 
\label{pro:nbh-trick}
In the $k$th iteration of the iterative GT algorithm (cf. {Algorithm \ref{alg:gt}}) we have for each point $x^{k+1}_i\in X^{k+1}$ that
$$B_\eps^{D^{k+1}}\lc x^{k+1}_i\rc\subseteq B_\eps\lc x^{k+1}_i\rc,$$
where $B_\eps^{D^{k+1}}\lc x^{k+1}_i\rc$ is the $\eps$-ball centered at $x_i^{k+1}$ with respect to the GT distance computed at the $(k+1)$th iteration:
$$B_\eps^{D^{k+1}}\lc x^{k+1}_i\rc\coloneqq\left\{x_j^{k+1}\in X^{k+1}:\,D^{k+1}\lc x_j^{k+1},x_i^{k+1}\rc\leq \eps\right\}$$
and $B_\eps\lc x^{k+1}_i\rc$ is the Euclidean $\eps$-ball centered at $x_i^{k+1}$:
$$B_\eps\lc x^{k+1}_i\rc\coloneqq\left\{x_j^{k+1}\in X^{k+1}:\,\norm{x_j^{k+1}-x_i^{k+1}}\leq \eps\right\}.$$
\end{proposition}

\begin{proof}
This follows directly from 
\begin{align*}
    \lc D^{k+1}\lc x_i^{k+1},x_j^{k+1}\rc\rc^2 &= \norm{x_i^{k+1}-x_j^{k+1}}^2 + \lc\dcov\lc\Sigma_{\alpha,D^{k}}^{\mathsmaller{(\eps)}}\left(x_i^{k+1}\right), \Sigma_{\alpha,D^{k}}^{\mathsmaller{(\eps)}}\left(x_j^{k+1}\right)\rc\rc^2\\
    &\geq \norm{x_i^{k+1}-x_j^{k+1}}^2.
\end{align*}
\end{proof}

Based on Proposition \ref{pro:nbh-trick}, we obtain an iterative GT algorithm which uses the neighborhood mechanism (called \emph{GT-neighborhood}) and its pseudocode is shown in Algorithm~\ref{alg:nnm}. The fact that $m_{\alpha,D^{k}}^{\mathsmaller{(\eps)}}\lc x_i^k\rc $ can be generated in Line 5 of Algorithm \ref{alg:nnm} is based on Proposition \ref{pro:nbh-trick}: in order to generate $m_{\alpha,D^{k}}^{\mathsmaller{(\eps)}}\lc x_i^k\rc $ we only need to consider points $x_j^k$ within $B_\eps\lc x_i^k\rc $.

\begin{algorithm}[htb]
\caption{Iterative GT with neighborhood mechanism (GT-neighborhood)}
\begin{algorithmic}[1]
\STATE \textbf{Input:} Points $X = \left\{ x_1, x_2, ..., x_n \right\} \in \mathbb{R}^{n\times m}$, probability measure $\alpha = \{ \alpha_1, \alpha_2, ..., \alpha_n\}$, distance matrix $D$
\STATE \textbf{Initialization:} $k=0$;
 $x^k_i=x_i$; for $i\in[n]$, for $x_j^{k}\in B_\eps(x_i^{k})$, let $D^k\lc x_i^k,x_j^k\rc =d^{\mathsmaller{(\eps,\lambda)}}_{{\alpha,D}}\lc x_i^k,x_j^k\rc $;
 \WHILE{ $k< \; \mathrm{max\_iter}$}
\FOR{$i\in[n]$}

\STATE $x_i^{k+1}=\mean\left(m_{\alpha,D^{k}}^{\mathsmaller{(\eps)}}\lc x_i^k\rc \right)$

\STATE Compute the covariance matrix $\Sigma^{\mathsmaller{(\eps)}}_{\alpha,D^k}\lc x_i^{k+1}\rc $

\FOR{$x_j^{k+1}\in B_\eps\lc x_i^{k+1}\rc $}
\STATE Let $D^{k+1}\lc x_i^{k+1},x_j^{k+1}\rc =d^{\mathsmaller{(\eps,\lambda)}}_{{\alpha,D^k}}\lc x_i^{k+1},x_j^{k+1}\rc $
\ENDFOR

\ENDFOR

\STATE $k = k + 1$
\ENDWHILE
\STATE \textbf{Output:} $X^k = \left\{ x_1^k, x_2^k, ..., x_n^k \right\} \in \mathbb{R}^{n\times m}$
\end{algorithmic}
\label{alg:nnm}
\end{algorithm}

\subsubsection{Neighborhood propagation} 
In line 5 of Algorithm \ref{alg:nnm}, in order to generate the probability measure $m_{\alpha,D^{k}}^{\mathsmaller{(\eps)}}\lc x_i^k\rc $, we first need to generate the $\eps$-ball $B_\eps^{D^k}\left(x_i^k\right)$ i.e., identify all points $x_j^k$ such that $D^k\left(x_i^k,x_j^k\right)\leq \eps$. Notice that once we determine that $x_j^k\in B_\eps^{D^k}\lc x_i^k\rc $, by symmetry of $D^k$, $x_i^k\in B_\eps^{D^k}\lc x_j^k\rc $. Hence, to determine $B_\eps^{D^k}\lc x_i^k\rc $, for all $i=1,\ldots,n$, we proceed as follows in order to avoid repetitive calculations. 
\begin{enumerate}
    \item We first initialize $B_i=\emptyset$ for all $i=1,\ldots,n$. 
    \item Then, for $i=1,\ldots,n$ and for any $j>i$, if $D^k\lc x_i^k,x_j^k\rc$ has been previously computed (i.e., $\norm{x_i^k-x_j^k}\leq \eps$) and $D^k\lc x_i^k,x_j^k\rc \leq \eps$, insert $x_j^k$ into both $B_i$ and $B_j$. 
\end{enumerate}
At the end of the procedure, it is obvious that $B_i=B_\eps^{D^k}\lc x_i^k\rc$ for all $i=1,\ldots,n$.
In this way, we do not need to determine each $B_\eps^{D^k}\lc x_i^k\rc$ independently and thus accelerate the process of generating all $\eps$-balls $B_\eps^{D^k}\lc x_i^k\rc$ for $i=1,\ldots,n$.

\subsubsection{Merging collocated points}
From an empirical stand point, after several successive applications of GT (or GT-neighborhood), the GT distances between some pairs of points converge to 0. When this happens, we propose to merge the collocated points into one single point and assign the sum of individual probabilities of all collocated points as the probability of this single point. 
This process reduces the total number of data points through subsequent iterations and thus accelerates the iterative GT algorithm.

\subsubsection{Validation}
In Table \ref{tab:compt} we present timing results for experiments that we carried out in order to validate our acceleration strategies.

\begin{table}[hbt]
\caption{\textbf{Validation of acceleration methods.} Consider the data set $X$ consisting of $200\times 200$ evenly spaced grid points inside the square $[0,1]\times [0,1]$: $X=\left\{\lc\frac{i}{199},\frac{j}{199}\rc:\,i,j=0,\cdots,199\right\}\subseteq\mathbb{R}^2$. Endow $X$ with the normalized empirical measure $\alpha$. Set $\lambda=1$ and $\eps=0.1$. 
$\tau$ denotes the current iteration number. 
Entries below show the running time of the GT algorithm with different combinations of acceleration strategies in each iteration. 
The experiments are performed on a Unix Server with 48 cores. We use C++ with openMP (Open Multi-Processing) to implement GT with parallel computing. 
(1) GT: full matrix computation of the GT distance; (2) GT-v1: GT with the neighborhood mechanism; (3) GT-v2: GT-v1 with neighborhood propagation; (4) GT-v3: GT-v1 with collocated points merged; (5) GT-v4: GT-v2  with collocated points merged. }

\vspace{0.1in}
    \centering
    \begin{tabular}{|c|c|c|c|c|c|}
        \hline
         & $\tau=1$ & $\tau=2$ & $\tau=3$ & 
        $\tau=4$ & $\tau=5$\\ \hline
        GT & 48.7s & - & - & - & - \\ \hline
        GT-v1 & 20.2s & 11.3s & 9s & 7.1s & 6.7s \\ \hline
        GT-v2 & 14.8s & 10s & 7.4s & 6.2s & 6.3s \\ \hline
        GT-v3 & 20.9s & 9.4s & 5.3s & 2.8s & 1.5s \\ \hline
        GT-v4 & 14.7s & 8.2s & 4.1s & 2.6s & 1.4s \\ \hline
    \end{tabular}
    \label{tab:compt}
\end{table}

\subsection{Complexity analysis}
In this section, we study the time complexity of the MS, LT-WT and GT algorithms under certain constraints which we describe next. We assume that all these methods are applied to a given point cloud $X\in\mathbb{R}^m$ of cardinality $n$. We fix a radius parameter $\eps>0$ for all these methods and assume that the maximum of cardinality of $\eps$-balls around points in $X$ is bounded by $N=N(\eps)$, i.e., $\max_{x\in X}\left|B_\eps(x)\cap X\right|\leq N(\eps).$ We summarize the complexity results in Table~\ref{tab:compformula}. 
\begin{table}[htb]
\caption{\textbf{Complexity comparison.} }
    \centering
    \begin{tabular}{|c|c|c|c|}
        \hline
          & MS & LT-WT & GT \\ \hline
        Cost  & $O(n^2m + nNm)$ & $O(n^2(N^3\log(N)))$  & $O(nNm^2+n^2 m^3)$ \\ \hline
    \end{tabular}
    \label{tab:compformula}
\end{table}

\paragraph{Time complexity of MS.} In the MS algorithm (\Cref{alg:ms}), we do not require the distance matrix w.r.t. $X$ as part of the input. Obviously, it takes time $O(m)$ to compute the Euclidean distance between any pair of points in $\mathbb{R}^m$. Now, for any $i\in[n]$ from line 4 in \Cref{alg:ms}, it takes time $O(nm)$ to construct the probability measure $m_{\alpha}^{\mathsmaller{(\eps)}}\lc x_i^k\rc $ where the factor $m$ comes from the computation of the Euclidean distances between $x_i^k$ and other data points. Then, it takes time at most $O(Nm)$ to compute the mean of $m_{\alpha}^{\mathsmaller{(\eps)}}\lc x_i^k\rc $ since the cardinality of the support of $m_{\alpha}^{\mathsmaller{(\eps)}}\lc x_i^k\rc $ is bounded by $N$ by our assumption. Therefore, the for-loop starting in line 4 of \Cref{alg:ms} can be accomplished in time at most $O(n^2m+nNm)$.

\paragraph{Time complexity of LT-WT.}
For one iteration of LT-WT (\Cref{alg:wt}), we have two parts to analyze:
\begin{itemize}
    \item Line 4 to line 6 in \Cref{alg:wt}: for each given $i\in[n]$, since we are given the distance matrix $D$ as part of the input, it takes time at most $O(n)$ to construct $m_{\alpha,D^k}^{\mathsmaller{(\eps)}}\lc x_i^k\rc $. Then, the whole for-loop can be accomplished in time at most $O(n^2)$.
    \item Line 7 to line 9 in \Cref{alg:wt}: for any $i,j\in[n]$, one needs to compute the Wasserstein distance between $m_{\alpha,D^k}^{\mathsmaller{(\eps)}}\lc x_i^k\rc $ and $m_{\alpha,D^k}^{\mathsmaller{(\eps)}}\lc x_j^k\rc $ described in line 8. Since the respective supports of these measures have cardinality bounded by $N$, this computation takes time at most $O(N^3\log(N))$~\cite{pele2009fast}. Then, the whole for-loop can be carried out in time at most $O(n^2N^3\log(N))$.
\end{itemize}
In summary, the total time complexity associated to one iteration of LT-WT is at most $O(n^2N^3\log(N))$. 

\paragraph{Time complexity of GT.} For one iteration of GT, the cost consists of two components: (1) the initialization step described in line 2 of \Cref{alg:gt} and (2) the steps described from line 4 to line 11. It is clear that the two components have the same time complexity and for simplicity, we only focus on analyzing the second component. In this component, we have two major parts which we analyze as follows.
\begin{itemize}
    \item Line 4 to line 7 in \Cref{alg:gt}: by \Cref{pro:nbh-trick}, for each given $i\in[n]$, the support of $m_{\alpha,D^{0}}^{\mathsmaller{(\eps)}}(x_i^0)$ is contained in the Euclidean $\eps$-ball of $x_i^0=x_i$ whose cardinality is bounded by $N$ due to our assumption at the beginning of this section. So it takes time at most $O(Nm^2)$ to compute the mean in line 5 and the covariance matrix in line 6. Then, the whole for-loop can be accomplished in time at most $O(nNm^2)$.
    \item Line 8 to line 10 in \Cref{alg:gt}: for any $i,j\in[n]$, one needs to compute the GT distance between $x_i^{1}$ and $x_j^{1}$ described in line 9. This computation involves matrix multiplication and matrix square root computation and hence its time complexity is at most $O(m^3)$~\citep{pan1998complexity,demmel2007fast}.  Then, the whole for-loop can be accomplished in time at most $O(n^2m^3)$.
\end{itemize}
In summary, the total time complexity associated to applying one iteration of GT is at most $O(nNm^2+n^2 m^3)$.

Note that from Table~\ref{tab:compformula}, when $N > m$, i.e., the estimate cardinality of neighborhoods is larger than the data dimension, the complexity order of the three methods is 
MS $<$ GT $<$ LT-WT.



\section{Experiments}\label{sec:experiments}
We now apply the Wasserstein Transform to various data sets and tasks. In all of our experiments, the probability measure $\alpha$ is the normalized empirical measure, and the radius adopted in each data set, $\eps$, varies across different experiment but, in any given experiment, it remains fixed (by default) throughout successive iterations. In all figures, we use $\tau$ to denote number of iterations.
We consider the performance of MS, LT-WT and GT, whenever applicable. In the sequel, we also use the notation GT-$\lambda$ whenever we need to report the value of $\lambda$ used in GT, e.g., GT-1 will mean $\lambda=1$.

\subsection{Clustering of a T-junction data set} 
\begin{wrapfigure}{r}{0.1\textwidth}
  \begin{center}
    \includegraphics[width=0.1\textwidth]{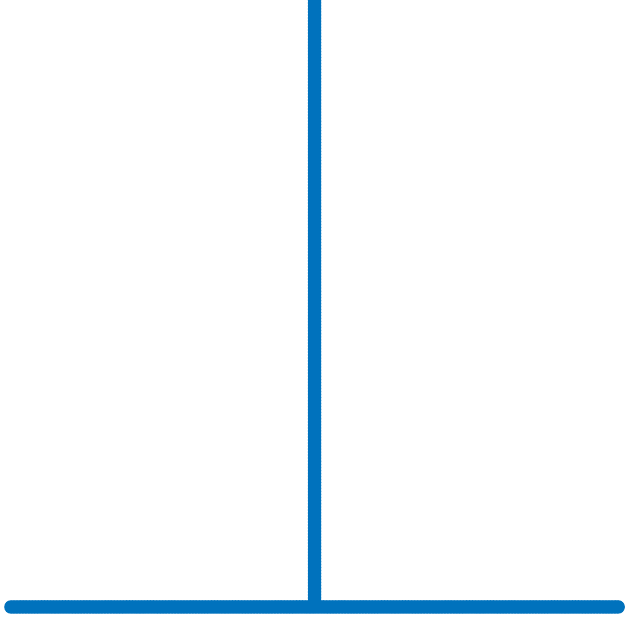}
  \end{center}
  \label{fig:tline-original}
\end{wrapfigure}
In this experiment we consider a clustering task on the \emph{T-junction} data set depicted on the right, which is composed of 200 evenly spaced points on the vertical line segment between $(0,1)$ and $(0,200)$, and 201 evenly spaced points on the horizontal line segment between $(-100,0)$ and $(100,0)$. Ideally, we would like to separate points from the two line segments and thus form two clusters, one for each line segment. We apply two iterations of MS, LT-WT1, LT-WT2, GT-1 as well as GT-5 to this data set with $\eps=10$. Then, we apply single-linkage clustering to the resulting metric spaces. See Figure~(\ref{fig:tlines-supp}) for our results.

We make the following observations: (1) Note that the respective 3D MDS plots induced by GT and LT-WT2 have comparable structures, which indicates similarity between their output distance matrices. This agrees with our analysis of the two line data set from Section \ref{sec: gt vs lt-wt}. (2) In the first iteration, all methods except for GT-5 shatter the data set into four parts. Only GT-5 is able to keep the horizontal line segment as one single cluster. (3) In the second iteration, only GT-5 and LT-WT2 can generate two correct clusters, one for each line segment (by choosing proper thresholds for their induced single-linkage dendrograms).

As noted in \Cref{rmk:GT-lambda}, $\lambda$ in GT controls the influence of local structures when comparing data points. We see through this experiment that fine-tuning $\lambda$ in GT indeed helps capture subtle geometric information in data sets and hence helps boosting performance in the task of clustering.

\begin{figure}[htb!]
    \centering
    \subfloat[2D and 3D MDS at $\tau=1$]{
    \label{fig:tline-mds-1-supp}
    \begin{minipage}[t][2.8cm][t]{0.08\textwidth}
                \centering
                MS \\ 
                \includegraphics[width=0.85\textwidth]{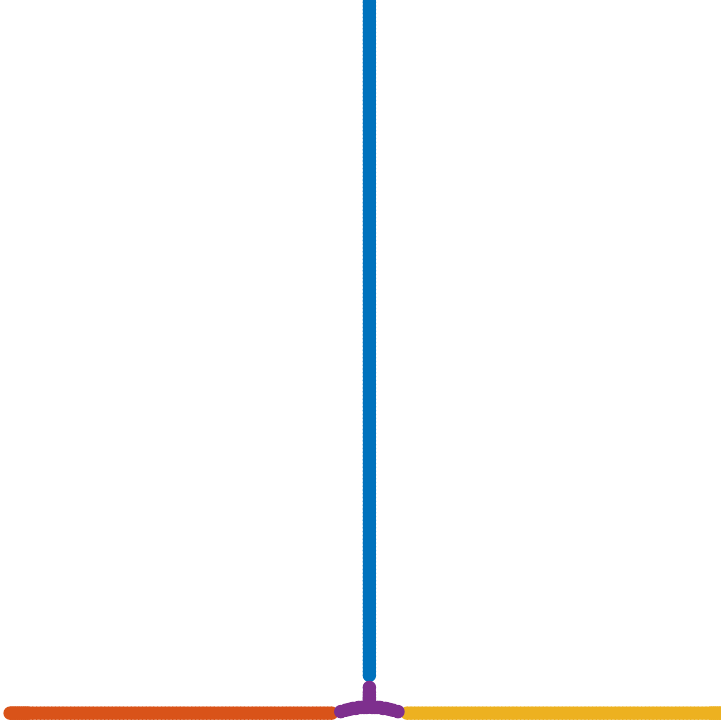} 
                \\ 
            \end{minipage}
            \begin{minipage}[t][2.8cm][t]{0.08\textwidth}
                \centering
                GT-1 
                \\  \includegraphics[width=1\textwidth]{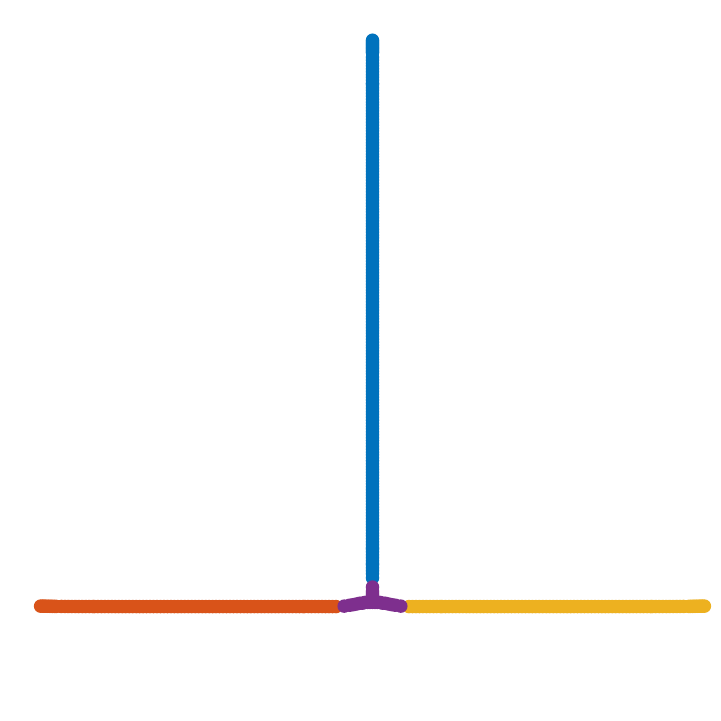}
                \\
              \includegraphics[width=1\textwidth]{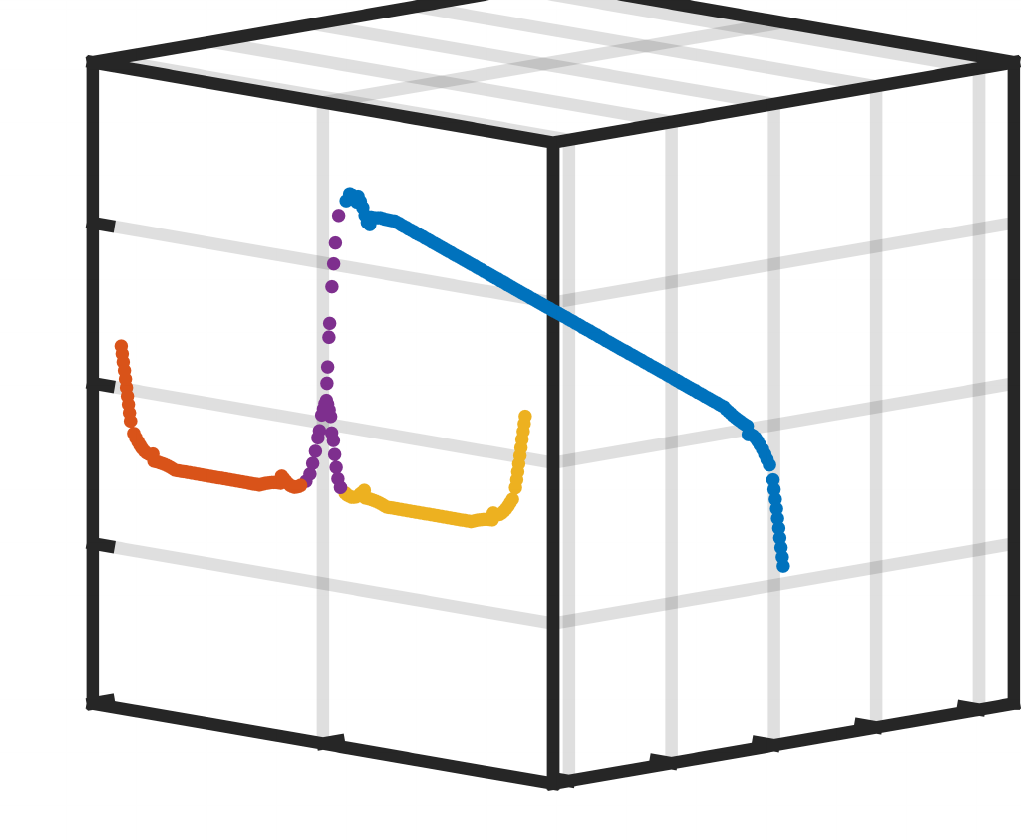} 
            \end{minipage}
            \begin{minipage}[t][2.8cm][t]{0.08\textwidth}
                \centering
                GT-5 \\ 
                \includegraphics[width=1\textwidth]{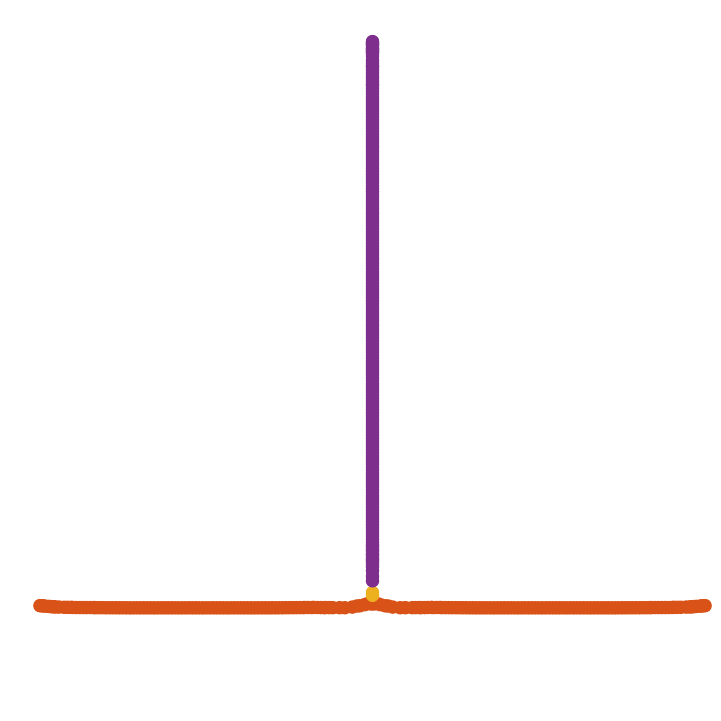}
                \\
              \includegraphics[width=1\textwidth]{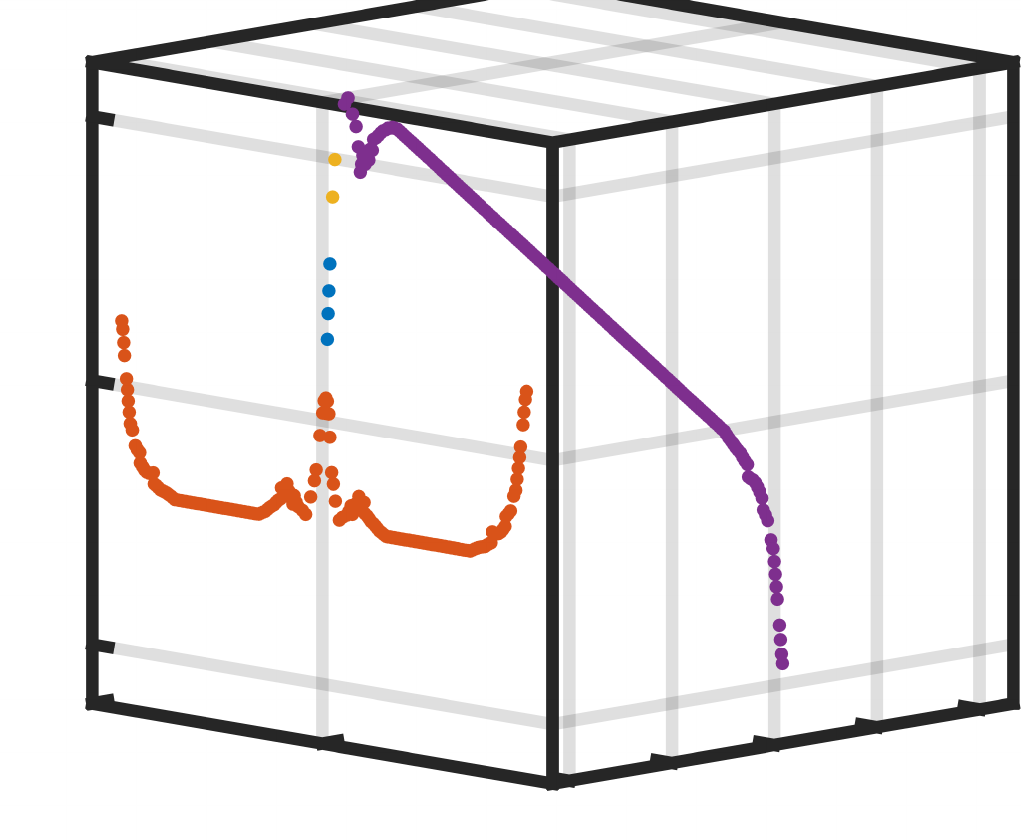} 
            \end{minipage}
            \begin{minipage}[t][2.8cm][t]{0.1\textwidth}
                \centering
                LT-WT2 \\ 
                \includegraphics[width=0.8\textwidth]{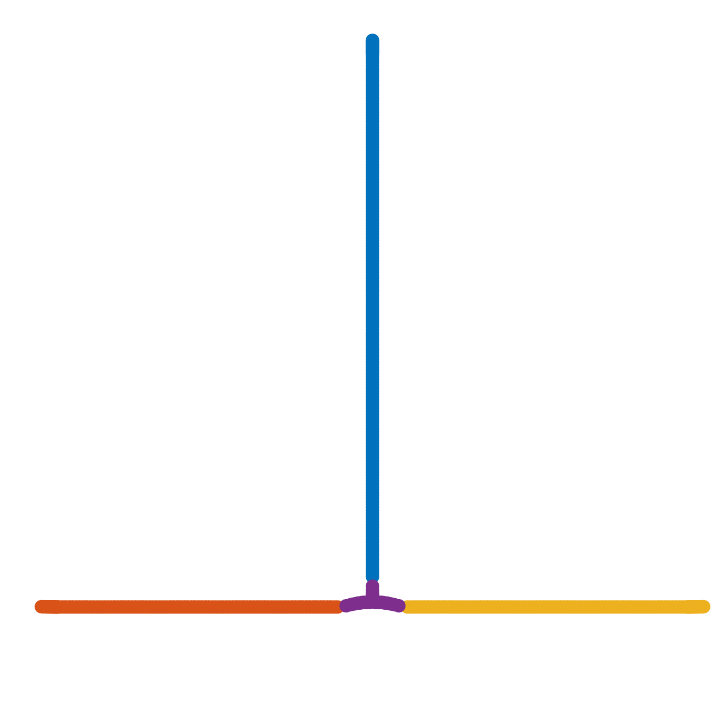}
                \\
              \includegraphics[width=0.8\textwidth]{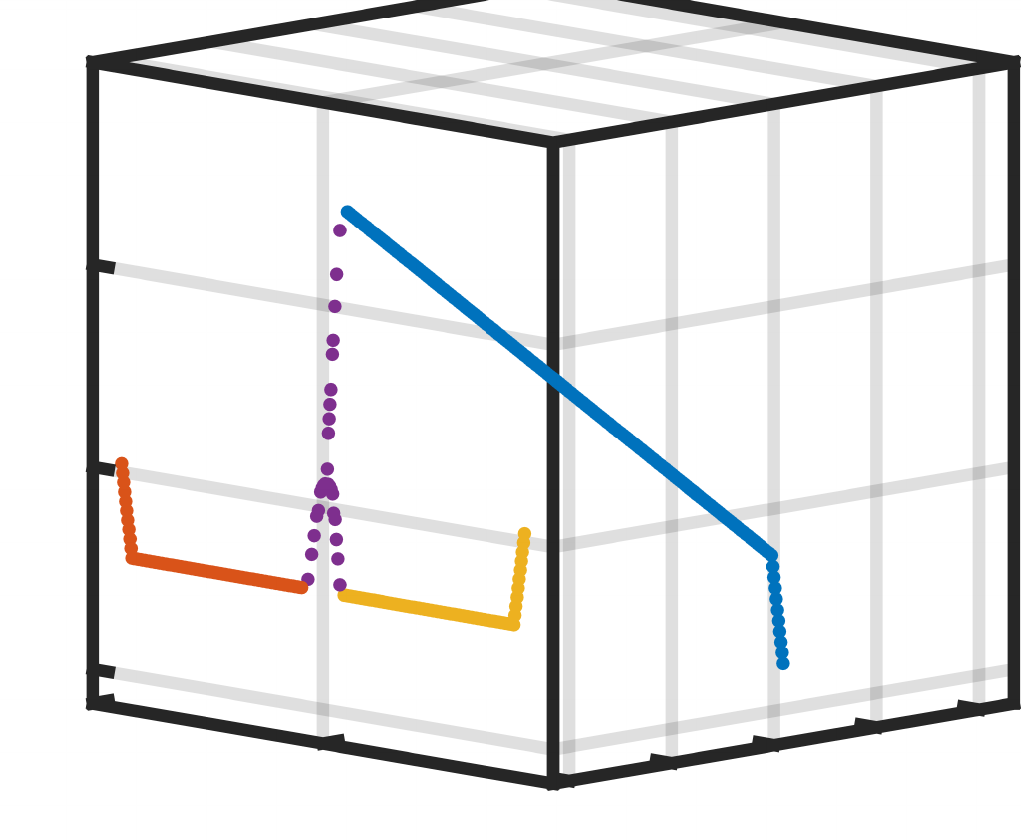} 
            \end{minipage}
            \begin{minipage}[t][2.8cm][t]{0.1\textwidth}
                \centering
                LT-WT1 \\ 
                \includegraphics[width=0.8\textwidth]{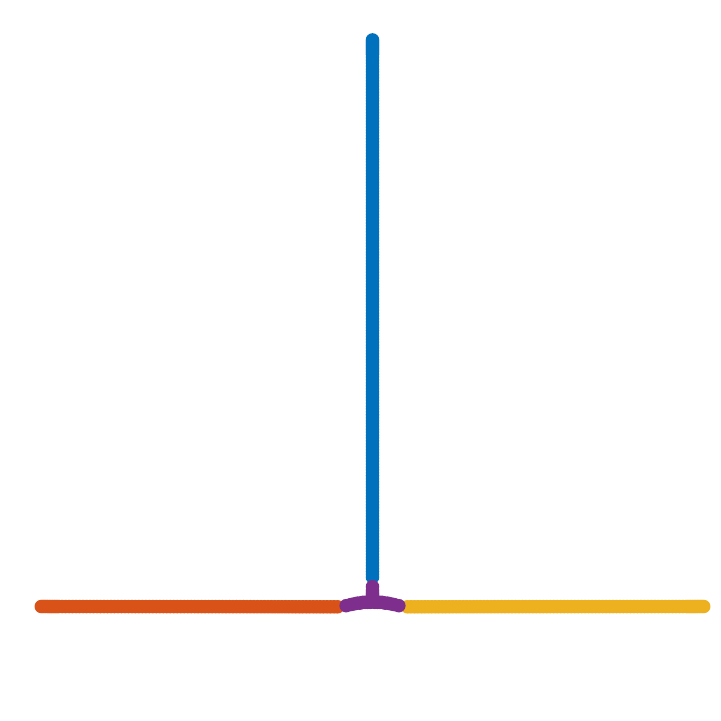}
                \\
              \includegraphics[width=0.8\textwidth]{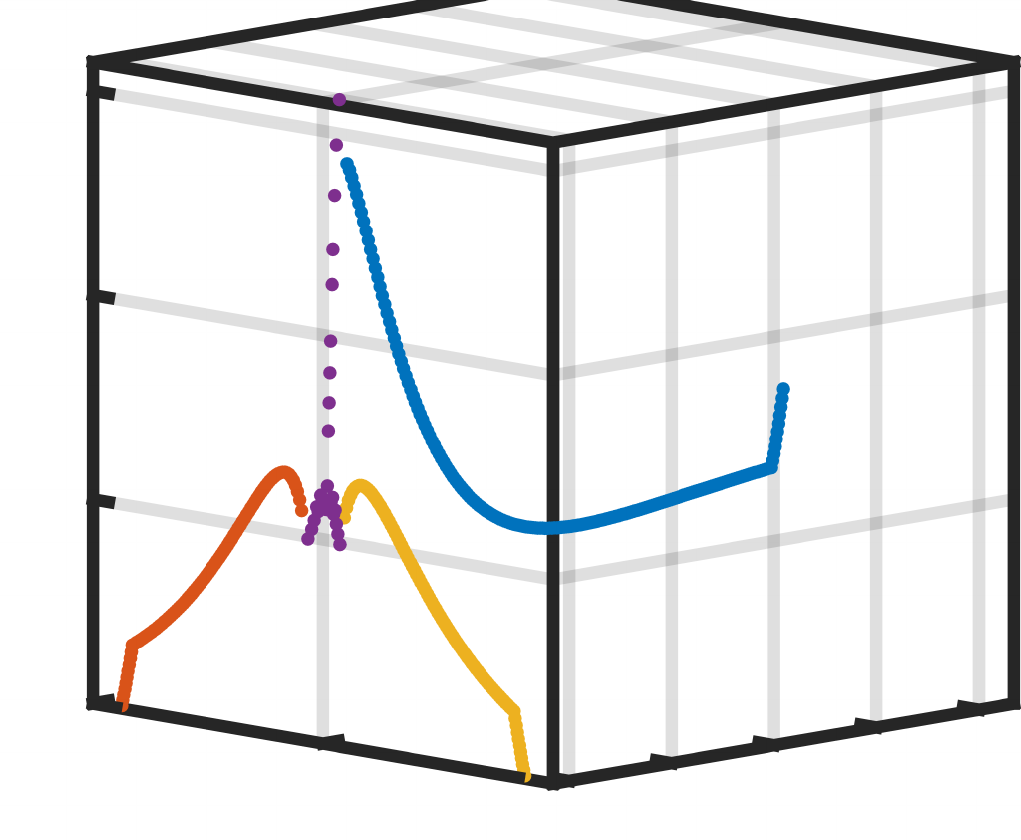} 
            \end{minipage} 
    }
    
    \subfloat[Dendrograms at $\tau=1$]{
    \label{fig:tline-dend-1-supp}  \begin{minipage}[t][3cm][t]{0.08\textwidth}
                \centering
                MS \\ 
                \includegraphics[width=1\textwidth]{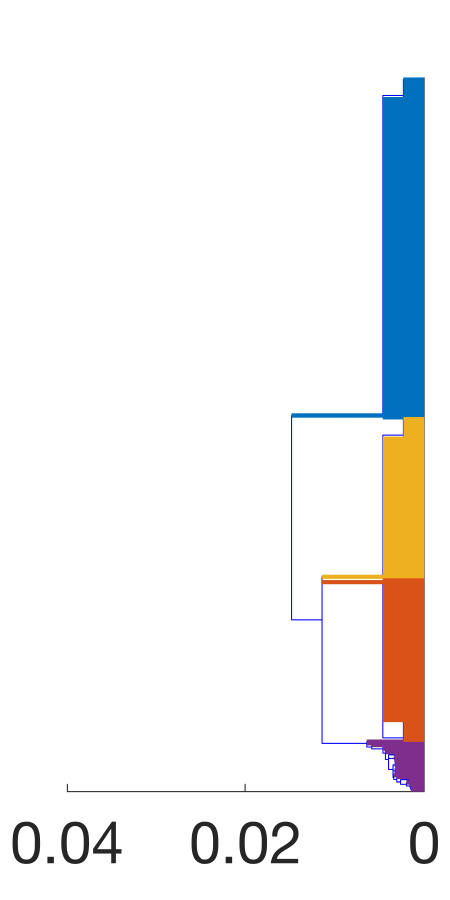}
            \end{minipage}
            \begin{minipage}[t][2.8cm][t]{0.08\textwidth}
                \centering
                GT-1 \\ 
                \includegraphics[width=1\textwidth]{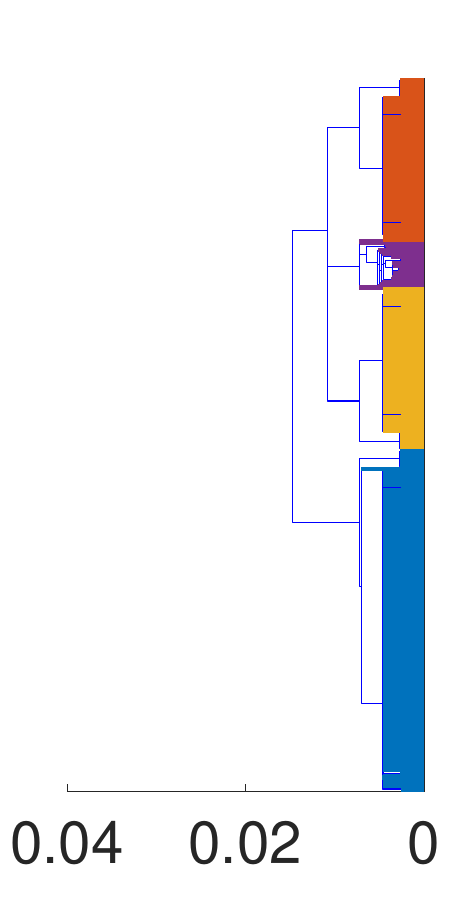}
            \end{minipage}
            \begin{minipage}[t][2.8cm][t]{0.08\textwidth}
                \centering
                GT-5 \\ 
                \includegraphics[width=1\textwidth]{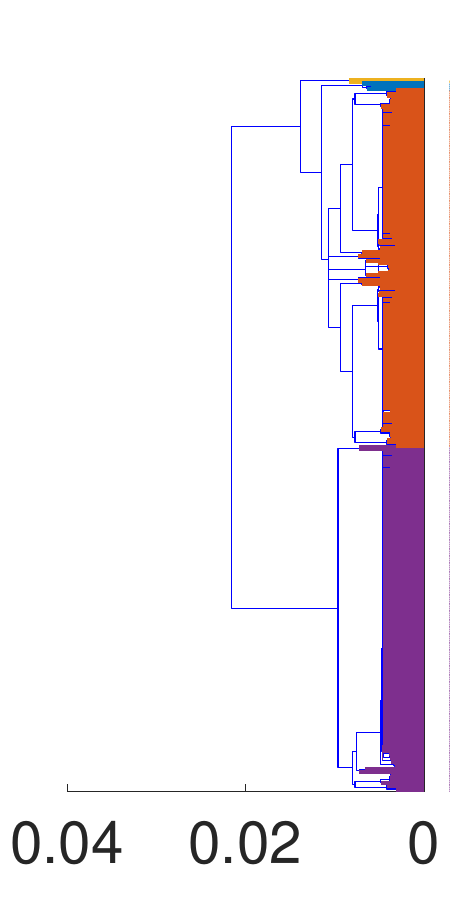}
            \end{minipage}  
            \begin{minipage}[t][2.8cm][t]{0.1\textwidth}
                \centering
                LT-WT2 \\ 
                \includegraphics[width=0.8\textwidth]{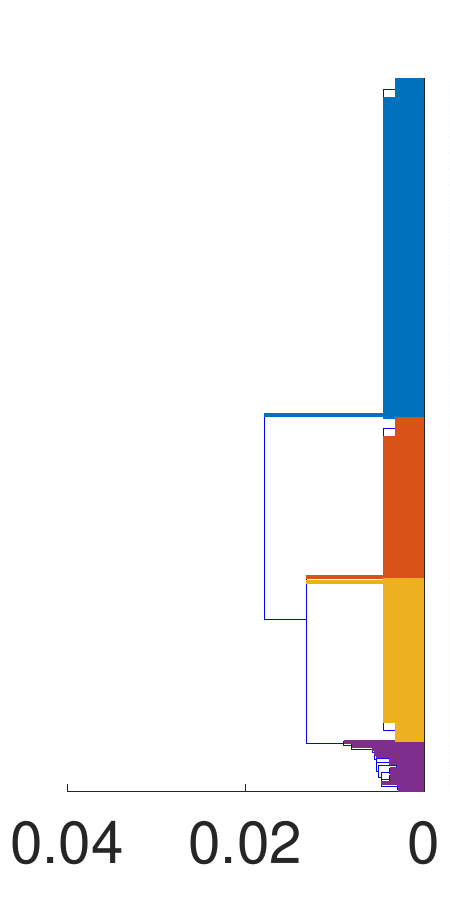}
            \end{minipage}
            \begin{minipage}[t][2.8cm][t]{0.1\textwidth}
                \centering
                LT-WT1 \\ 
                \includegraphics[width=0.8\textwidth]{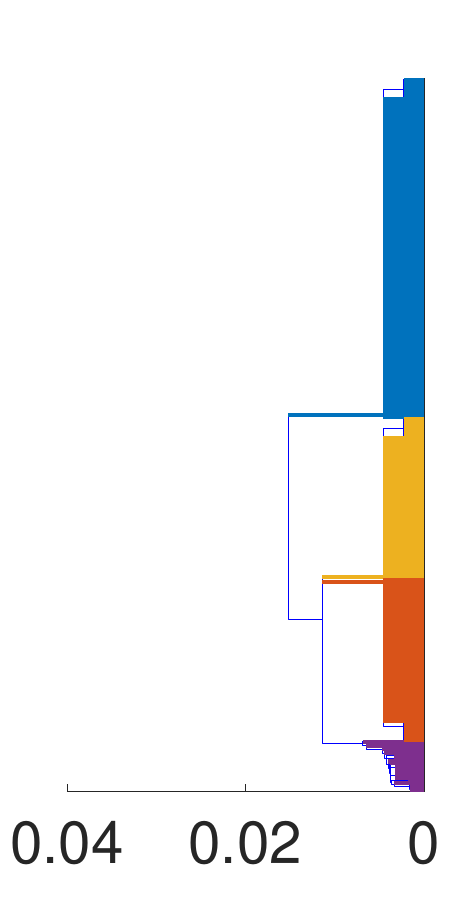}
            \end{minipage}    
    }

    \subfloat[2D and 3D MDS at $\tau=2$]{
    \label{fig:tline-mds-2-supp}
     \begin{minipage}[t][2.8cm][t]{0.08\textwidth}
                \centering
                MS \\ 
                \includegraphics[width=0.85\textwidth]{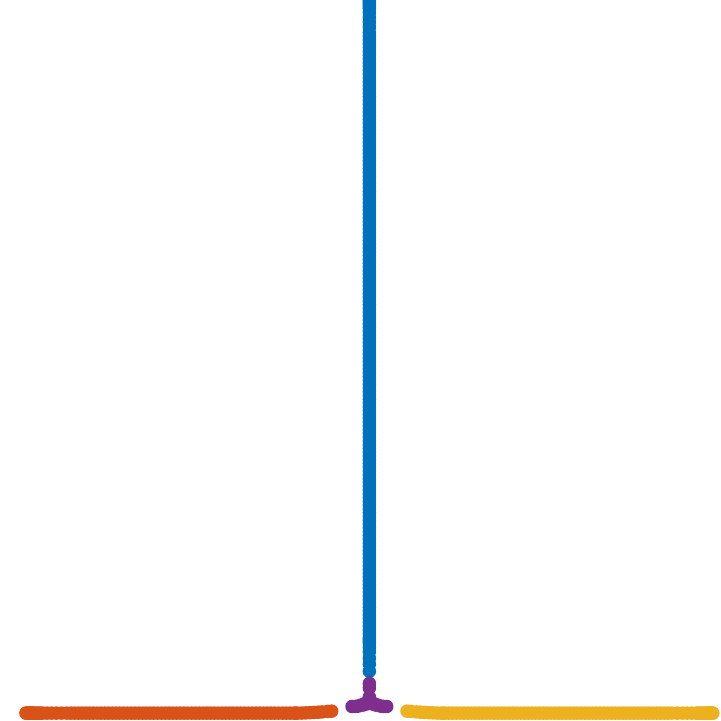} 
                \\ 
            \end{minipage}
            \begin{minipage}[t][2.8cm][t]{0.08\textwidth}
                \centering
                GT-1 
                \\  \includegraphics[width=1\textwidth]{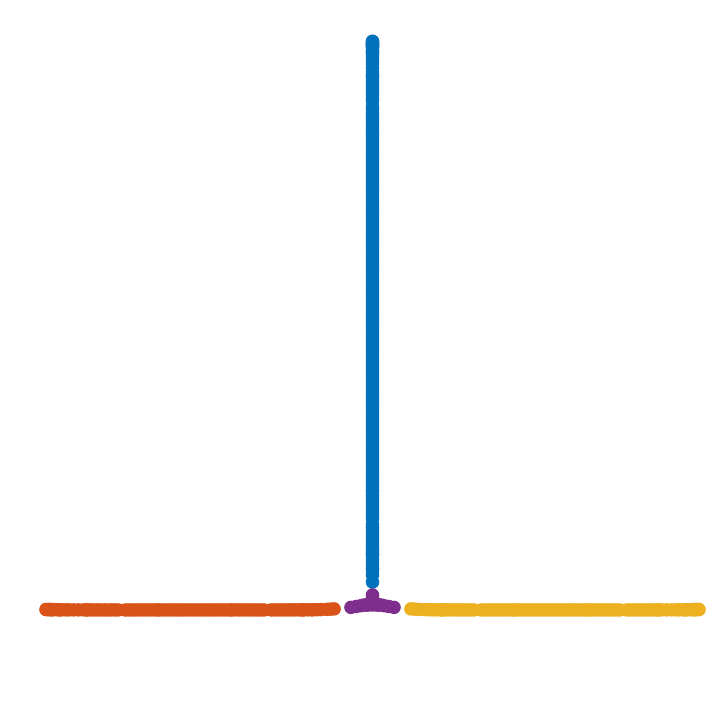}
                \\
              \includegraphics[width=1\textwidth]{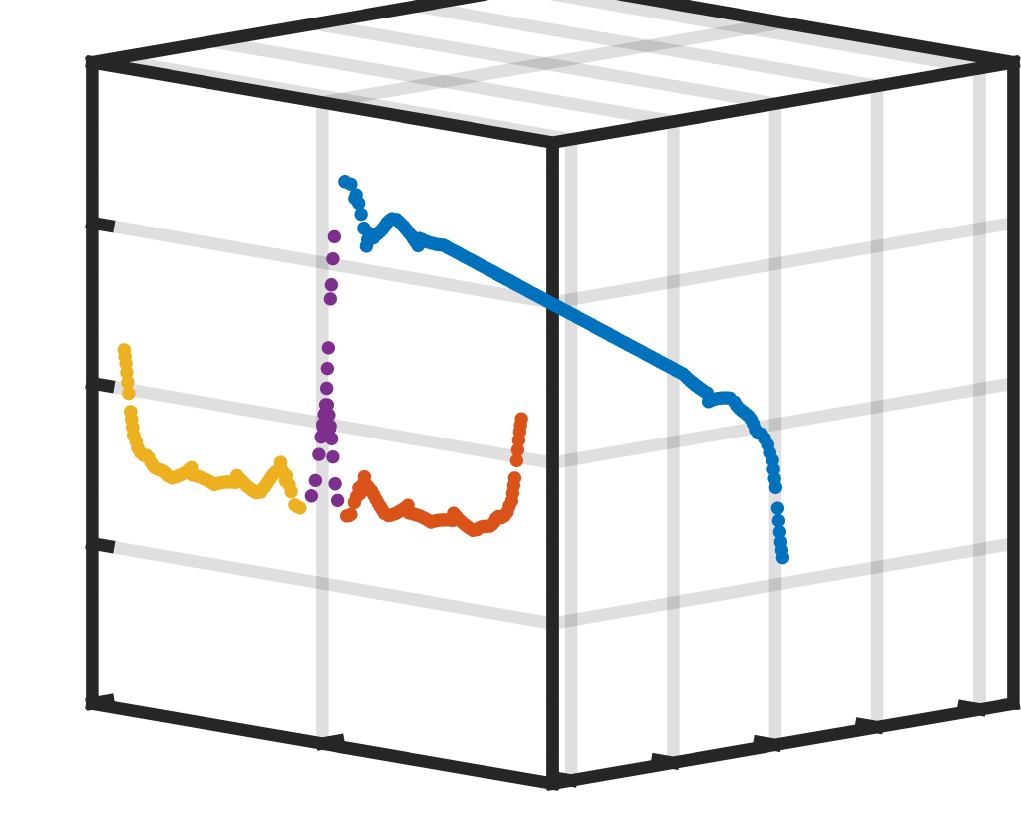} 
            \end{minipage}
            \begin{minipage}[t][2.8cm][t]{0.08\textwidth}
                \centering
                GT-5 \\ 
                \includegraphics[width=1\textwidth]{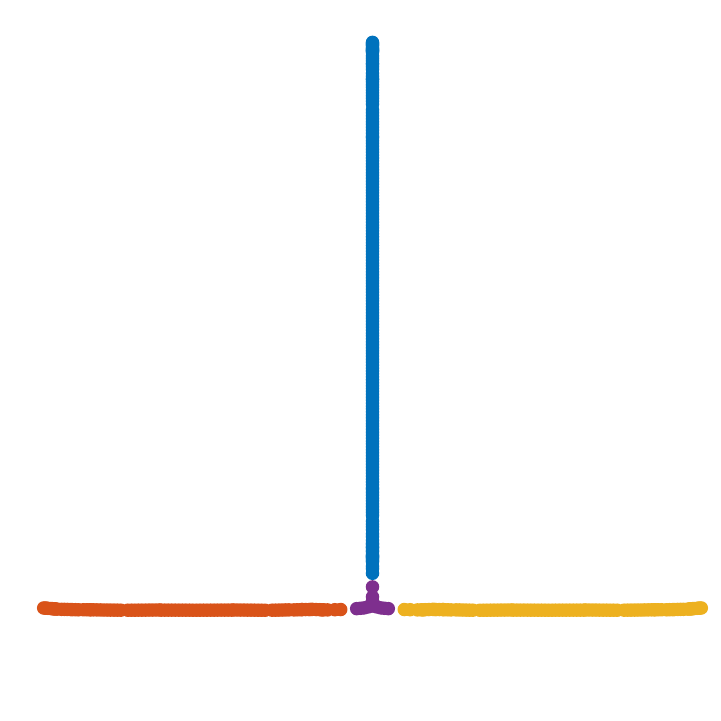}
                \\
              \includegraphics[width=1\textwidth]{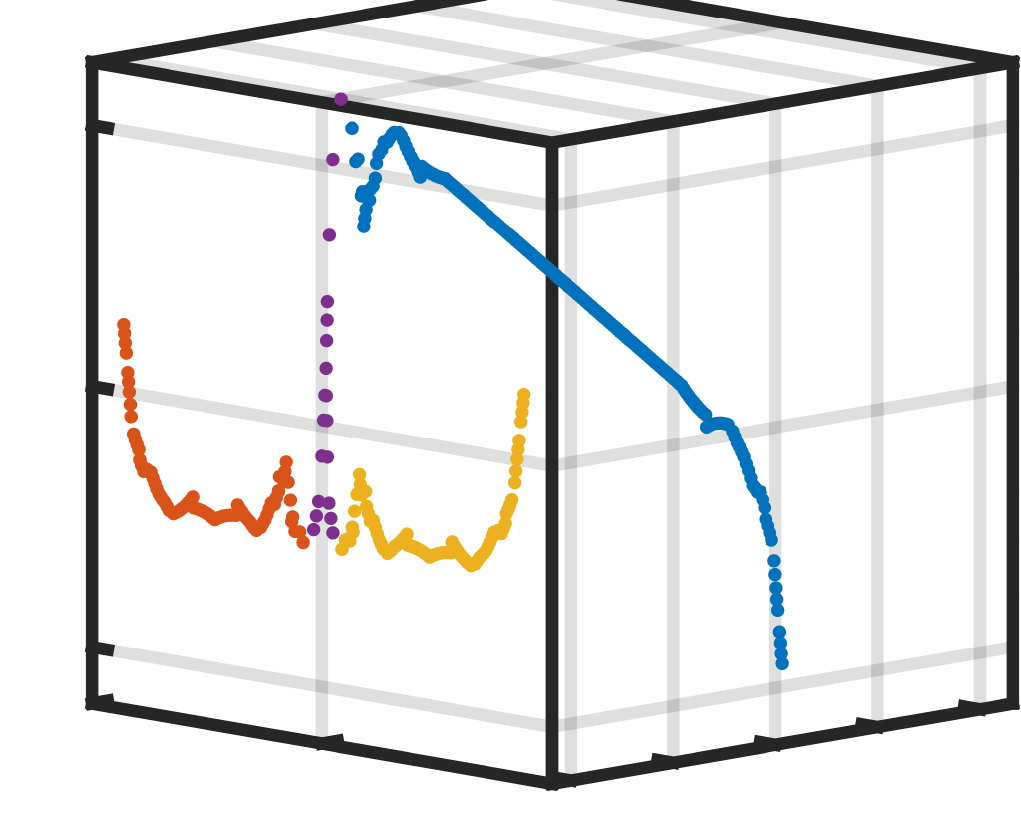} 
            \end{minipage}
            \begin{minipage}[t][2.8cm][t]{0.1\textwidth}
                \centering
                LT-WT2 \\ 
                \includegraphics[width=0.8\textwidth]{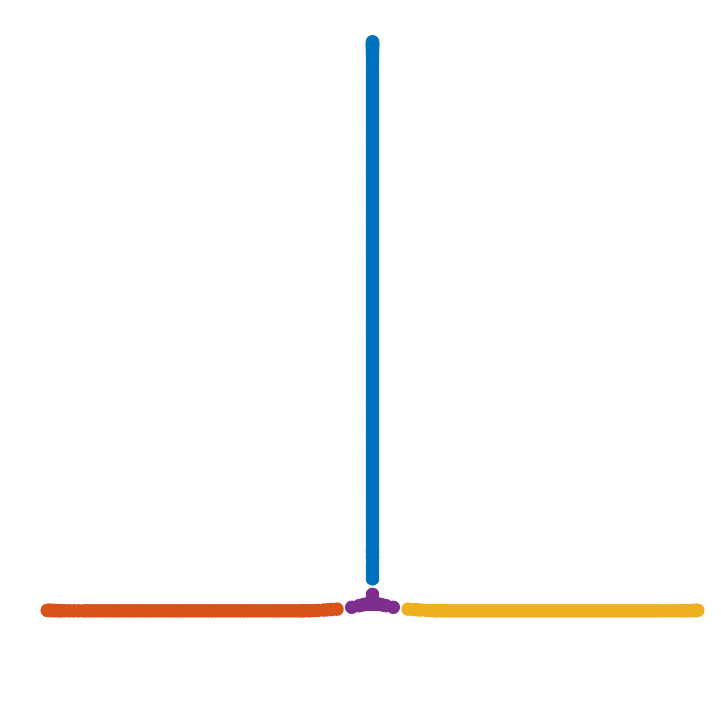}
                \\
              \includegraphics[width=0.8\textwidth]{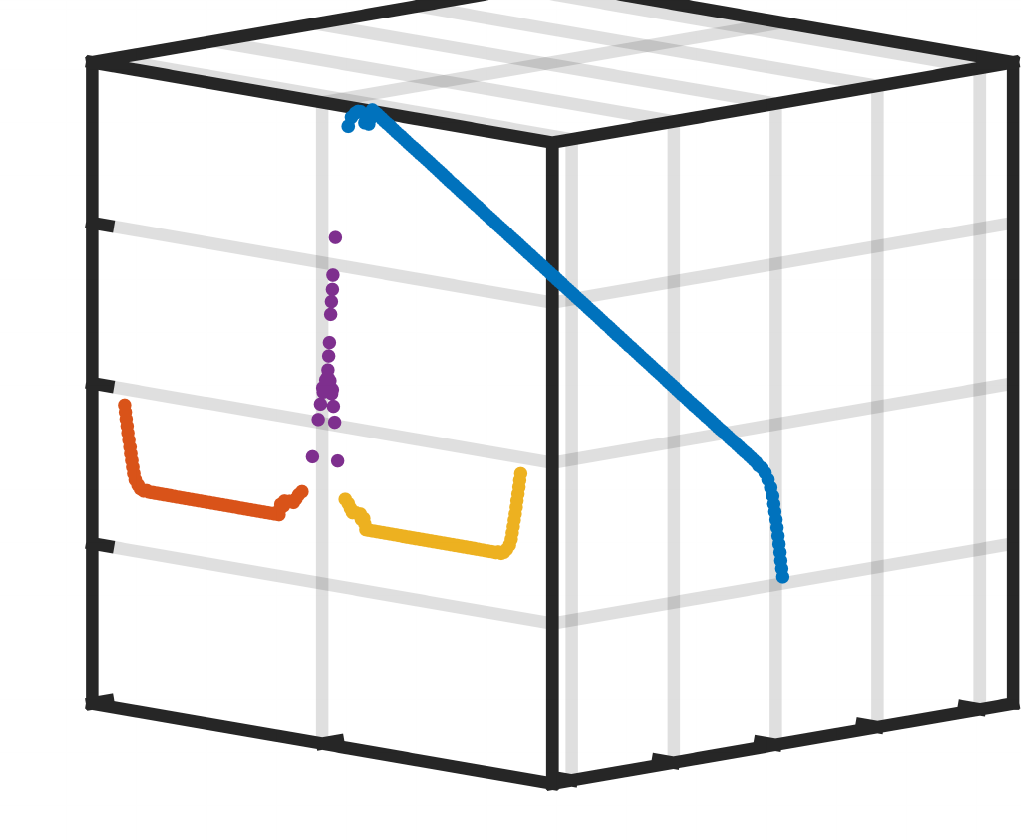} 
            \end{minipage}
            \begin{minipage}[t][2.55cm][t]{0.1\textwidth}
                \centering
                LT-WT1 \\ 
                \includegraphics[width=0.8\textwidth]{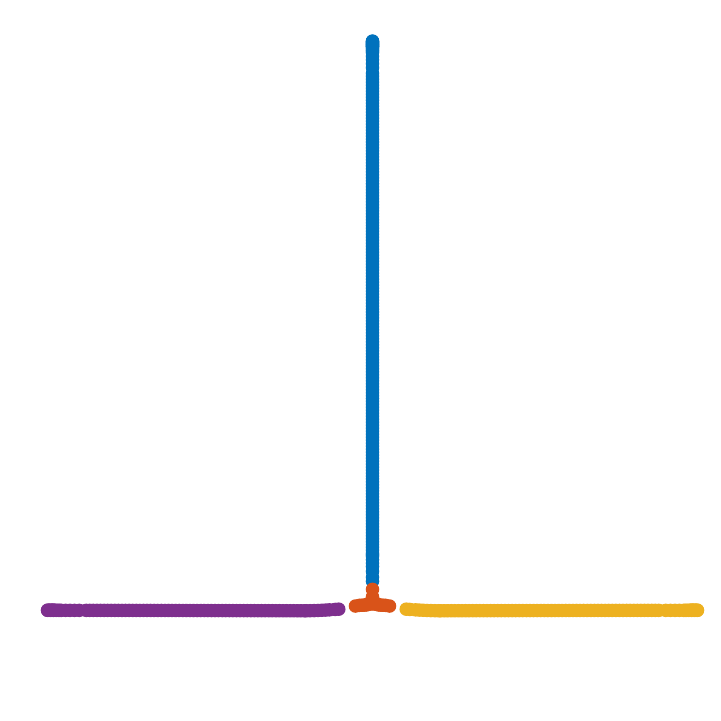}
                \\
              \includegraphics[width=0.8\textwidth]{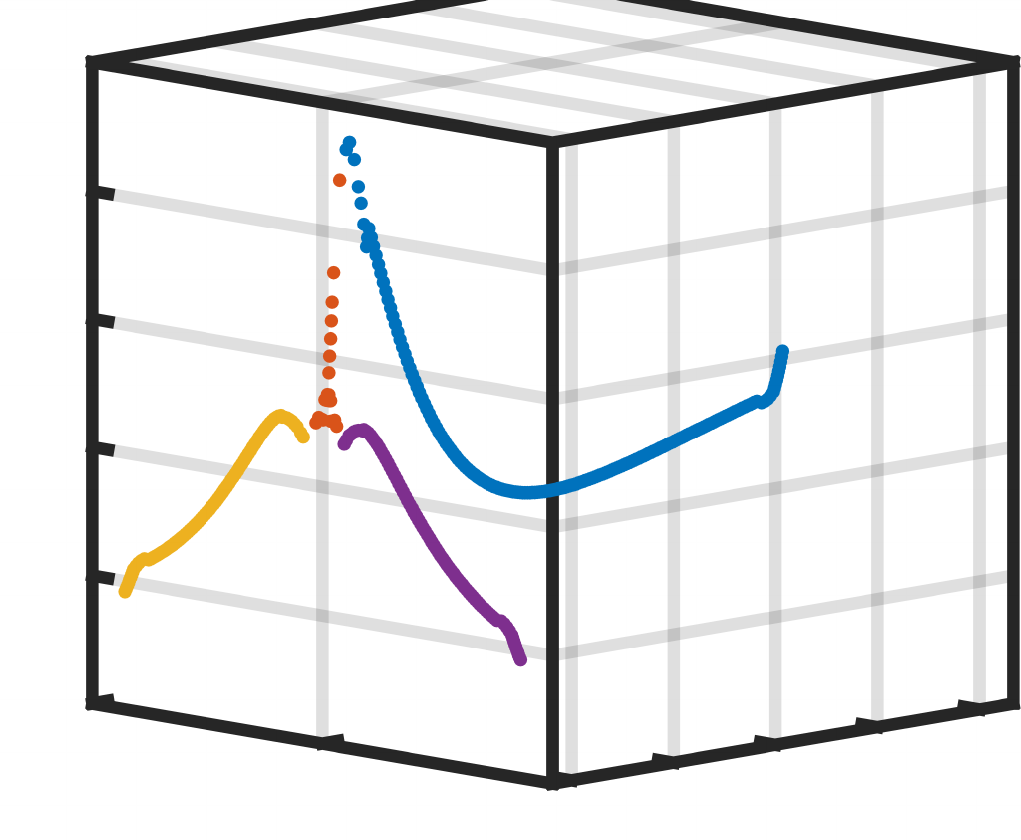} 
            \end{minipage} 
    }
    
    \subfloat[Dendrograms at $\tau=2$]{
    \label{fig:tline-dend-2-supp}
        \begin{minipage}[t][2.8cm][t]{0.08\textwidth}
                \centering
                MS \\ 
                \includegraphics[width=1\textwidth]{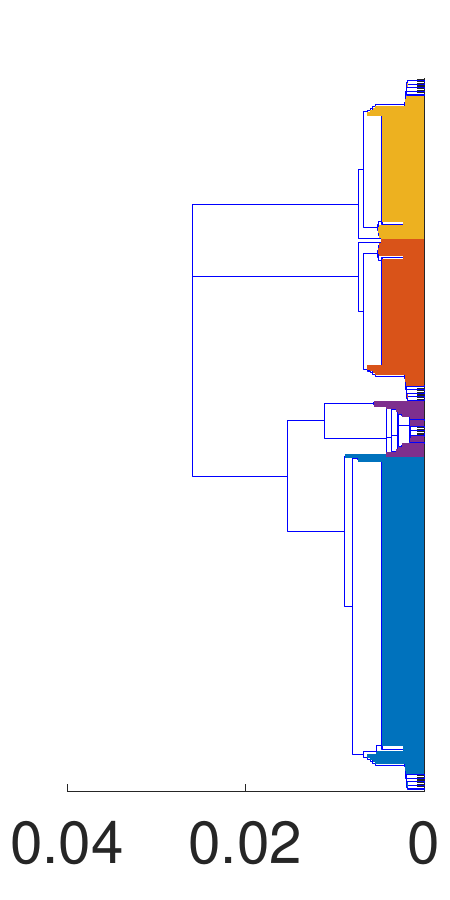}
            \end{minipage}
            \begin{minipage}[t][2.8cm][t]{0.08\textwidth}
                \centering
                GT-1 \\ 
                \includegraphics[width=1\textwidth]{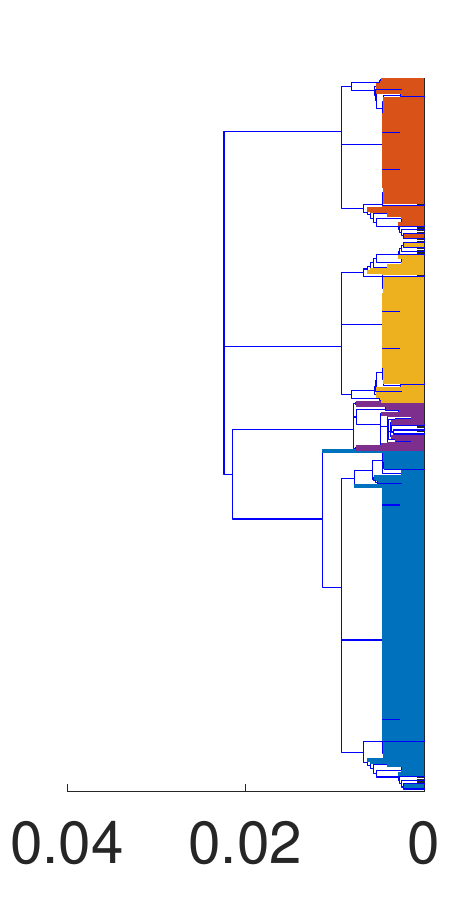}
            \end{minipage}
            \begin{minipage}[t][2.8cm][t]{0.08\textwidth}
                \centering
                GT-5 \\ 
                \includegraphics[width=1\textwidth]{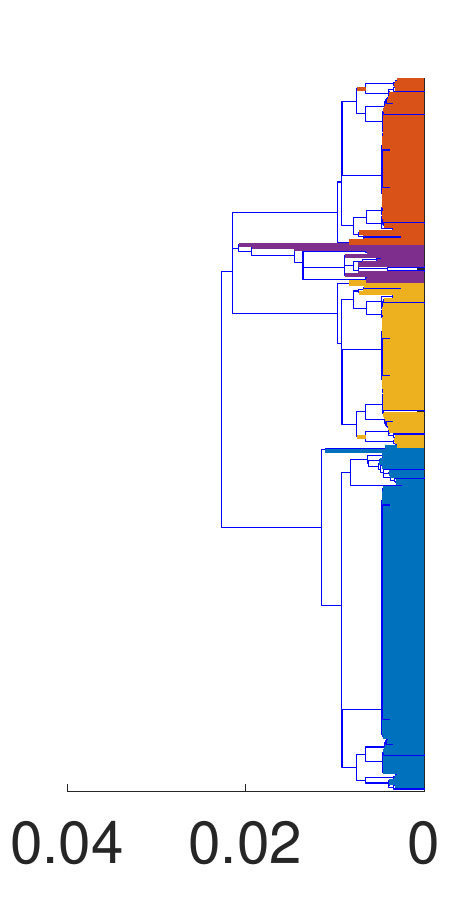}
            \end{minipage}  
            \begin{minipage}[t][2.8cm][t]{0.1\textwidth}
                \centering
                LT-WT2 \\ 
                \includegraphics[width=0.8\textwidth]{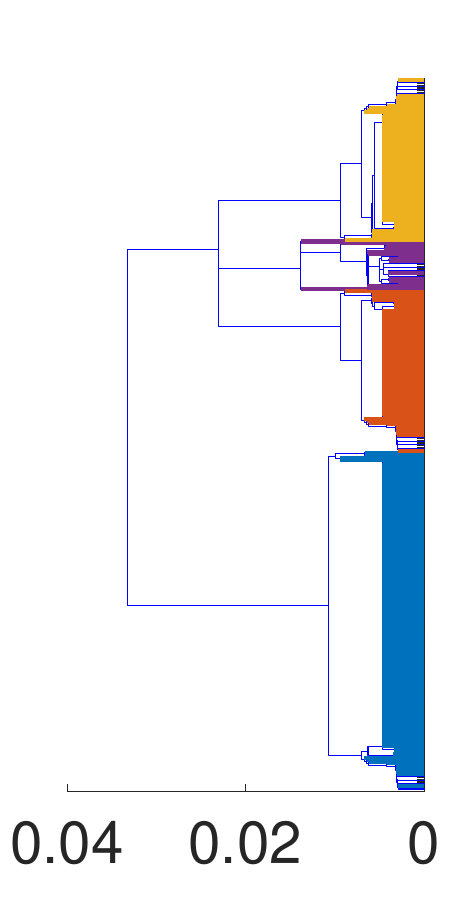}
            \end{minipage}
            \begin{minipage}[t][3cm][t]{0.1\textwidth}
                \centering
                LT-WT1 \\ 
                \includegraphics[width=0.8\textwidth]{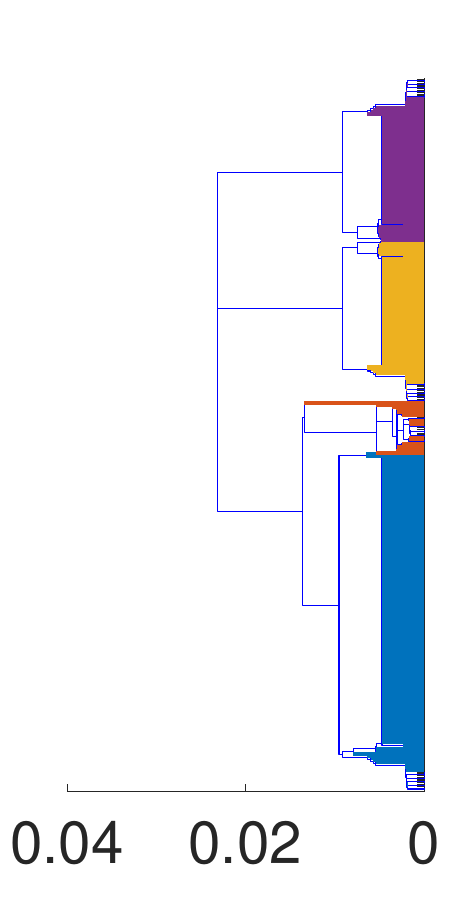}
            \end{minipage}     
    }
    \caption{\textbf{T-junction clustering.} 
     {
    (a): The first column shows the updated point cloud based on 1 iteration of MS; the next four columns show the 2D and 3D MDS plots of distance matrices obtained via GT with $\lambda$=1, GT with $\lambda$=5, LT-WT2 and LT-WT1 after 1 iteration, respectively. Different colors in (a) represent each of the different clusters which are obtained by slicing the single-linkage dendrograms illustrated in (b). 
    (b): The five columns show the single-linkage dendrograms based on distances generated via the methods in (a).
    (c): The first column shows the updated point cloud based on MS after 2 iterations; the next four columns show the 2D and 3D MDS plots of distance matrices based on GT with $\lambda$=1, GT with $\lambda$=5, LT-WT2 and LT-WT1 after 2 iterations, respectively. Different colors in (c) represent each of the different clusters which are obtained by slicing the dendrograms illustrated in (d). 
    (d): The five columns demonstrate the single-linkage dendrograms based on distances generated via methods in (c). 
    } }
    \label{fig:tlines-supp}
\end{figure}

\subsection{Ameliorating the chaining effect} 
We consider the clustering task on a data set with two blobs connected by a chain (i.e., a dumbbell). Each blob is composed of 300 uniform grid points in a circle of radius 1 and the chain is composed of 200 uniform grid points along a line segment with length 2. The single-linkage dendrogram cannot separate the two blobs as shown in \Cref{fig:ellip-iters}. This phenomenon is known as the chaining effect as mentioned in the introduction. We now show that the Wasserstein Transform  helps separate the two blobs by improving the underlying distance matrices of the data set, and as a consequence it improves the quality of single-linkage dendrograms throughout iterations. Here we only show the result of iterative GT (see \Cref{fig:ellip-iters}). We remark that LT-WT performs as well in the task of ameliorating chaining effect; see results of iterative LT-WT in \cite{memoli2019wasserstein} under slightly different settings of data sets.
\begin{figure}[htb]
    \centering
      \begin{minipage}[t][2.8cm][t]{0.08\textwidth}
          \centering
    		    $\tau=0$    
    		    \includegraphics[width=1\textwidth]{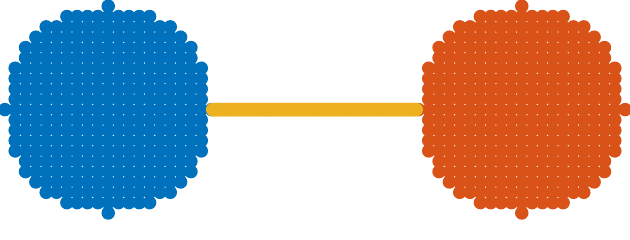} 
              \begin{minipage}[t][0.73cm][t]{0.08\textwidth}
              \centering
              
              \end{minipage}
              
             \includegraphics[width=1\textwidth]{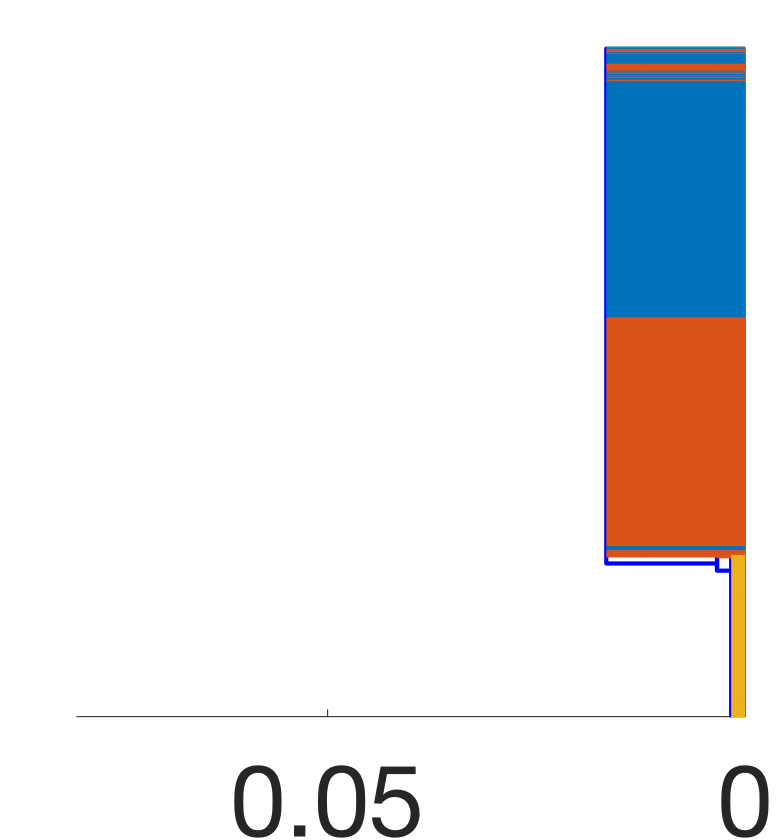}
          \end{minipage}
          \begin{minipage}[t]{0.08\textwidth}
          \centering
             $\tau=1$ 
            \includegraphics[width=1\textwidth]{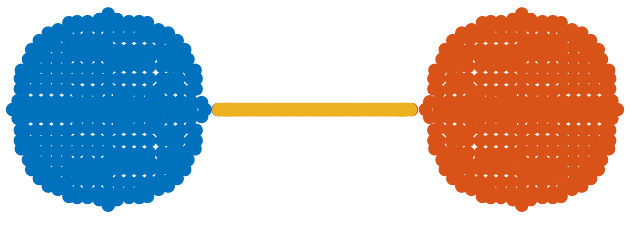}
               
             \includegraphics[width=1\textwidth]{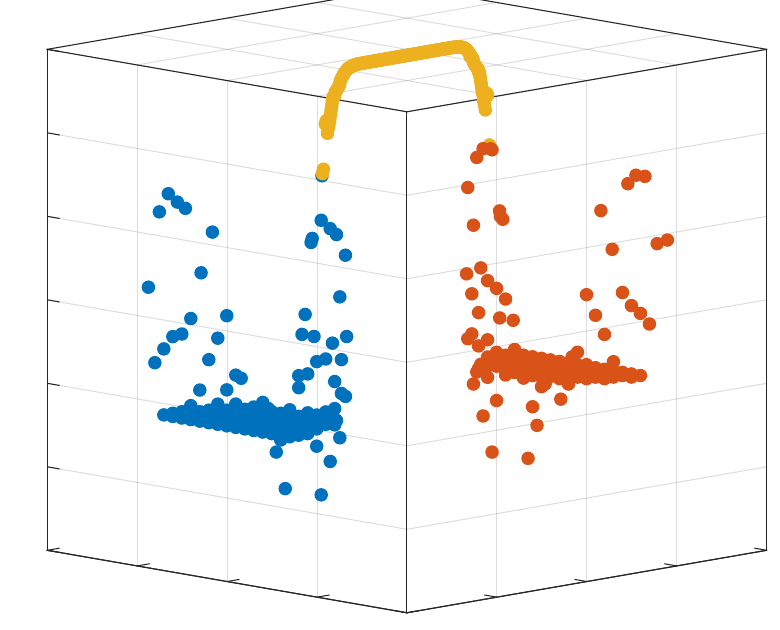}
              
             \includegraphics[width=1\textwidth]{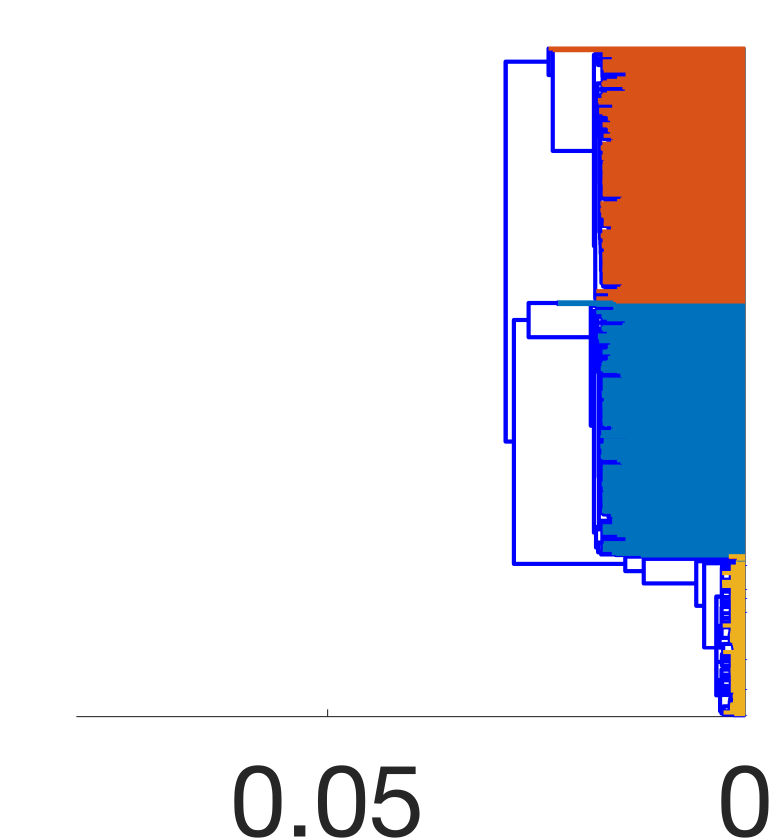}
          \end{minipage}
          \begin{minipage}[t]{0.08\textwidth}
          \centering
             $\tau=2$   
             \includegraphics[width=1\textwidth]{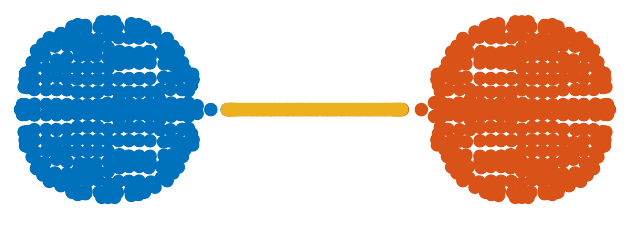}  
             \includegraphics[width=1\textwidth]{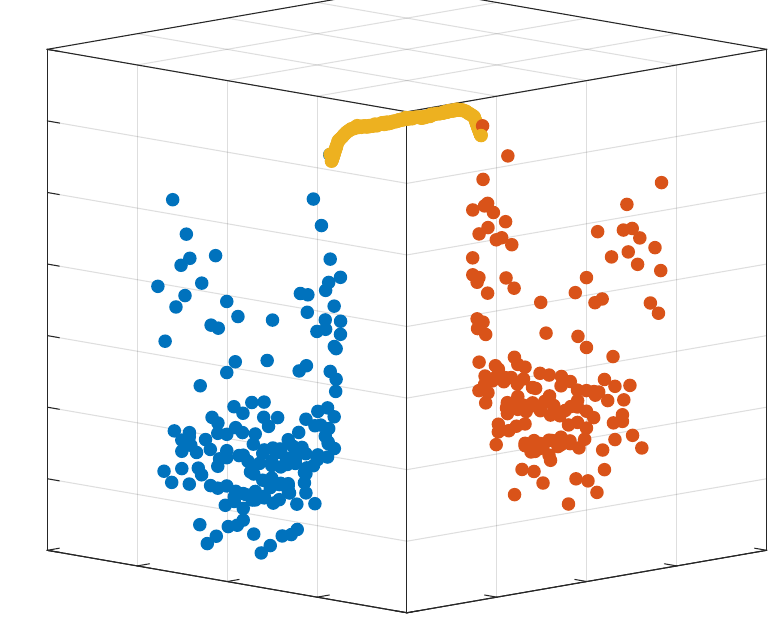} 
              
             \includegraphics[width=1\textwidth]{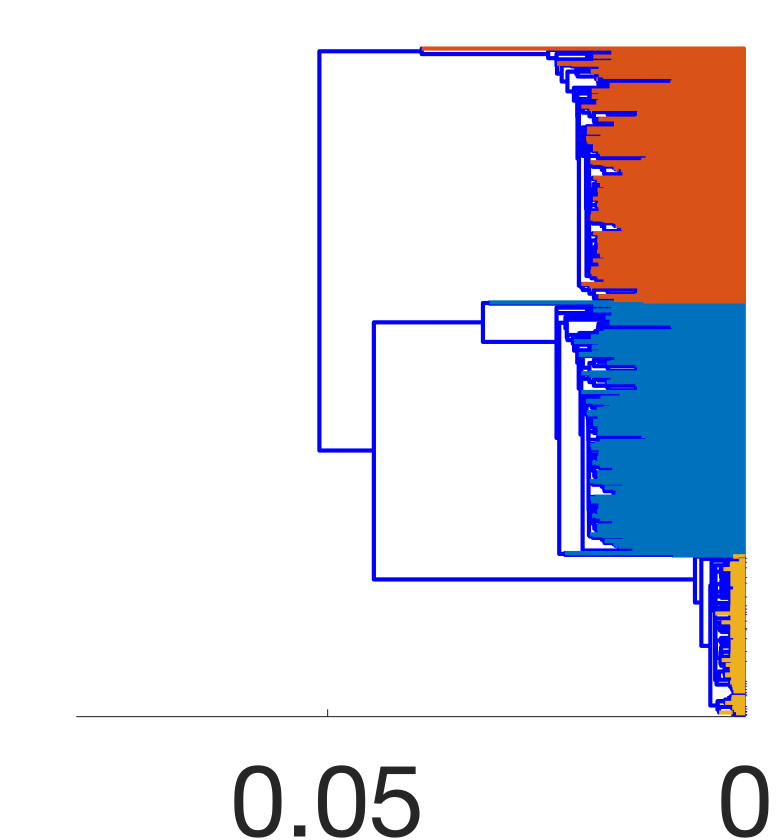}
          \end{minipage}
        \begin{minipage}[t]{0.08\textwidth}
          \centering
             $\tau=3$   
             \includegraphics[width=1\textwidth]{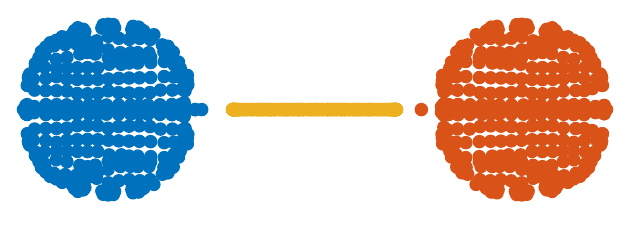}  
             
             \includegraphics[width=1\textwidth]{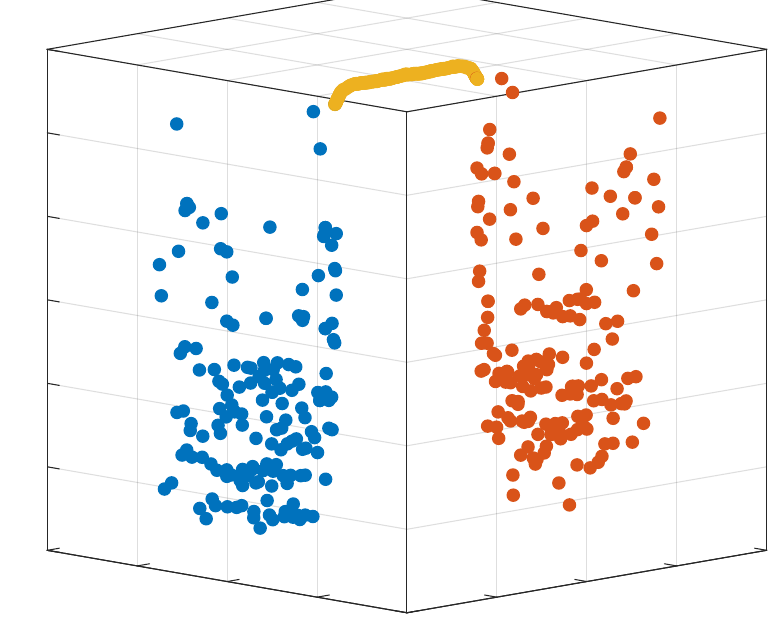}
             
             \includegraphics[width=1\textwidth]{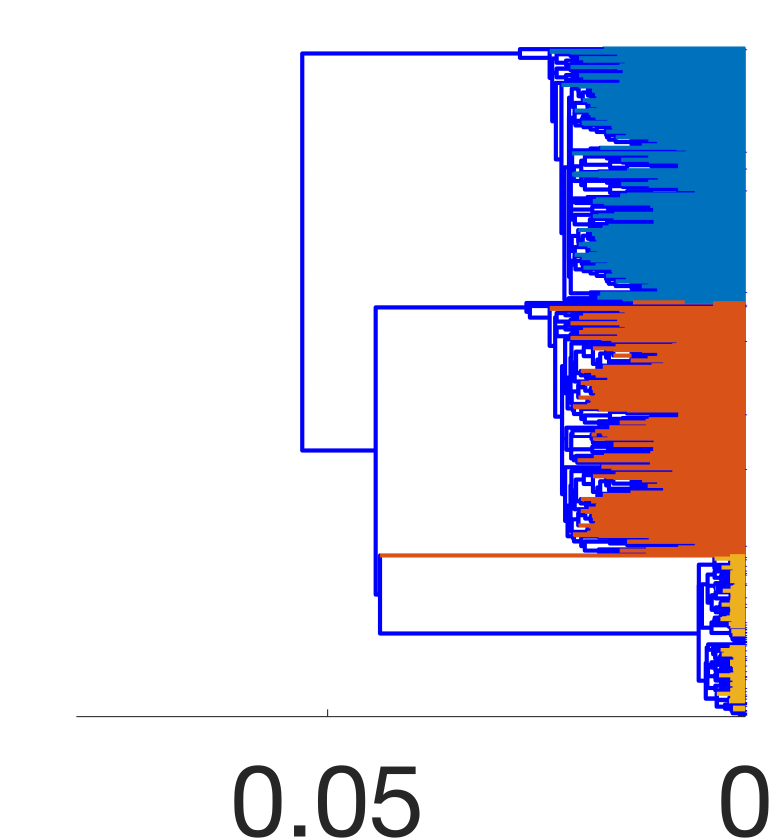}
          \end{minipage}
          \begin{minipage}[t]{0.08\textwidth}
          \centering
             $\tau=4$ 
             \includegraphics[width=1\textwidth]{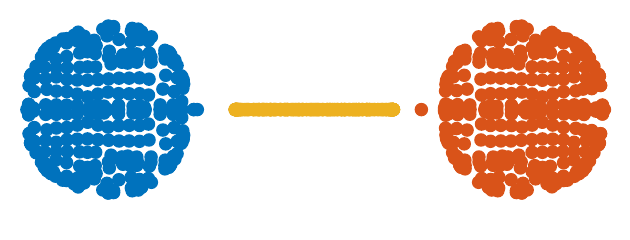}
             
             \includegraphics[width=1\textwidth]{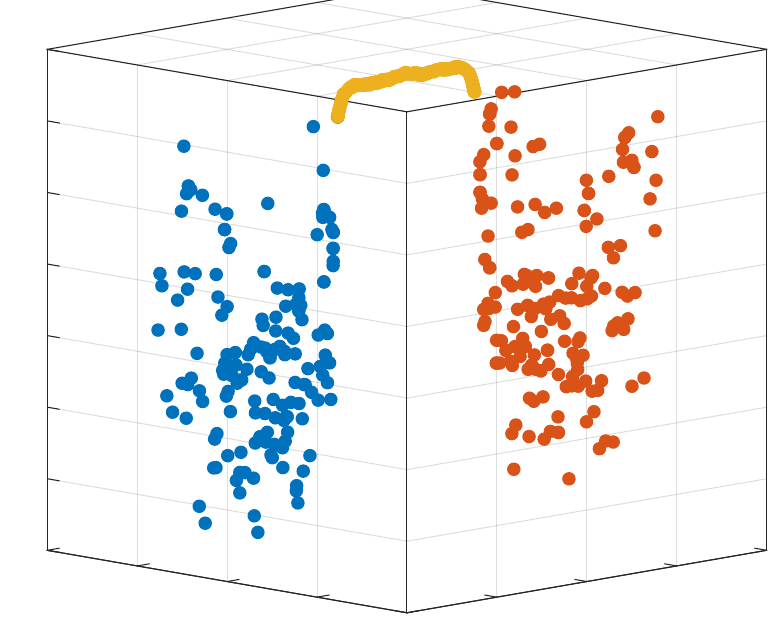}
             
             \includegraphics[width=1\textwidth]{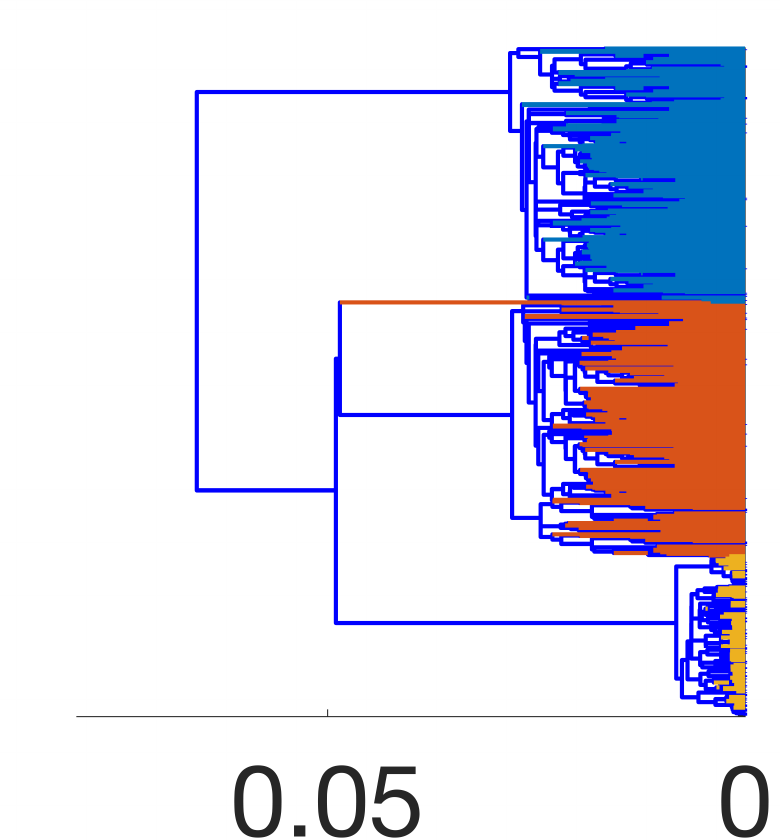}
          \end{minipage}  
    \caption{\textbf{GT and Chaining effect.} The top and middle rows respectively show 2D and 3D MDS plots after after successive GT iterations. The bottom row shows the corresponding single-linkage dendrograms. We set $\eps=0.2$ in this experiment.} 
    \label{fig:ellip-iters} 
\end{figure}

We then examine the performance of MS, LT-WT and GT in the task of ameliorating chaining effect from another perspective.
We apply linear transformations  $T=\begin{psmallmatrix}
a_1 & 0\\0 & a_2
\end{psmallmatrix}$ with different \emph{eccentricity} $e\coloneqq a_2/a_1$ to the abovementioned dumbbell shape data set to generate several new data sets. Then, we apply MS, LT-WT1, LT-WT2 and GT for 1 iteration to these newly obtained data sets and study how their geometry (reflected by the eccentricity in this case) influences the single-linkage clustering structure induced by these methods. 
Results are shown in Figure~(\ref{fig:supp-ballchains}). We notice that whereas MS, LT-WT1, LT-WT2 and GT-1 all induce similar single-linkage dendrograms, in the extreme condition when $e\coloneqq 1/0.2$, GT-5 outperforms all other methods in separating two blobs from the chain connecting them. 
\begin{figure}[htb]
    \centering
    \subfloat[MS]{
    \fbox{
        \begin{minipage}[t]{0.09\textwidth}
            \centering
            \begin{minipage}[t]{1\textwidth}
                \centering
                $e$:1/1 \\ 
                \includegraphics[width=1\textwidth]{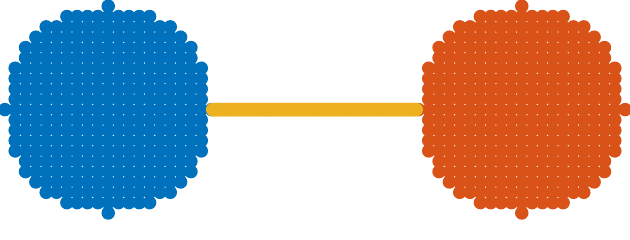} 
            \end{minipage}
            \begin{minipage}[t][0.55cm][t]{1\textwidth}
                \centering
                \includegraphics[width=1\textwidth]{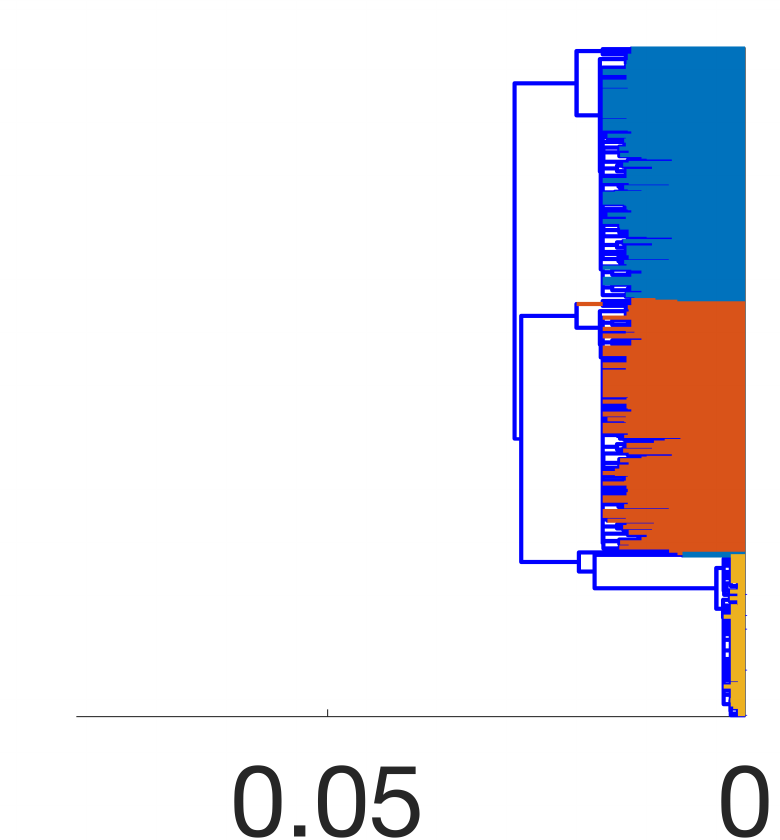}
            \end{minipage}
            \begin{minipage}[t]{1\textwidth}
                \centering
                $e$:1/1 \\ 
                \includegraphics[width=1\textwidth]{figures/ellip/ellip-ori-seq-1-10.pdf} 
            \end{minipage}
            \begin{minipage}[t][0.55cm][t]{1\textwidth}
                \centering
                \includegraphics[width=1\textwidth]{figures/ellip/ms-ellip-seq-dend-1-10.pdf}
            \end{minipage}
        \end{minipage}
        \begin{minipage}[t]{0.09\textwidth}
            \centering
            \begin{minipage}[t]{1\textwidth}
                \centering
                $e$:1/0.8 \\ 
                \includegraphics[width=1\textwidth]{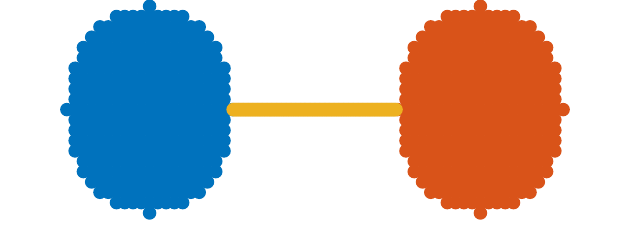} 
            \end{minipage}
            \begin{minipage}[t][0.55cm][t]{1\textwidth}
                \centering
                \includegraphics[width=1\textwidth]{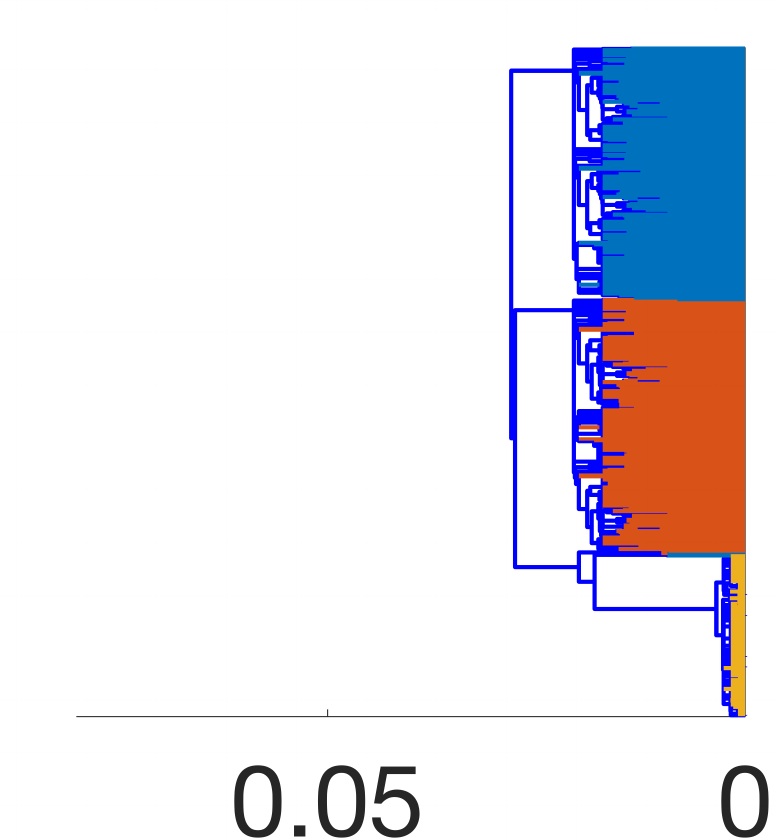}
            \end{minipage}
            \begin{minipage}[t]{1\textwidth}
                \centering
                $e$:0.8/1 \\ 
                \includegraphics[width=1\textwidth]{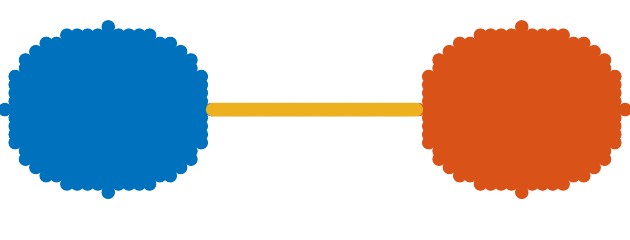} 
            \end{minipage}
            \begin{minipage}[t][0.55cm][t]{1\textwidth}
                \centering
                \includegraphics[width=1\textwidth]{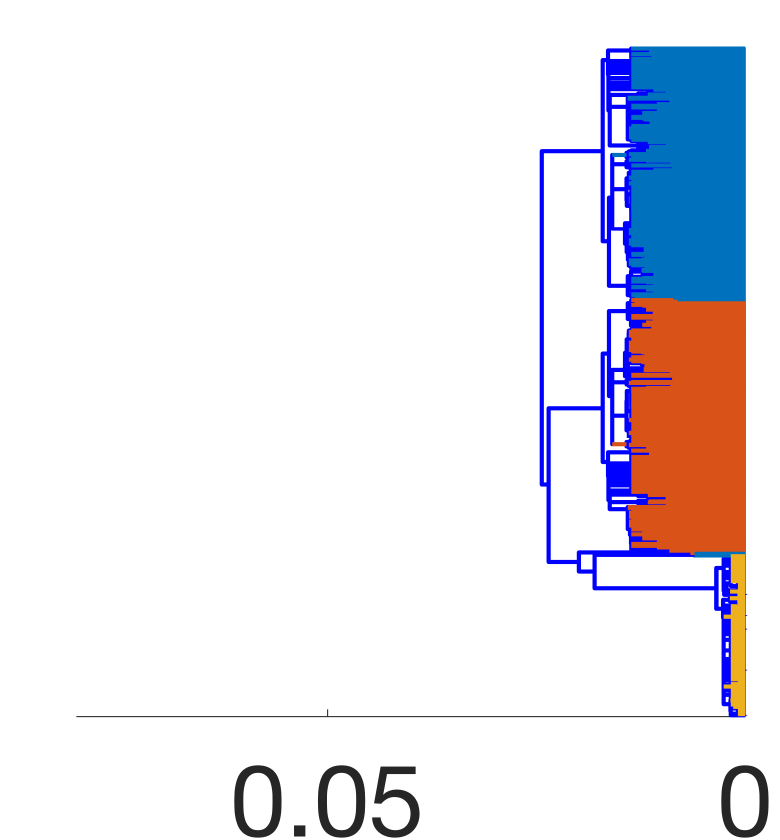}
            \end{minipage}
        \end{minipage}
        \begin{minipage}[t]{0.09\textwidth}
            \centering
            \begin{minipage}[t]{1\textwidth}
                \centering
                $e$:1/0.6 \\ 
                \includegraphics[width=1\textwidth]{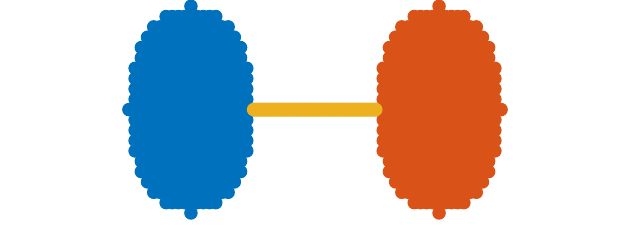} 
            \end{minipage}
            \begin{minipage}[t][0.55cm][t]{1\textwidth}
                \centering
                \includegraphics[width=1\textwidth]{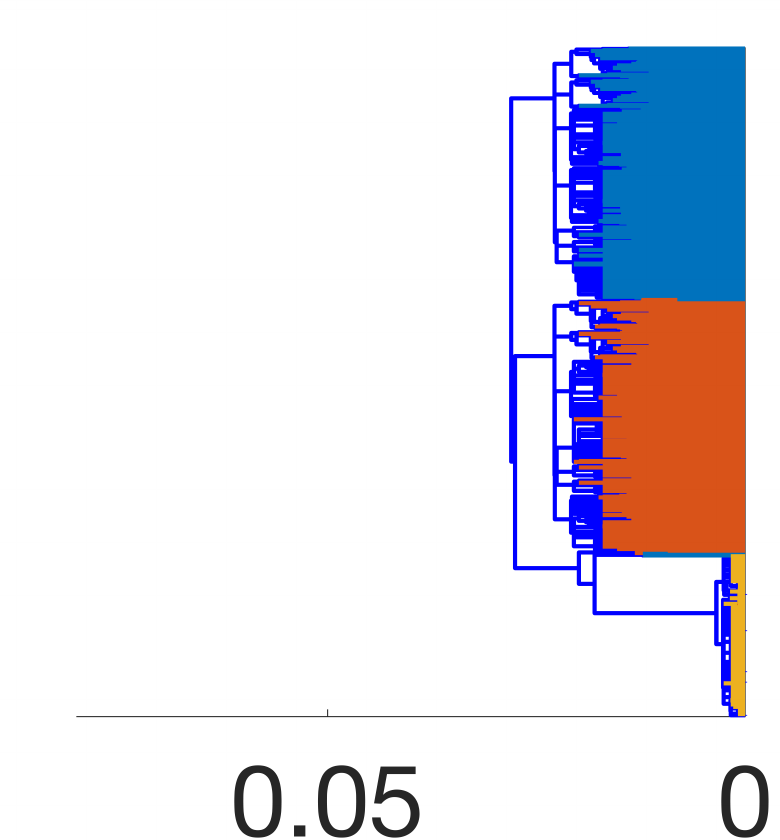}
            \end{minipage}
            \begin{minipage}[t]{1\textwidth}
                \centering
                $e$:0.6/1 \\ 
                \includegraphics[width=1\textwidth]{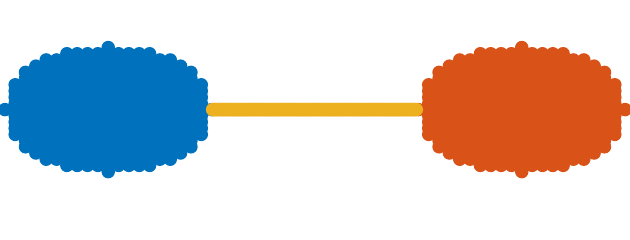} 
            \end{minipage}
            \begin{minipage}[t][0.55cm][t]{1\textwidth}
                \centering
                \includegraphics[width=1\textwidth]{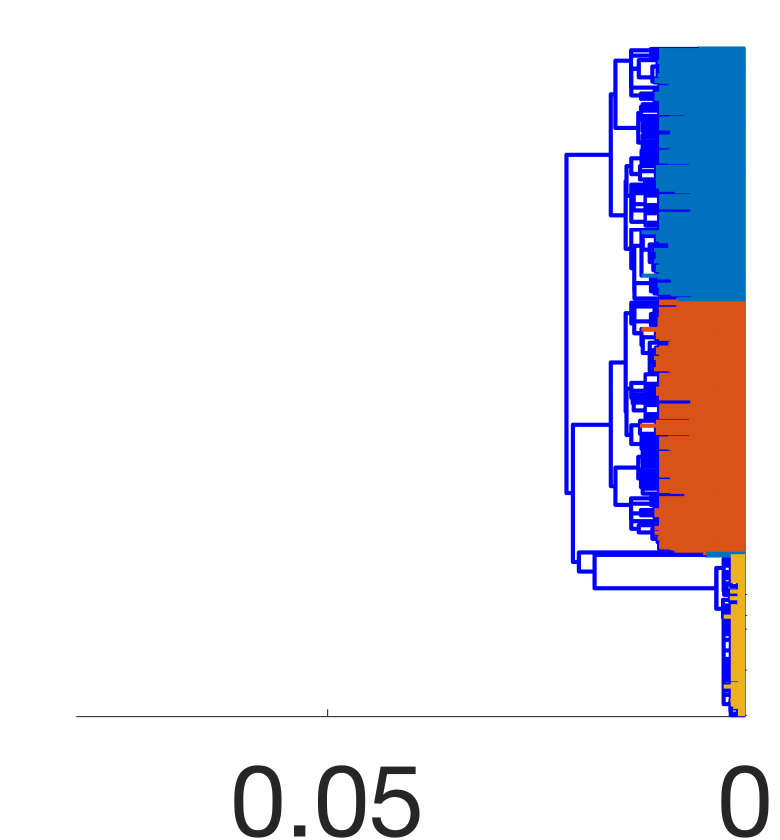}
            \end{minipage}
        \end{minipage}
        \begin{minipage}[t]{0.09\textwidth}
            \centering
            \begin{minipage}[t]{1\textwidth}
                \centering
                $e$:1/0.4 \\ 
                \includegraphics[width=1\textwidth]{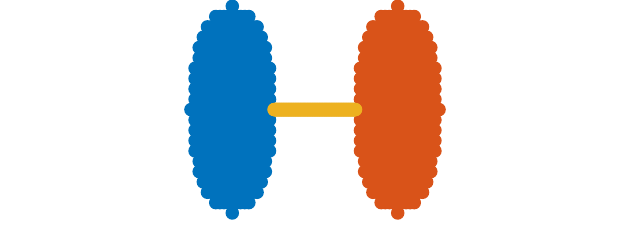} 
            \end{minipage}
            \begin{minipage}[t][0.55cm][t]{1\textwidth}
                \centering
                \includegraphics[width=1\textwidth]{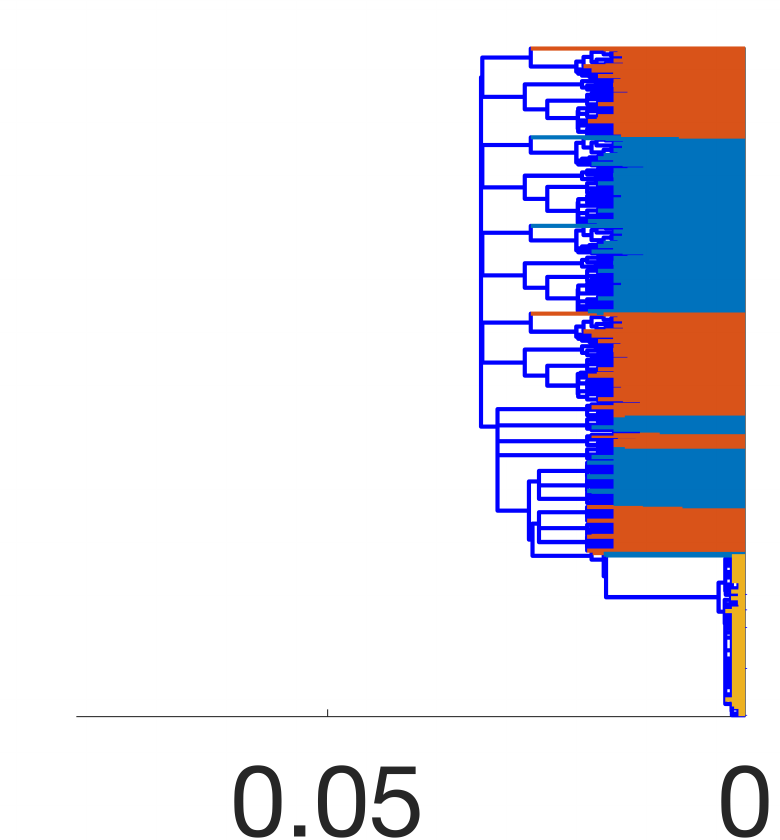}
            \end{minipage}
            \begin{minipage}[t]{1\textwidth}
                \centering
                $e$:0.4/1 \\ 
                \includegraphics[width=1\textwidth]{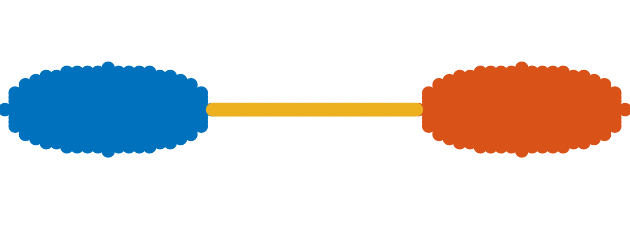} 
            \end{minipage}
            \begin{minipage}[t][0.55cm][t]{1\textwidth}
                \centering
                \includegraphics[width=1\textwidth]{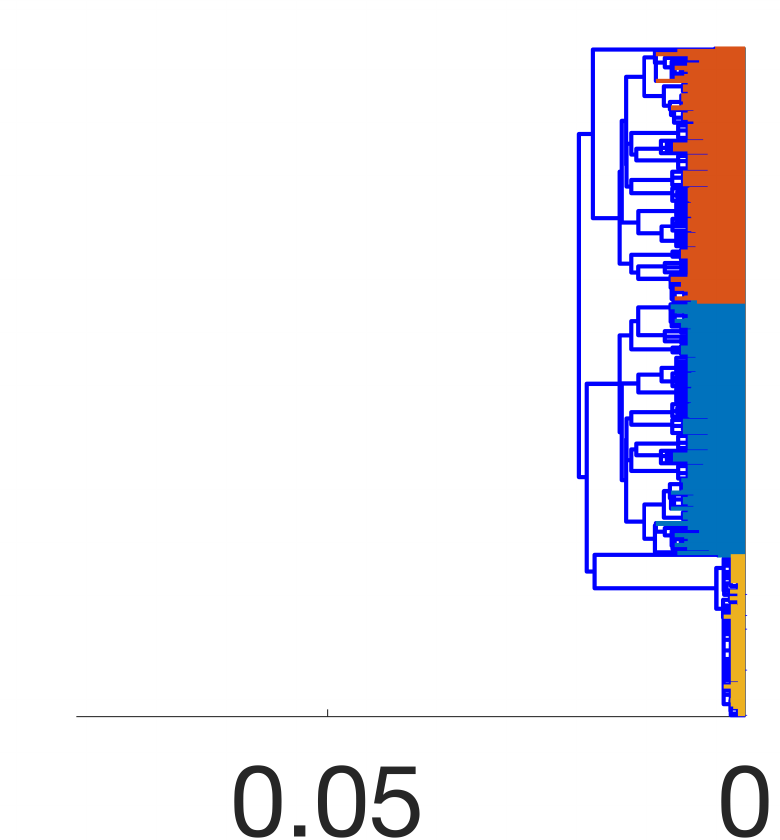}
            \end{minipage}
        \end{minipage}
        \begin{minipage}[t]{0.09\textwidth}
            \centering
            \begin{minipage}[t]{1\textwidth}
                \centering
                $e$:1/0.2 \\ 
                \includegraphics[width=1\textwidth]{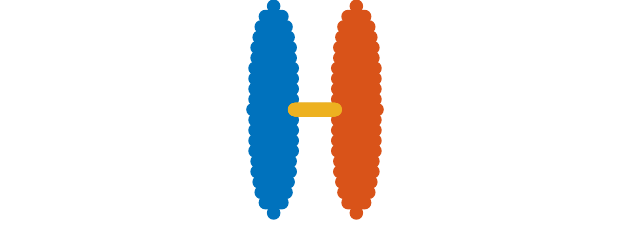} 
            \end{minipage}
            \begin{minipage}[t][0.55cm][t]{1\textwidth}
                \centering
                \includegraphics[width=1\textwidth]{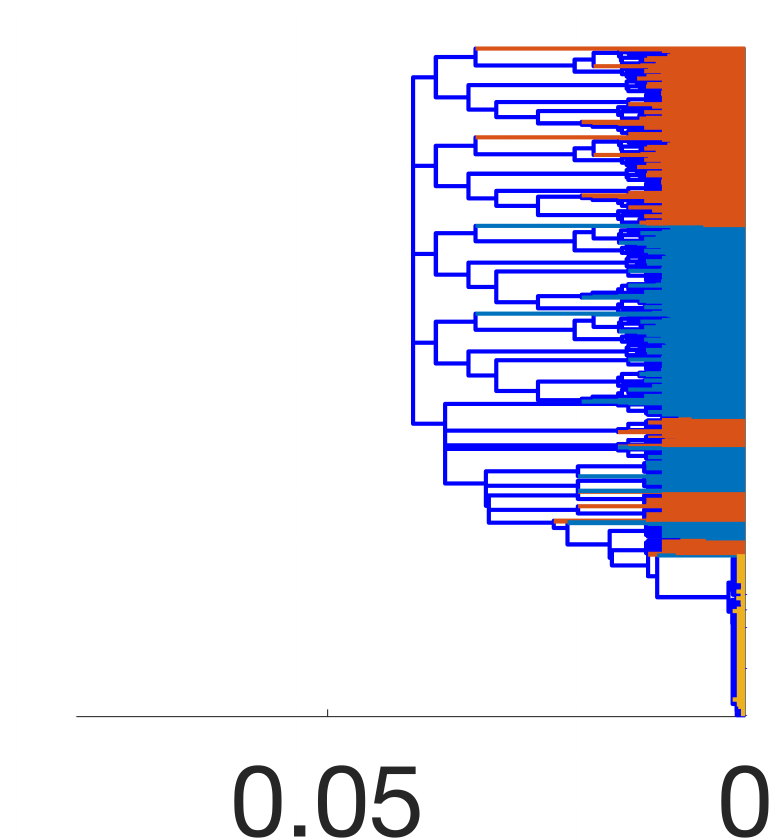}
            \end{minipage}
            \begin{minipage}[t]{1\textwidth}
                \centering
                $e$:0.2/1 \\ 
                \includegraphics[width=1\textwidth]{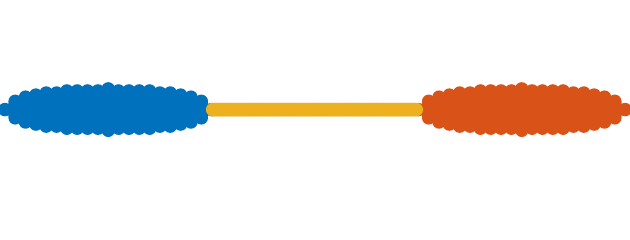} 
            \end{minipage}
            \begin{minipage}[t][0.55cm][t]{1\textwidth}
                \centering
                \includegraphics[width=1\textwidth]{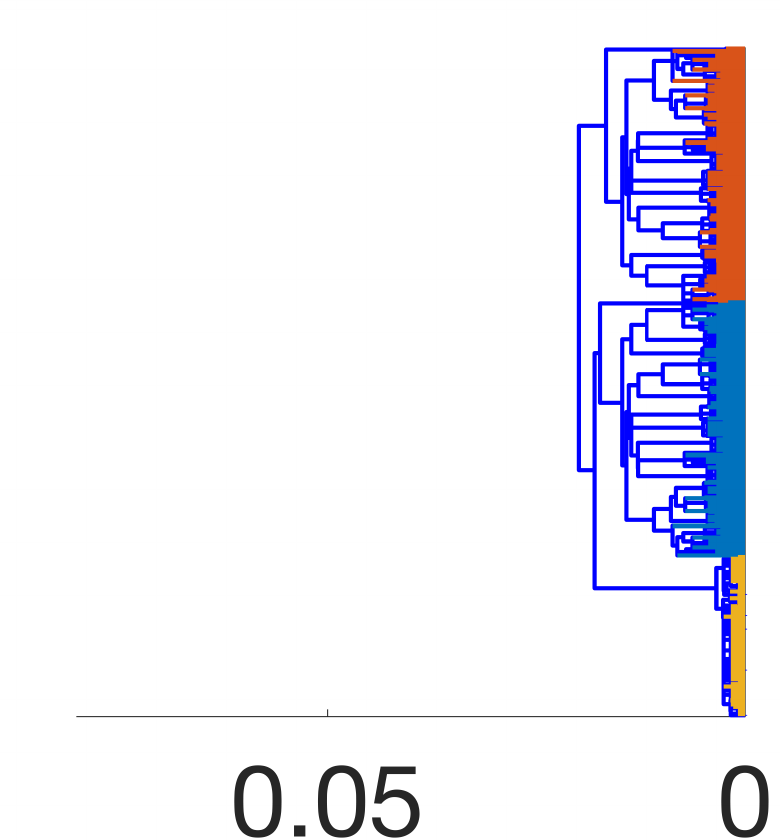}
            \end{minipage}
        \end{minipage}
    } } 
    
    \subfloat[GT-1]{
    \label{fig:ballchains-gt1}
    \fbox{
        \begin{minipage}[t]{0.09\textwidth}
            \centering
            \begin{minipage}[t]{1\textwidth}
                \centering
                $e$:1/1 \\ 
                \includegraphics[width=1\textwidth]{figures/ellip/ellip-ori-seq-1-10.pdf} 
            \end{minipage}
            \begin{minipage}[t][0.55cm][t]{1\textwidth}
                \centering
                \includegraphics[width=1\textwidth]{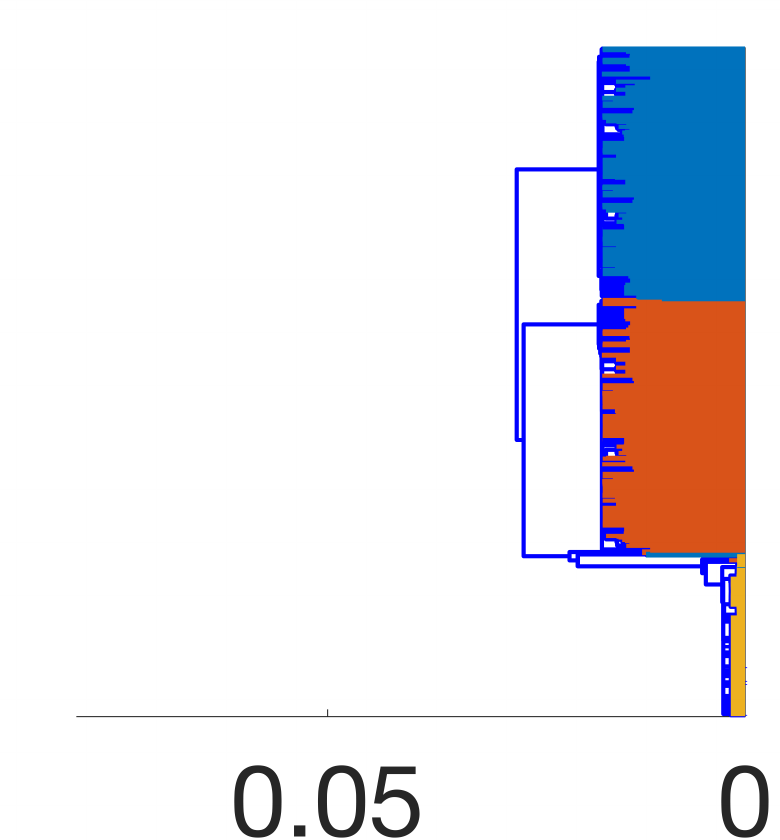}
            \end{minipage}
            \begin{minipage}[t]{1\textwidth}
                \centering
                $e$:1/1 \\ 
                \includegraphics[width=1\textwidth]{figures/ellip/ellip-ori-seq-1-10.pdf} 
            \end{minipage}
            \begin{minipage}[t][0.55cm][t]{1\textwidth}
                \centering
                \includegraphics[width=1\textwidth]{figures/ellip/gtv-ellip-lamb-1-seq-dend-1-10.pdf}
            \end{minipage}
        \end{minipage}
        \begin{minipage}[t]{0.09\textwidth}
            \centering
            \begin{minipage}[t]{1\textwidth}
                \centering
                $e$:1/0.8 \\ 
                \includegraphics[width=1\textwidth]{figures/ellip/ellip-ori-seq-1-8.pdf} 
            \end{minipage}
            \begin{minipage}[t][0.55cm][t]{1\textwidth}
                \centering
                \includegraphics[width=1\textwidth]{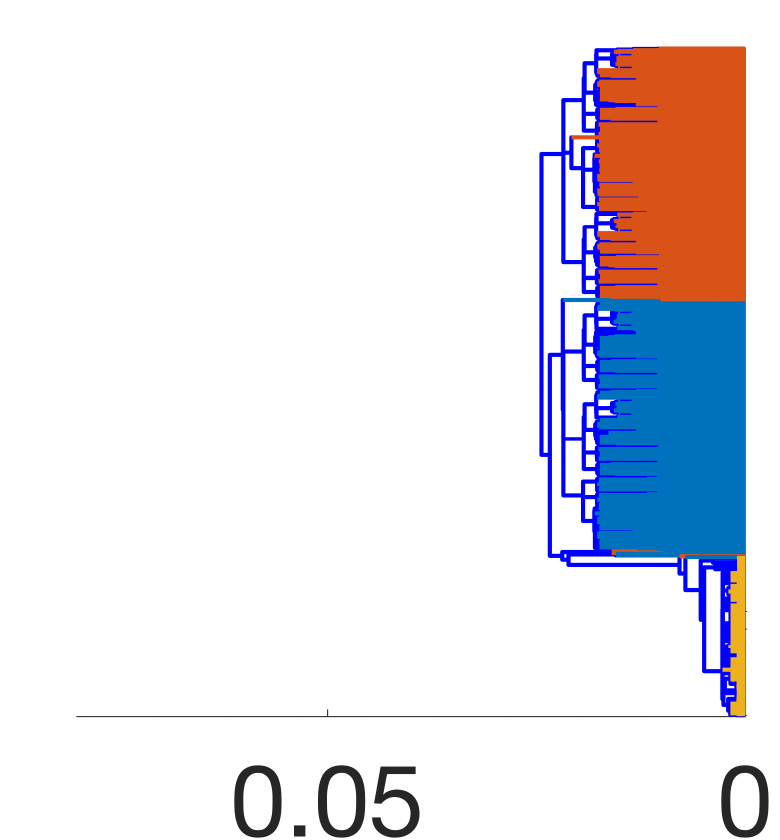}
            \end{minipage}
            \begin{minipage}[t]{1\textwidth}
                \centering
                $e$:0.8/1 \\ 
                \includegraphics[width=1\textwidth]{figures/ellip/ellip-ori-seq-8-1.pdf} 
            \end{minipage}
            \begin{minipage}[t][0.55cm][t]{1\textwidth}
                \centering
                \includegraphics[width=1\textwidth]{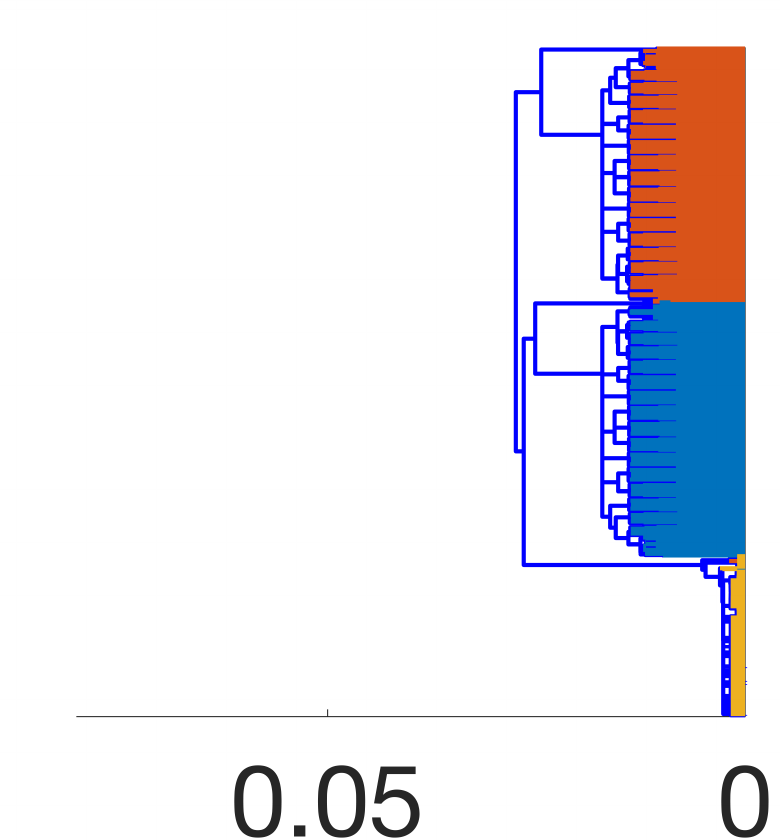}
            \end{minipage}
        \end{minipage}
        \begin{minipage}[t]{0.09\textwidth}
            \centering
            \begin{minipage}[t]{1\textwidth}
                \centering
                $e$:1/0.6 \\ 
                \includegraphics[width=1\textwidth]{figures/ellip/ellip-ori-seq-1-6.pdf} 
            \end{minipage}
            \begin{minipage}[t][0.55cm][t]{1\textwidth}
                \centering
                \includegraphics[width=1\textwidth]{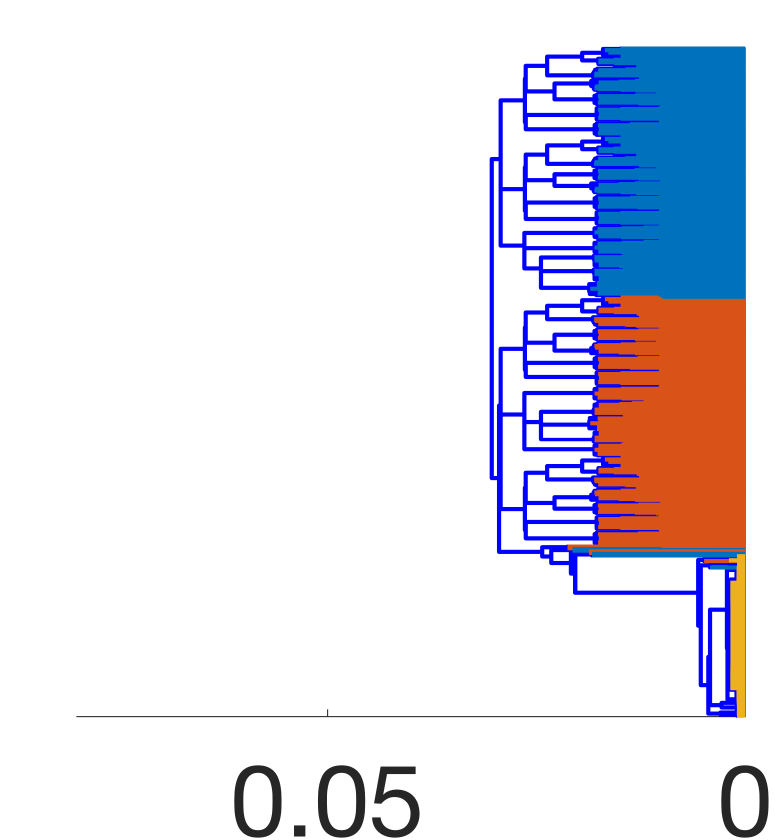}
            \end{minipage}
            \begin{minipage}[t]{1\textwidth}
                \centering
                $e$:0.6/1 \\ 
                \includegraphics[width=1\textwidth]{figures/ellip/ellip-ori-seq-6-1.pdf} 
            \end{minipage}
            \begin{minipage}[t][0.55cm][t]{1\textwidth}
                \centering
                \includegraphics[width=1\textwidth]{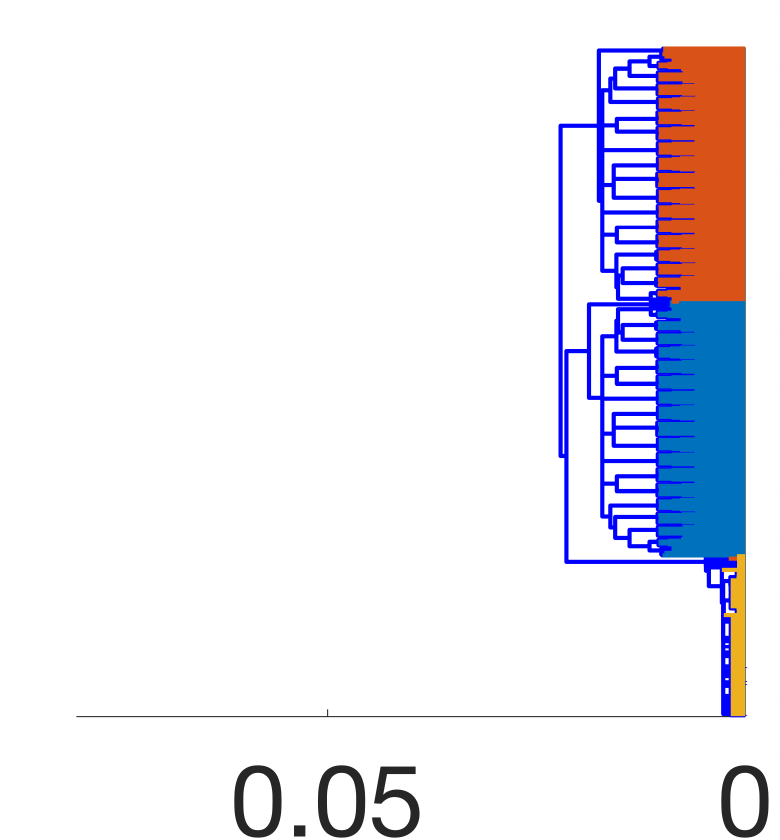}
            \end{minipage}
        \end{minipage}
        \begin{minipage}[t]{0.09\textwidth}
            \centering
            \begin{minipage}[t]{1\textwidth}
                \centering
                $e$:1/0.4 \\ 
                \includegraphics[width=1\textwidth]{figures/ellip/ellip-ori-seq-1-4.pdf} 
            \end{minipage}
            \begin{minipage}[t][0.55cm][t]{1\textwidth}
                \centering
                \includegraphics[width=1\textwidth]{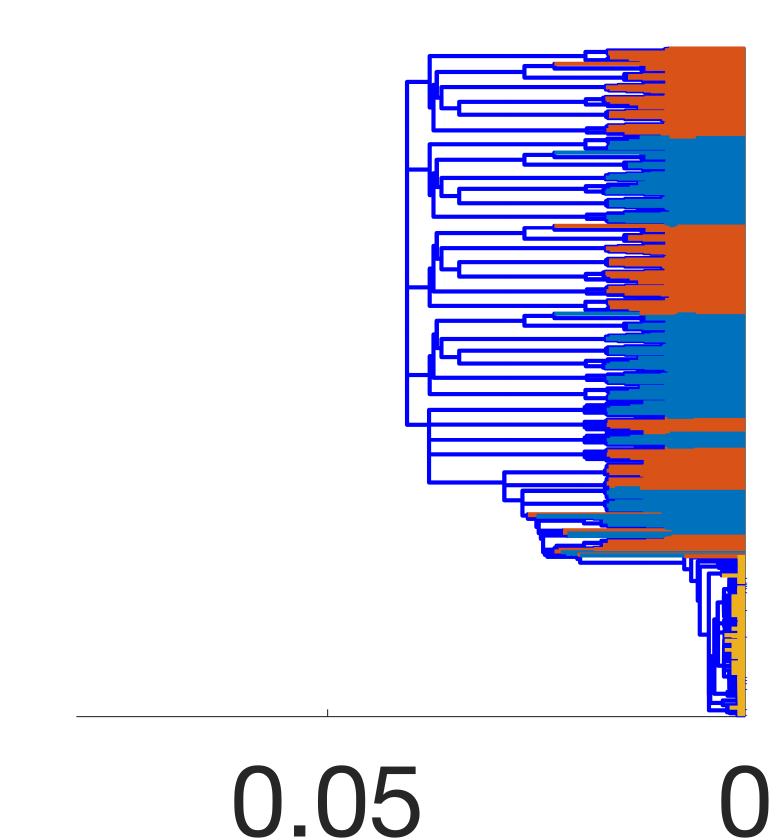}
            \end{minipage}
            \begin{minipage}[t]{1\textwidth}
                \centering
                $e$:0.4/1 \\ 
                \includegraphics[width=1\textwidth]{figures/ellip/ellip-ori-seq-4-1.pdf} 
            \end{minipage}
            \begin{minipage}[t][0.55cm][t]{1\textwidth}
                \centering
                \includegraphics[width=1\textwidth]{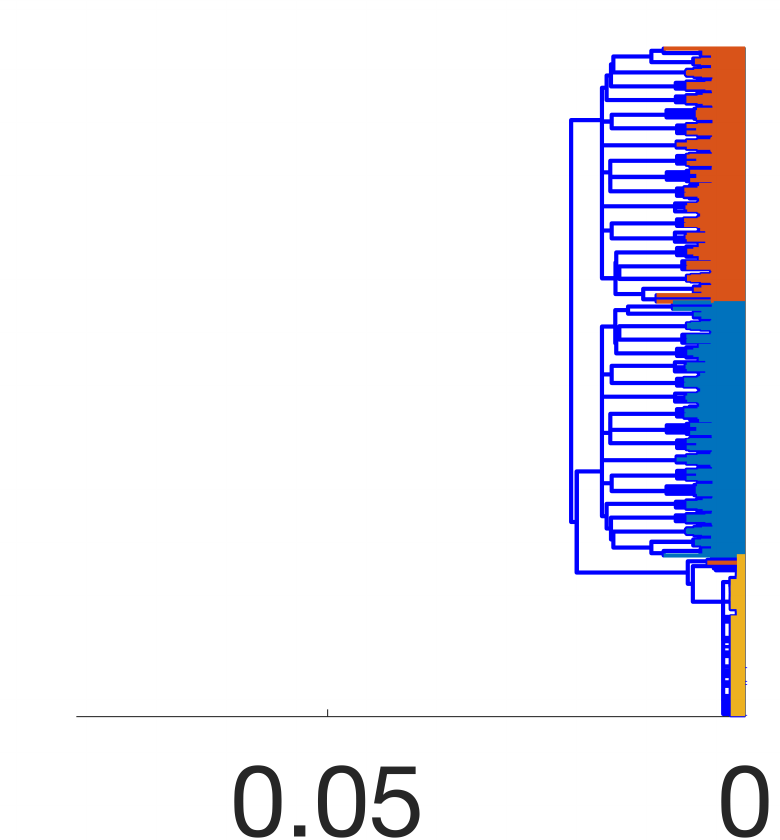}
            \end{minipage}
        \end{minipage}
        \begin{minipage}[t]{0.09\textwidth}
            \centering
            \begin{minipage}[t]{1\textwidth}
                \centering
                $e$:1/0.2 \\ 
                \includegraphics[width=1\textwidth]{figures/ellip/ellip-ori-seq-1-2.pdf} 
            \end{minipage}
            \begin{minipage}[t][0.55cm][t]{1\textwidth}
                \centering
                \includegraphics[width=1\textwidth]{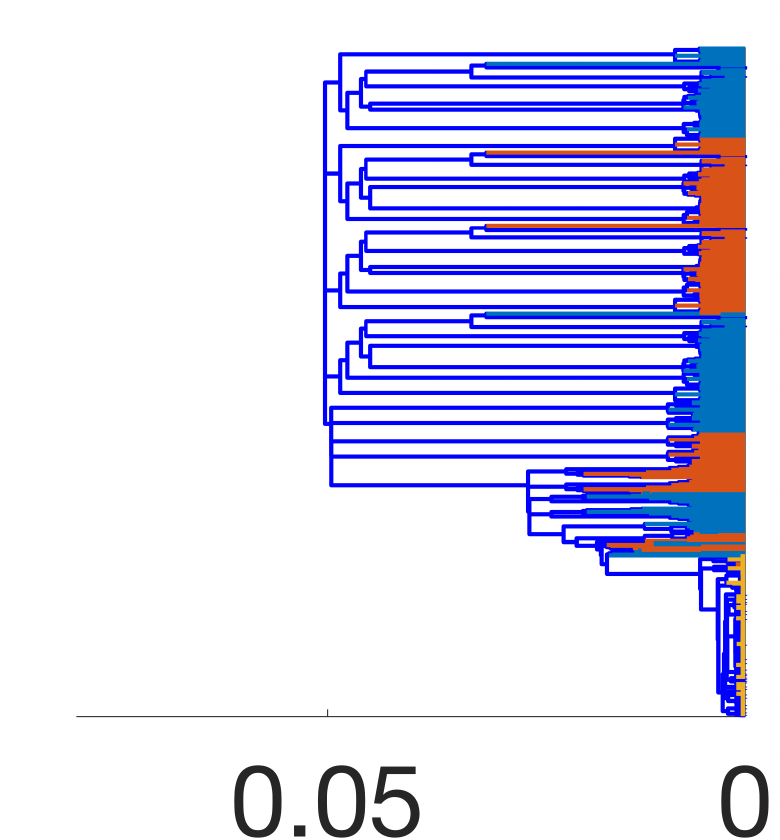}
            \end{minipage}
            \begin{minipage}[t]{1\textwidth}
                \centering
                $e$:0.2/1 \\ 
                \includegraphics[width=1\textwidth]{figures/ellip/ellip-ori-seq-2-1.pdf} 
            \end{minipage}
            \begin{minipage}[t][0.55cm][t]{1\textwidth}
                \centering
                \includegraphics[width=1\textwidth]{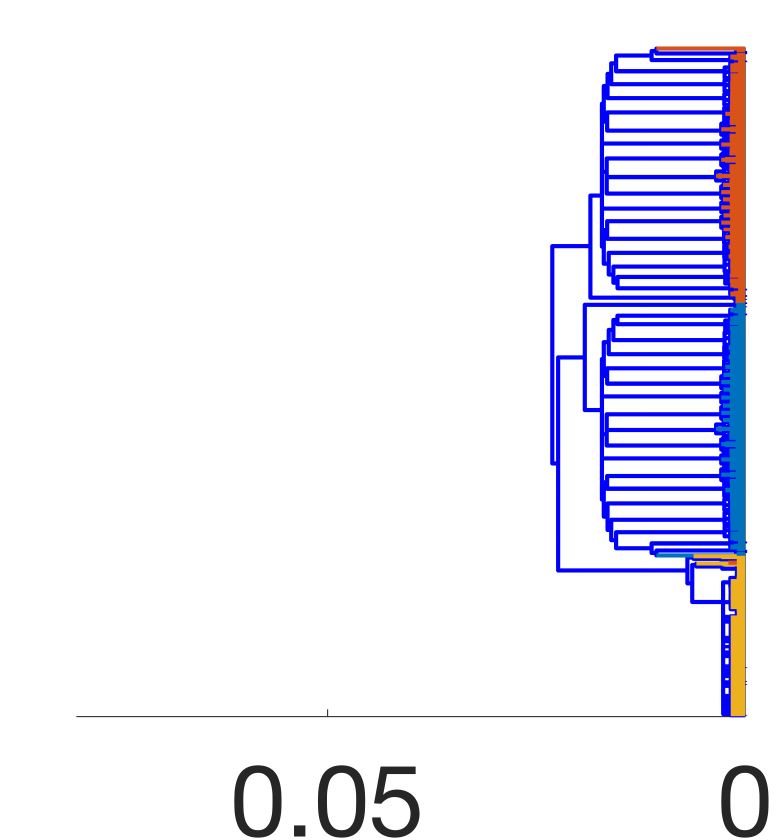}
            \end{minipage}
        \end{minipage}
    } }
    \subfloat[GT-5]{
    \label{fig:ballchains-gt5}
    \fbox{
        \begin{minipage}[t]{0.09\textwidth}
            \centering
            \begin{minipage}[t]{1\textwidth}
                \centering
                $e$:1/1 \\ 
                \includegraphics[width=1\textwidth]{figures/ellip/ellip-ori-seq-1-10.pdf} 
            \end{minipage}
            \begin{minipage}[t][0.55cm][t]{1\textwidth}
                \centering
                \includegraphics[width=1\textwidth]{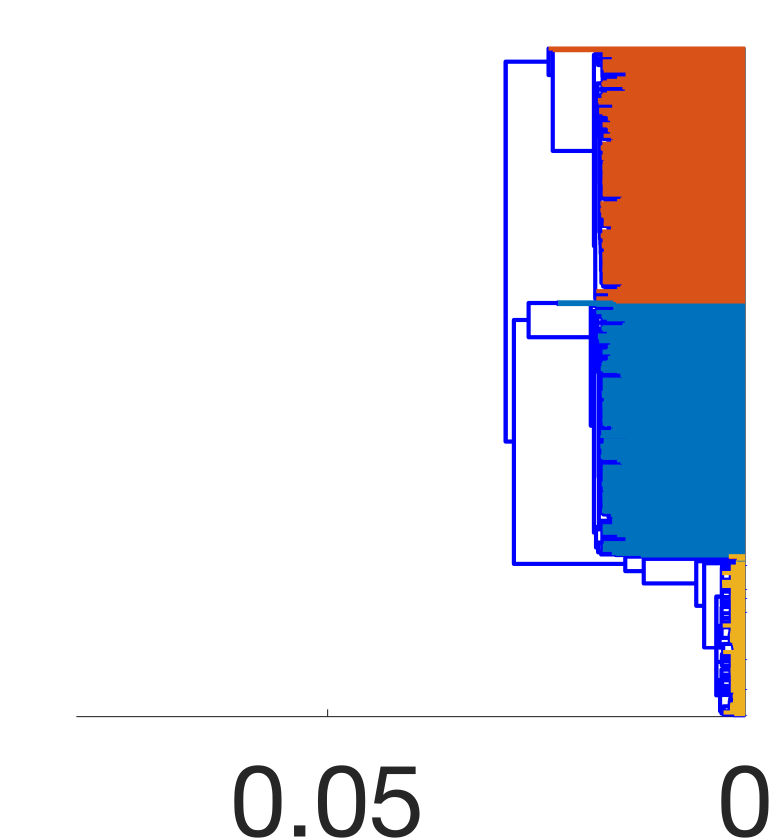}
            \end{minipage}
            \begin{minipage}[t]{1\textwidth}
                \centering
                $e$:1/1 \\ 
                \includegraphics[width=1\textwidth]{figures/ellip/ellip-ori-seq-1-10.pdf} 
            \end{minipage}
            \begin{minipage}[t][0.55cm][t]{1\textwidth}
                \centering
                \includegraphics[width=1\textwidth]{figures/ellip/gtv-ellip-lamb-5-seq-dend-1-10.pdf}
            \end{minipage}
        \end{minipage}
        \begin{minipage}[t]{0.09\textwidth}
            \centering
            \begin{minipage}[t]{1\textwidth}
                \centering
                $e$:1/0.8 \\ 
                \includegraphics[width=1\textwidth]{figures/ellip/ellip-ori-seq-1-8.pdf} 
            \end{minipage}
            \begin{minipage}[t][0.55cm][t]{1\textwidth}
                \centering
                \includegraphics[width=1\textwidth]{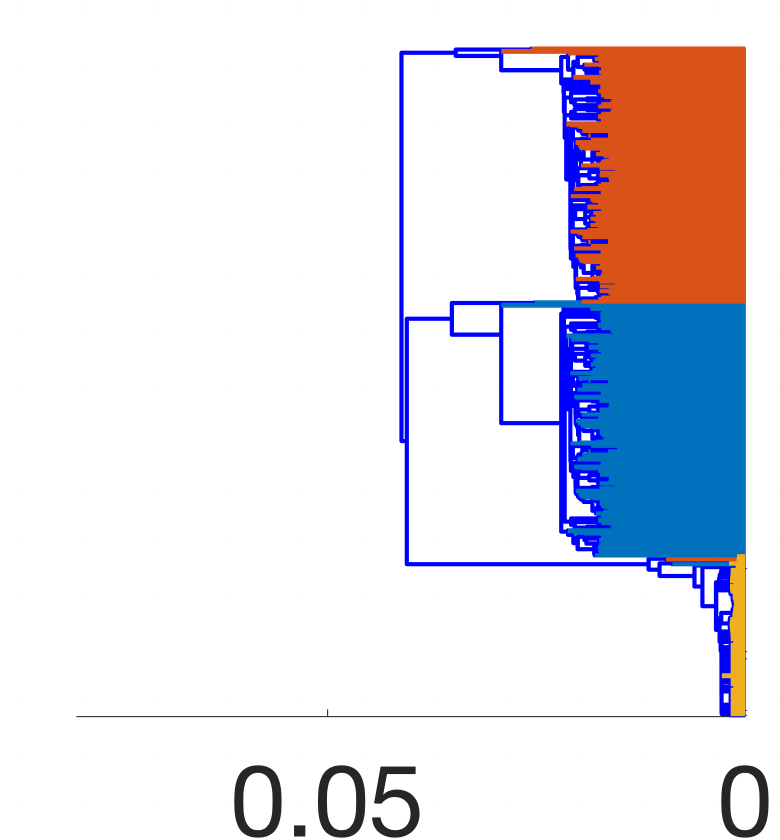}
            \end{minipage}
            \begin{minipage}[t]{1\textwidth}
                \centering
                $e$:0.8/1 \\ 
                \includegraphics[width=1\textwidth]{figures/ellip/ellip-ori-seq-8-1.pdf} 
            \end{minipage}
            \begin{minipage}[t][0.55cm][t]{1\textwidth}
                \centering
                \includegraphics[width=1\textwidth]{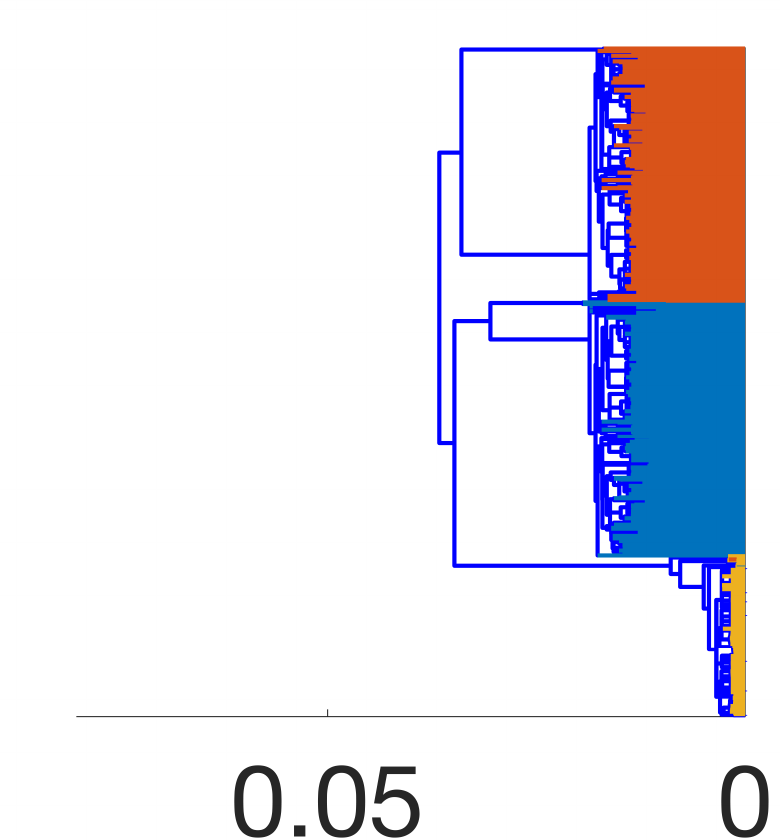}
            \end{minipage}
        \end{minipage}
        \begin{minipage}[t]{0.09\textwidth}
            \centering
            \begin{minipage}[t]{1\textwidth}
                \centering
                $e$:1/0.6 \\ 
                \includegraphics[width=1\textwidth]{figures/ellip/ellip-ori-seq-1-6.pdf} 
            \end{minipage}
            \begin{minipage}[t][0.55cm][t]{1\textwidth}
                \centering
                \includegraphics[width=1\textwidth]{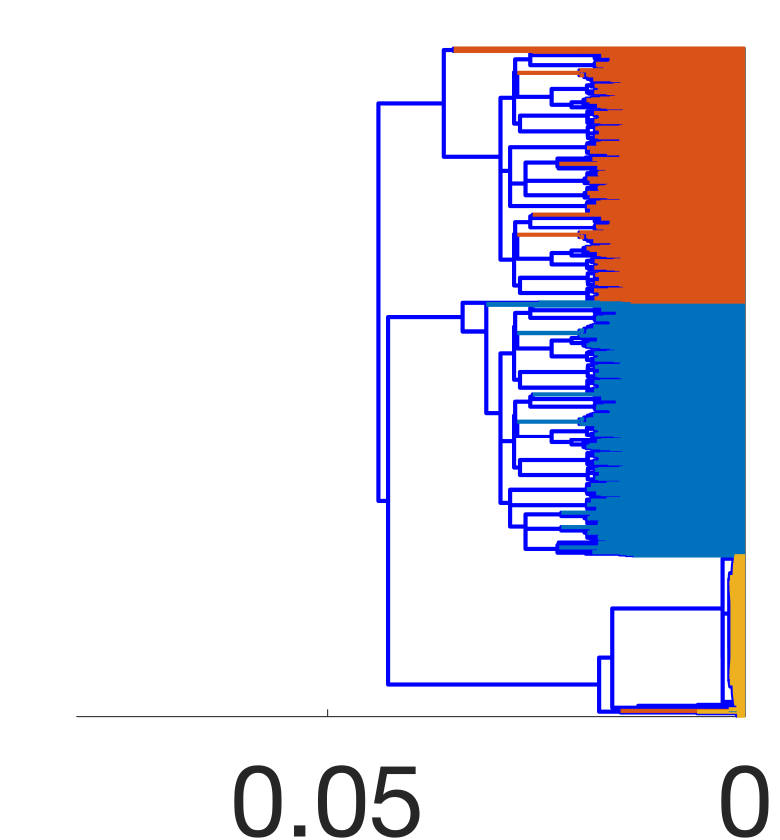}
            \end{minipage}
            \begin{minipage}[t]{1\textwidth}
                \centering
                $e$:0.6/1 \\ 
                \includegraphics[width=1\textwidth]{figures/ellip/ellip-ori-seq-6-1.pdf} 
            \end{minipage}
            \begin{minipage}[t][0.55cm][t]{1\textwidth}
                \centering
                \includegraphics[width=1\textwidth]{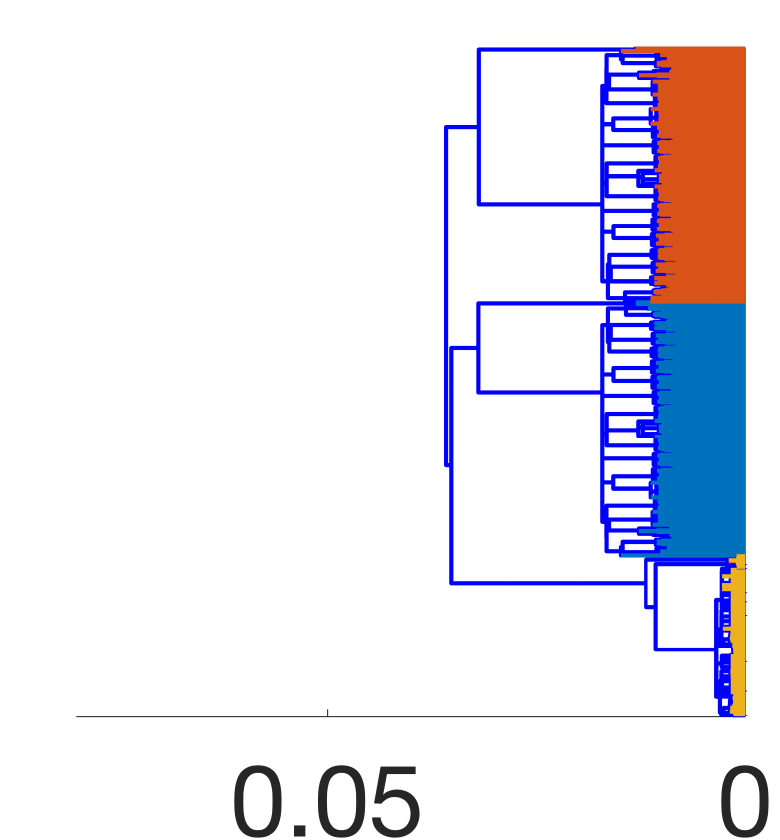}
            \end{minipage}
        \end{minipage}
        \begin{minipage}[t]{0.09\textwidth}
            \centering
            \begin{minipage}[t]{1\textwidth}
                \centering
                $e$:1/0.4 \\ 
                \includegraphics[width=1\textwidth]{figures/ellip/ellip-ori-seq-1-4.pdf} 
            \end{minipage}
            \begin{minipage}[t][0.55cm][t]{1\textwidth}
                \centering
                \includegraphics[width=1\textwidth]{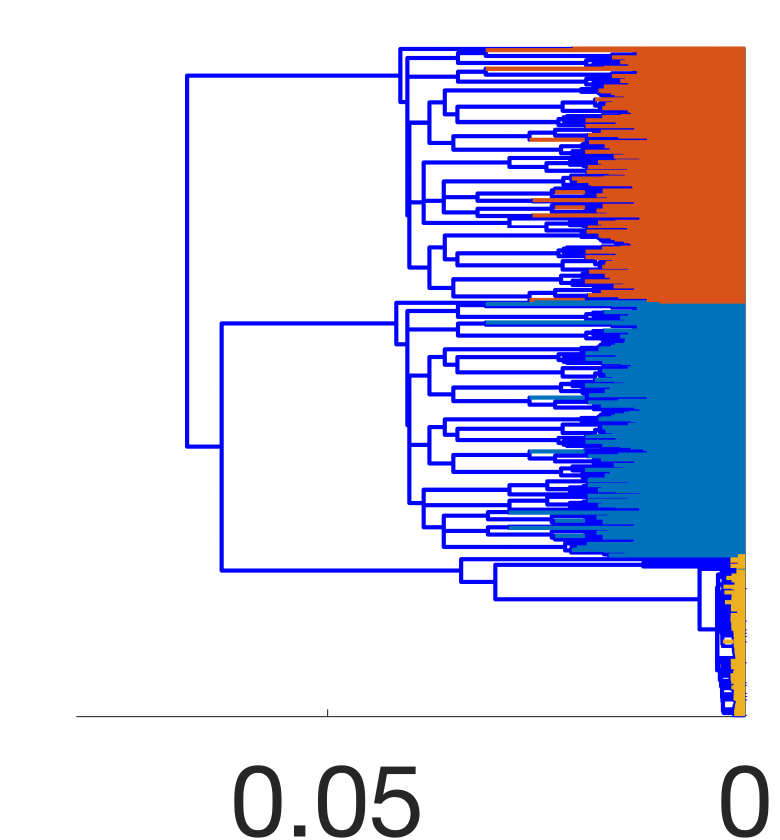}
            \end{minipage}
            \begin{minipage}[t]{1\textwidth}
                \centering
                $e$:0.4/1 \\ 
                \includegraphics[width=1\textwidth]{figures/ellip/ellip-ori-seq-4-1.pdf} 
            \end{minipage}
            \begin{minipage}[t][0.55cm][t]{1\textwidth}
                \centering
                \includegraphics[width=1\textwidth]{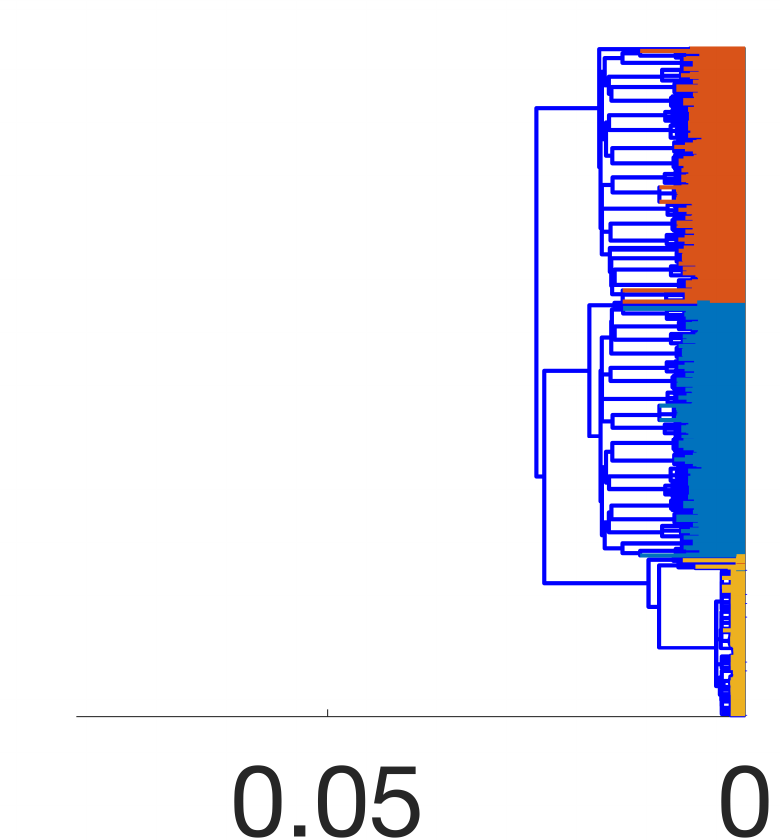}
            \end{minipage}
        \end{minipage}
        \begin{minipage}[t]{0.09\textwidth}
            \centering
            \begin{minipage}[t]{1\textwidth}
                \centering
                $e$:1/0.2 \\ 
                \includegraphics[width=1\textwidth]{figures/ellip/ellip-ori-seq-1-2.pdf} 
            \end{minipage}
            \begin{minipage}[t][0.55cm][t]{1\textwidth}
                \centering
                \includegraphics[width=1\textwidth]{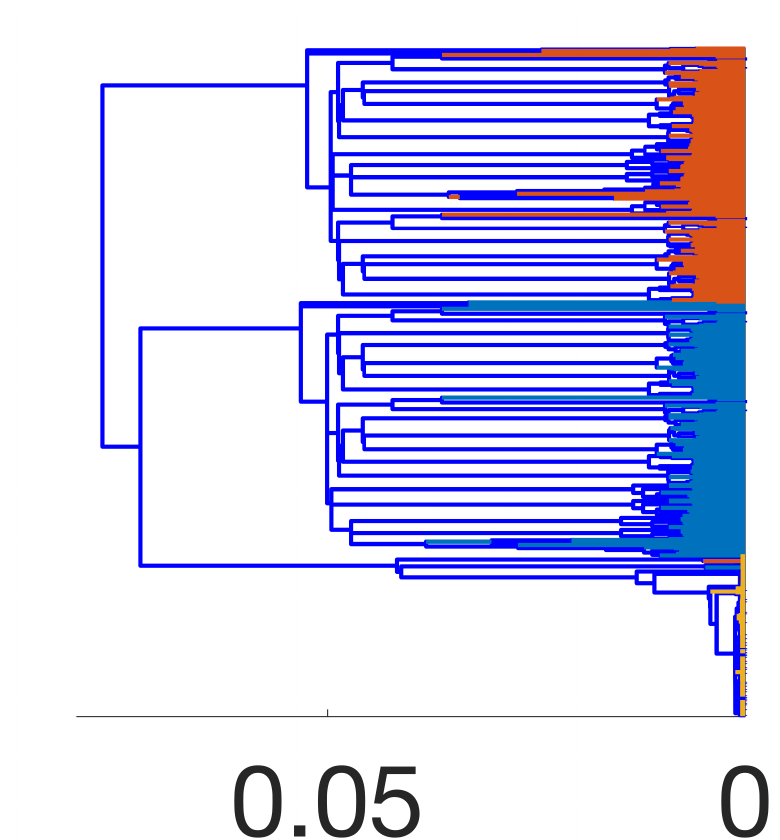}
            \end{minipage}
            \begin{minipage}[t]{1\textwidth}
                \centering
                $e$:0.2/1 \\ 
                \includegraphics[width=1\textwidth]{figures/ellip/ellip-ori-seq-2-1.pdf} 
            \end{minipage}
            \begin{minipage}[t][0.55cm][t]{1\textwidth}
                \centering
                \includegraphics[width=1\textwidth]{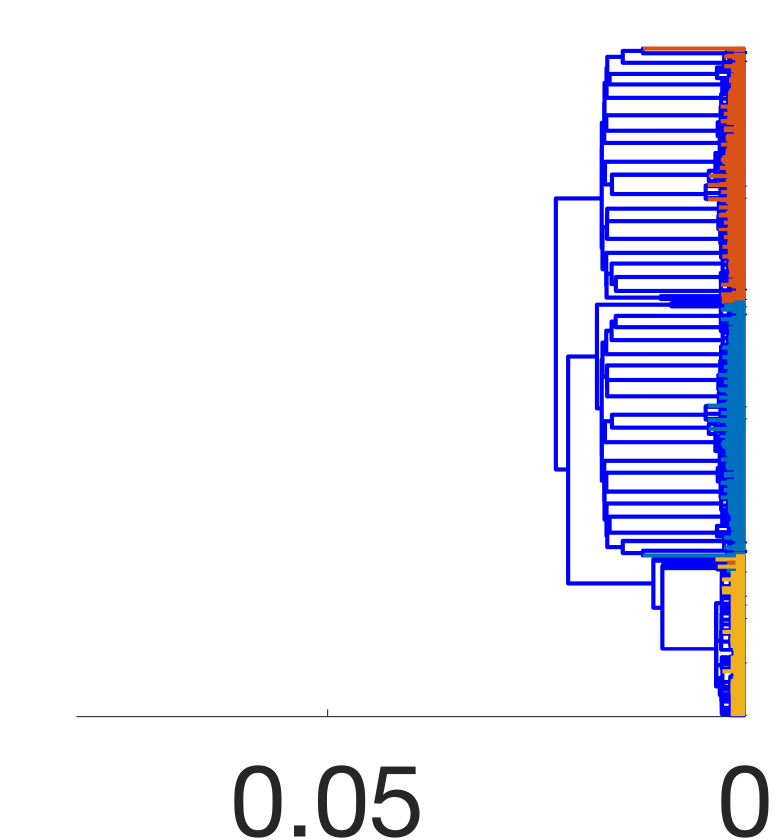}
            \end{minipage}
        \end{minipage}
    } }
    
    \subfloat[LT-WT2]{
    \fbox{
        \begin{minipage}[t]{0.09\textwidth}
            \centering
            \begin{minipage}[t]{1\textwidth}
                \centering
                $e$:1/1 \\ 
                \includegraphics[width=1\textwidth]{figures/ellip/ellip-ori-seq-1-10.pdf} 
            \end{minipage}
            \begin{minipage}[t][0.55cm][t]{1\textwidth}
                \centering
                \includegraphics[width=1\textwidth]{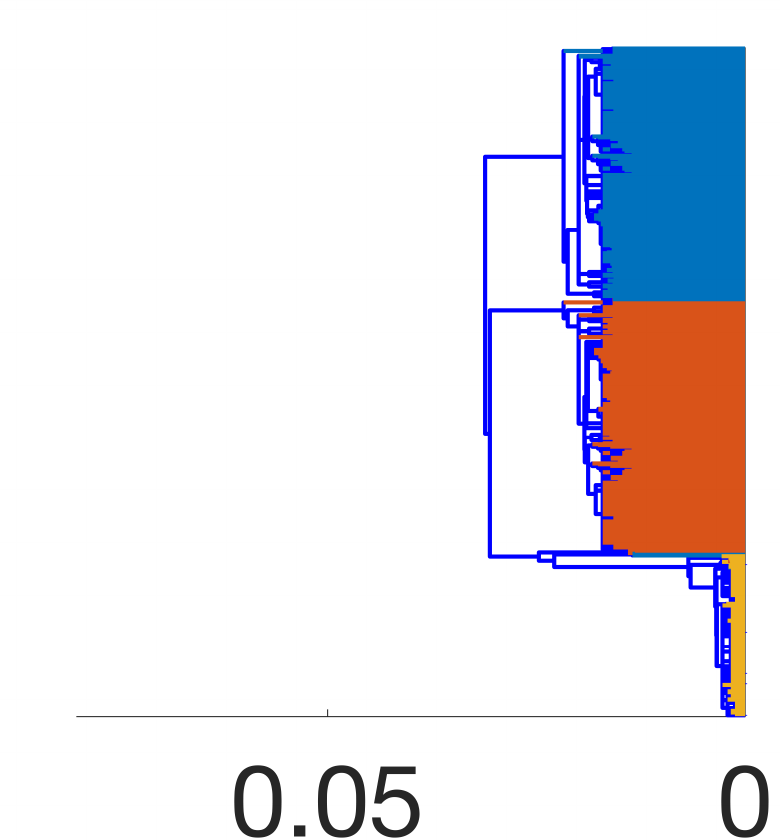}
            \end{minipage}
            \begin{minipage}[t]{1\textwidth}
                \centering
                $e$:1/1 \\ 
                \includegraphics[width=1\textwidth]{figures/ellip/ellip-ori-seq-1-10.pdf} 
            \end{minipage}
            \begin{minipage}[t][0.55cm][t]{1\textwidth}
                \centering
                \includegraphics[width=1\textwidth]{figures/ellip/wt2-fix-emd-ellip-seq-dend-1-10.pdf}
            \end{minipage}
        \end{minipage}
        \begin{minipage}[t]{0.09\textwidth}
            \centering
            \begin{minipage}[t]{1\textwidth}
                \centering
                $e$:1/0.8 \\ 
                \includegraphics[width=1\textwidth]{figures/ellip/ellip-ori-seq-1-8.pdf} 
            \end{minipage}
            \begin{minipage}[t][0.55cm][t]{1\textwidth}
                \centering
                \includegraphics[width=1\textwidth]{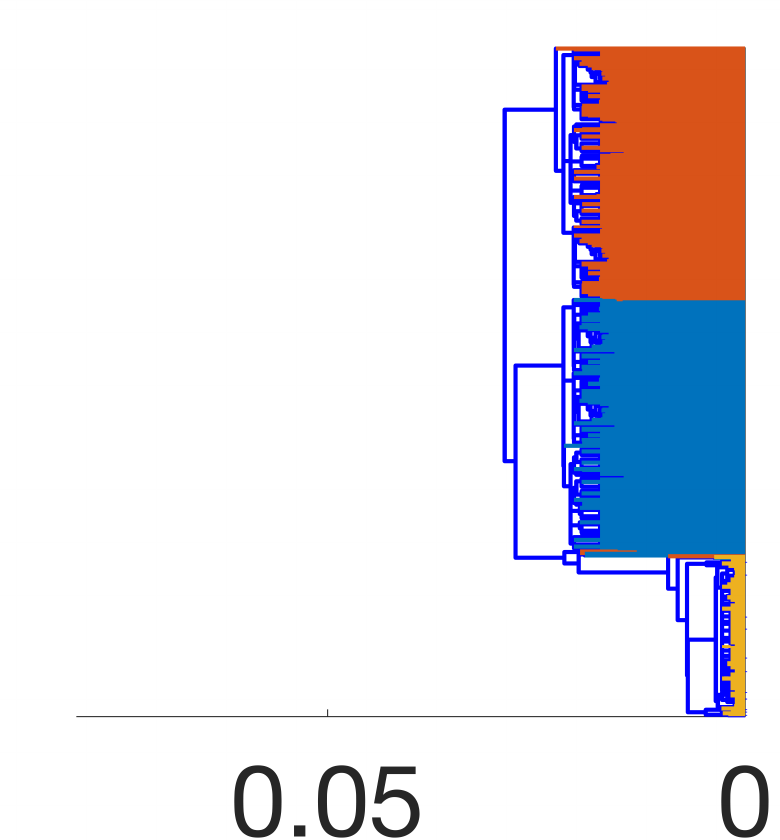}
            \end{minipage}
            \begin{minipage}[t]{1\textwidth}
                \centering
                $e$:0.8/1 \\ 
                \includegraphics[width=1\textwidth]{figures/ellip/ellip-ori-seq-8-1.pdf} 
            \end{minipage}
            \begin{minipage}[t][0.55cm][t]{1\textwidth}
                \centering
                \includegraphics[width=1\textwidth]{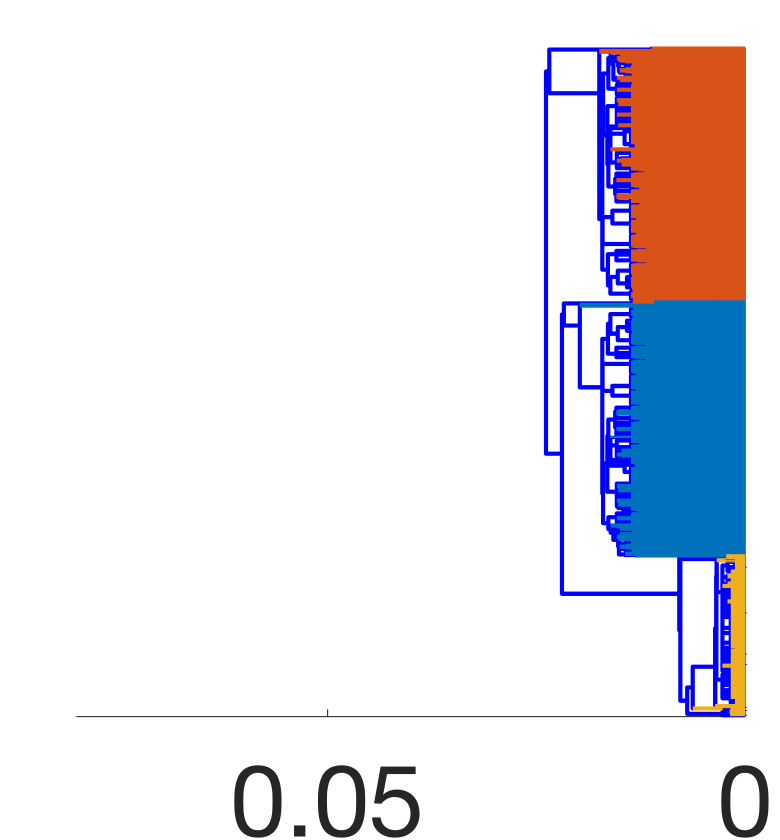}
            \end{minipage}
        \end{minipage}
        \begin{minipage}[t]{0.09\textwidth}
            \centering
            \begin{minipage}[t]{1\textwidth}
                \centering
                $e$:1/0.6 \\ 
                \includegraphics[width=1\textwidth]{figures/ellip/ellip-ori-seq-1-6.pdf} 
            \end{minipage}
            \begin{minipage}[t][0.55cm][t]{1\textwidth}
                \centering
                \includegraphics[width=1\textwidth]{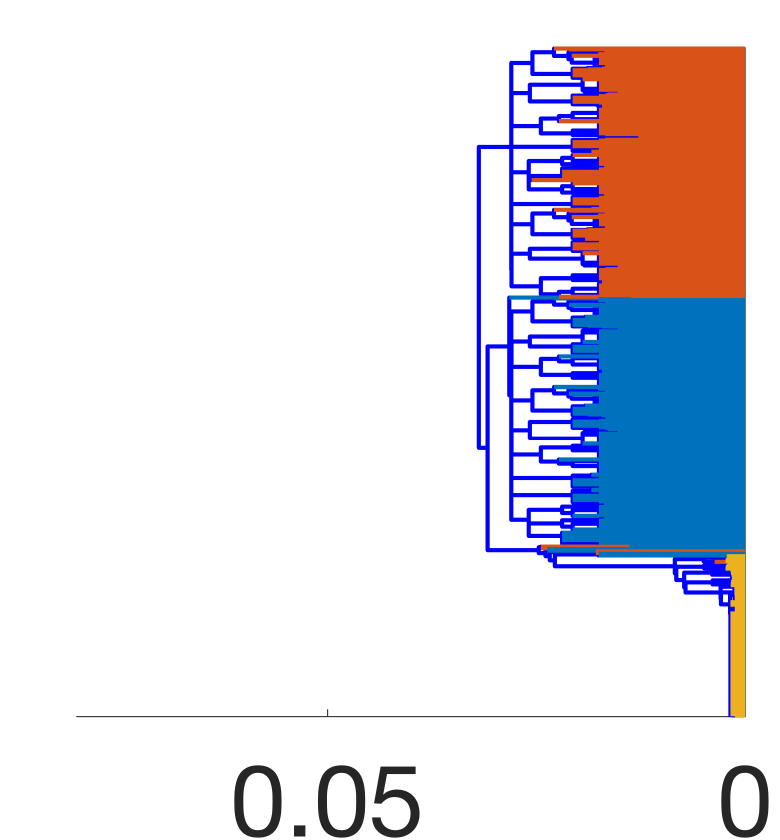}
            \end{minipage}
            \begin{minipage}[t]{1\textwidth}
                \centering
                $e$:0.6/1 \\ 
                \includegraphics[width=1\textwidth]{figures/ellip/ellip-ori-seq-6-1.pdf} 
            \end{minipage}
            \begin{minipage}[t][0.55cm][t]{1\textwidth}
                \centering
                \includegraphics[width=1\textwidth]{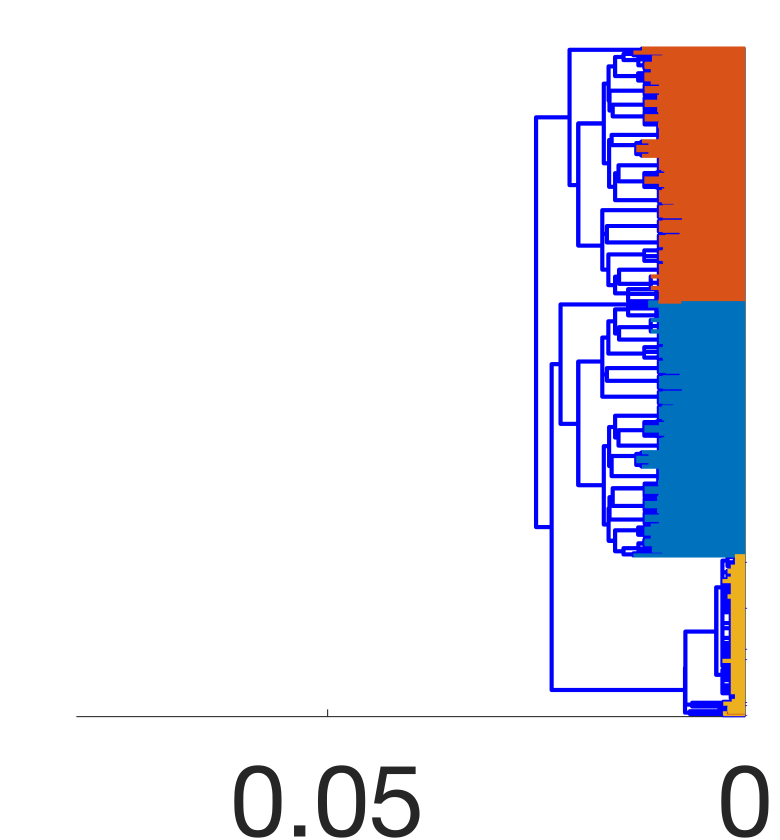}
            \end{minipage}
        \end{minipage}
        \begin{minipage}[t]{0.09\textwidth}
            \centering
            \begin{minipage}[t]{1\textwidth}
                \centering
                $e$:1/0.4 \\ 
                \includegraphics[width=1\textwidth]{figures/ellip/ellip-ori-seq-1-4.pdf} 
            \end{minipage}
            \begin{minipage}[t][0.55cm][t]{1\textwidth}
                \centering
                \includegraphics[width=1\textwidth]{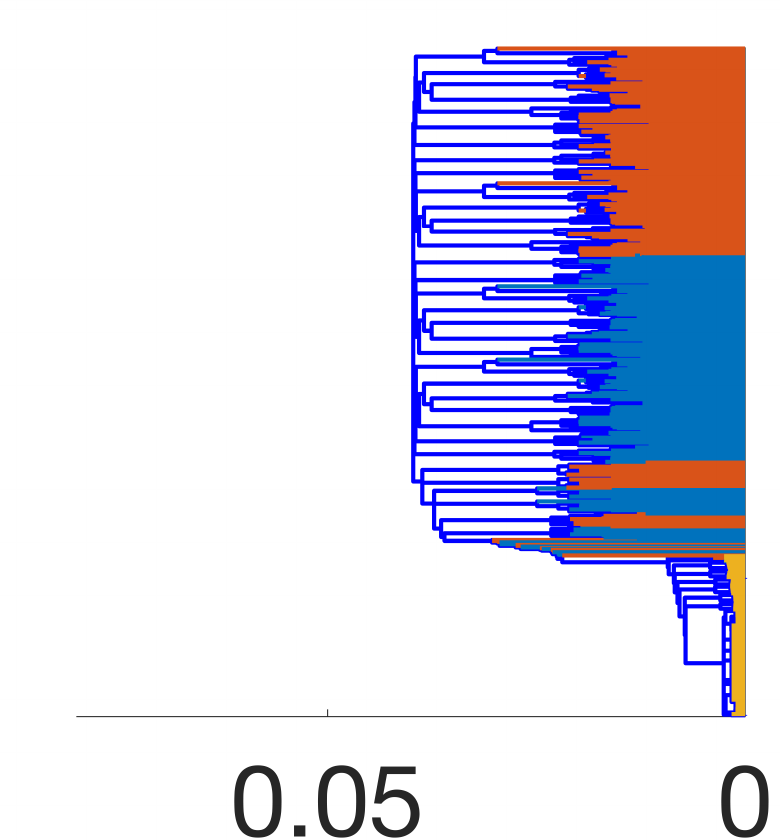}
            \end{minipage}
            \begin{minipage}[t]{1\textwidth}
                \centering
                $e$:0.4/1 \\ 
                \includegraphics[width=1\textwidth]{figures/ellip/ellip-ori-seq-4-1.pdf} 
            \end{minipage}
            \begin{minipage}[t][0.55cm][t]{1\textwidth}
                \centering
                \includegraphics[width=1\textwidth]{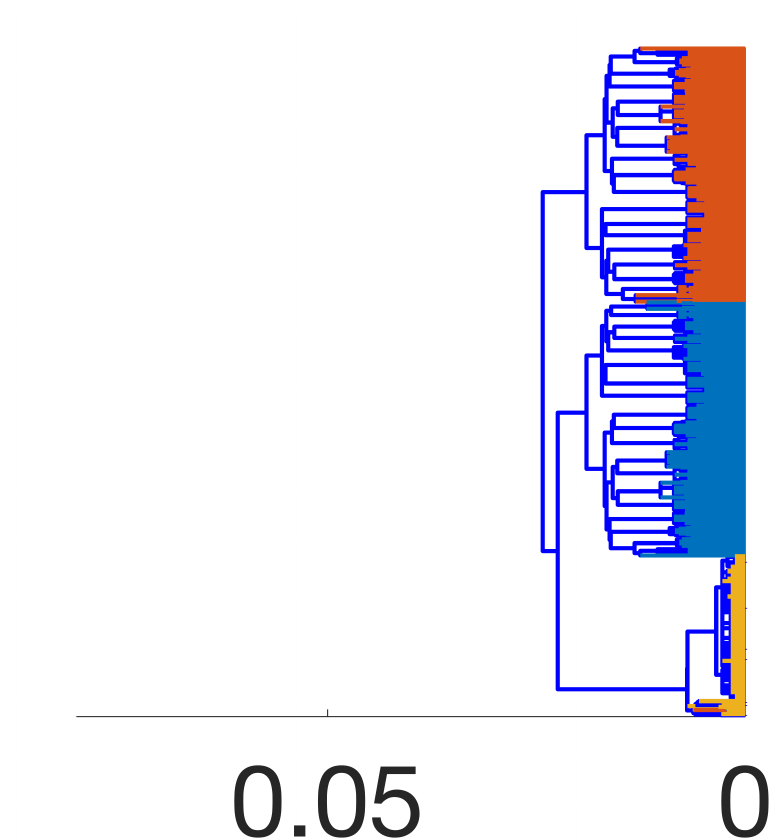}
            \end{minipage}
        \end{minipage}
        \begin{minipage}[t]{0.09\textwidth}
            \centering
            \begin{minipage}[t]{1\textwidth}
                \centering
                $e$:1/0.2 \\ 
                \includegraphics[width=1\textwidth]{figures/ellip/ellip-ori-seq-1-2.pdf} 
            \end{minipage}
            \begin{minipage}[t][0.55cm][t]{1\textwidth}
                \centering
                \includegraphics[width=1\textwidth]{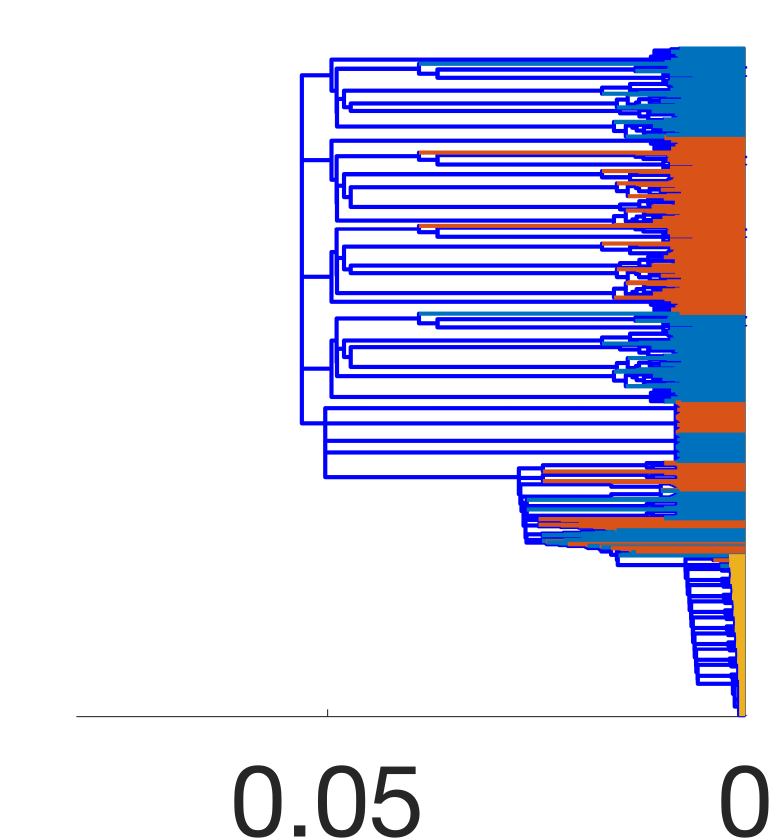}
            \end{minipage}
            \begin{minipage}[t]{1\textwidth}
                \centering
                $e$:0.2/1 \\ 
                \includegraphics[width=1\textwidth]{figures/ellip/ellip-ori-seq-2-1.pdf} 
            \end{minipage}
            \begin{minipage}[t][0.55cm][t]{1\textwidth}
                \centering
                \includegraphics[width=1\textwidth]{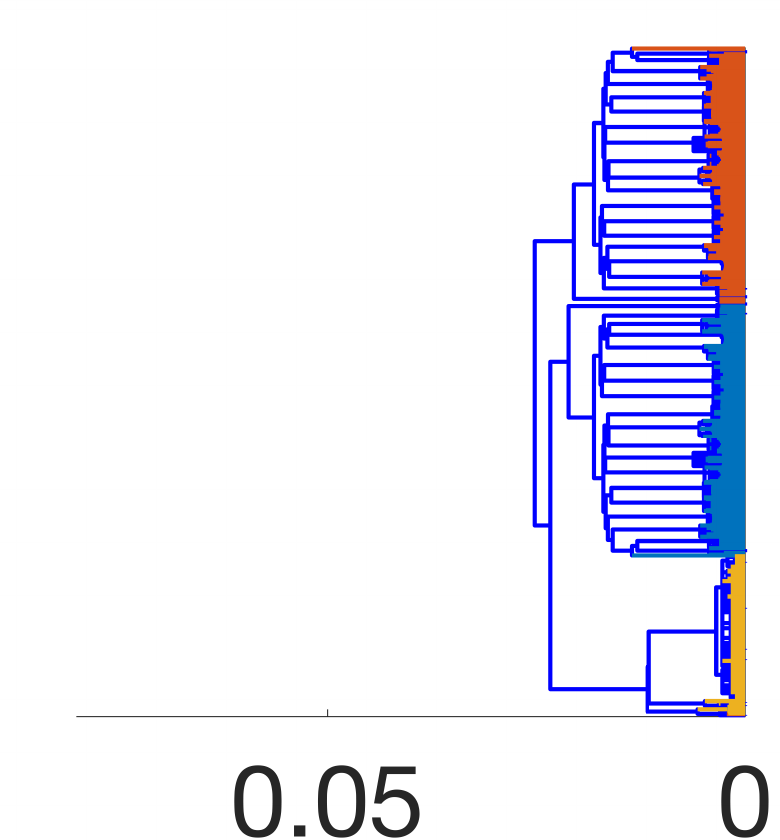}
            \end{minipage}
        \end{minipage}
    } }
    \subfloat[LT-WT1]{
    \fbox{
        \begin{minipage}[t]{0.09\textwidth}
            \centering
            \begin{minipage}[t]{1\textwidth}
                \centering
                $e$:1/1 \\ 
                \includegraphics[width=1\textwidth]{figures/ellip/ellip-ori-seq-1-10.pdf} 
            \end{minipage}
            \begin{minipage}[t][0.55cm][t]{1\textwidth}
                \centering
                \includegraphics[width=1\textwidth]{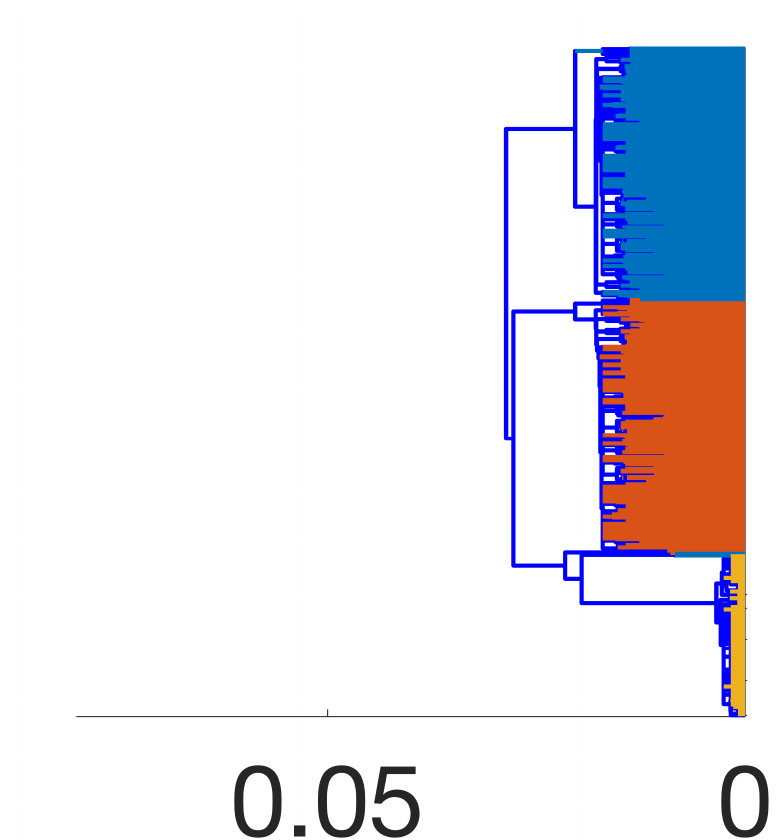}
            \end{minipage}
            \begin{minipage}[t]{1\textwidth}
                \centering
                $e$:1/1 \\ 
                \includegraphics[width=1\textwidth]{figures/ellip/ellip-ori-seq-1-10.pdf} 
            \end{minipage}
            \begin{minipage}[t][0.55cm][t]{1\textwidth}
                \centering
                \includegraphics[width=1\textwidth]{figures/ellip/wt1-fix-emd-ellip-seq-dend-1-10.pdf}
            \end{minipage}
        \end{minipage}
        \begin{minipage}[t]{0.09\textwidth}
            \centering
            \begin{minipage}[t]{1\textwidth}
                \centering
                $e$:1/0.8 \\ 
                \includegraphics[width=1\textwidth]{figures/ellip/ellip-ori-seq-1-8.pdf} 
            \end{minipage}
            \begin{minipage}[t][0.55cm][t]{1\textwidth}
                \centering
                \includegraphics[width=1\textwidth]{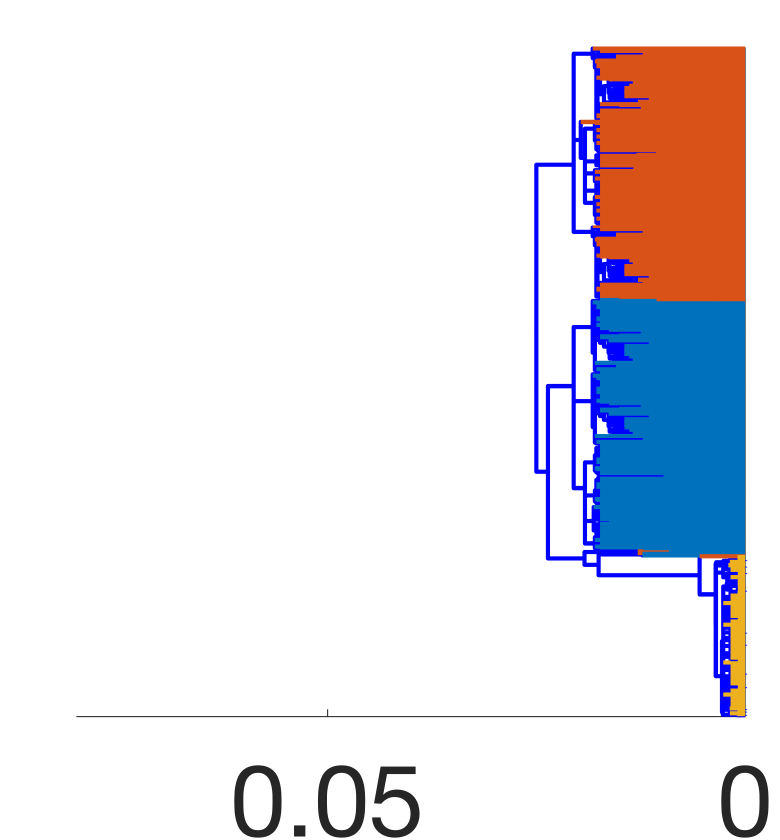}
            \end{minipage}
            \begin{minipage}[t]{1\textwidth}
                \centering
                $e$:0.8/1 \\ 
                \includegraphics[width=1\textwidth]{figures/ellip/ellip-ori-seq-8-1.pdf} 
            \end{minipage}
            \begin{minipage}[t][0.55cm][t]{1\textwidth}
                \centering
                \includegraphics[width=1\textwidth]{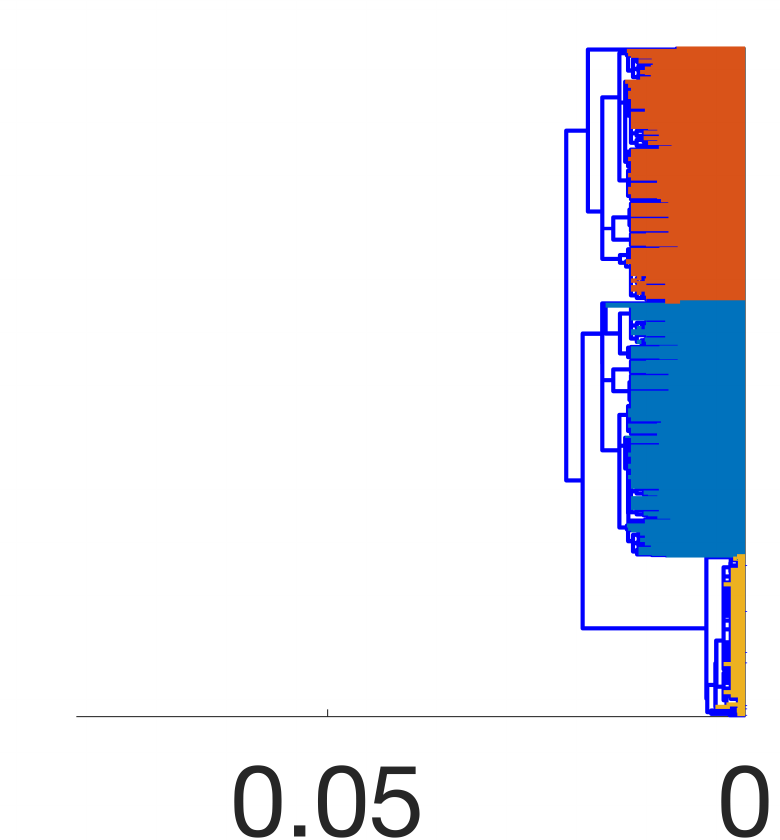}
            \end{minipage}
        \end{minipage}
        \begin{minipage}[t]{0.09\textwidth}
            \centering
            \begin{minipage}[t]{1\textwidth}
                \centering
                $e$:1/0.6 \\ 
                \includegraphics[width=1\textwidth]{figures/ellip/ellip-ori-seq-1-6.pdf} 
            \end{minipage}
            \begin{minipage}[t][0.55cm][t]{1\textwidth}
                \centering
                \includegraphics[width=1\textwidth]{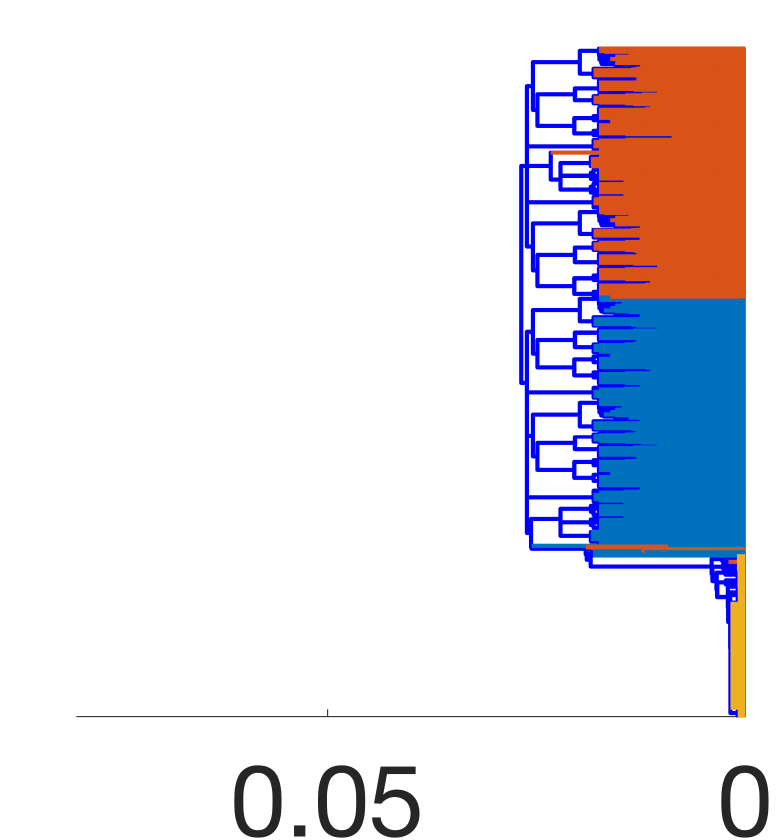}
            \end{minipage}
            \begin{minipage}[t]{1\textwidth}
                \centering
                $e$:0.6/1 \\ 
                \includegraphics[width=1\textwidth]{figures/ellip/ellip-ori-seq-6-1.pdf} 
            \end{minipage}
            \begin{minipage}[t][0.55cm][t]{1\textwidth}
                \centering
                \includegraphics[width=1\textwidth]{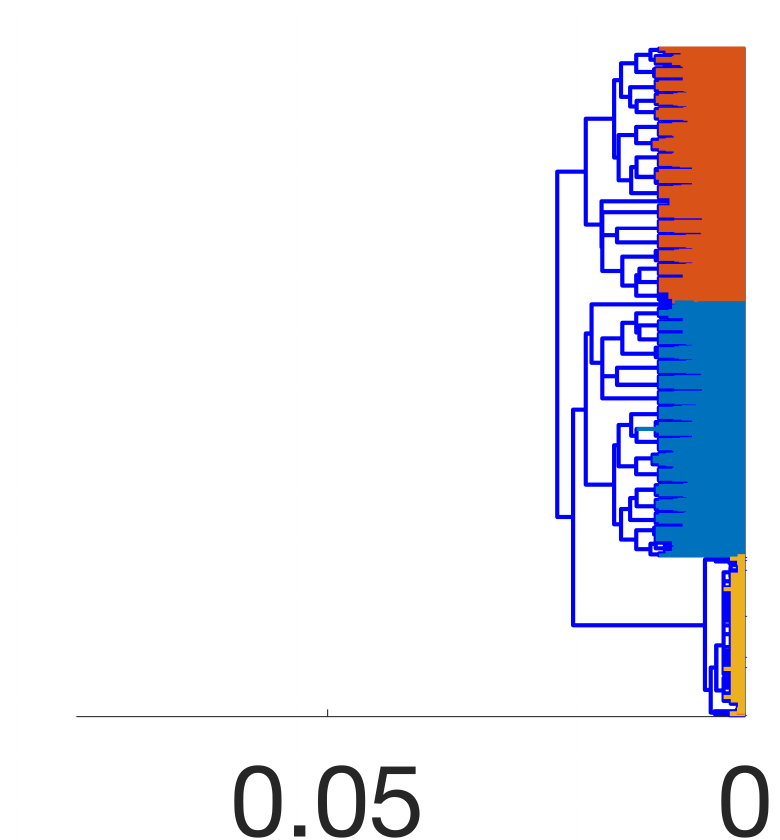}
            \end{minipage}
        \end{minipage}
        \begin{minipage}[t]{0.09\textwidth}
            \centering
            \begin{minipage}[t]{1\textwidth}
                \centering
                $e$:1/0.4 \\ 
                \includegraphics[width=1\textwidth]{figures/ellip/ellip-ori-seq-1-4.pdf} 
            \end{minipage}
            \begin{minipage}[t][0.55cm][t]{1\textwidth}
                \centering
                \includegraphics[width=1\textwidth]{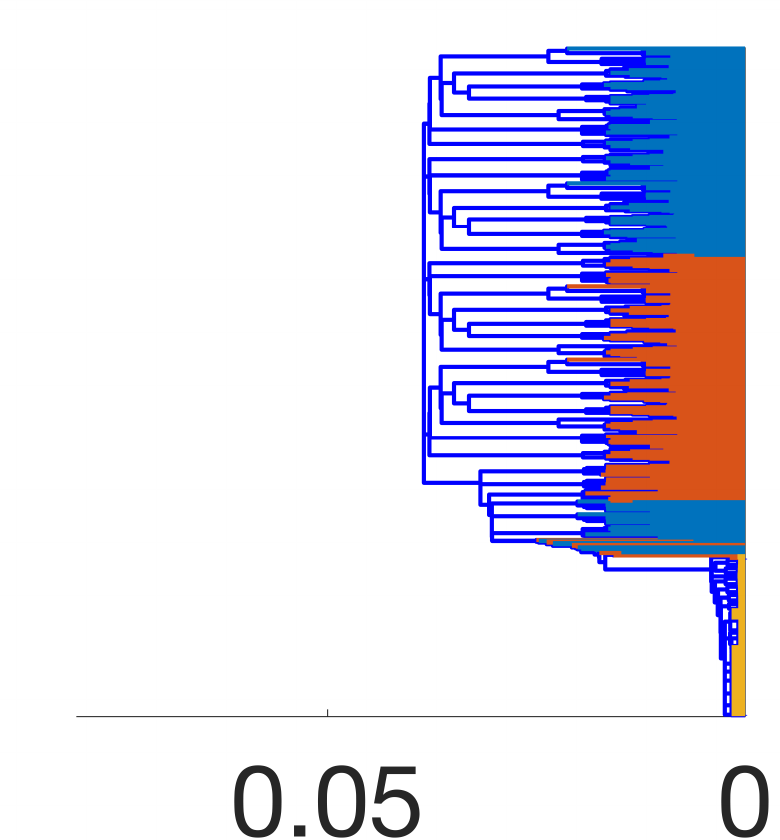}
            \end{minipage}
            \begin{minipage}[t]{1\textwidth}
                \centering
                $e$:0.4/1 \\ 
                \includegraphics[width=1\textwidth]{figures/ellip/ellip-ori-seq-4-1.pdf} 
            \end{minipage}
            \begin{minipage}[t][0.55cm][t]{1\textwidth}
                \centering
                \includegraphics[width=1\textwidth]{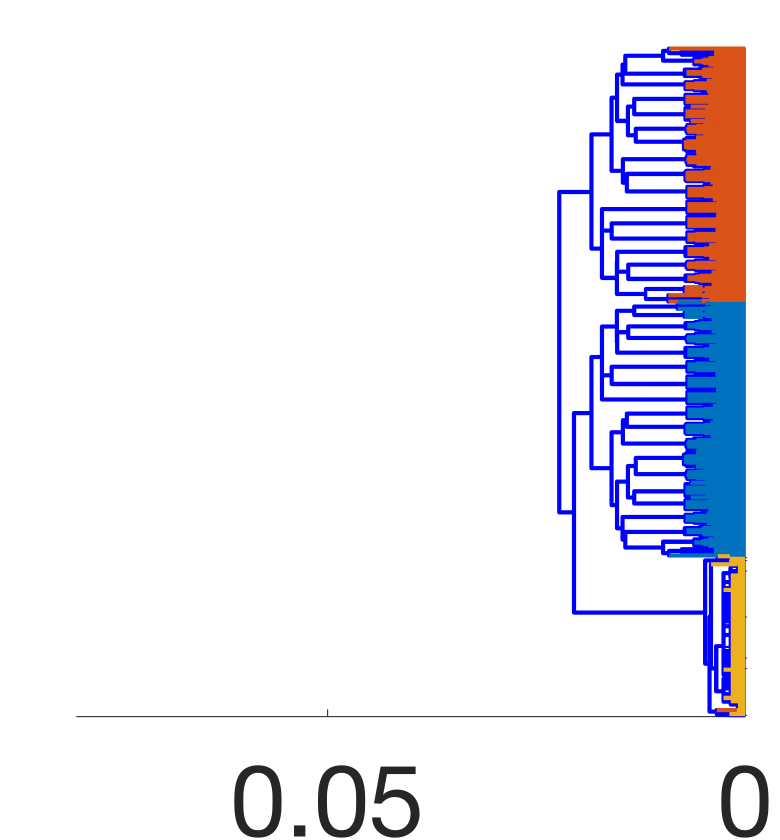}
            \end{minipage}
        \end{minipage}
        \begin{minipage}[t]{0.09\textwidth}
            \centering
            \begin{minipage}[t]{1\textwidth}
                \centering
                $e$:1/0.2 \\ 
                \includegraphics[width=1\textwidth]{figures/ellip/ellip-ori-seq-1-2.pdf} 
            \end{minipage}
            \begin{minipage}[t][0.55cm][t]{1\textwidth}
                \centering
                \includegraphics[width=1\textwidth]{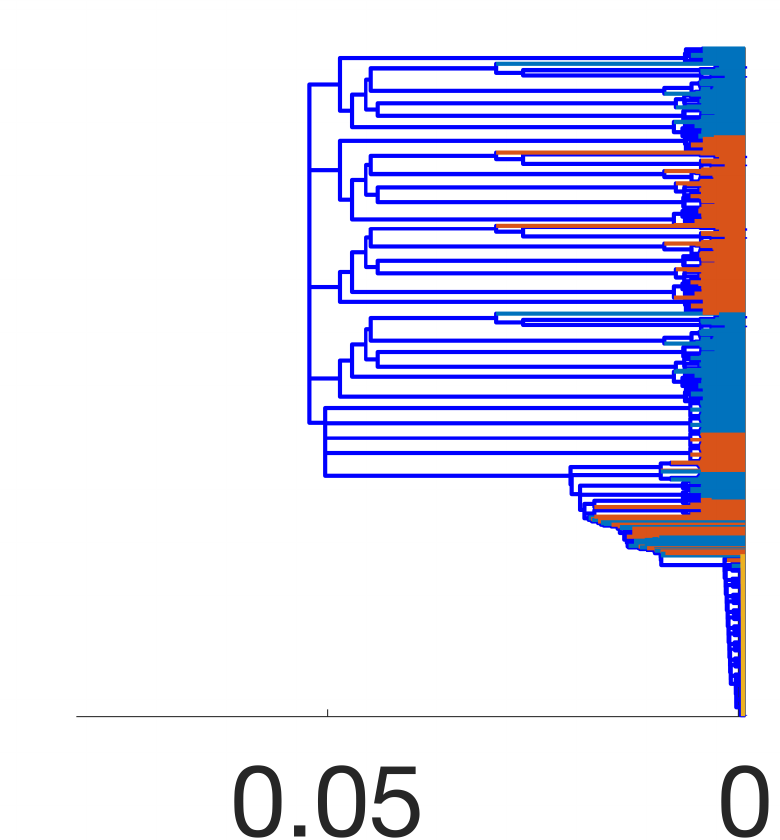}
            \end{minipage}
            \begin{minipage}[t]{1\textwidth}
                \centering
                $e$:0.2/1 \\ 
                \includegraphics[width=1\textwidth]{figures/ellip/ellip-ori-seq-2-1.pdf} 
            \end{minipage}
            \begin{minipage}[t][0.55cm][t]{1\textwidth}
                \centering
                \includegraphics[width=1\textwidth]{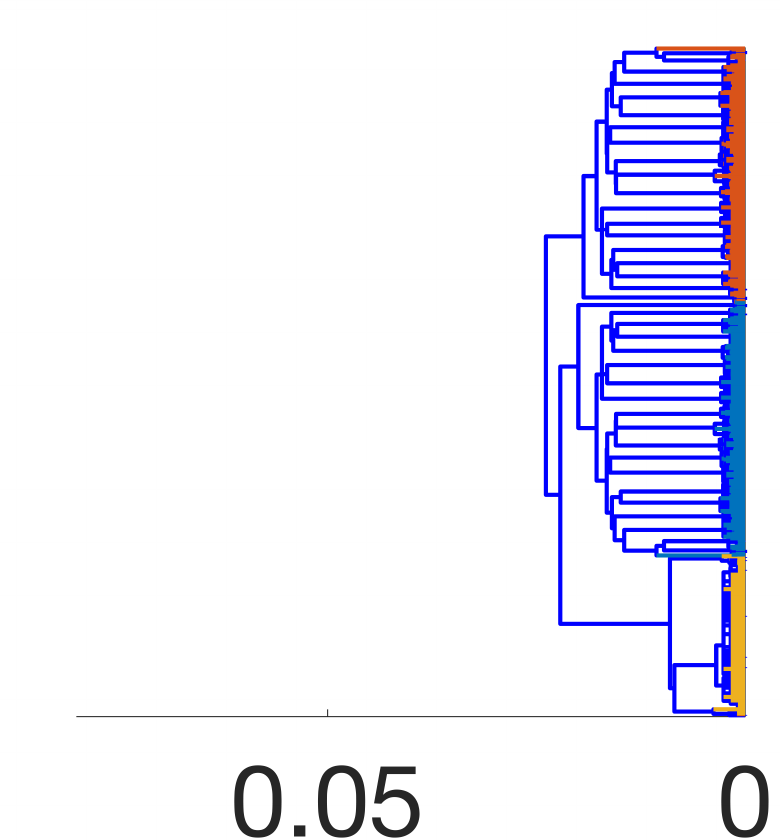}
            \end{minipage}
        \end{minipage}
    } }
    \caption{\textbf{Chaining effect: influence of geometry.} For each subgraph, each (dumbbell-shape) point cloud represents the result after applying the linear transformation $T$ with eccentricity $e$ to the original data set ($e=1$). All the corresponding dendrograms have the same x-axis limit.}
    \label{fig:supp-ballchains}
\end{figure}

\subsection{Noise removal} 
In this section, we consider two data sets corrupted by two different kinds of noise, respectively. We apply MS, GT, LT-WT1 and LT-WT2 to these data sets to test the denoising ability of these methods.

\paragraph{Noisy spiral.}
\label{app:denoising-spiral}
In this example, we analyze a data set consisting of 600 points sampled from a spiral lying in the square $[-30, 30]\times [-30,30]$. This data set is corrupted by 150 extra outliers uniformly sampled in the square $[-30, 30]\times [-30,30]$. We apply MS, GT, LT-WT1 and LT-WT2 to this data set in the course of 4 iterations. The parameter $\eps$ is set to be 4. Results are shown in Figure (\ref{fig:supp-spiral}). 

\begin{figure}[htb]
    \centering
    \subfloat[MS]{
		\begin{minipage}[t]{0.08\textwidth}
		    \centering
		    $\tau=0$ \\ 
            \includegraphics[width=1\textwidth]{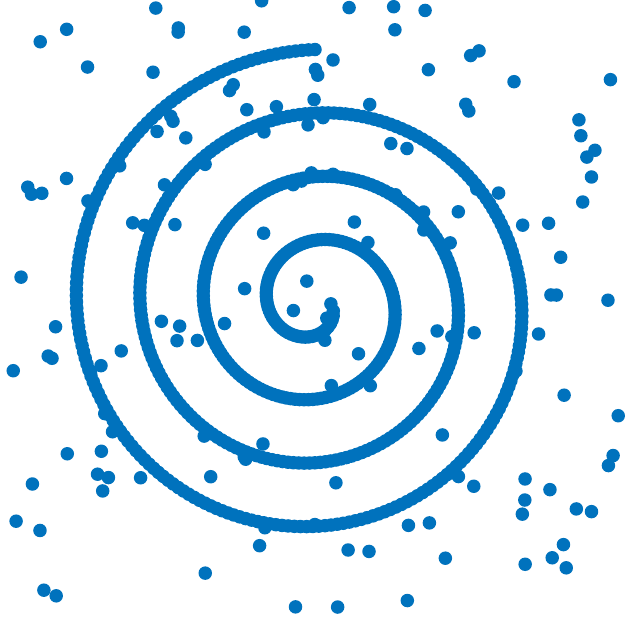}
        \end{minipage}
        \begin{minipage}[t]{0.08\textwidth}
		    \centering
		    $\tau=1$ \\ 
            \includegraphics[width=1\textwidth]{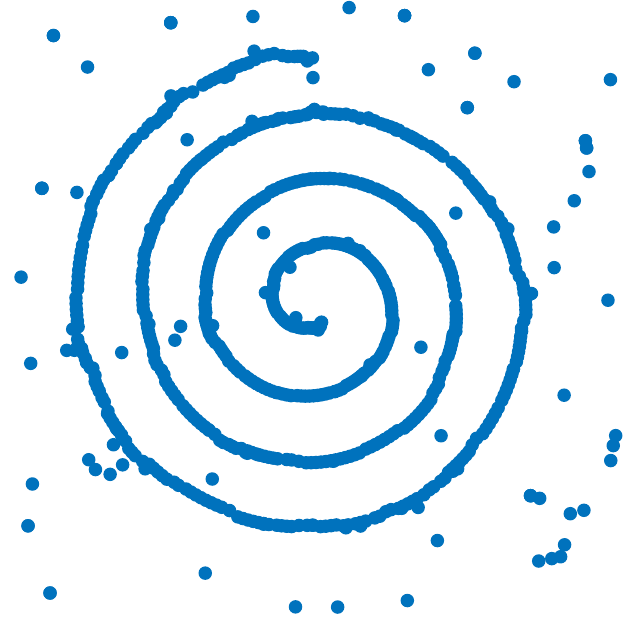}
        \end{minipage}
        \begin{minipage}[t]{0.08\textwidth}
		    \centering
		    $\tau=2$ \\ 
            \includegraphics[width=1\textwidth]{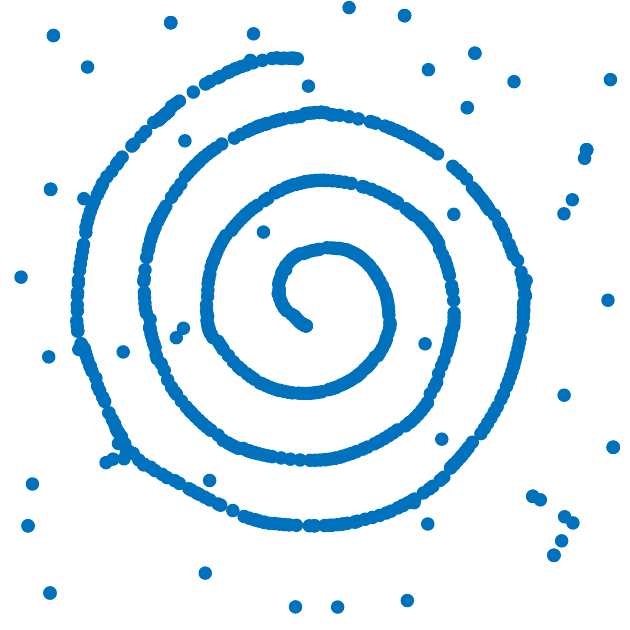}
        \end{minipage}
        \begin{minipage}[t]{0.08\textwidth}
		    \centering
		    $\tau=3$ \\ 
            \includegraphics[width=1\textwidth]{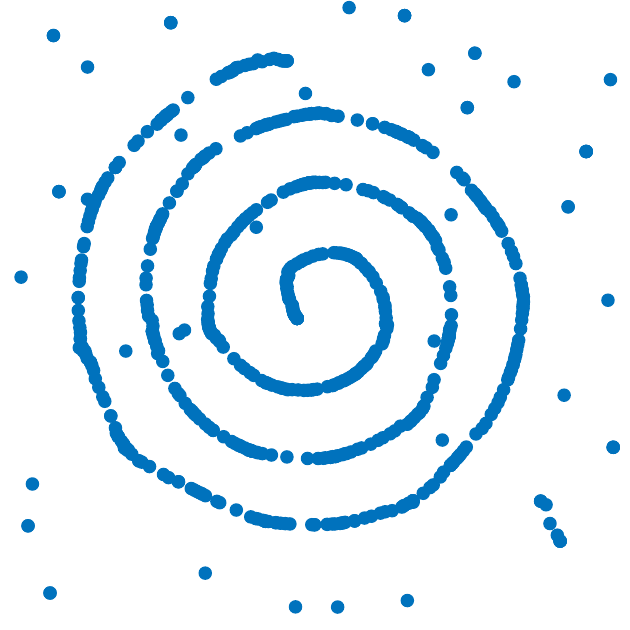}
        \end{minipage}
        \begin{minipage}[t]{0.08\textwidth}
		    \centering
		    $\tau=4$ \\ 
            \includegraphics[width=1\textwidth]{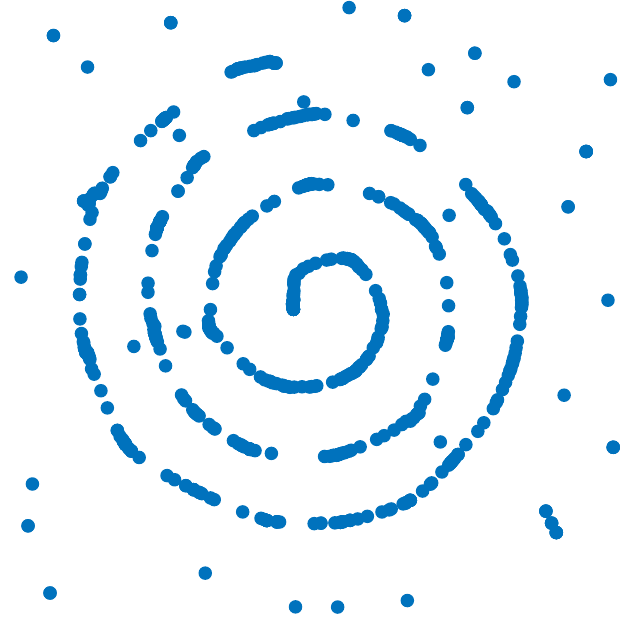}
        \end{minipage}
	}

	\subfloat[GT-1]{
		\begin{minipage}[t]{0.08\textwidth}
		    \centering
            \includegraphics[width=1\textwidth]{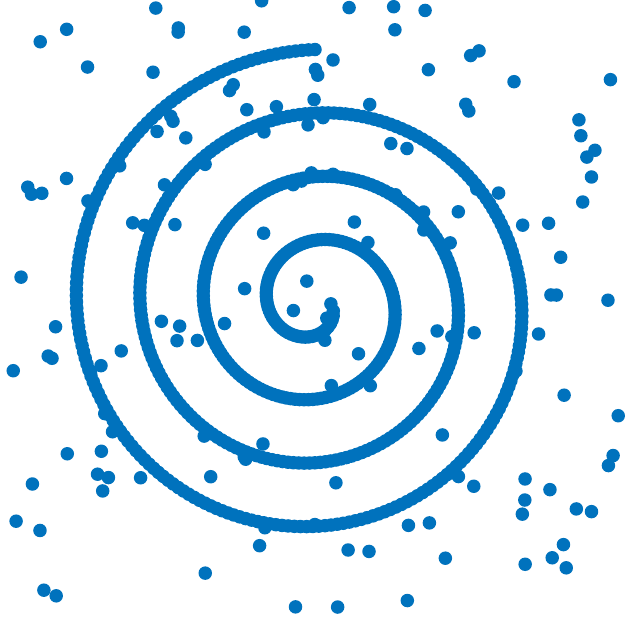}
        \end{minipage}
        \begin{minipage}[t]{0.08\textwidth}
		    \centering
            \includegraphics[width=1\textwidth]{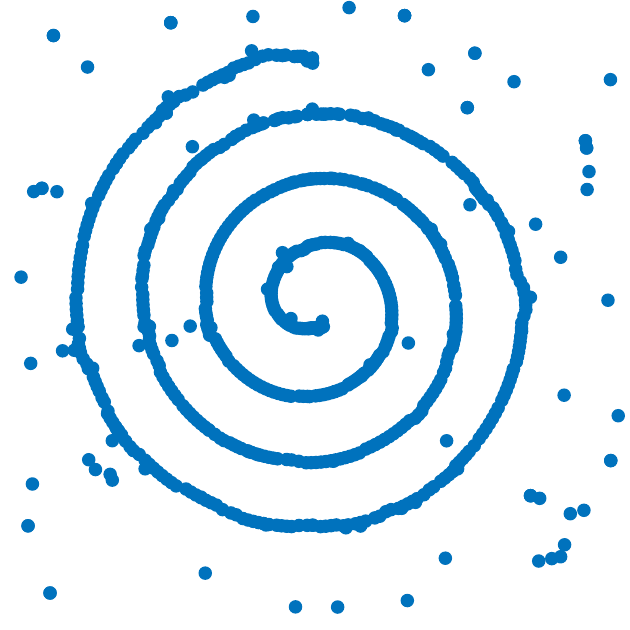}
        \end{minipage}
        \begin{minipage}[t]{0.08\textwidth}
		    \centering
            \includegraphics[width=1\textwidth]{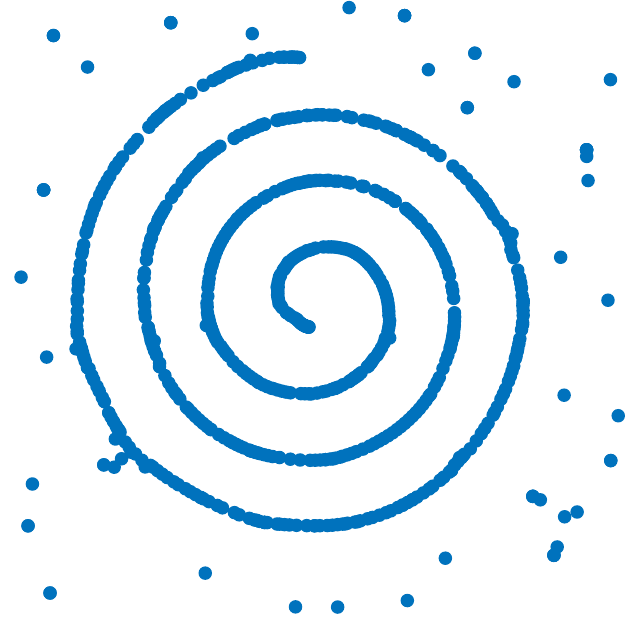}
        \end{minipage}
        \begin{minipage}[t]{0.08\textwidth}
		    \centering
            \includegraphics[width=1\textwidth]{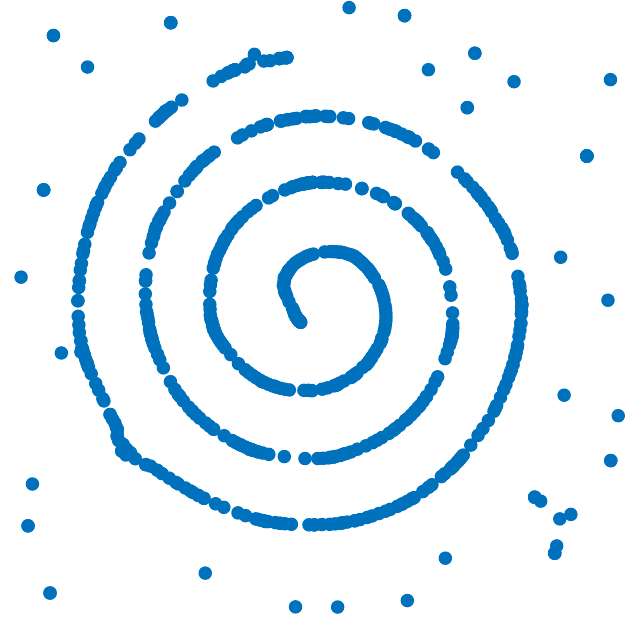}
        \end{minipage}
        \begin{minipage}[t]{0.08\textwidth}
		    \centering
            \includegraphics[width=1\textwidth]{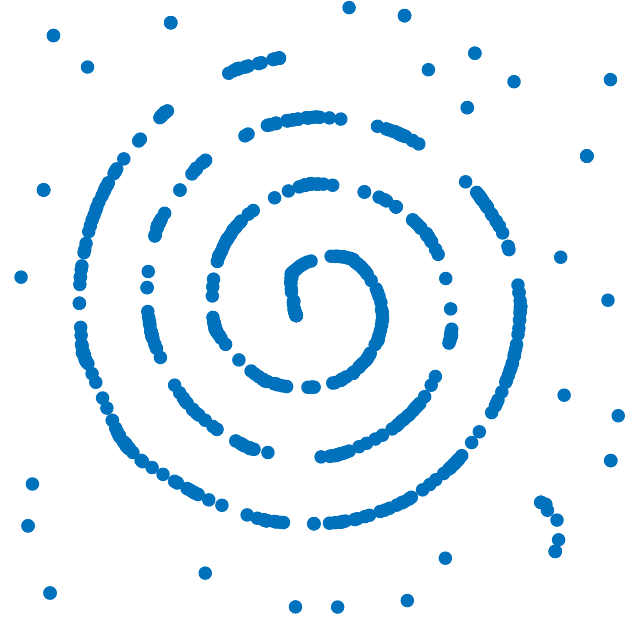}
        \end{minipage}
	}
	
	\subfloat[LT-WT2]{
        \begin{minipage}[t]{0.08\textwidth}
		    \centering
            \includegraphics[width=1\textwidth]{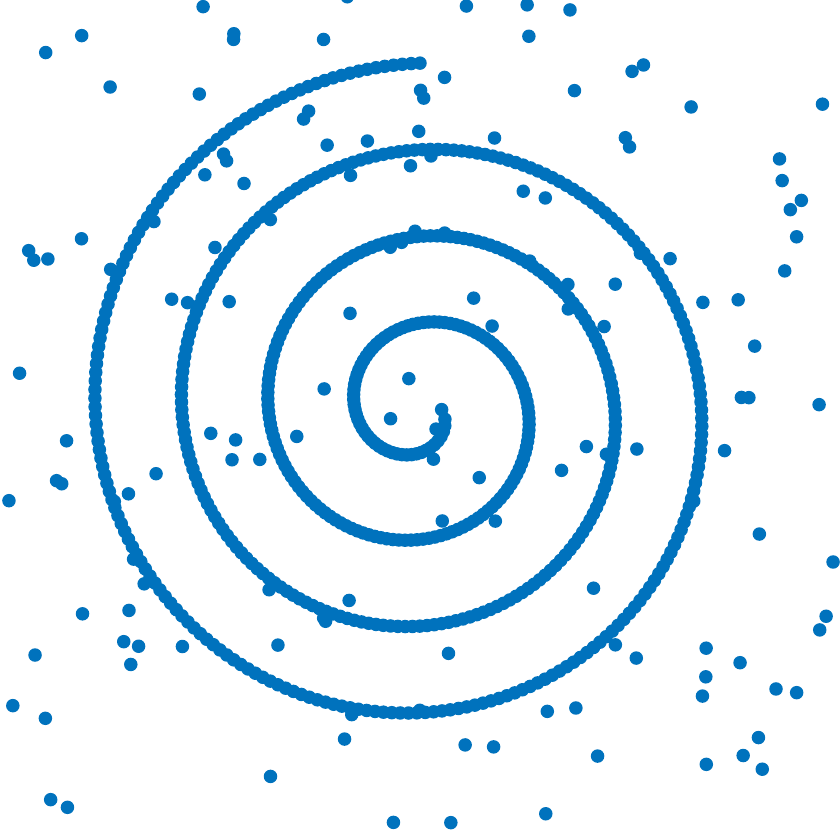}
        \end{minipage}
        \begin{minipage}[t]{0.08\textwidth}
		    \centering
            \includegraphics[width=1\textwidth]{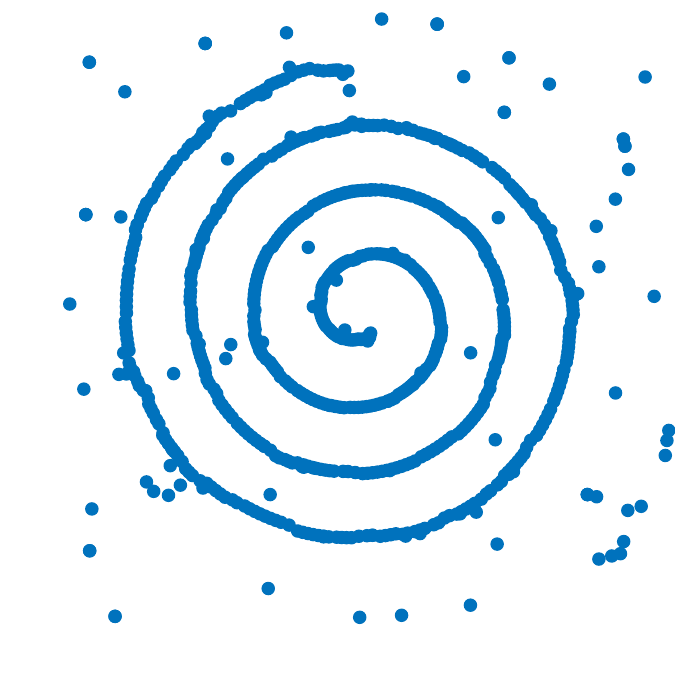}
        \end{minipage}
        \begin{minipage}[t]{0.08\textwidth}
		    \centering
            \includegraphics[width=1\textwidth]{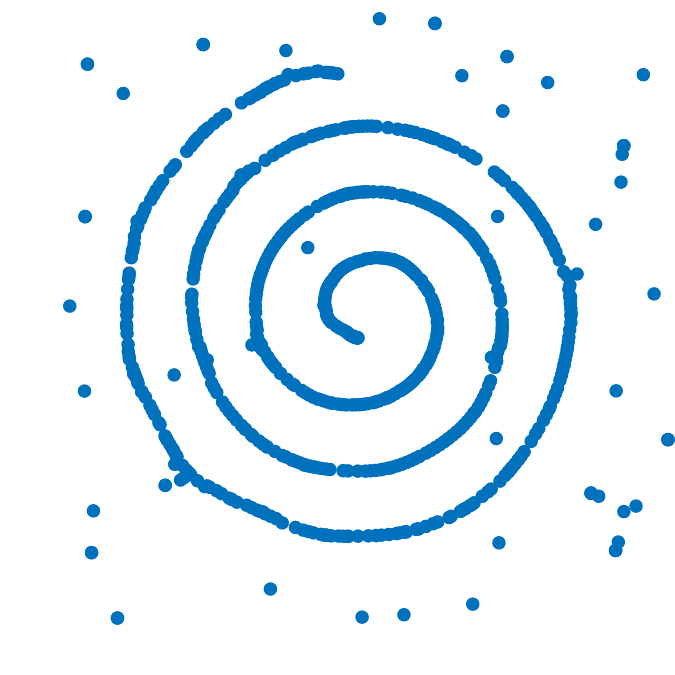}
        \end{minipage}
        \begin{minipage}[t]{0.08\textwidth}
		    \centering
            \includegraphics[width=1\textwidth]{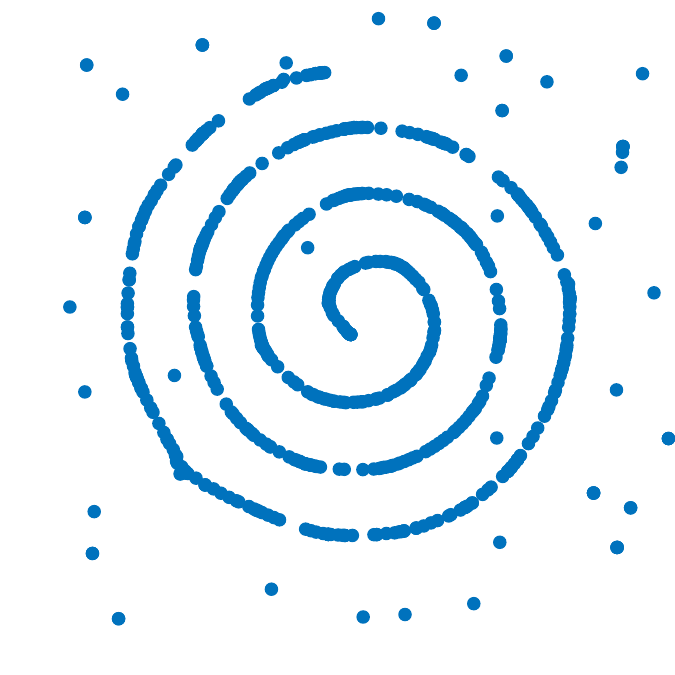}
        \end{minipage}
        \begin{minipage}[t]{0.08\textwidth}
		    \centering
            \includegraphics[width=1\textwidth]{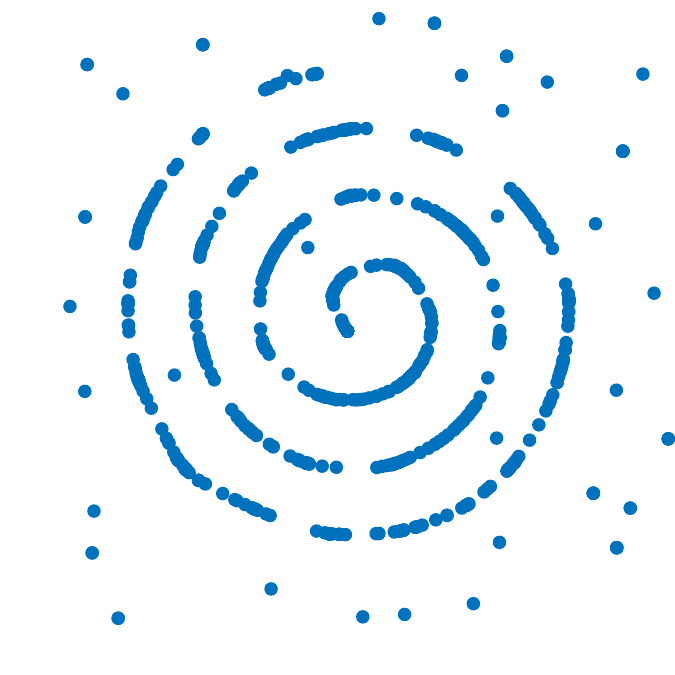}
        \end{minipage}
	}

    \subfloat[LT-WT1]{
        \begin{minipage}[t]{0.08\textwidth}
            \centering
            \includegraphics[width=1\textwidth]{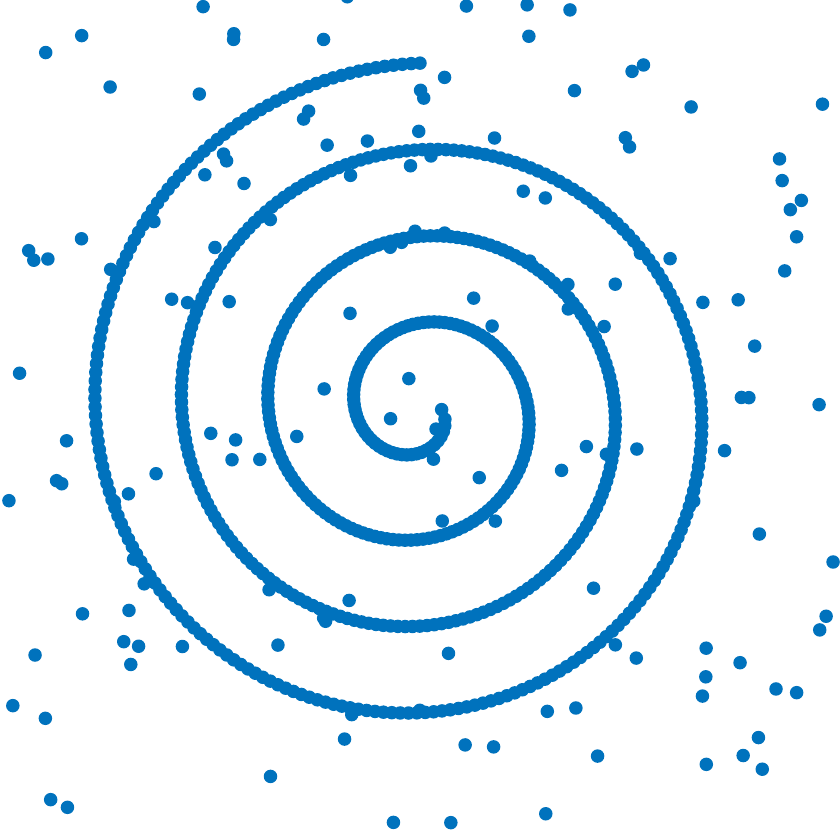}
        \end{minipage}
        \begin{minipage}[t]{0.08\textwidth}
            \centering
            \includegraphics[width=1\textwidth]{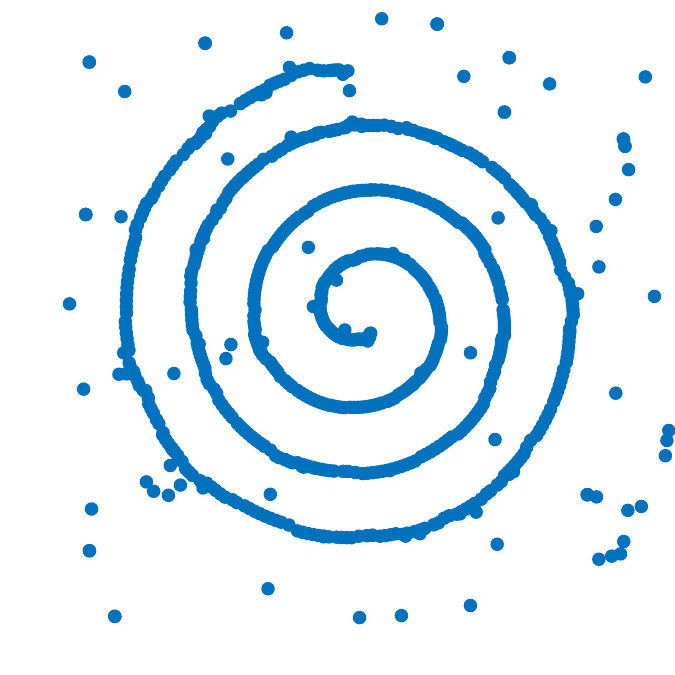}
        \end{minipage}
        \begin{minipage}[t]{0.08\textwidth}
            \centering
            \includegraphics[width=1\textwidth]{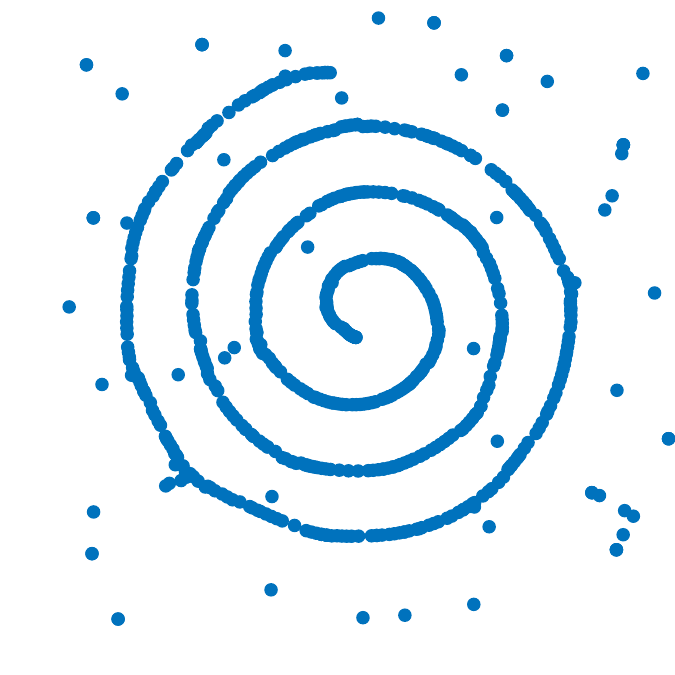}
        \end{minipage}
        \begin{minipage}[t]{0.08\textwidth}
            \centering
            \includegraphics[width=1\textwidth]{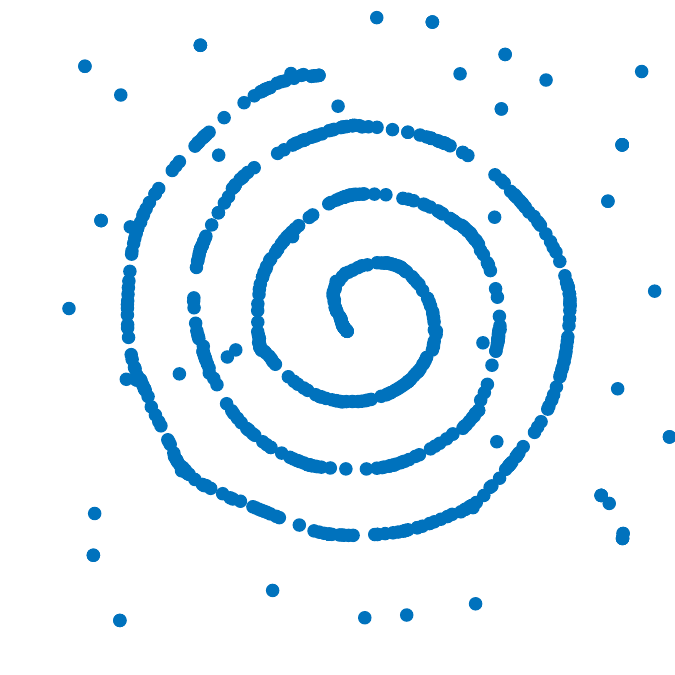}
        \end{minipage}
        \begin{minipage}[t]{0.08\textwidth}
            \centering
            \includegraphics[width=1\textwidth]{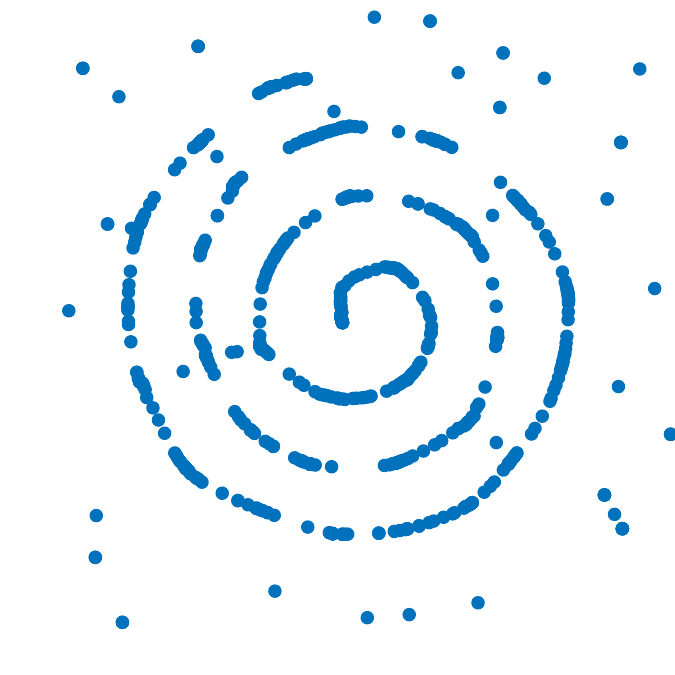}
        \end{minipage}
    }
    \caption{\textbf{Denoising of a spiral with outliers.} }
    \label{fig:supp-spiral}
\end{figure}

\paragraph{Noisy concentric circles.}
In this example, we consider the data set consisting of 500 points sampled from two concentric circles in the square $[-2, 2]\times [-2,2]$ where we sample 250 points from each circle. We corrupt the data set by introducing small additive noise to these 500 points. We apply MS, GT, LT-WT1 and LT-WT2 to this data set in the course of 4 iterations. The parameter $\eps$ is set to be 0.7. Results are shown in Figure (\ref{fig:supp-concen}). 

\begin{figure}[htb]
    \centering
    \subfloat[MS]{
        \begin{minipage}[t]{0.08\textwidth}
            \centering
            $\tau=0$ \\ 
            \includegraphics[width=1\textwidth]{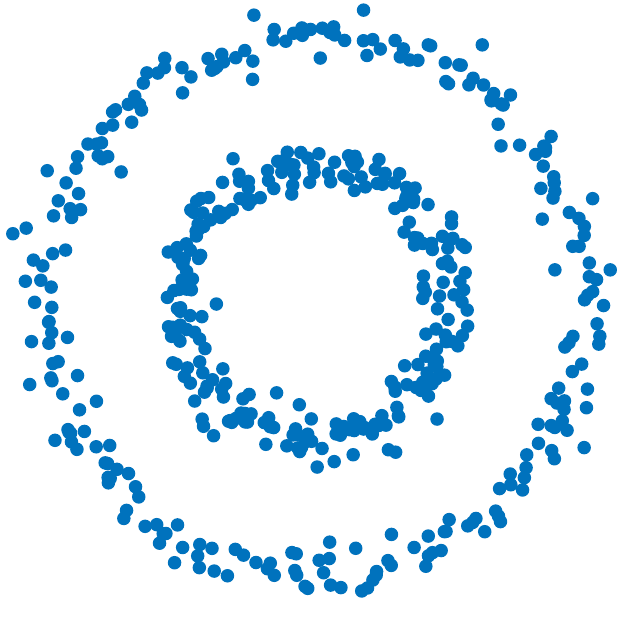}
        \end{minipage}
        \begin{minipage}[t]{0.08\textwidth}
            \centering
            $\tau=1$ \\ 
            \includegraphics[width=1\textwidth]{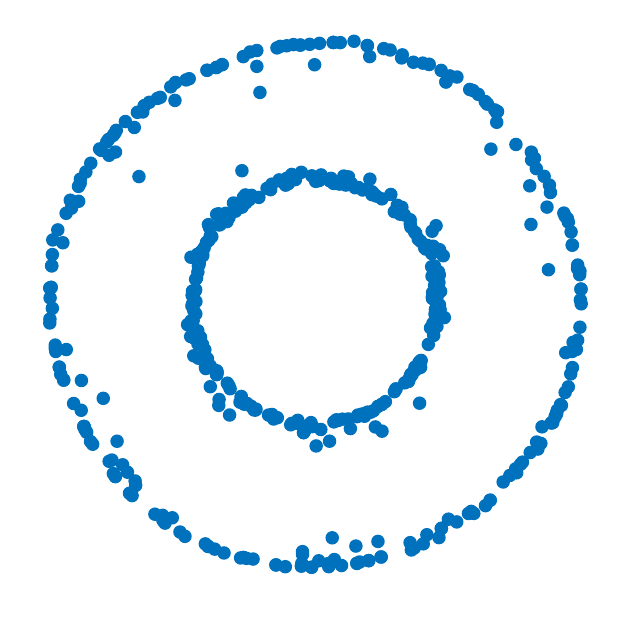}
        \end{minipage}
        \begin{minipage}[t]{0.08\textwidth}
            \centering
            $\tau=2$ \\ 
            \includegraphics[width=1\textwidth]{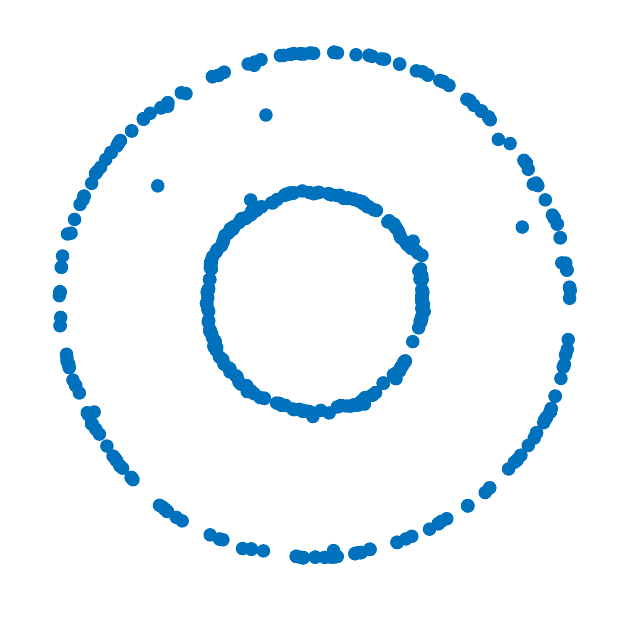}
        \end{minipage}
        \begin{minipage}[t]{0.08\textwidth}
            \centering
            $\tau=3$ \\ 
            \includegraphics[width=1\textwidth]{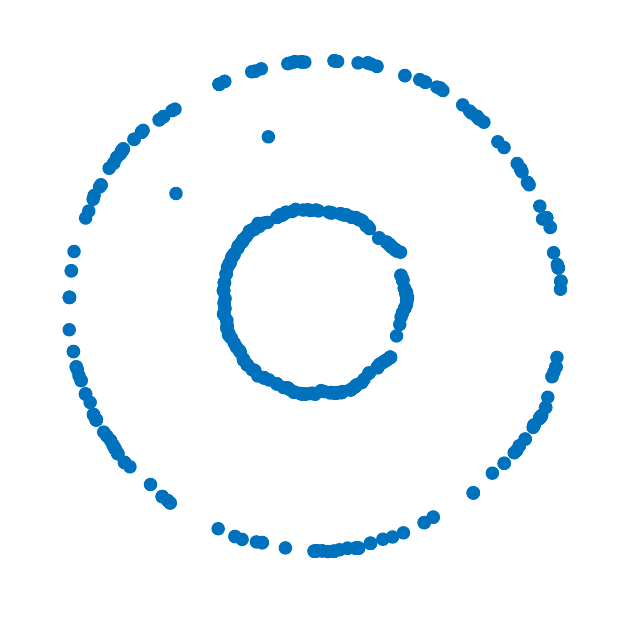}
        \end{minipage}
        \begin{minipage}[t]{0.08\textwidth}
            \centering
            $\tau=4$ \\ 
            \includegraphics[width=1\textwidth]{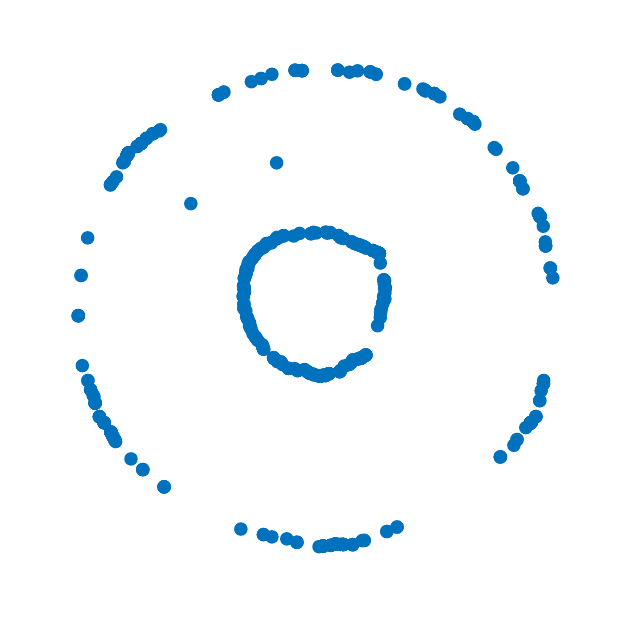}
        \end{minipage}
    }

    \subfloat[GT-2.5]{
        \begin{minipage}[t]{0.08\textwidth}
            \centering
            \includegraphics[width=1\textwidth]{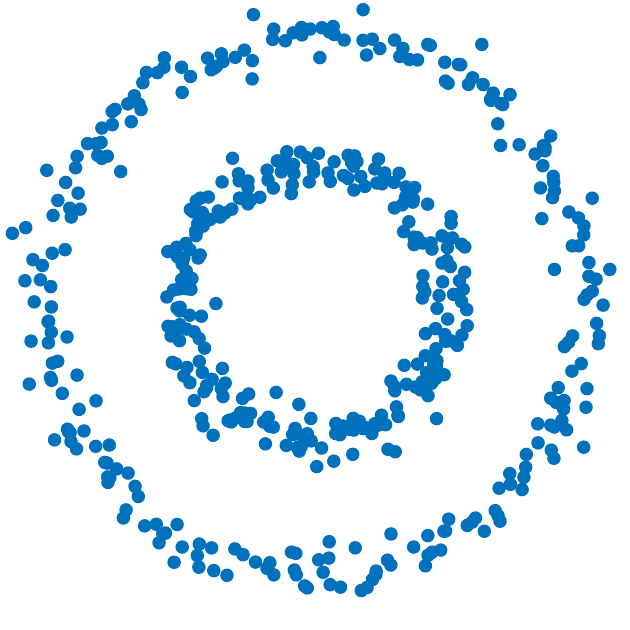}
        \end{minipage}
        \begin{minipage}[t]{0.08\textwidth}
            \centering
            \includegraphics[width=1\textwidth]{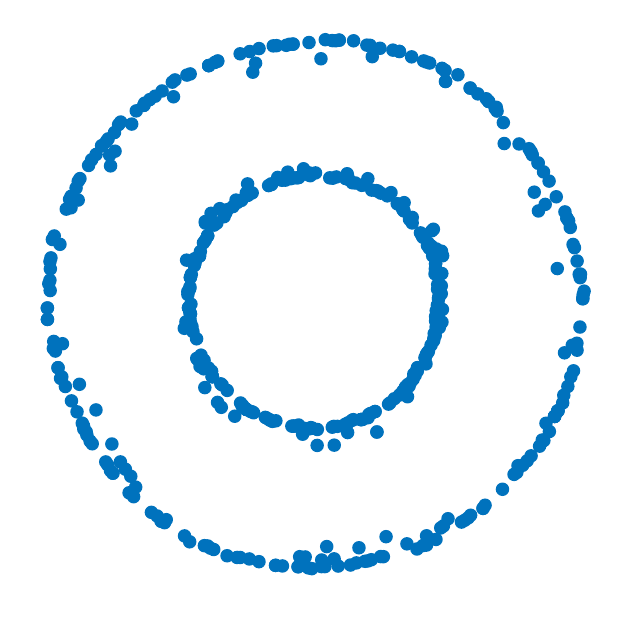}
        \end{minipage}
        \begin{minipage}[t]{0.08\textwidth}
            \centering
            \includegraphics[width=1\textwidth]{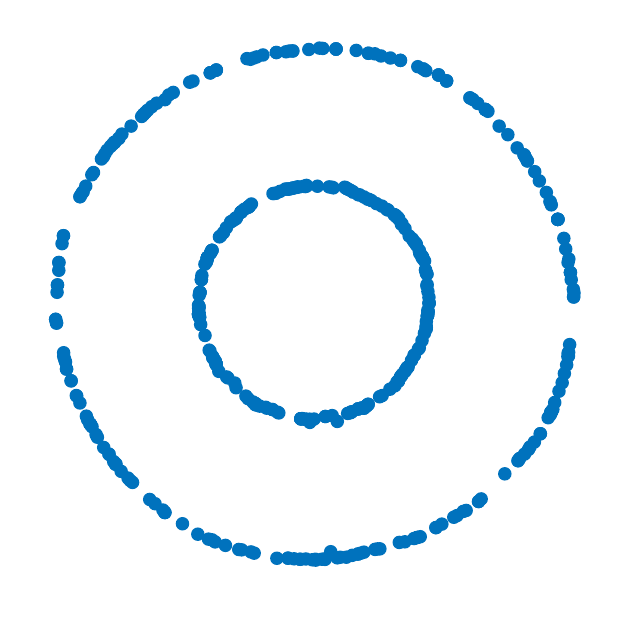}
        \end{minipage}
        \begin{minipage}[t]{0.08\textwidth}
            \centering
            \includegraphics[width=1\textwidth]{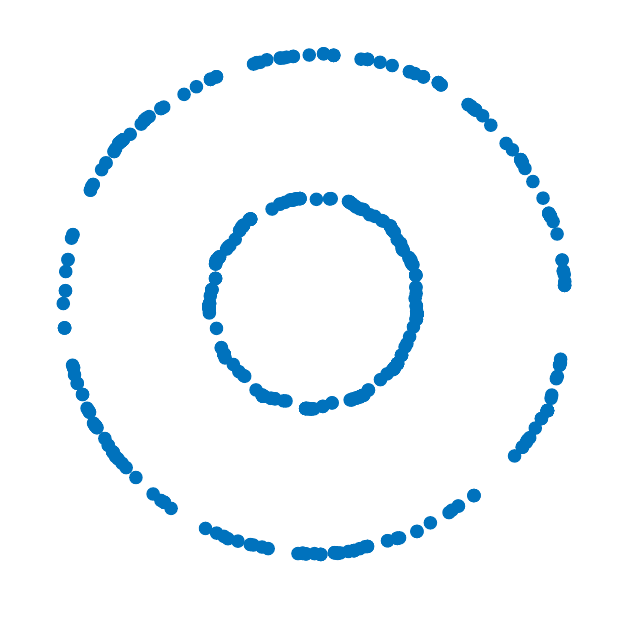}
        \end{minipage}
        \begin{minipage}[t]{0.08\textwidth}
            \centering
            \includegraphics[width=1\textwidth]{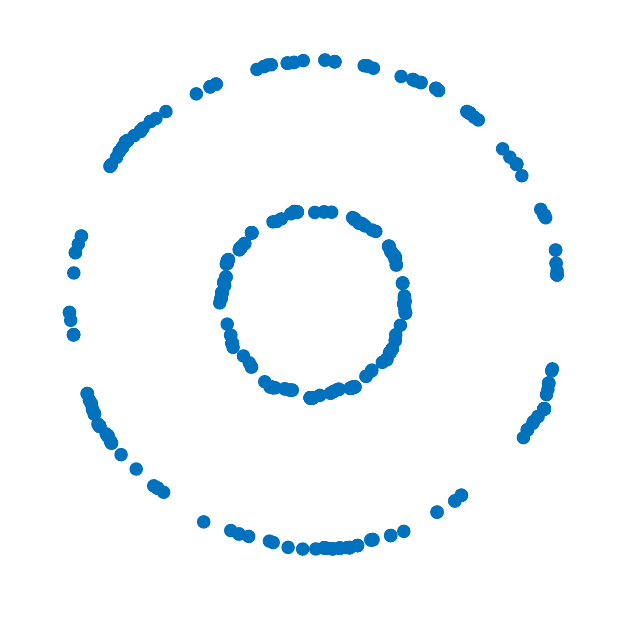}
        \end{minipage}
    }

    \subfloat[LT-WT2]{
        \begin{minipage}[t]{0.08\textwidth}
            \centering
            \includegraphics[width=1\textwidth]{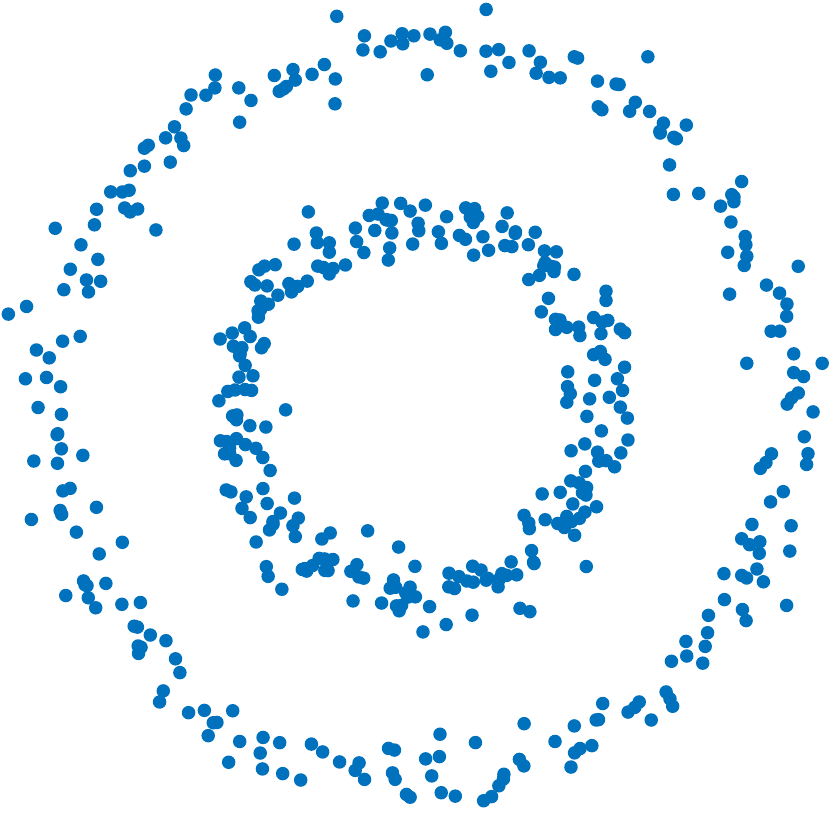}
        \end{minipage}
        \begin{minipage}[t]{0.08\textwidth}
            \centering
            \includegraphics[width=1\textwidth]{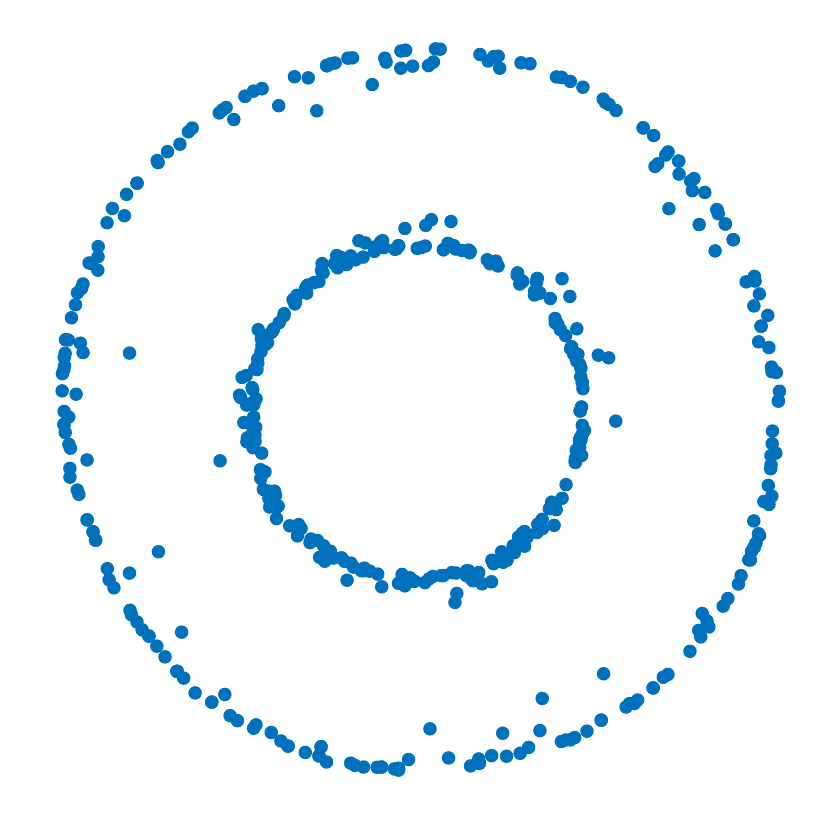}
        \end{minipage}
        \begin{minipage}[t]{0.08\textwidth}
            \centering
            \includegraphics[width=1\textwidth]{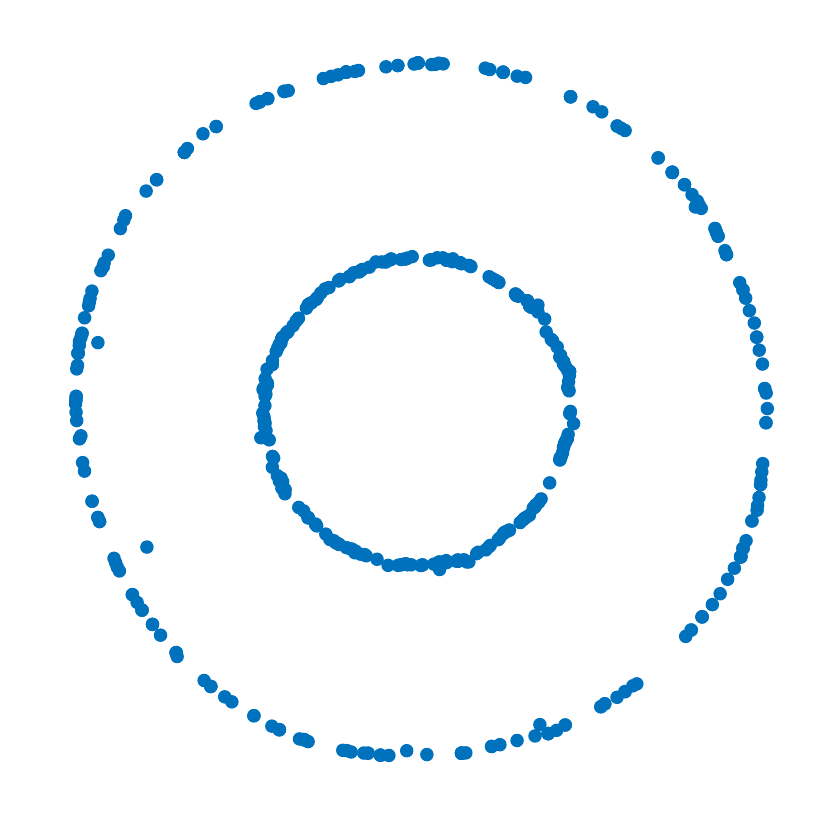}
        \end{minipage}
        \begin{minipage}[t]{0.08\textwidth}
            \centering
            \includegraphics[width=1\textwidth]{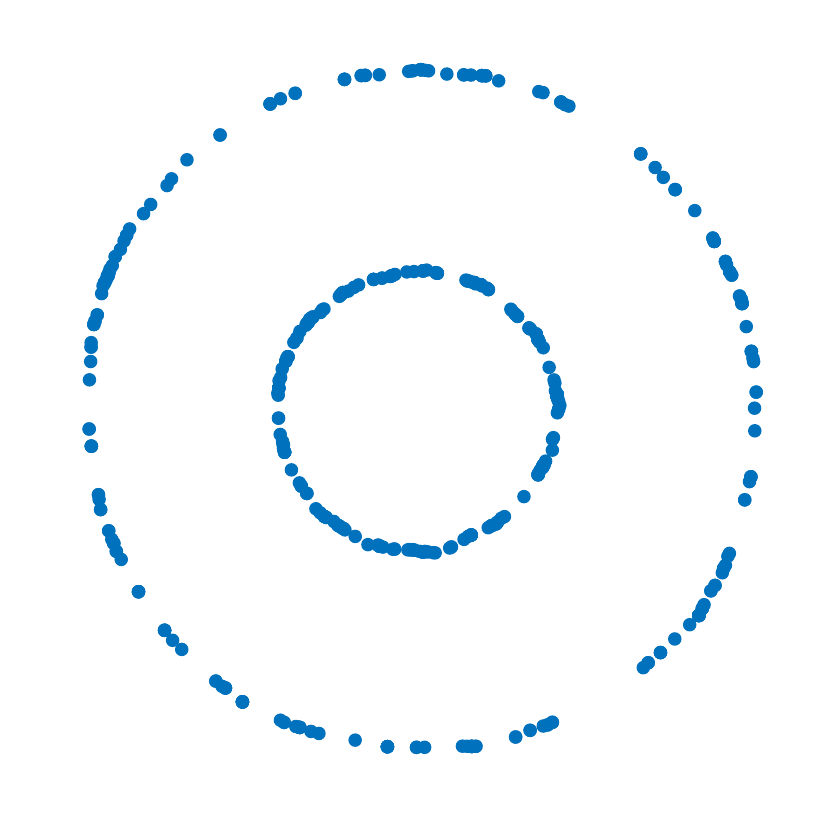}
        \end{minipage}
        \begin{minipage}[t]{0.08\textwidth}
            \centering
            \includegraphics[width=1\textwidth]{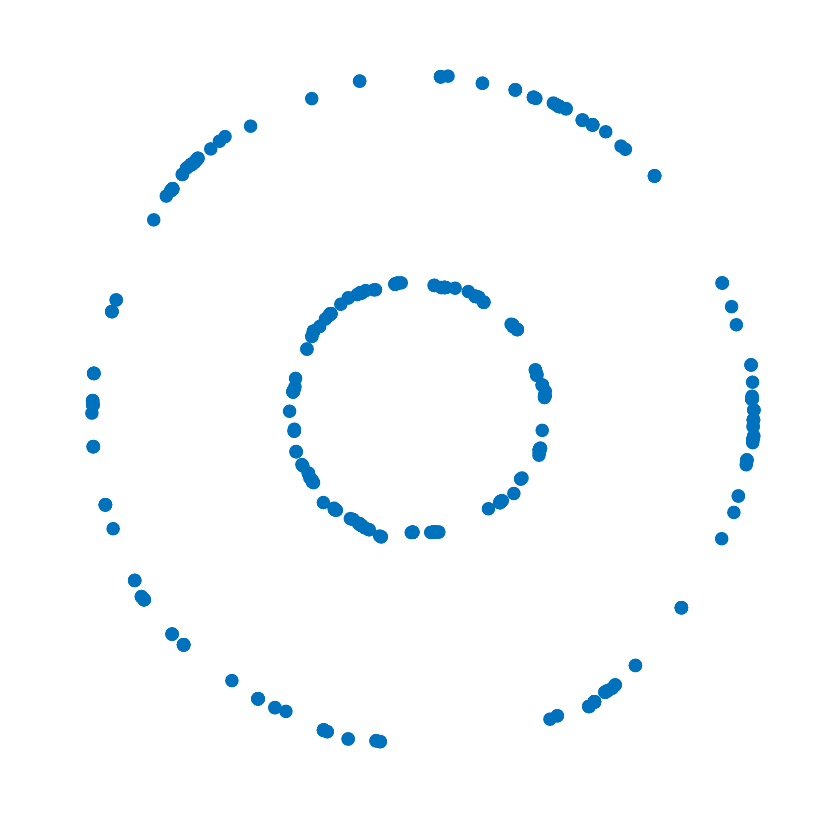}
        \end{minipage}
    }

    \subfloat[LT-WT1]{
        \begin{minipage}[t]{0.08\textwidth}
            \centering
            \includegraphics[width=1\textwidth]{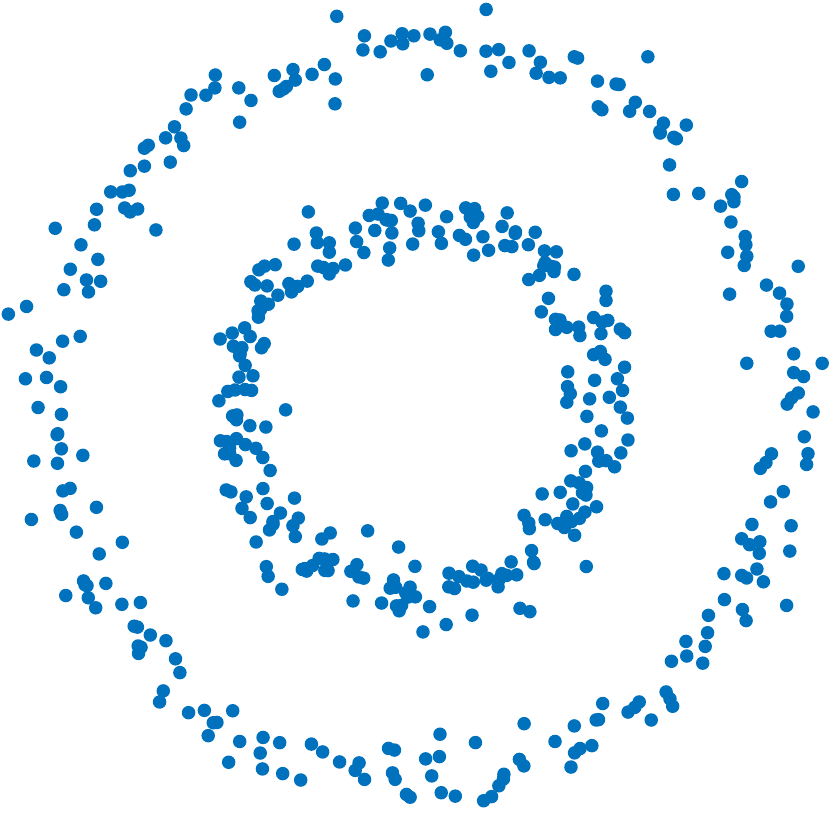}
        \end{minipage}
        \begin{minipage}[t]{0.08\textwidth}
            \centering
            \includegraphics[width=1\textwidth]{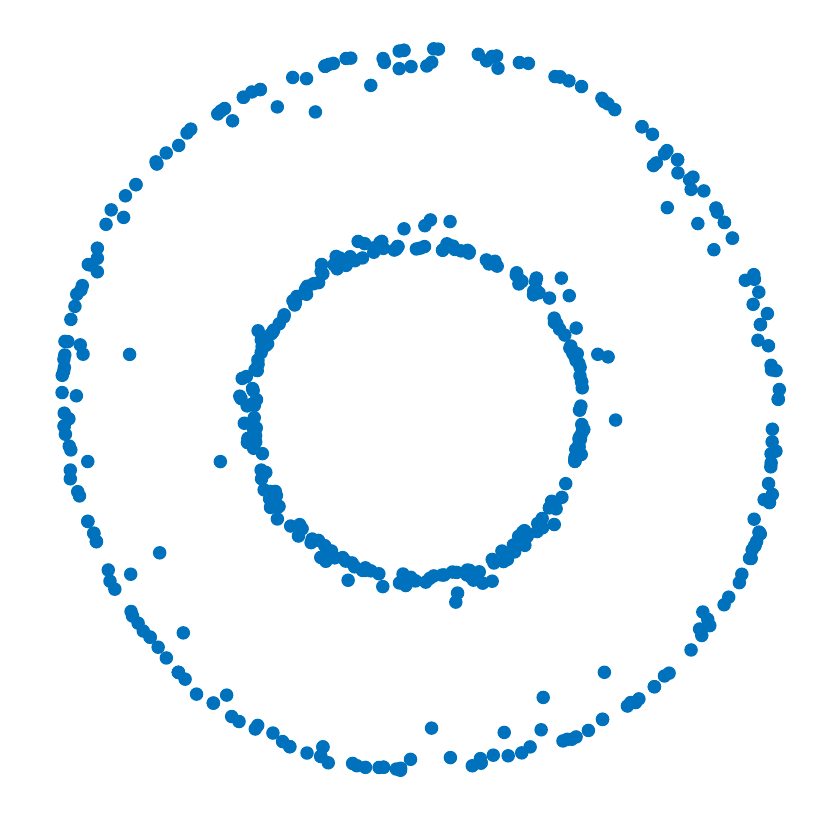}
        \end{minipage}
        \begin{minipage}[t]{0.08\textwidth}
            \centering
            \includegraphics[width=1\textwidth]{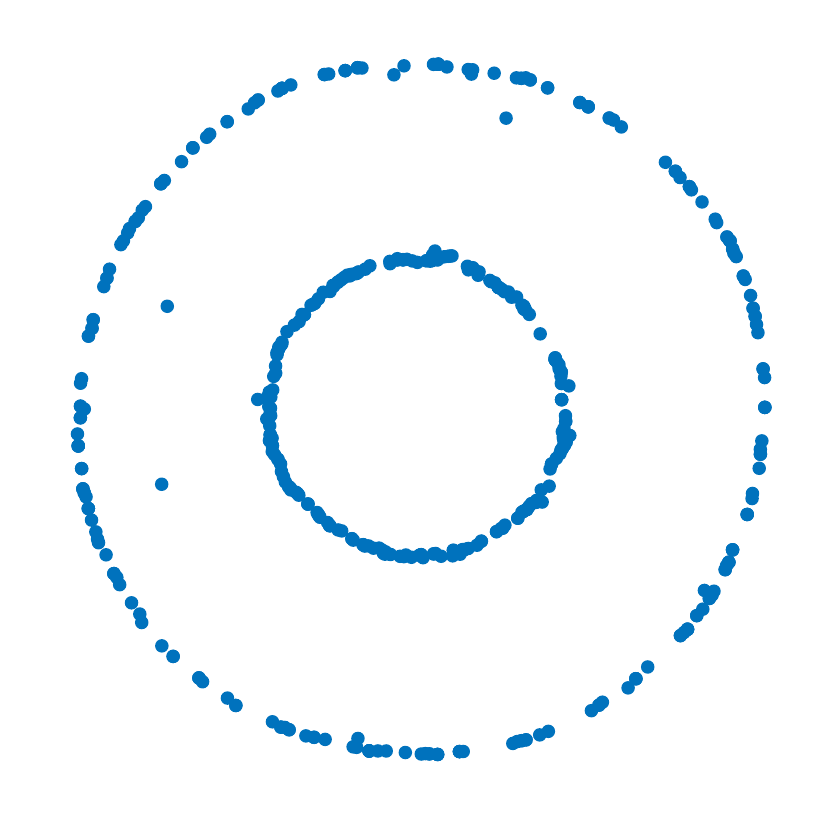}
        \end{minipage}
        \begin{minipage}[t]{0.08\textwidth}
            \centering
            \includegraphics[width=1\textwidth]{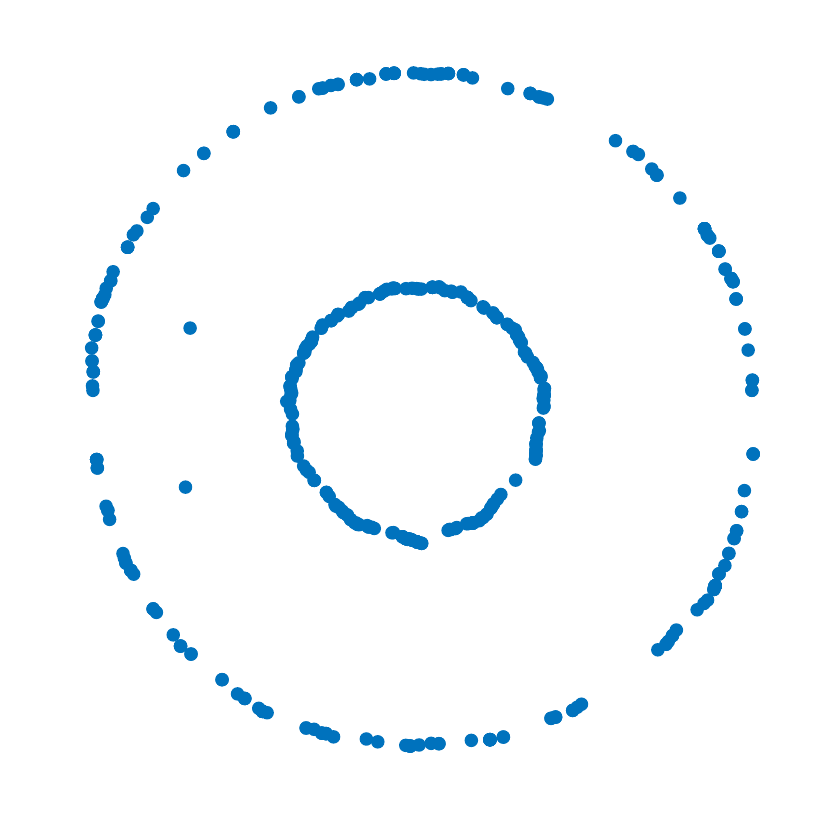}
        \end{minipage}
        \begin{minipage}[t]{0.08\textwidth}
            \centering
            \includegraphics[width=1\textwidth]{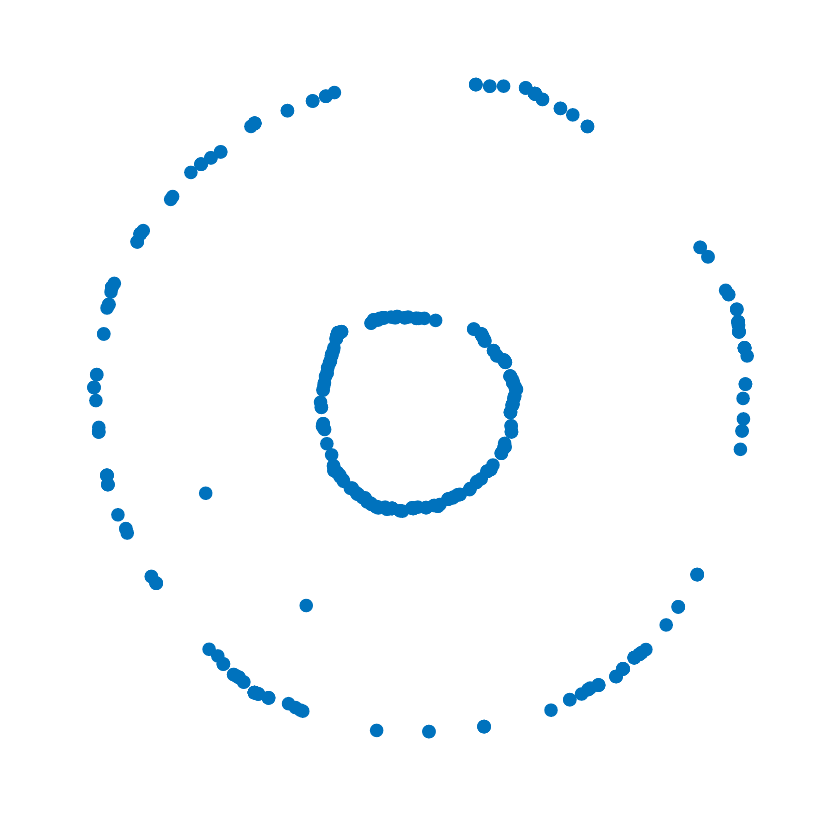}
        \end{minipage}
    }
    \caption{\textbf{Denoising of concentric circles.}}
    \label{fig:supp-concen}
\end{figure}

\paragraph{Observations.} 
Under these two types of noise, we see from Figure~(\ref{fig:supp-spiral}) and Figure~(\ref{fig:supp-concen}) that all methods resolve the shapes of the spiral as well as that of the concentric circles, to some extent. Notice how GT is able to better resolve the shapes of the two data sets compared to the other methods. In particular, it is better at displacing points towards the high density area around the spiral and the concentric circles. Besides its superior performance over LT-WT in these denoising tasks, we also emphasize that compared with LT-WT, GT is more computationally efficient thanks to the closed form formula for the GT distance (cf. Equation (\ref{eq:gtd2})).


\subsection{Image segmentation}\label{sec:img-seg}
Image segmentation is an important application domain in computer 
vision~\citep{szeliski2010computer} with
mean shift~\citep{comaniciu2002mean,demirovic2019implementation} being an effective method in this area. We first review MS as applied to image segmentation and then show how to comparatively apply GT to this task. Note that LT-WT is not applicable in this experiment since the process of image segmentation involves feature updates whereas LT-WT only updates/retains distance matrices. 

\paragraph{Mean shift for image segmentation.}
There are multiple variants of MS for image segmentation \citep{tao2007color,demirovic2019implementation}. We follow the MS implementation in \citep{demirovic2019implementation} which will be explained as follows.

Given an image $X$, each pixel $x\in X$ consists of two types of features: the \emph{spatial feature} and the \emph{range feature}, denoted by $x^s$ and $x^r$, respectively. $x^s$ is represented by a point in $\R^2$, whereas $x^r$ is represented by a point in the three dimensional L*u*v* color space \citep{comaniciu2002mean}, in which Euclidean distances approximate perceptual differences in color better than the Euclidean distance in the RGB space. 

The MS from \citep{demirovic2019implementation} has two bandwidth parameters: one spatial bandwidth parameter $\eps^s$ and one range bandwidth parameter $\eps^r$. 
For any pixel $x=(x^s,x^r)$, we define its \emph{$(\eps^s,\eps^r)$-neighborhood} to be the set of all pixels $y=(y^s,y^r)$ such that 
$$\| y^s - x^s \| \leq \eps^s\text{ and }\| y^r - x^r \| \leq \eps^r.$$
Then, the MS algorithm from \citep{demirovic2019implementation} is described below:

\begin{enumerate}
\item Data: pixels are given in the form of 5-dimensional feature points $\{x_i = (x_i^s, x_i^r)\}_{i=1}^M$;
    \item For each $x_i = (x_i^s, x_i^r)$, let $T^s(x_i)\coloneqq x_i^s$ and $T^r(x_i)\coloneqq x_i^r$. Further define $T(x_i)\coloneqq (T^s(x_i),T^r(x_i))$. Repeat the following steps to update $T(x_i)$ for each $i=1,\ldots,M$ until $T(x_i)$ converges:
    \begin{enumerate}

        \item determine the \emph{$(\eps^s,\eps^r)$-range-neighborhood} of $T(x_i)=(T^s(x_i),T^r(x_i))$, which consists of all pixels $x_j=\lc x_j^s,x_j^r\rc$ such that $\| x_j^s - x_i^s \| \leq \eps^s$ and $\| x_j^r - T^r(x_i) \| \leq \eps^r$;
        \item update $T(x_i)$ as the mean of its $(\eps^s,\eps^r)$-range-neighborhood.
    \end{enumerate}
    \item Compute a clustering of points $\{x_i\}_{i=1}^M$: we construct a graph $G$ such that its vertex set is $\{T(x_i)\}_{i=1}^M$ and its edge set consists of all pairs $\{T(x_i),T(x_j)\}$ such that $\norm{T^s(x_i)-T^s(x_j)}\leq\eps^s$ and $\norm{T^r(x_i)-T^r(x_j)}\leq\eps^r$. Then, the set of connected components of $G$ forms a clustering of $\{T(x_i)\}_{i=1}^M$. This in turn induces a clustering/segmentation of the pixels $x_i$s according to the following rule: $x_i$ and $x_j$ belong to the same cluster iff $T(x_i)$ and $T(x_j)$ belong to the same cluster.
\end{enumerate}

\paragraph{GT for image segmentation.}
We adapt GT to the image segmentation task following a strategy similar to the one used in the MS algorithm described above as follows:
\begin{enumerate}
\item Data: pixels are given in the form of 5-dimensional feature points $\{x_i = (x_i^s, x_i^r)\}_{i=1}^M$;
    \item Initialization:  compute the $2\times 2$ covariance matrix $\Sigma(x_i^s)$ of spatial features of points in the $(\eps^s,\eps^r)$-neighborhood of $x_i$.
    \item For each $x_i = (x_i^s, x_i^r)$, let $T^s(x_i)\coloneqq x_i^s$ and $T^r(x_i)\coloneqq x_i^r$. Further define $T(x_i)\coloneqq (T^s(x_i),T^r(x_i))$. Repeat the following steps to update $T(x_i)$ for each $i=1,\ldots,M$ until $T(x_i)$ converges:
    \begin{enumerate}
        \item for each $x_j$ within the $(\eps^s,\eps^r)$-range-neighborhood of $T(x_i)$ (i.e., $\| x_j^s - x_i^s \| \leq \eps^s$ and $\| x_j^r - T^r(x_i) \| \leq \eps^r$), compute via the following variant of the GT distance (cf. Equation~(\ref{eq:gtd})) between the spatial features of $x_j$ and $T(x_i)$:
        \begin{equation}\label{eq:image-gt}
         d^{\mathsmaller{(\eps,\lambda)}}_{\alpha,d_s} \lc x_j^s, T^s(x_i)\rc\coloneqq\lc\norm{x_i^s-x_j^s}^2+\lambda\cdot\lc\dcov\left(\Sigma\left(x_j^s\right),\Sigma(T^s(x_i))\right)\rc^2\rc^\frac{1}{2}, 
        \end{equation}
        where $d_s$ refers to the Euclidean distance on the spatial features.

        \item determine the \emph{$(\eps^s,\eps^r)$-GT-neighborhood} of $T(x_i)$, which consists of all pixels $x_j=(x_j^s,x_j^r)$ satisfying both $d^{\mathsmaller{(\eps,\lambda)}}_{\alpha,d_s} \left(x_j^s, T^s(x_i)\right) \leq \eps^s$ and $\left\| T^r(x_i) - x_j^r \right\| \leq \eps^r$; 
        \item update $T(x_i)$ as the mean of the $(\eps^s,\eps^r)$-GT-neighborhood and compute the 2-dimensional covariance matrix $\Sigma(T^s(x_i))$ of spatial features of points in the $(\eps^s,\eps^r)$-GT-neighborhood of $T(x_i)$.
    \end{enumerate}
    \item Compute a clustering of points $\{x_i\}_{i=1}^M$: this step is the same as the corresponding one in the MS case.
\end{enumerate}

We remark that when $\lambda=0$, the above GT based segmentation algorithm reduces to the MS based segmentation algorithm given above.

\paragraph{Experiments.} We apply both GT and MS to the image segmentation task on the standard cameraman images with different resolutions. The results are shown in Figure~(\ref{fig:supp-cameraman}). Notice that when the image is of high resolution (Figure (\ref{fig:supp-image1})), GT performs as well as MS. When the image is of low resolution (Figure (\ref{fig:supp-image2})), we see that GT induces a reasonably better segmentation than MS does. The labels marked on the test image 
correspond to the major different segments that MS and GT recognize.

\begin{figure}[htb]
    \centering
    \subfloat[Test image1]{
        \label{fig:supp-image1}
        \includegraphics[width=0.15\textwidth]{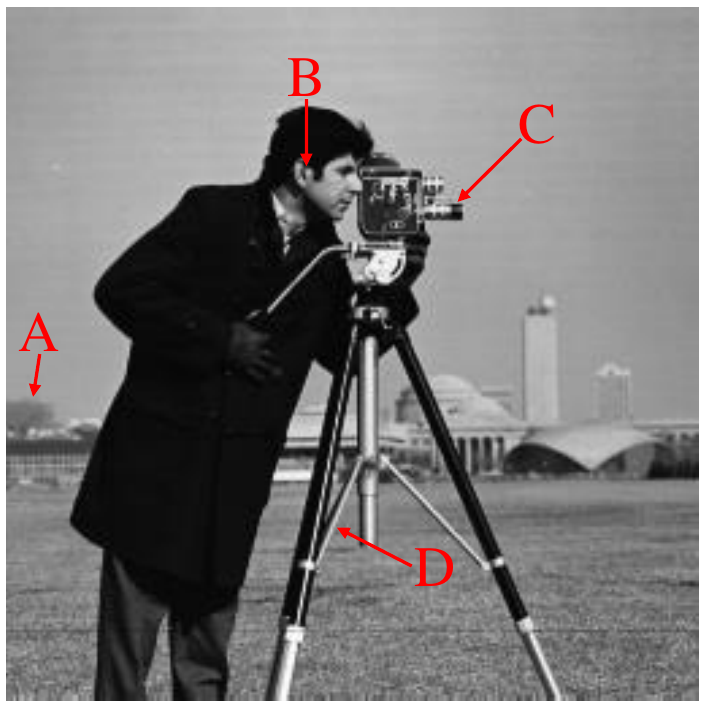}
	}
	\subfloat[MS]{
        \includegraphics[width=0.15\textwidth]{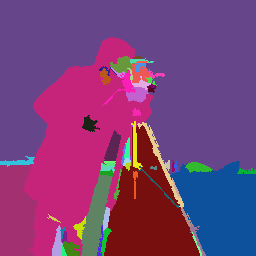}
	}
	\subfloat[GT]{
        \includegraphics[width=0.15\textwidth]{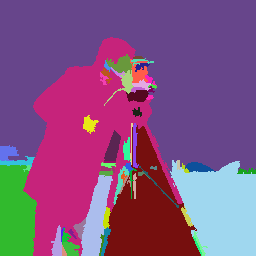}
	}
	\subfloat[Test image2]{
        \label{fig:supp-image2}
        \includegraphics[width=0.15\textwidth]{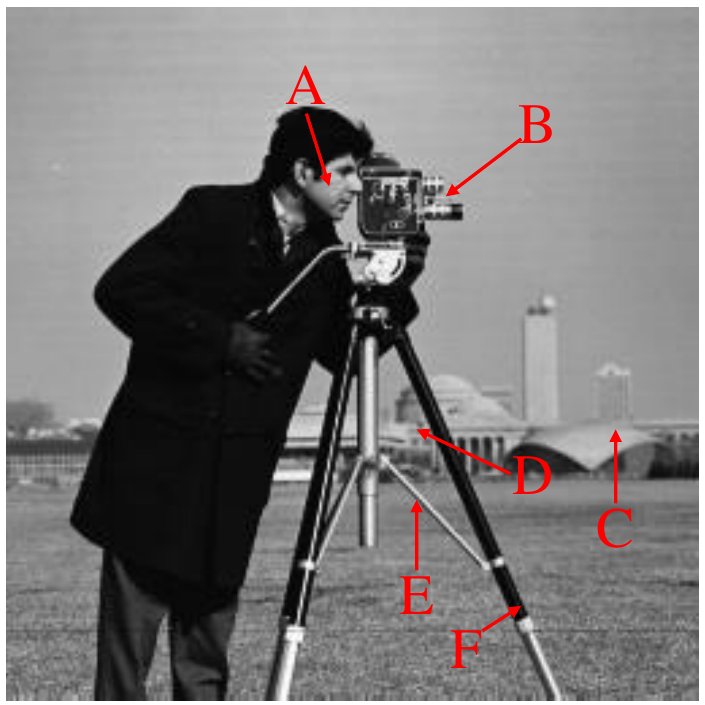}
	}
	\subfloat[MS]{
        \includegraphics[width=0.15\textwidth]{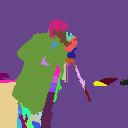}
	}
	\subfloat[GT]{
        \includegraphics[width=0.15\textwidth]{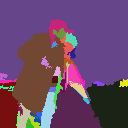}
	}
    \caption{\textbf{Image segmentation. }
    (a) Test image1 (cameraman $256 \times 256$ grayscale).  (b) MS segmentation with $\eps^s=8$ and $\eps^r=8$. 
    (c) GT segmentation with $\eps^s=8, \eps^r = 8, \lambda=4.8$. (d) Test image2 (cameraman $128\times 128$ grayscale). (e) MS segmentation with $\eps^s=6$ and $\eps^r=6$. (f) GT segmentation with $\eps^s=6, \eps^r = 6, \lambda=5$.  Check differences between GT and MS results at the labeled areas in test images. } 
    \label{fig:supp-cameraman}
\end{figure}

\subsection{Boosting  word embeddings in NLP}
\label{sec:word-emb}
Word embedding methods are important techniques in natural language processing~\citep{mikolov2013distributed,Vilnis2014WordRV,muzellec2018generalizing,pennington2014glove,Devlin2019BERTPO}. A basic instantiation of this idea is that one vectorizes each word in a given corpus by mapping it to a feature vector in a context sensitive way. 
Such ideas are applied widely in many NLP tasks, such as machine translation~\citep{zou2013bilingual}, word analogy~\citep{pmlr-v97-allen19a}, and name entity recognition~\citep{das2017named}. 

There are many open-source word
embeddings \footnote{\url{https://nlp.stanford.edu/projects/glove/}} \footnote{\url{https://gluon-nlp.mxnet.io/model\_zoo/bert}}  which have been pre-trained on very large and rich corpora (such as wikipedia). Given any new corpus $\mathcal{C}$, instead of training an embedding from scratch (which can be computationally intensive~\citep{anand2019asynchronous}), perhaps naively, one would hope to directly adapt $\mathcal{C}$ to a pre-trained embedding in order to obtain a new word embedding which may perform better than the original embedding on standard tasks or specific tasks related to the theme of $\mathcal{C}$. As a proof of concept, in this experiment, we explore the possibility of applying GT to boosting a pre-trained word embedding with a relatively small corpus $\mathcal{C}$ in order to improve the performance of the pre-trained embedding on standard word similarity tasks.

\subsubsection{GT for word embeddings}
We consider a given pre-trained embedding (in our case, we are using the GloVe embedding \citep{pennington2014glove} pre-trained on Wikipedia2014 and Gigaword 5.) as a map $\Omega:\mathrm{Dict}\rightarrow \R^m$ where $\mathrm{Dict}$ is the universe of all words under consideration. We normalize $\Omega$ and still denote the new map by $\Omega$ so that for each $w\in \mathrm{Dict}$, $\norm{\Omega{(w)}}=1$. For each word $w\in \mathrm{Dict}$, we abuse notation and also use $w$ to represent the embedding $\Omega(w)\in\R^m$. Given a certain corpus $\mathcal{C}$, and a word $w$ in $\mathcal{C}$, let $c_\mathcal{C}(w)$ denote the multiset of all words in the corpus $\mathcal{C}$ which are found in the context of $w$ with a given fixed window size $W$. We regard $c_\mathcal{C}(w)$ as the neighborhood of the word $w$ (note that this context neighborhood is a conceptual neighborhood which is different from our previous notion of neighborhood defined using distances). We compute the covariance matrix $\Sigma(w)$ for each $c_\mathcal{C}(w)$\footnote{Strictly speaking, $\Sigma(w)$ is not the covariance matrix of $c_\mathcal{C}(w)$ but instead the covariation of points in $c_\mathcal{C}(w)$ around $w$.} described in~\cite{Vilnis2014WordRV} as follows:
\begin{equation}\label{eq:cov-nlp}
    \Sigma(w) \coloneqq \frac{1}{|c_\mathcal{C}(w)|} \sum_{v\in c_\mathcal{C}(w)} (v - w)(v - w)^\mathrm{T}.
\end{equation}
If there exists no context word for $w$, i.e., if $c_\mathcal{C}(w)=\emptyset$, we set $\Sigma(w) = {0}$, the $m\times m$ dimensional zero matrix. 
Then, the GT distance between any pair of words $w_1,w_2\in\mathcal{C}$ is defined as follows: 
\begin{equation}\label{eq:gt word}
    d^{\mathsmaller{(W,\lambda)}}(w_1,w_2) \coloneqq \lc\|w_1 - w_2\|^2 + \lambda \,  \lc\dcov(\Sigma(w_1), \Sigma(w_2))\rc^2\rc^\frac{1}{2}.
\end{equation}
We use $-d^{\mathsmaller{(W,\lambda)}}(w_1,w_2)$ as the \emph{similarity score} between two words $w_1$ and $w_2$ (note the minus sign).

\paragraph{Preprocessing.}
We let $\mathcal{C}$ be the corpus text8~\footnote{http://mattmahoney.net/dc/text8.zip}. We preprocess the corpus text8 in two steps: (1) we discard rare words which occur less than 5 times in $\mathcal{C}$; (2) we discard frequent words following the strategy proposed in~\citep{mikolov2013distributed}: each word $w$ in the corpus $\mathcal{C}$ will be discarded with probability depending negatively on the frequency of $w$.

\paragraph{Evaluation on benchmarks.}
We evaluate the quality of any given word embedding on 13 different standard word similarity benchmarks: 
MC-30~\citep{miller1991contextual}, MEN-TR-3k~\citep{bruni2014multimodal}, MTurk-287~\citep{radinsky2011word}, 
MTurk-771~\citep{halawi2012large}, RG-65~\citep{rubenstein1965contextual}, RW-STANFORD~\citep{luong2013better}, SIMPLEX-999~\citep{hill-etal-2015-simlex}, SimVerb-3500~\citep{gerz2016simverb}, 
VERB-143~\citep{baker2014unsupervised}, WS-353~\citep{finkelstein2001placing}, WS-YP-130~\citep{yang2006verb}. In these benchmarks, similarity scores between certain pairs of words are provided. We refer to them as \emph{human similarity scores} (`ground truth'). Then, we calculate the Spearman rank correlation coefficient~\citep{spearman1961proof} between the human similarity scores and the similarity scores on the word pairs induced from the given word embedding. 
If one word embedding has high Spearmen rank correlation coefficient, then we regard this word embedding as of high quality.

\paragraph{Experiment and results.} We apply GT described above (GloVe+GT) to obtain similarity scores and evaluate the results on the benchmarks mentioned above. In Table \ref{tab:word-compare}, for each evaluation word similarity benchmark, we compare the Spearman rank correlation coefficients corresponding to GloVe+GT with those corresponding to the pre-trained GloVe Embeddings (GloVe), GloVe Embeddings (GloVe*text8) \citep{pennington2014glove} and Word2Vec (W2V*text8) \citep{mikolov2013distributed} trained on the preprocessed corpus text8. We have the following observations: 
{GloVe+GT} outperforms GloVe in most of the evaluation data sets, and has comparable performance on the remaining data sets. Moreover, GloVe+GT outperforms models GloVe*text8 and W2V*text8 trained specifically on text8 in most evaluation data sets.

We also compare the Spearman rank correlation coefficient of GloVe+GT with the ones given by elliptical embeddings (Ell) \citep{muzellec2018generalizing} and diagonal Gaussian embeddings (W2G) \citep{Vilnis2014WordRV} trained on the larger corpora ukWaC and WaCkypedia. These methods are in a similar spirit to our approach in that they train word embeddings into the space of probability measures of some ambient space. We point out, again, that the small corpus text8 has improved the performance of GloVe so that the performance of GloVe+GT is comparable with the performance of Ell and W2G which were trained on much larger corpora (cf. Table \ref{tab:word-compare}).

\begin{remark}
\citet{singh2020context} have carried out an experiment whose setting is similar to ours. They considered an embedding trained on a given corpus $\mathcal{C}$ and used the context information of $\mathcal{C}$ itself (instead of another small corpus as in our experiment) to further generate similarity scores via the Wasserstein distance in a way similar to our application of GT above. We remark that our approach only requires a small corpus to improve the word embedding performance on word similarity tasks, rather than reusing the large training corpus in both stages. As for computation, their methods have to rely on Sinkhorn algorithm to approximate the Wasserstein distance \citep{cuturi2013sinkhorn}, whereas we can carry out exact calculation of Wasserstein distance efficiently due to the closed form formula for the GT distance (\Cref{eq:gt word}).
\end{remark}

\begin{table*}[htb]
\caption{\textbf{Spearman rank correlation coefficients for word similarity data sets.} The results for  ``{Ell}" and ``{W2G}" are directly from ~\cite{muzellec2018generalizing}.}
    \centering
    \begin{tabular}{| c| c| c| c| c|| c| c|}
    \hline
    \textbf{Data set} & \textbf{GloVe} & \textbf{GloVe*text8} & \textbf{W2V*text8} & \textbf{GloVe+GT}  & \textbf{Ell} & \textbf{W2G}\\
    \hline
        MC-30 & 0.56  & 0.34 & 0.57 &  \textbf{0.67} & 0.65 & 0.59  \\ \hline 
        MEN-TR-3k & \textbf{0.65}  & 0.37 & 0.59 & \textbf{0.65}  & 0.65 & 0.65 \\ \hline 
        MTurk-287 & 0.61 & 0.49 & 0.61 &  \textbf{0.62} & 0.59 & 0.61  \\ \hline 
        MTurk-771 & 0.55 & 0.36 & 0.50 &  \textbf{0.56}& 0.56 & 0.57  \\ \hline 
         RG-65 & 0.60 & 0.33 &  0.56 &  \textbf{0.62}& 0.65 & 0.69  \\ \hline 
         RW-STANFORD & 0.34  & 0.20 & 0.25 &  \textbf{0.38} & 0.29 & 0.40 \\ \hline 
         SIMLEX-999 & 0.26 &  0.13 & 0.22 &  \textbf{0.27} & 0.24 & 0.25 \\ \hline 
        SimVerb-3500 & \textbf{0.15} & 0.07 & 0.08  & 0.14  & - & - \\ \hline
        VERB-143 & 0.25  & 0.28 & \textbf{0.32}  & 0.24  & - & - \\ \hline 
         WS-353-ALL & 0.49  & 0.43 & \textbf{0.62} & 0.51  & 0.66 & 0.53 \\ \hline 
        WS-353-REL & 0.46  & 0.41  & \textbf{0.59} & 0.47  & 0.71 & 0.61\\ \hline 
        WS-353-SIM & 0.57  & 0.51 & \textbf{0.66}  & 0.60  & 0.60 & 0.48 \\ \hline 
        WS-YP-130 & \textbf{0.37} & 0.19 &  0.23  & \textbf{0.37}  & 0.25 & 0.37 \\ \hline 
    \end{tabular}
        \label{tab:word-compare}
\end{table*}
\section{Relegated proofs}\label{sec:wt-proof}

\begin{proof}[Proof of Remark \ref{rem:taylor-1d}]
We let $F_x$ and $F_{x'}$ represent the cumulative distribution functions of $\me_{\alpha,d_X}(x)$ and $\me_{\alpha,d_X}(x')$, respectively. By \Cref{ex:dw on R},
\begin{equation*}
   \dW{1}\lc\me_{\alpha,d_X}(x),\me_{\alpha,d_X}(x')\rc=\int_{-\infty}^\infty|F_x(t)-F_{x'}(t)|dt.
\end{equation*}

We have explicitly
$$ F_x(t)=
\begin{cases}
0,       &      {t\leq x-\eps}\\
\frac{\alpha\lc[x-\eps,t]\rc}{\alpha\lc[x-\eps,x+\eps]\rc},     &       {x-\eps \leq t < x+\eps}\\
1,    &      {x+\eps \leq t }
\end{cases}  $$
and 
$$ F_{x'}(t)=
\begin{cases}
0,       &      {t\leq x'-\eps}\\
\frac{\alpha\lc[x'-\eps,t]\rc}{\alpha\lc[x'-\eps,x'+\eps]\rc},     &       {x'-\eps \leq t < x'+\eps}\\
1,    &      {x'+\eps \leq t }
\end{cases}  .$$

It is easy to see that $F_x(t)=F_{x'}(t)$ for  $t\in(-\infty,x-\eps]\cup[x'+\eps,\infty)$. 
If we assume that $\eps$ is small enough such that $x+\eps\leq x'-\eps$, then
\begin{align*}
\int_{-\infty}^\infty|F_x(t)-F_{x'}(t)|dt 
=&\lc\int_{-\infty}^{x-\eps}+\int_{x-\eps}^{x'+\eps}+\int_{x'+\eps}^\infty\rc|F_x(t)-F_{x'}(t)|dt\\
 = &\int_{x-\eps}^{x'+\eps}|F_x(t)-F_{x'}(t)|dt\\
 =&\int_{x-\eps}^{x'-\eps}F_x(t)dt+\int_{x'-\eps}^{x'+\eps}|1-F_{x'}(t)|dt\\
 = & x'-x+\underbrace{\int_{x-\eps}^{x+\eps}\frac{\alpha\lc[x-\eps,t]\rc}{\alpha\lc[x-\eps,x+\eps]\rc}dt}_{H(x)}
 -\underbrace{\int_{x'-\eps}^{x'+\eps}\frac{\alpha\lc[x'-\eps,t]\rc}{\alpha\lc[x'-\eps,x'+\eps]\rc}dt}_{H(x')}
\end{align*}

Recall that $\alpha$ has density function $f$. Consider the Taylor expansion
\[f(s)=f(x)+f'(x)s+\frac{f''(x)}{2}s^2+O(s^3),\,\,\forall s\in[x-\eps,x+\eps],\]
then we have that

\begin{align*}
H(x) & =\frac{\int_{x-\eps}^{x+\eps}\int_{x-\eps}^{t} f(s)\,ds\,dt}{\int_{x-\eps}^{x+\eps} f(s)ds}=\frac{\int_{-\eps}^{\eps}\int_{-\eps}^{t} f(x+s)\,ds\,dt}{\int_{-\eps}^{\eps} f(x+s)ds}\\
& =\frac{\int_{-\eps}^{\eps}\int_{-\eps}^{t} \lc f(x)+f'(x)s+\frac{f''(x)}{2}s^2+O(s^3)\rc ds\,dt}{\int_{-\eps}^{\eps} \lc f(x)+f'(x)s+\frac{f''(x)}{2}s^2+O(s^3)\rc ds}\\
& = \frac{2\eps^2 f(x)+\frac{2\eps^3}{3}f'(x)+\frac{\eps^4}{3}f''(x)+O(\eps^5)}{2\eps f(x)+\frac{\eps^3}{3}f''(x)+O(\eps^4)}\\
& =\eps-\frac{f'(x)}{3f(x)}\eps^2+O(\eps^3).
\end{align*}
Similarly, 
$$H(x')=\eps-\frac{f'(x')}{3f(x')}\eps^2+O(\eps^3).$$
Therefore
\begin{align*}
\deps(x,x') & =\dW{1}\lc\me_{\alpha,d_X}(x),\me_{\alpha,d_X}(x')\rc\\
& =x'-x+H(x)-H(x')\\
& =x'-x+\frac{1}{3}\left(\frac{f'(x')}{f(x')}-\frac{f'(x)}{f(x)}\right)\eps^2+O(\eps^3).
\end{align*}
\end{proof}
\begin{proof}[Proof of Proposition \ref{prop:WT as quotient}]
We say two pseudometric spaces are isometric if their canonically induced metri spaces are isometric. In this proposition, it suffices to prove that for any $x,x'\in X$ we have
$$\deps(x,x')=u_{X_\eps}([x]_{\eps},[x']_{\eps}).$$

If $u_X(x,x')\leq \eps$, then we have that
\[B_\eps(x')=[x']_{\eps}=[x]_{\eps}=B_\eps(x).\]
Therefore, $\me_{\alpha,d_X}(x)=\me_{\alpha,d_X}(x')$ and thus $\deps(x,x')=0=u_{X_\eps}([x]_{\eps},[x']_{\eps})$.

If otherwise $u_X(x,x')>\eps$, then $B_\eps(x)\cap B_\eps(x')=\emptyset$. Moreover, for any $y\in B_\eps(x)$ and $y'\in B_\eps(x')$, 
\[u_X(y,y')=u_X(x,x').\]
Then, for any coupling $\pi$ between $\me_{\alpha,d_X}(x)$ and $\me_{\alpha,d_X}(x')$, we have that 
\begin{align*}
   \int_{X\times X}u_X(y,y')\pi(dy\times dy')&=\int_{B_\eps(x)\times B_\eps(x')} u_X(y,y')\pi(dy\times dy')\\
   &=\int_{B_\eps(x)\times B_\eps(x')} u_X(x,x')\pi(dy\times dy')\\
   &=u_X(x,x').
\end{align*}
Therefore, $\deps(x,x')=\dW{1}\lc \me_{\alpha,d_X}(x),\me_{\alpha,d_X}(x')\rc=u_X(x,x')=u_{X_\eps}([x]_{\eps},[x']_{\eps})$.
\end{proof}
\begin{proof}[Proof of Lemma \ref{prop:line}]
Since $\mu_{\alpha_i}$ is the mean of $\alpha_i$ for $i=1,2$, the leftmost inequality in the statement of the proposition follows directly from Lemma \ref{lm:bound}.

We now compute $\dW{2}(\gamma_{\alpha_1},\gamma_{\alpha_2})$ explicitly. Without loss of generality, we assume $l_1$ is parametrized by $l_1(t)=(t\cdot s_1,0)$ and $l_2$ is parametrized by $l_2(t)=(a_0+t\cdot s_2\cos{\theta},b_0+t\cdot s_2\sin{\theta})$ for $t\in[0,1]$. Then, $\mu_{\alpha_1}=\left(\frac{s_1}{2},0\right)$, $\mu_{\alpha_2}=\left(a_0+\frac{s_2\cos{\theta}}{2},b_0+\frac{s_2\sin{\theta}}{2}\right)$, $\Sigma_{\alpha_1}=\begin{pmatrix}\frac{s_1^2}{12}&0\\0&0\end{pmatrix}$ and $\Sigma_{\alpha_2}=\begin{pmatrix}\frac{s_2^2}{12}(\cos{\theta})^2&\frac{s_2^2}{12}\sin{\theta}\cos{\theta}\\\frac{s_2^2}{12}\sin{\theta}\cos{\theta}&\frac{s_2^2}{12}(\sin{\theta})^2\end{pmatrix}$.
Then, by definition of $\dcov$ in \Cref{eq:dcov}, it is easy to check that 
\[\dcov(\Sigma_{\alpha_1},\Sigma_{\alpha_2})=\lc\frac{s_1^2}{12}+\frac{s_2^2}{12}-\frac{s_1s_2}{6}\cos{\theta}\rc^{\frac{1}{2}}. \]
Therefore, 
\begin{align*}
    \dW{2}^2(\gamma_{\alpha_1},\gamma_{\alpha_2})=&\norm{\mu_{\alpha_1}-\mu_{\alpha_2}}^2+\lc\dcov(\Sigma_{\alpha_1},\Sigma_{\alpha_2})\rc^2\\
    =&\frac{s_1^2}{3}+\frac{s_2^2}{3}-\frac{2s_1s_2\cos{\theta}}{3}\\
    +&a_0^2+b_0^2+a_0s_2\cos{\theta}+b_0s_2\sin{\theta}-a_0s_1.
\end{align*}

Next we compute $\dW{2}(\alpha_1,\alpha_2)$. Consider the linear map $T:l_1\rightarrow l_2$ defined by sending $ l_1(t)$ to $ l_2(t)$ for all $t\in[0,1]$. See Figure \ref{fig:GT} for an illustration. Then, it is easy to check that $T_\#\alpha_1=\alpha_2$. This gives rise to a transport plan $\pi=(\mathrm{Id}\times T)_\#\alpha_1\in\mathcal{C}(\alpha_1,\alpha_2)$, where $\mathrm{Id}: l_1\rightarrow l_1$ is the identity map on $ l_1$. Then, 
\begin{align*}
   &\dW{2}^2(\alpha_1,\alpha_2)\leq\int_{(x,x')\in l_1\times l_2}\norm{x-x'}^2\pi(dx\times dx')\\
   &=\int_{x\in l_1}\norm{x-T(x)}^2\alpha_1(dx)\\
   &=\int_0^1 \left|a_0+ts_2\cos{\theta}-ts_1\right|^2+\left|b_0+ts_2\sin{\theta}\right|^2dt\\
   &=\frac{s_1^2}{3}+\frac{s_2^2}{3}-\frac{2s_1s_2\cos{\theta}}{3}\\
    &+a_0^2+b_0^2+a_0s_2\cos{\theta}+b_0s_2\sin{\theta}-a_0s_1\\
    &=\dW{2}^2(\gamma_{\alpha_1},\gamma_{\alpha_2}).
\end{align*}
By Lemma \ref{lm:bound}, we know $\dW{2}(\alpha_1,\alpha_2)\geq\dW{2}(\gamma_{\alpha_1},\gamma_{\alpha_2})$ and thus $\dW{2}(\alpha_1,\alpha_2)=\dW{2}(\gamma_{\alpha_1},\gamma_{\alpha_2})$
\end{proof}

\begin{figure}[h]
    \centering
    \includegraphics[width=0.35\linewidth]{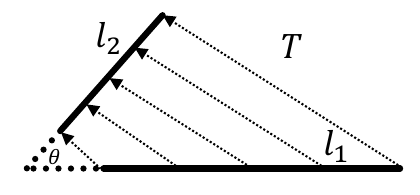} 
    \caption{\textbf{Illustration of the linear map $T$ from $l_1$ to $l_2$.} 
    }
    \label{fig:GT}
\end{figure}
\begin{proof}[Proof of Theorem \ref{thm:ellipsoid}]
We only consider the case when the dimension $m=1$. Then, for an arbitrary $x_0\in \R$, we have
\[\mu^{\mathsmaller{(\eps)}}_{{\alpha}}(x_0)=\frac{\int_{-\eps}^\eps(x_0+t)f(x_0+t)dt}{\int_{-\eps}^\eps f(x_0+t)dt}, \]
and 
\[\Sigma_\alpha^{\mathsmaller{(\eps)}}(x_0)=\frac{\int_{-\eps}^{\eps}\left(t+x_0-\mu^{\mathsmaller{(\eps)}}_{{\alpha}}(x_0)\right)^2f(t+x_0)dt}{\int_{-\eps}^\eps f(x_0+t)dt}. \]
By replacing $f(x_0+t)$ with the following Taylor expansion around $x_0$
\[f(x_0+t)=f(x_0)+f'(x_0)t+\frac{f''(x_0)}{2}t^2+O(t^3),\,\,\forall t\in[-\eps,\eps],\]
we obtain
\begin{align*}
  \Sigma_\alpha^{\mathsmaller{(\eps)}}(x_0)&=\frac{1}{3}\eps^2+\frac{-5(f'(x_0))^2+2f(x_0)f''(x_0)}{45f(x_0)^2}\eps^4+O(\eps^6)\\
  &=\frac{\eps^2}{3}(1+h(x_0)\eps^2+O(\eps^4)),
\end{align*}
where $h(x_0)\coloneqq\frac{-5(f'(x_0))^2+2f(x_0)f''(x_0)}{9f(x_0)^2}$. Since $m=1$, for any $x_1\in B_\eps(x_0)$, we have
\begin{align}
    &\lc\dcov\lc\Sigma_\alpha^{\mathsmaller{(\eps)}}(x_0),\Sigma_\alpha^{\mathsmaller{(\eps)}}(x_1)\rc\rc^2=\left|\sqrt{\Sigma_\alpha^{\mathsmaller{(\eps)}}(x_0)}-\sqrt{\Sigma_\alpha^{\mathsmaller{(\eps)}}(x_1)}\right|^2\label{eq:1d_trick}\\
    &=\lc\frac{h(x_0)-h(x_1)}{2\sqrt{3}}\eps^3+O(\eps^5)\rc^2\\
    &=\frac{(h'(x_0))^2}{12}\eps^6|x_0-x_1|^2+O(\eps^9),
\end{align}
where the first equality holds because $m=1$ and the last equality follows from the assumption that $|x_0-x_1|\leq\eps$ and a Taylor expansion of $h$ at $x_0$
\[h(x_1)=h(x_0)+h'(x_1)(x_1-x_0)+\frac{h''(x_0)}{2}(x_1-x_0)^2+O(|x_1-x_0|^3),\,\,\forall x_1\in[x_0-\eps,x_0+\eps].\]

So, if $x_1\in B_\eps^{\lambda,\alpha}(x_0)$, we have
\begin{align}
    \eps^2&\geq |x_0-x_1|^2+\eps^{-6}\lc\dcov(\Sigma_\alpha^{\mathsmaller{(\eps)}}(x_0),\Sigma_\alpha^{\mathsmaller{(\eps)}}(x_1))\rc^2\\
    &=\lc1+\frac{(h'(x_0))^2}{12}\rc|x_0-x_1|^2+O(\eps^3),\label{eq:1d-ell}
\end{align}
Let $x_1=x_0+a\eps$, then we have $|a|\leq\sqrt{\frac{12}{12+(h'(x_0))^2}}=:a_0$ by discarding the higher order terms. This implies that $B_\eps^{\lambda,\alpha}(x_0)$ is approximately an Euclidean ball $B_{a_0}(x_0)=[x_0-a_0,x_0+a_0]$. More rigorously, we prove the following statement.
\begin{claim}\label{claim:ball = close}
For any decreasing sequence $\{\eps_n\}_{n=1}^\infty$ approaching $0$, we have that
\[B_{a_0}(x_0)=\overline{\liminf_{n\rightarrow\infty}\mathfrak{S}_{\eps_n}\lc B_{\eps_n}^{\lambda,\alpha}(x_0)\rc}.\]
\end{claim}

For any $a\in\R$ such that $|a|< a_0$, we let $x_n\coloneqq x_0+a\eps_n$. Then,
\[d^{\mathsmaller{(\eps,\eps^{-6})}}_{{\alpha,d_X}}(x_0,x_n)=\lc1+\frac{(h'(x_0))^2}{12}\rc|x_0-x_1|^2+O(\eps_n^3)=\lc\frac{a}{a_0}\rc^2\eps_n^2+O(\eps_n^3).\]
When $n$ is large enough, we have that $d^{\mathsmaller{(\eps,\eps^{-6})}}_{{\alpha,d_X}}(x_0,x_n)\leq \eps_n^2$ and thus $x_n\in B_{\eps_n}^{\lambda,\alpha}(x_0)$. This implies that $x_0+a\in\mathfrak{S}_{\eps_n}\lc B_{\eps_n}^{\lambda,\alpha}(x_0)\rc$. Hence, the open ball $(x_0-a_0,x_0+a_0)\subseteq\liminf_{n\rightarrow\infty}\mathfrak{S}_{\eps_n}\lc B_{\eps_n}^{\lambda,\alpha}(x_0)\rc$. Thus, $B_{a_0}(x_0)\subseteq\overline{\liminf_{n\rightarrow\infty}\mathfrak{S}_{\eps_n}\lc B_{\eps_n}^{\lambda,\alpha}(x_0)\rc}$. 

Conversely, let $x\in \liminf_{n\rightarrow\infty}\mathfrak{S}_{\eps_n}\lc B_{\eps_n}^{\lambda,\alpha}(x_0)\rc$. Then, for any $n$ large enough, there exists $x_n\in B_{\eps_n}^{\lambda,\alpha}(x_0) $ such that $x = x_0 +\frac{1}{\eps_n}(x_n-x_0)$. By \Cref{eq:1d-ell}, we have that $\left|\frac{1}{\eps_n}(x_n-x_0)\right|^2+O(\eps_n)\leq a_0^2$. Therefore, \[|x-x_0|=\left|\frac{1}{\eps_n}(x_n-x_0)\right|=\lim_{n\rightarrow \infty}\left|\frac{1}{\eps_n}(x_n-x_0)\right|\leq a_0.\]
This means that $x\in B_{a_0}(x_0)$ and thus $\liminf_{n\rightarrow\infty}\mathfrak{S}_{\eps_n}\lc B_{\eps_n}^{\lambda,\alpha}(x_0)\rc\subseteq B_{a_0}(x_0)$. Therefore, we conclude the proof of Claim \ref{claim:ball = close}. Since the claim holds for arbitrary sequence $\{\eps_n\}_{n=1}^\infty$, we further conclude that $\overline{\liminf_{\eps\rightarrow 0}\mathfrak{S}_{\eps}\lc B_{\eps}^{\lambda,\alpha}(x_0)\rc}= B_{a_0}(x_0)$.

When $m>1$, there is no formula analogous to \Cref{eq:1d_trick} that helps simplify the computation of the Taylor expansion of $\dcov$, yet through a direct and tedious calculation, one can still compute the Taylor expansion of $\dcov$ around ${x_0}\in\R^m$ and show that there exists an $m$-dimensional PSD matrix $H({x_0})$ (which coincides with $\frac{(h'(x_0))^2}{12}$ when $m=1$) depending on $f$ such that for ${x_1}\in B_\eps^{\lambda,\alpha}({x_0})$ we have the following analogy to \Cref{eq:1d-ell}:
\[({x_1}-{x_0})^\mathrm{T}(I_m+H({x_0}))({x_1}-{x_0})+O(\eps^3)\leq\eps^2, \]
where $I_m$ is the $m$-dimensional identity matrix. 
Then, an argument similar to the one for the 1-dimensional case indicates that 
$$\overline{\liminf_{\eps\rightarrow 0}\mathfrak{S}_{\eps}\lc B_{\eps}^{\lambda,\alpha}(x_0)\rc}= \left\{{x}_0+{a_0}:\,{a_0}^\mathrm{T}(I_m+H({x_0}))\,{a_0}\leq 1,\,\forall a_0\in\R^m\right\},$$ 
which is an ellipsoid centered at ${x}_0$.
\end{proof}

\begin{proof}[Proof of \Cref{thm:lip kernel stab}]
We first prove the following lemma.

\begin{lemma}\label{thm:stab-lip}
Suppose $\mathcal{K}$ is $L$-Lipschitz on $\R$, then there exists $M>0$ depending on $\mathcal{K}$ and $X$ such that for $\alpha,\beta\in \mathcal{P}(X)$, we have for any $x\in X$ that
$$\dW{1}\lc m^K_{\alpha,d_X} (x), m^K_{\beta,d_X} (x)\rc\leq\frac{2L\,\diam(X)+M}{\max(M_\alpha(x),M_\beta(x))}\dW{1}(\alpha,\beta),$$
where $M_\alpha(x)\coloneqq\int_XK(x,y)\alpha(dy)$ and similarly for $M_\beta(x)$.
\end{lemma}
\begin{proof}The distance function will reach its maximum and minimum on $X\times X$ since $X$ is compact. Therefore, $K=\mathcal{K}\circ d$ is bounded and hence there exists $M>0$ such that $K\leq M$. Now, let $\pi\in\mathcal{C}^\mathrm{opt}_1(\alpha,\beta)$ and let $f:X\rightarrow\R$ be a 1-Lipschitz function such that $|f|\leq\diam(X)$. Then, we have that 
\begin{align*}
&\left|\int_Xf(y) m^K_{\alpha,d_X} (x)(dy)-\int_Xf(y) m^K_{\beta,d_X} (x)(dy)\right| \\
=& \left|\frac{\int_XK(x,y)f(y)\alpha(dy)}{\int_XK(x,y)\alpha(dy)}-\frac{\int_XK(x,y)f(y)\beta(dy)}{\int_XK(x,y)\beta(dy)}\right|\\
=& \frac{1}{M_\alpha(x) M_\beta(x)}\left|\int_XK(x,y)f(y)\alpha(dy)M_\beta(x)-\int_XK(x,y)f(y)\beta(dy)M_\alpha(x)\right|\\
\leq& \frac{1}{M_\alpha(x) M_\beta(x)}\int_X\big|K(x,y_1)f(y_1)M_\beta(x)- K(x,y_2)f(y_2)M_\alpha(x)\big|\pi(dy_1\times dy_2)\\
\leq & \frac{1}{M_\alpha(x) M_\beta(x)}\int_X\big|K(x,y_1)f(y_1)M_\beta(x)- K(x,y_2)f(y_1)M_\beta(x)\big|\pi(dy_1\times dy_2)\\
+&\frac{1}{M_\alpha(x) M_\beta(x)}\int_X\big|K(x,y_2)f(y_1)M_\beta(x)- K(x,y_2)f(y_2)M_\beta(x)\big|\pi(dy_1\times dy_2)\\
+&\frac{1}{M_\alpha(x) M_\beta(x)}\int_X\big|K(x,y_2)f(y_2)M_\beta(x)- K(x,y_2)f(y_2)M_\alpha(x)\big|\pi(dy_1\times dy_2)\\
\leq & \frac{1}{M_\alpha(x) M_\beta(x)}\int_X\diam(X)M_\beta(x) L\, d_X(y_1,y_2)\pi(dy_1\times dy_2)\\
+&\frac{1}{M_\alpha(x) M_\beta(x)}\int_XMM_\beta(x) d_X(y_1,y_2)\pi(dy_1\times dy_2)\\
+&\frac{1}{M_\alpha(x) M_\beta(x)}\int_X\diam(X)|M_\beta(x)- M_\alpha(x)|K(x,y_2)\pi(dy_1\times dy_2)\\
=& \frac{L\diam(X)+M}{M_\alpha(x)}\dW{1}(\alpha,\beta)+\frac{\diam(X)|M_\beta(x)-M_\alpha(x)|}{M_\alpha(x)}\\
\leq & \frac{2L\diam(X)+M}{M_\alpha(x)}\dW{1}(\alpha,\beta),
\end{align*}

The last step follows from Kantorovich duality for $\alpha$ and $\beta$ (cf. \Cref{eq:kant duality}) and the fact that $K(x,\cdot)$ is $L$-Lipschitz. Similarly,
$$\left|\int_Xf(y) m^K_{\alpha,d_X} (x)(dy)-\int_Xf(y) m^K_{\beta,d_X} (x)(dy)\right| \leq \frac{2L\diam(X)+M}{M_\beta(x)}\dW{1}(\alpha,\beta). $$

Hence by Kantorovich duality (cf. \Cref{eq:kant duality}) again, we have that

$$\dW{1}\lc  m^K_{\alpha,d_X} (x), m^K_{\beta,d_X} (x)\rc \leq\frac{2L\diam(X)+M}{\max(M_\alpha(x),M_\beta(x))}\dW{1}(\alpha,\beta).$$
\end{proof}

Then, we turn to prove Theorem \ref{thm:lip kernel stab}.
Since $\mathcal{K}$ is continuous, there exists $N>0$ such that $\mathcal{K}([0,\diam(X)])\geq N,$ which implies that $K(x,x')=\mathcal{K}(d_X(x,x'))\geq N,\forall x,x'\in X$. Thus we have for any $\alpha\in \mathcal{P}(X)$, 
$$M_\alpha(x)=\int_XK(x,y)\,\alpha(dy)\geq\int_XN\,\alpha(dy)=N,$$
which proves the statement (by applying Lemma \ref{thm:stab-lip}).

\end{proof}

\begin{proof}[Proof of Corollary \ref{coro:stb-metric-lipschitz}]
For any $x,x'\in X$, we have that
\begin{align*}
 \left|d_{\alpha,d_X}^K(x,x')-d_{\beta,d_X}^K(x,x')\right|  &=|\dW{1}\lc m^K_{\alpha,d_X} (x), m^K_{\alpha,d_X} (x')\rc -\dW{1}\lc m^K_{\beta,d_X} (x), m^K_{\beta,d_X} (x')\rc| \\
&\leq \left|\dW{1}\lc m^K_{\alpha,d_X} (x), m^K_{\alpha,d_X} (x')\rc -\dW{1}\lc m^K_{\alpha,d_X} (x), m^K_{\beta,d_X} (x')\rc\right|\\
&+\left|\dW{1}\lc m^K_{\alpha,d_X} (x), m^K_{\beta,d_X} (x')\rc-\dW{1}\lc m^K_{\beta,d_X} (x), m^K_{\beta,d_X} (x')\rc\right|\\
& \leq \dW{1}\lc m^K_{\alpha,d_X} (x'), m^K_{\beta,d_X} (x')\rc+\dW{1}\lc m^K_{\alpha,d_X} (x), m^K_{\beta,d_X} (x)\rc\\
&\leq \frac{4L\,\diam(X)+2M}{N}\dW{1}(\alpha,\beta),
\end{align*}
where the last inequality follows from \Cref{thm:lip kernel stab}.
\end{proof}

\begin{proof}[Proof of Theorem \ref{thm:lip kernel is lip}]
The proof is in the same spirit as the proof of Theorem \ref{thm:lip kernel stab}. Note that
\begin{align*}
& \dW{1}\lc m^K_{\alpha,d_X} (x), m^K_{\alpha,d_X} (x')\rc\\
=&\sup \left\{\left|\int_Xf(y)  m^K_{\alpha,d_X} (x)(dy)-\int_Xf(y) m^K_{\alpha,d_X} (x')(dy)\right|:f\text{ is 1-Lipschitz and }|f|\leq\diam(X).\right\} 
\end{align*}
Since $d_X$ reaches its maximum and minimum on $X\times X$, $K=\mathcal{K}\circ d$ is bounded and hence there exists $M>0$ such that $K\leq M$.

Since $\mathcal{K}$ is positive and continuous, there exists $N>0$ such that $\mathcal{K}([0,\diam(X)])\geq N,$ which gives us $K(x,y)\geq N,\forall x,y\in X$. 

Then, $N\leq M_\alpha(x)\leq M$ and
\[|M_\alpha(x)-M_\alpha(x')|\leq\int_X|K(x,y)-K(x',y)|\alpha(dy)\leq Ld_X(x,x')\]

Then, we have that
\begin{align*}
&\left|\int_Xf(y) m^K_{\alpha,d_X} (x)(dy)-\int_Xf(y) m^K_{\alpha,d_X}(x')(dy)\right| \\
=& \left|\frac{\int_XK(x,y)f(y)\alpha(dy)}{\int_XK(x,y)\alpha(dy)}-\frac{\int_XK(x',y)f(y)\alpha(dy)}{\int_XK(x',y)\alpha(dy)}\right|\\
\leq& \frac{1}{M_\alpha(x) M_\alpha(x')}\int\left| K(x,y)f(y)M_\alpha(x')- K(x',y)f(y)M_\alpha(x)\right|\alpha(dy)\\
\leq & \frac{1}{M_\alpha(x) M_\alpha(x')}\int_X\big|K(x,y)f(y)M_\alpha(x')- K(x,y)f(y)M_\alpha(x)\big|\alpha(dy)\\
+&\frac{1}{M_\alpha(x) M_\alpha(x')}\int_X\big|K(x,y)f(y)M_\alpha(x)- K(x',y)f(y)M_\alpha(x)\big|\alpha(dy)\\
\leq & \frac{1}{M_\alpha(x) M_\alpha(x')}\int_X\diam(X)M|M_\alpha(x)-M_\alpha(x')|\alpha(dy)\\
+&\frac{1}{M_\alpha(x) M_\alpha(x')}\int_X\diam(X)M_\alpha(x)L\, d_X(x,x')\alpha(dy)\\
=&\frac{\diam(X)M|M_\alpha(x)-M_\alpha(x')|}{M_\alpha(x)M_\alpha(x')}+\frac{\diam(X)L \,d_X(x,x')}{M_\alpha(x')}\\
\leq & \frac{L(M+N)\diam(X)}{N^2}d_X(x,x').
\end{align*}
\end{proof}

\begin{proof}[Proof of Theorem \ref{thm:sampling convergence}]
First by the triangle inequality for the Gromov-Wasserstein distance,
\begin{align*}
    &\dgw{1}\lc (X,d_{\alpha_n,d_X}^K,\alpha_n),\lc X,d_{\alpha,d_X}^K,\alpha\rc\rc \\
    \leq &\dgw{1}\lc (X,d_{\alpha_n,d_X}^K,\alpha_n),\lc X,d_{\alpha,d_X}^K,\alpha_n\rc\rc +\dgw{1}\lc \lc X,d_{\alpha,d_X}^K,\alpha_n\rc,\lc X,d_{\alpha,d_X}^K,\alpha\rc\rc.
\end{align*}
For the second term, by \cite[Theorem 5.1 (c)]{memoli2011gromov} and Theorem \ref{thm:lip kernel is lip} we have that
$$\dgw{1}\lc \lc X,d_{\alpha,d_X}^K,\alpha_n\rc,\lc X,d_{\alpha,d_X}^K,\alpha\rc\rc \leq \dW{1}^{\lc X,d_{\alpha,d_X}^K\rc}\lc \alpha_n,\alpha\rc \leq C_1\,\dW{1}^{(X,d_X)}\lc \alpha_n,\alpha\rc, $$
where $C_1$ denotes the positive multiplicative coefficient in Theorem \ref{thm:lip kernel is lip} depending on $K$ and $(X,d_X)$.

For the first term, by \cite[Theorem 5.1 (d)]{memoli2011gromov} we have that

$$ \dgw{1}\lc (X,d_{\alpha_n,d_X}^K,\alpha_n),\lc X,d_{\alpha,d_X}^K,\alpha_n\rc\rc \leq \frac{1}{2}\norm{d_{\alpha_n,d_X}^K-d_{\alpha,d_X}^K}_{L^1(\alpha_n\otimes\alpha_n)}\leq\frac{1}{2}\norm{d_{\alpha_n,d_X}^K-d_{\alpha,d_X}^K}_{L^\infty(\alpha_n\otimes\alpha_n)}.$$

Given any $x,x'\in\supp(\alpha_n)$, we have by definition that
\[d_{\alpha_n,d_X}^K(x,x')=\dW{1}^{(X,d_X)}\lc m^K_{\alpha_n,d_X} (x), m^K_{\alpha_n,d_X} (x')\rc\text{ and }d_{\alpha,d_X}^K(x,x')=\dW{1}^{(X,d_X)}\lc m^K_{\alpha,d_X} (x), m^K_{\alpha,d_X} (x')\rc.\] 
Hence by triangle inequality, we have that
$$ \left|d_{\alpha_n,d_X}^K(x,x')-d_{\alpha,d_X}^K(x,x')\right|\leq \dW{1}^{(X,d_X)}\lc m^K_{\alpha_n,d_X} (x), m^K_{\alpha,d_X} (x)\rc+\dW{1}^{(X,d_X)}\lc m^K_{\alpha_n,d_X} (y), m^K_{\alpha,d_X} (y)\rc.$$
By Theorem \ref{thm:lip kernel stab}, we have that 
\[\dW{1}^{(X,d_X)}\lc m^K_{\alpha_n,d_X}(x), m^K_{\alpha,d_X}(x)\rc\leq C_2\,\dW{1}^{(X,d_X)}(\alpha_n,\alpha)\]
for some constant $C_2$ depending on $K$ and $(X,d_X)$. This implies that

$$ \norm{d_{\alpha_n,d_X}^K-d_{\alpha,d_X}^K}_{L^\infty(\alpha_n\otimes\alpha_n)}\leq 2C_2\,\dW{1}^{(X,d_X)}(\alpha_n,\alpha).$$
Therefore, 
\begin{equation}\label{eq:GW stb}
    \dgw{1}\lc \lc X,d_{\alpha_n,d_X}^K,\alpha_n\rc,\lc X,d_{\alpha,d_X}^K,\alpha\rc\rc \leq (C_1+ 2C_2)\,\dW{1}^{(X,d_X)}(\alpha_n,\alpha).
\end{equation}
The right hand side (and thus the left hand side) approaches $0$ a.s. due to \cite{varadarajan1958convergence}, which concludes the proof.
\end{proof}
\begin{proof}[Proof of Corollary \ref{coro:convergence rate}]
To prove this corollary, we need the notion of the Prokhorov distance defined in \Cref{sec:prokhorov} and the following convergence rate result.

\begin{lemma}[{\cite[Theorem 4.1]{dudley1969speed}}]\label{lm:rate of convergence}
Let $X$ be a separable metric space. Let $\alpha\in\mathcal{P}(X)$ and let $X_1,X_2,\cdots$ be i.i.d. random variables with distribution $\alpha$. Denote by $\alpha_n\coloneqq\frac{1}{n}\sum_{i=1}^n\delta_{X_i}$ the empirical measure. Then, for any $\eps>0$, we have for $n$ large enough that 
\[\mathbb{E}(\dP(\alpha_n,\alpha))\leq n^{-\frac{1}{k(\alpha)+2+\eps}}.\]
\end{lemma}

By Equation \eqref{eq:GW stb}, there exists $C'>0$ depending on $K$ and $(X,d_X)$ such that
$$\dgw{1}\lc (X,d_{\alpha_n,d_X}^K,\alpha_n),\lc X,d_{\alpha,d_X}^K,\alpha\rc\rc \leq C'\,\dW{1}^{(X,d_X)}(\alpha_n,\alpha).$$
Therefore, by Lemma \ref{lemma:dp} and Lemma \ref{lm:rate of convergence}, 
\begin{align*}
    \mathbb{E}\lc \dgw{1}\lc (X,d_{\alpha_n,d_X}^K,\alpha_n),\lc X,d_{\alpha,d_X}^K,\alpha\rc\rc \rc&\leq C'\mathbb{E}\lc \dW{1}^{(X,d_X)}(\alpha_n,\alpha)\rc\\
    &\leq C'(1+\diam(X))\mathbb{E}\lc \dP^{(X,d_X)}(\alpha_n,\alpha)\rc\\
    &\leq C\,n^{-\frac{1}{k(\alpha)+2+\eps}},
\end{align*}
where $C\coloneqq C'(1+\diam(X))$ depends only on $K$ and $(X,d_X)$.
\end{proof}


\begin{proof}[Proof of Theorem \ref{thm:stb}]
We need the following lemma for the proof.

\begin{lemma}\label{rem:lb-meas}
Let $X$ be a compact metric space and let $\alpha\in\pf(X)$. Assume that $\diam(X)<D$ for some $D>0$. For $\Lambda>0$, if $\alpha\in\mathcal{P}^{\mathcal{BG}(\Lambda)}(X)$, then we have that
$\alpha(B_r(x))\geq \psi_{\Lambda,D}(r)$, for all $x\in X$ and $r>0$.
\end{lemma}
\begin{proof}[Proof of Lemma \ref{rem:lb-meas}]
Let $r_1=D,r_2=r$ in Definition \ref{def:bg}. We then have that when $r\leq D$,
\[\frac{\alpha(B_{D}(x))}{\alpha(B_{r}(x))}\leq \left(\frac{D}{r}\right)^\Lambda.\]
Notice that $X=B_D(x)$, hence we have $\alpha(B_D(x))=\alpha(X)=1$. Therefore
\[\alpha(B_r(x))\geq\lc\frac{r}{D}\rc^\Lambda\geq\min\left(1,\left(\frac{r}{D}\right)^\Lambda\right) .\] 
When $r>D$, obviously we have $\alpha(B_r(x))=\alpha(X)=1\geq\min\left(1,\left(\frac{r}{D}\right)^\Lambda\right)$.
\end{proof}

To analyze $\dW{1}\lc \me_{\alpha,d_X}(x),\me_{\beta,d_X}(x)\rc $, we first analyze the Prokhorov distance $\dP$.


\begin{claim}\label{claim6wt}
For any $x\in X$, $\dP\lc \me_{\alpha,d_X}(x),\me_{\beta,d_X}(x)\rc \leq \Phi_{\Lambda,D,\eps}\lc \dP(\alpha,\beta)\rc.$
\end{claim}
\begin{proof}
Suppose $\dP(\alpha,\beta)<\eta$ for some $\eta>0$. Fix $x\in X$ and assume without loss of generality that $\beta(B_\eps(x))\leq\alpha(B_\eps(x))$.
Then,  invoke the expression $$\dP\lc\me_{\alpha,d_X}(x),\me_{\beta,d_X}(x)\rc=\inf\lb\delta>0:\me_{\alpha,d_X}(x)(A)\leq\me_{\beta,d_X}(x)(A^\delta)+\delta,\forall A\subseteq X\rb.$$ 
For any $A\subseteq X$  we have the following inclusions:
\begin{align*}(A\cap B_\eps(x))^\eta &\subseteq A^\eta \cap (B_\eps(x))^\eta \subseteq A^\eta \cap B_{\eps+\eta}(x) =A^\eta \cap \lc B_\eps(x) \cup \lc B_{\eps+\eta}(x)\backslash B_\eps(x)\rc  \rc  \\
&\subseteq A^\eta \cap B_\eps(x) \bigcup A^\eta \cap B_{\eps+\eta}(x)\backslash
B_\eps(x) \subseteq A^\eta \cap B_\eps(x) \bigcup B_{\eps+\eta}(x)\backslash B_\eps(x).
\end{align*}
Then, by monotonicity of measure and the fact that $\dP(\alpha,\beta)<\eta$, we have
\begin{align*}
\me_{\alpha,d_X}(x)(A) &= \frac{\alpha(A\cap B_\eps(x))}{\alpha(B_\eps(x))}\leq \frac{\beta((A\cap B_\eps(x))^\eta)+\eta}{\alpha(B_\eps(x))}\leq \frac{\beta((A\cap B_\eps(x))^\eta)+\eta}{\beta(B_\eps(x))}\\
& \leq \frac{\beta(A^\eta\cap B_\eps(x))+\beta(B_\eps^\eta(x)\backslash B_\eps(x))}{\beta(B_\eps(x))}+\frac{\eta}{\beta(B_\eps(x))}\\
& \leq \frac{\beta(A^\eta\cap B_\eps(x))}{\beta(B_\eps(x))}+\frac{\beta(B_{\eps+\eta}(x))+\eta}{\beta(B_\eps(x))}-1\\
& \leq \me_{\beta,d_X}(x)(A^\eta)+\lc 1+\frac{\eta}{\eps}\rc ^\Lambda-1+\frac{\eta}{\beta(B_\eps(x))}\leq\me_{\beta,d_X}(x)(A^\eta)+\Phi_{\Lambda,D}(\eta),
\end{align*}
where the last inequality follows from Lemma \ref{rem:lb-meas}. Note that since $\lc 1+\frac{\eta}{\eps}\rc ^\Lambda-1\geq 0$, and $\psi_{\Lambda, D}(\eps)\leq 1$, then $\Phi_{\Lambda,D}(\eta)\geq \eta$. Let $\xi\coloneqq\Phi_{\Lambda,D}(\eta)$. Then, from the inequalities above, we have that 
$$\me_{\alpha,d_X}(x)(A)\leq\me_{\beta,d_X}(x)(A^\eta)+\xi\leq \me_{\beta,d_X}(x)(A^\xi)+\xi.$$
Therefore, $\dP\lc\me_{\alpha,d_X}(x),\me_{\beta,d_X}(x)\rc\leq\xi=\Phi_{\Lambda,D}(\eta)$. Then, by letting $\eta\rightarrow \dP(\alpha,\beta)$ we have $\dP\lc\me_{\alpha,d_X}(x),\me_{\beta,d_X}(x)\rc\leq \Phi_{\Lambda,D,\eps}\lc \dP(\alpha,\beta)\rc $, which concludes the proof.
\end{proof}

We now finish the proof of  Theorem \ref{thm:stb}. Since  $\supp\lc \me_{\alpha,d_X}(x)\rc $ and $\supp\lc \me_{\beta,d_X}(x)\rc $ are both contained in $B_\eps(x)$ and $\mathrm{diam}(B_\eps(x))\leq 2\eps$, we have from Remark \ref{rmk:pr} that
\[\dW{1}\lc\me_{\alpha,d_X}(x),\me_{\beta,d_X}(x)\rc \leq \lc 1+2\eps\rc \dP\lc\me_{\alpha,d_X}(x),\me_{\beta,d_X}(x)\rc.\]
Now, from this inequality, by  Claim \ref{claim6wt} above we in turn obtain 
$$\dW{1}\lc \me_{\alpha,d_X}(x),\me_{\beta,d_X}(x)\rc  \leq \lc 1+2\eps\rc\Phi_{\Lambda,D,\eps}\lc \dP(\alpha,\beta)\rc.$$ 
Finally,  since $\Phi_{\Lambda,D,\eps}(\eta)$ is an increasing function of $\eta$, by Lemma \ref{lemma:dp} we obtain the statement of the theorem.
\end{proof}

\begin{proof}[Proof of Corollary \ref{coro:stb-metric}]
For any $x,x'\in X$, we have that
\begin{align*}
 \left|\deps(x,x')-\depsbeta(x,x')\right|  &=|\dW{1}\lc \me_{\alpha,d_X}(x),\me_{\alpha,d_X}(x')\rc -\dW{1}\lc\me_{\beta,d_X}(x),\me_{\beta,d_X}(x')\rc| \\
&\leq \left|\dW{1}\lc \me_{\alpha,d_X}(x),\me_{\alpha,d_X}(x')\rc -\dW{1}\lc\me_{\alpha,d_X}(x),\me_{\beta,d_X}(x')\rc\right|\\
&+\left|\dW{1}\lc\me_{\alpha,d_X}(x),\me_{\beta,d_X}(x')\rc-\dW{1}\lc\me_{\beta,d_X}(x),\me_{\beta,d_X}(x')\rc\right|\\
& \leq \dW{1}\lc\me_{\alpha,d_X}(x'),\me_{\beta,d_X}(x')\rc+\dW{1}\lc\me_{\alpha,d_X}(x),\me_{\beta,d_X}(x)\rc\\
&\leq 2(1+2\eps)\Phi_{\Lambda,D,\eps}\lc \sqrt{\dW{1}(\alpha,\beta)}\rc,
\end{align*}
where the last inequality follows from Theorem \ref{thm:stb}.
\end{proof}

\begin{proof}[Proof of \Cref{thm:ms-smooth}]
Let $\mathcal{K}_1(t)\coloneqq \mathcal{K}\lc\frac{t^2}{\eps^2}\rc$. Then, for $t\in[0,D]$, it is easy to see that $\mathcal{K}_1$ is $\frac{2\,L\,D}{\eps^2}$-Lipschitz. Since $K(x,y)=\mathcal{K}_1(d(x,y))$, we have by Theorem \ref{thm:lip kernel stab} that
\[\dW{1}\left(m_{\alpha,d_X}^{K_\eps}(x), m_{\beta,d_X}^{K_\eps}(x)\right)\leq\frac{2\,CD+M}{N}\dW{1}(\alpha,\beta), \]
where $M,N$ are positive constants depending on $\mathcal{K}_1$ (thus on $\mathcal{K}$ and $\eps$) and $X$.
Then, by Lemma \ref{lm:bound}, we have that
\begin{align*}
    \left\|\mean\lc m_{\alpha,d_X}^{K_\eps}(x)\rc-\mean\lc m_{\beta,d_X}^{K_\eps}(x)\rc\right\|\leq \dW{1}\left(m_{\alpha,d_X}^{K_\eps}(x), m_{\beta,d_X}^{K_\eps}(x)\right)\leq\frac{2\,CD+M}{N}\dW{1}(\alpha,\beta).
\end{align*}
\end{proof}
\begin{proof}[Proof of \Cref{thm:stb-ms}]
By Lemma \ref{lm:bound} and Theorem \ref{thm:stb}, we have  $\forall x\in X$ that
\begin{align*}
 &\norm{\mean\lc\me_{\alpha,d_X}(x)\rc-\mean\lc\me_{\beta,d_X}(x)\rc}\\
 \leq &\dW{1}\lc\me_{\alpha,d_X}(x),\me_{\beta,d_X}(x)\rc\\
\leq&(1+2\eps)\Phi_{\Lambda,D,\eps}\lc \sqrt{\dW{1}(\alpha,\beta)}\rc.   
\end{align*}
\end{proof}

\begin{proof}[Proof of Theorem \ref{thm:GT-stable}]\label{proof of gt stable}
We first provide the definition of the function $\Psi_{\Lambda,D,\eps}^{c,A,m}$ in the statement of the theorem.
For any $\Lambda,\eps>0$, we define for any $t\geq 0$
$$\Upsilon_{\Lambda,\eps}(t)\coloneqq\lc\lc1+\frac{t}{\eps}\rc^\Lambda-1\rc+t.$$
Moreover, given any constants $A,c>0$ and $m\in\mathbb{N}$, let 
$$\Upsilon_{\Lambda,D,\eps}^{c,A,m}(t)\coloneqq \frac{\eps^m\,\nu_m\,c\,A}{\psi_{\Lambda,D}(\eps)}\,t+\frac{\eps^2}{\psi_{\Lambda,D}^2(\eps)}\Upsilon_{\Lambda,\eps}\left(\sqrt{t}\right),$$
where $\nu_m$ denotes the volume of the unit ball in $\R^m$ and $\psi_{\Lambda,D}$ is the function defined in \Cref{eq:psi l d}.
Note that both $\Upsilon_{\Lambda,\eps}$ and $\Upsilon_{\Lambda,D,\eps}^{c,A,m}$ are increasing functions which attain value 0 when $t=0$.

For $t\geq 0$, let 
$$\Psi_{\Lambda,D,\eps}^{c,A,m}(t)\coloneqq \Upsilon_{\Lambda,D,\eps}^{c,A,m}(t)+2\eps(1+2\eps)\Phi_{\Lambda,D,\eps}\left(\sqrt{t}\right),$$
where $\Phi_{\Lambda,D,\eps}$ is defined in \Cref{eq:phi l d e}. 

The proof of the theorem is based on the following series of lemmas. 
\begin{lemma}\label{lm:tr}
For symmetric positive semi-definite matrices $A,B$, we have
\[\mathrm{tr}\lc A+B-2\lc A^{ \frac{1}{2} }BA^{ \frac{1}{2} }\rc^{ \frac{1}{2} }\rc\leq\norm{A^{ \frac{1}{2} }-B^{ \frac{1}{2} }}_F^2,\]
where $\norm{\cdot}_F$ denotes the Frobenius norm of matrices.
\end{lemma}

\begin{proof}
Expand the right hand side of the inequality we obtain
\begin{align*}
\norm{A^{ \frac{1}{2} }-B^{ \frac{1}{2} }}_F^2 = & \mathrm{tr}\lc A+B-A^{ \frac{1}{2} }B^{ \frac{1}{2} }-B^{ \frac{1}{2} }A^{ \frac{1}{2} }\rc\\
=&\mathrm{tr}\lc A+B-2B^{ \frac{1}{2} }A^{ \frac{1}{2} }\rc
\end{align*}
Hence it suffices to prove 
\[\mathrm{tr}\lc \lc A^{ \frac{1}{2} }BA^{ \frac{1}{2} }\rc^{ \frac{1}{2} }\rc\geq\tr\lc B^{ \frac{1}{2} }A^{ \frac{1}{2} }\rc.\]
Let $M=B^{ \frac{1}{2} }A^{ \frac{1}{2} }$, then $\lc A^{ \frac{1}{2} }BA^{ \frac{1}{2} }\rc^{ \frac{1}{2} }=\lc M^\mathrm{T}M\rc^{ \frac{1}{2} }$. Assume that $M$ has size $m\times m$. Then, denote by $\{\sigma_i\}_{i=1,\cdots,m}$ the multiset of singular values of $M$, and denote by $\{\lambda_i\}_{i=1,\cdots,m}$ the multiset of eigenvalues of $M$, then 
\begin{align*}
\mathrm{tr}\lc \lc A^{ \frac{1}{2} }BA^{ \frac{1}{2} }\rc^{ \frac{1}{2} }\rc & =\tr\lc\lc M^\mathrm{T}M\rc^{ \frac{1}{2} }\rc\\
& =\sum_{i=1}^m\sigma_i\geq\sum_{i=1}^m|\lambda_i|\\
& \geq\tr\lc M\rc=\tr\lc B^{ \frac{1}{2} }A^{ \frac{1}{2} }\rc.
\end{align*}
The first inequality follows directly from \cite[Theorem 2.3.6]{bhatia2013matrix}.
\end{proof}

\begin{lemma}\label{lm:sqrn}
For symmetric positive semi-definite $m\times m$ matrices $A,B$ with dimension $m$, we have
\[\norm{A^{ \frac{1}{2} }-B^{ \frac{1}{2} }}_F^2\leq m\norm{A-B}_F.\]
\end{lemma}

\begin{proof}
It is shown in \cite[page 290]{bhatia2013matrix} that by using the operator norm $\norm{\cdot}$ of matrices one has
\[\norm{A^{ \frac{1}{2} }-B^{ \frac{1}{2} }}^2\leq \norm{A-B}.\]
By using the following relation for any $m\times m$ matrix $M$ (see \cite[page 7]{bhatia2013matrix})
\[\norm{M}\leq \norm{M}_F\leq\sqrt{m}\norm{M},\]
we obtain
\[\norm{A^{ \frac{1}{2} }-B^{ \frac{1}{2} }}_F^2\leq m\norm{A-B}_F.\]
\end{proof}
For any $x\in X$, we define a \emph{covariation matrix} w.r.t. $x$ as follows:
\begin{equation}\label{eq:til-cov}
    \Tilde{\Sigma}_{\alpha,d_X}^{\mathsmaller{(\eps)}}(x)\coloneqq\frac{1}{\eps^m\,\nu_m}\int_{B_\eps^{d_X}(x)}\lc{y}-x\rc\otimes \lc {y}- x\rc\,\alpha(d {y}).
\end{equation}
See Definition \ref{def:mean and covariance} for the definition of the symbol $\otimes$.
Note that $\Tilde{\Sigma}_{\alpha,d_X}^{\mathsmaller{(\eps)}}(x)$ is different from the local covariance matrix ${\Sigma}_{\alpha,d_X}^{\mathsmaller{(\eps)}}(x)$ defined in Section \ref{sec:gt}. 

The following result establishes the stability of the matrix $\Tilde{\Sigma}_{\alpha,d_X}^{\mathsmaller{(\eps)}}(x)$ w.r.t. the measure.

\begin{lemma}[Theorem 3 in \cite{martinez2020shape}]\label{lm:cmbound}
Assume $\alpha,\beta\in\pf^c(X)$. Then, there is a constant $A=A(\eps,m,D)$ such that 
\[\sup_{x\in\R^m}\norm{\tilde{\Sigma}_{\alpha,d_X}^{\mathsmaller{(\eps)}}(x)-\tilde{\Sigma}_{\beta,d_X}^{\mathsmaller{(\eps)}}(x)}_F\leq c\, A\cdot \dW{\infty}(\alpha,\beta).\]
\end{lemma}

To utilize Lemma \ref{lm:cmbound}, we define another matrix as follows which mediates between $\Tilde{\Sigma}_{\alpha,d_X}^{\mathsmaller{(\eps)}}(x)$ and ${\Sigma}_{\alpha,d_X}^{\mathsmaller{(\eps)}}(x)$:
\[\hat{\Sigma}_{\alpha,d_X}^{\mathsmaller{(\eps)}}(x)\coloneqq\int_{\mathbb{R}^d}\lc{y}-x\rc\otimes \lc {y}- x\rc\,\me_{{\alpha,d_X}}(x)(d {y}). \]
Note that $\hat{\Sigma}_{\alpha,d_X}^{\mathsmaller{(\eps)}}(x)=\frac{\eps^m\,\nu_m}{\alpha(B_\eps^{d_X}(x))}\tilde{\Sigma}_{\alpha,d_X}^{\mathsmaller{(\eps)}}(x)$ and 
$$\hat{\Sigma}_{\alpha,d_X}^{\mathsmaller{(\eps)}}(x)={\Sigma}_{\alpha,d_X}^{\mathsmaller{(\eps)}}(x)+\lc \mu_{\alpha,d_X}^{\mathsmaller{(\eps)}}(x)-x\rc\otimes\lc\mu_{\alpha,d_X}^{\mathsmaller{(\eps)}}(x)-x\rc.$$

\begin{lemma}\label{lm:stb-hat-sigma}
Let $\alpha,\beta\in\pf^c(X)\cap\mathcal{P}^{\mathcal{BG}(\Lambda)}(X)$. Then, we have that
\[\sup_{x\in\R^m}\norm{\hat{\Sigma}_{\alpha,d_X}^{\mathsmaller{(\eps)}}(x)-\hat{\Sigma}_{\beta,d_X}^{\mathsmaller{(\eps)}}(x)}_F\leq \Upsilon_{\Lambda,D,\eps}^{c,A,m}(\dW{\infty}(\alpha,\beta)),\]
where $A$ is the same $A$ appearing in Lemma \ref{lm:cmbound}.
\end{lemma}

\begin{proof}[Proof of Lemma \ref{lm:stb-hat-sigma}]
For notational simplicity, we let $\alpha_x\coloneqq\alpha(\bxe)$ and $\beta_x\coloneqq\beta(\bxe)$.
\begin{align*}
    \norm{\hat{\Sigma}_{\alpha,d_X}^{\mathsmaller{(\eps)}}(x)-\hat{\Sigma}_{\beta,d_X}^{\mathsmaller{(\eps)}}(x)}_F=&{\eps^m\,\nu_m}\norm{\frac{\tilde{\Sigma}_{\alpha,d_X}^{\mathsmaller{(\eps)}}(x)}{\alpha_x}-\frac{\tilde{\Sigma}_{\beta,d_X}^{\mathsmaller{(\eps)}}(x)}{\beta_x}}_F\\
    =&\frac{\eps^m\,\nu_m}{\alpha_x\,\beta_x}\norm{{\tilde{\Sigma}_{\alpha,d_X}^{\mathsmaller{(\eps)}}(x)}{\beta_x}-{\tilde{\Sigma}_{\beta,d_X}^{\mathsmaller{(\eps)}}(x)}{\alpha_x}}_F\\
    \leq&\underbrace{\frac{\eps^m\,\nu_m}{\alpha_x}\norm{{\tilde{\Sigma}_{\alpha,d_X}^{\mathsmaller{(\eps)}}(x)}-{\tilde{\Sigma}_{\beta,d_X}^{\mathsmaller{(\eps)}}(x)}}_F}_{T_1}+\underbrace{\frac{\eps^m\,\nu_m}{\alpha_x\,\beta_x}|\alpha_x-\beta_x|\norm{{\tilde{\Sigma}_{\beta,d_X}^{\mathsmaller{(\eps)}}(x)}}_F}_{T_2}
\end{align*}
By Lemma \ref{rem:lb-meas}, we have that $\alpha_x,\beta_x\geq\psi_{\Lambda,D}(\eps)$. Hence, together with Lemma \ref{lm:cmbound}, we have
\[T_1\leq \frac{\eps^m\,\nu_m\,c\,A}{\psi_{\Lambda,D}(\eps)}\dW{\infty}(\alpha,\beta). \]

To estimate $|\alpha_x-\beta_x|$, we utilize the {Prokhorov distance} $\dP$. Without loss of generality, we assume that $\beta_x\geq \alpha_x$. Let $\xi\coloneqq \dP(\alpha,\beta)$. Then,
\begin{align*}
    &\beta(\bxe)-\alpha(\bxe)\\
    \leq&\alpha\left((\bxe)^\xi\right)+\xi-\alpha(\bxe)\\
    \leq &\alpha(\bxe)\lc\frac{\alpha\left(B_{\eps+\xi}^{d_X}(x)\right)}{\alpha(\bxe)}-1\rc+\xi\\
    \leq& \lc\lc1+\frac{\xi}{\eps}\rc^\Lambda-1\rc+\xi=\Upsilon_{\Lambda,\eps}(\xi)\\
    \leq &\Upsilon_{\Lambda,\eps}\lc\sqrt{\dW{1}(\alpha,\beta)}\rc
    \leq\Upsilon_{\Lambda,\eps}\lc\sqrt{\dW{\infty}(\alpha,\beta)}\rc.
\end{align*}
Since $\Upsilon_{\Lambda,\eps}$ is increasing, the second to last inequality follows from the fact that $(\dP)^2\leq \dW{1}$ (cf. Lemma \ref{lemma:dp}) and the last inequality follows from the fact that $\dW{p}\leq \dW{q}$ whenever $1\leq p\leq q\leq \infty$ \citep{givens1984class}.

Since $y$ is constructed in $\bxe$ in Equation (\ref{eq:til-cov}), we have $\norm{y-x}\leq \eps$ and thus $\norm{\tilde{\Sigma}_{\beta,d_X}^{\mathsmaller{(\eps)}}(x)}_F\leq \frac{\eps^2}{{\eps^{m}\nu_m}}$. Therefore,
\[T_2\leq  \frac{\eps^2}{\psi_{\Lambda,D}^2(\eps)}\Upsilon_{\Lambda,\eps}\left(\sqrt{\dW{\infty}(\alpha,\beta)}\right).\]
Hence, $T_1+T_2\leq \Upsilon_{\Lambda,D,\eps}^{c,A,m}(\dW{\infty}(\alpha,\beta))$
\end{proof}

Now, we are ready to establish a key lemma for proving Theorem \ref{thm:GT-stable}.

\begin{lemma}\label{lm:key-cov-est}
Let $\alpha,\beta\in\pf^{c}(X)\cap\mathcal{P}^{\mathcal{BG}(\Lambda)}(X)$. Then,
\[\sup_{x\in\R^m}\norm{{\Sigma}_{\alpha,d_X}^{\mathsmaller{(\eps)}}(x)-{\Sigma}_{\beta,d_X}^{\mathsmaller{(\eps)}}(x)}_F\leq \Psi_{\Lambda,D,\eps}^{c,A,m}(\dW{\infty}(\alpha,\beta)).\]
\end{lemma}

\begin{proof}
For simplicity of notation, we let $\mu_\alpha\coloneqq\mu^{\mathsmaller{(\eps)}}_{\alpha,d_X}(x)$ and $\mu_\beta\coloneqq\mu^{\mathsmaller{(\eps)}}_{\beta,d_X}(x)$.
\begin{align*}
    &\norm{{\Sigma}_{\alpha,d_X}^{\mathsmaller{(\eps)}}(x)-{\Sigma}_{\beta,d_X}^{\mathsmaller{(\eps)}}(x)}_F\leq\norm{\hat{\Sigma}_{\alpha,d_X}^{\mathsmaller{(\eps)}}(x)-\hat{\Sigma}_{\beta,d_X}^{\mathsmaller{(\eps)}}(x)}_F+\norm{(\mu_\alpha-x)^{\otimes^2}-(\mu_\beta-x)^{\otimes^2}}_F\\
    \leq &\Upsilon_{\Lambda,D,\eps}^{c,A,m}(\dW{\infty}(\alpha,\beta))+\norm{(\mu_\alpha-x)\otimes(\mu_\alpha-\mu_\beta)}_F+\norm{(\mu_\alpha-\mu_\beta)\otimes(\mu_\beta-x)}_F\\
    \leq &\Upsilon_{\Lambda,D,\eps}^{c,A,m}(\dW{\infty}(\alpha,\beta))+2\eps\norm{\mu_\alpha-\mu_\beta}.
\end{align*}
By Theorem \ref{thm:stb-ms}, we obtain
\begin{align*}
    \norm{{\Sigma}_{\alpha,d_X}^{\mathsmaller{(\eps)}}(x)-{\Sigma}_{\beta,d_X}^{\mathsmaller{(\eps)}}(x)}_F\leq& \Upsilon_{\Lambda,D,\eps}^{c,A,m}(\dW{\infty}(\alpha,\beta))+2\eps(1+2\eps)\Phi_{\Lambda,D,\eps}\left(\sqrt{\dW{1}(\alpha,\beta)}\right)\\
    \leq &\Psi_{\Lambda,D,\eps}^{c,A,m}(\dW{\infty}(\alpha,\beta)).
\end{align*}
We use the fact that $\dW{1}\leq \dW{\infty}$ in the last inequality.
\end{proof}

Finally, we finish the proof of Theorem \ref{thm:GT-stable}.
For any $x\in X$, one has
\begin{align*}
    \dW{2}\left(\gamma_{\alpha,d_X}^{\mathsmaller{(\eps,\lambda)}}(x),\gamma_{\beta,d_X}^{\mathsmaller{(\eps,\lambda)}}(x)\right)=&\sqrt{\lambda\,\lc\dcov\left(\Sigma_{\alpha,d_X}^{\mathsmaller{(\eps)}}(x),\Sigma_{\beta,d_X}^{\mathsmaller{(\eps)}}(x)\right)\rc^2}\\
    \leq&\sqrt{\lambda} \norm{\left(\Sigma_{\alpha,d_X}^{\mathsmaller{(\eps)}}(x)\right)^\frac{1}{2}-\left(\Sigma_{\beta,d_X}^{\mathsmaller{(\eps)}}(x)\right)^\frac{1}{2}}_F\\
    \leq& \sqrt{m\,\lambda} \norm{\Sigma_{\alpha,d_X}^{\mathsmaller{(\eps)}}(x)-\Sigma_{\beta,d_X}^{\mathsmaller{(\eps)}}(x)}_F^\frac{1}{2}\\
    \leq& \sqrt{m\,\lambda\,\Psi_{\Lambda,D,\eps}^{c,A,m}(\dW{\infty}(\alpha,\beta))}.
\end{align*}
The first inequality follows from Lemma \ref{lm:tr}. The second inequality follows from Lemma \ref{lm:sqrn}. The last inequality follows from Lemma \ref{lm:key-cov-est}.

Now, for any $x,x'\in X$, one has
\begin{align*}
    \left|\dgt(x,x')-d^{\mathsmaller{(\eps,\lambda)}}_{{\beta,d_X}}(x,x')\right|= & \big|\dW{2}\lc\gamma_{\alpha,d_X}^{\mathsmaller{(\eps,\lambda)}}(x),\gamma_{\alpha,d_X}^{\mathsmaller{(\eps,\lambda)}}(x')\rc-\dW{2}\lc\gamma_{\alpha,d_X}^{\mathsmaller{(\eps,\lambda)}}(x),\gamma_{\beta,d_X}^{\mathsmaller{(\eps,\lambda)}}(x')\rc\\
    +&\dW{2}\lc\gamma_{\alpha,d_X}^{\mathsmaller{(\eps,\lambda)}}(x),\gamma_{\beta,d_X}^{\mathsmaller{(\eps,\lambda)}}(x')\rc-\dW{2}\lc\gamma_{\beta,d_X}^{\mathsmaller{(\eps,\lambda)}}(x),\gamma_{\beta,d_X}^{\mathsmaller{(\eps,\lambda)}}(x')\rc\big|\\
    \leq &\dW{2}\lc\gamma_{\alpha,d_X}^{\mathsmaller{(\eps,\lambda)}}(x'),\gamma_{\beta,d_X}^{\mathsmaller{(\eps,\lambda)}}(x)\rc+\dW{2}\lc\gamma_{\alpha,d_X}^{\mathsmaller{(\eps,\lambda)}}(x),\gamma_{\beta,d_X}^{\mathsmaller{(\eps,\lambda)}}(x)\rc\\
    \leq & 2\sqrt{m\,\lambda\,\Psi_{\Lambda,D,\eps}^{c,A,m}(\dW{\infty}(\alpha,\beta))}.
\end{align*}
\end{proof}

\begin{proof}[Proof of Theorem \ref{thm:gt-stable-smooth}]
For any $x\in X$, let
\begin{equation}\label{eq:til-cov-smooth}
    \Tilde{\Sigma}_{\alpha,d_X}^{K_\eps}(x)\coloneqq\int_{\R^m}\lc{y}-x\rc\otimes \lc {y}- x\rc K_\eps(x,y)\,\alpha(d {y}). 
\end{equation}

\begin{lemma}[{\cite[Theorem 1]{martinez2020shape}}]\label{thm:tilde-sigma}
There exists a constant $A_\mathcal{K}>0$ only depending on $\mathcal{K}$ such that for any $\alpha,\beta\in\mathcal{P}_f(X)$, we have
\[\norm{\Tilde{\Sigma}_{\alpha,d_X}^{K_\eps}(x)-\Tilde{\Sigma}_{\beta,d_X}^{K_\eps}(x)}_F\leq A_\mathcal{K}\,\dW{1}(\alpha,\beta).\]
\end{lemma}

\begin{lemma}\label{lm:key-smooth-stb}
There exists a positive constant $C>0$ depending on $\mathcal{K},\eps,D$ and $X$ such that for any $\alpha,\beta\in\pf(X)$
\[\norm{{\Sigma}_{\alpha,d_X}^{K_\eps}(x)-{\Sigma}_{\beta,d_X}^{K_\eps}(x)}_F\leq C\,\dW{1}(\alpha,\beta).\]
\end{lemma}

\begin{proof}
For notational simplicity, let $M_\alpha(x)\coloneqq\int_XK_\eps(x,y)\alpha(dy)$, $\mu_\alpha\coloneqq\mu_{\alpha,d_X}^{K_\eps}(x)$, $\Sigma_\alpha\coloneqq{\Sigma}_{\alpha,d_X}^{K_\eps}(x)$ and $\tilde{\Sigma}_\alpha\coloneqq\tilde{\Sigma}_{\alpha,d_X}^{K_\eps}(x)$.
Note that $\Sigma_\alpha=\frac{\tilde{\Sigma}_\alpha}{M_\alpha(x)}-(\mu_\alpha-x)\otimes(\mu_\alpha-x)$. Then,
\begin{align*}
    \norm{\Sigma_\alpha-\Sigma_\beta}_F&\leq\norm{\frac{1}{M_\alpha(x)}\tilde{\Sigma}_\alpha-\frac{1}{M_\beta(x)}\tilde{\Sigma}_\beta}_F\\
    &+\norm{(\mu_\alpha-x)^{\otimes^2}-(\mu_\beta-x)^{\otimes^2}}\\
    &\leq \underbrace{\frac{1}{M_\alpha(x)}\norm{\tilde{\Sigma}_\alpha-\tilde{\Sigma}_\beta}_F}_{T_1}+\underbrace{\frac{|M_\alpha(x)-M_\beta(x)|}{M_\alpha(x)\,M_\beta(x)}\norm{\tilde{\Sigma}_\beta}_F}_{T_2}\\
    &+\underbrace{\norm{\mu_\alpha-\mu_\beta}\lc\norm{\mu_\alpha-x}+\norm{\mu_\beta-x}\rc}_{T_3}
\end{align*}

Since $\mathcal{K}$ is continuous and positive by assumption, there exists $W=W(\eps,D)>0$ such that for any $t\in\left[0,\frac{D^2}{\eps^2}\right]$, $\mathcal{K}(t)\geq W$. Hence, $M_\alpha(x)\geq W$. Then, by Lemma \ref{thm:tilde-sigma}, we have that $T_1\leq \frac{A_\mathcal{K}}{W}\dW{1}(\alpha,\beta).$

Since $\mathcal{K}$ is $L$-Lipschitz, we have that for any $y_1,y_2\in X$,
\begin{align*}
    &|K_\eps(x,y_1)-K_\eps(x,y_2)|\\
    =&\frac{1}{C_m(\eps)}\left|\mathcal{K}\lc\frac{\norm{y_1-x}^2}{\eps^2}\rc-\mathcal{K}\lc\frac{\norm{y_2-x}^2}{\eps^2}\rc\right|\\
    \leq &\frac{L}{\eps^2}\left|\norm{y_1-x}^2-\norm{y_2-x}^2\right|\\
    =&\frac{L}{\eps^2}\norm{y_1-y_2}\norm{2x-y_1-y_2}\\
    \leq &\frac{2\,L\,D}{\eps^2}\norm{y_1-y_2}.
\end{align*}
So $K_\eps(x,\cdot)$ is a $\frac{2\,LD}{\eps^2}$-Lipschitz function on $X$. Then, by Kantorovich duality (cf. \Cref{eq:kant duality}), we have
\begin{align*}
    |M_\alpha(x)-M_\beta(x)|&=\left|\int_XK_\eps(x,y)\alpha(dy)-\int_XK_\eps(x,y)\beta(dy)\right|\\
    &\leq \frac{2\,LD}{\eps^2}\,\dW{1}(\alpha,\beta).
\end{align*}
As for $\frac{1}{M_\beta(x)}\norm{\tilde{\Sigma}_\beta}_F$, we know from Equation (\ref{eq:til-cov-smooth}) that
\begin{align*}
    \frac{\norm{\tilde{\Sigma}_\beta}_F}{M_\beta(x)}&\leq \frac{\int_{X}\norm{\lc{y}-x\rc^{\otimes^2}}_F K_\eps(x,y)\,\beta(d {y})}{\int_XK_\eps(x,y)\,\beta(dy)}\leq D^2.
\end{align*}
Thus, $T_2\leq \frac{2\,L\,D^3}{\eps^2\,W}\,\dW{1}(\alpha,\beta).$

As for $T_3$, we first have by Theorem \ref{thm:ms-smooth} that
\[\norm{\mu_\alpha-\mu_\beta}\leq\frac{2\,CD+M}{N}\dW{1}(\alpha,\beta). \]
Then, we have that
\begin{align*}
    \norm{\mu_\alpha-x}&=\norm{\frac{\int_{X}(y-x)\,K_\eps(x,y)\,\alpha(dy)}{\int_{X}K_\eps(x,y)\,\alpha(dy)}}\\
    &\leq \frac{\int_{X}\norm{y-x}\,K_\eps(x,y)\,\alpha(dy)}{\int_{X}K_\eps(x,y)\,\alpha(dy)}\leq D.
\end{align*}
Hence, $T_3\leq \frac{2D(2CD+M)}{N}\,\dW{1}(\alpha,\beta).$

By adding up the three inequalities regarding upper bounds of $T_1,T_2$ and $T_3$, we conclude the proof.
\end{proof}

Then, based on Lemma \ref{lm:key-smooth-stb}, with a similar proof as the one for Theorem \ref{thm:GT-stable}, we obtain the stability theorem (Theorem \ref{thm:gt-stable-smooth}).
\end{proof}

\bibliographystyle{plainnat}
\bibliography{biblio-wt}

\end{document}